\documentclass[runningheads]{llncs}
\usepackage{algorithmic}
\usepackage[linesnumbered,ruled,vlined]{algorithm2e}
\usepackage{multirow}
\usepackage{eccv}
\usepackage{filecontentsdef}

\usepackage{eccvabbrv}

\usepackage{graphicx}
\usepackage{booktabs}

\usepackage[accsupp]{axessibility}  

\usepackage{hyperref}

\usepackage{orcidlink}

\begin{document}

\title{A Theoretical and Empirical Study on the Convergence of Adam with an \emph{"Exact"} Constant Step Size in Non-Convex Settings} 

\titlerunning{Study on the Convergence of Adam}

\author{Alokendu Mazumder\inst{1} \and
Rishabh Sabharwal\inst{2}\thanks{denotes equal contribution} \and
Manan Tayal\inst{1,*} \and Bhartendu Kumar\inst{3,*} \and Punit Rathore\inst{1}}

\authorrunning{Mazumder et al.}

\institute{Robert Bosch Centre for Cyber-Physical Systems, Indian Institute of Science, Bengaluru, India \and
Netaji Subhas University of Technology, New Delhi, India
\\
\and
TCS Research, Bengaluru, India\\
}

\maketitle

\begin{abstract}
  In neural network training, RMSProp and Adam remain widely favoured optimisation algorithms. One of the keys to their performance lies in selecting the correct step size, which can significantly influence their effectiveness. Additionally, questions about their theoretical convergence properties continue to be a subject of interest. In this paper, we theoretically analyse a constant step size version of Adam in the non-convex setting and discuss why it is important for the convergence of Adam to use a fixed step size. This work demonstrates the derivation and effective implementation of a constant step size for Adam, offering insights into its performance and efficiency in non-convex optimisation scenarios. \textbf{(i)} First, we provide proof that these adaptive gradient algorithms are guaranteed to reach criticality for smooth non-convex objectives with constant step size, and we give bounds on the running time. Both deterministic and stochastic versions of Adam are analysed in this paper. We show sufficient conditions for the derived constant step size to achieve asymptotic convergence of the gradients to zero with minimal assumptions. Next, \textbf{(ii)} we design experiments to empirically study Adam's convergence with our proposed constant step size against state-of-the-art step size schedulers on classification tasks. Lastly, \textbf{(iii)} we also demonstrate that our derived constant step size has better abilities in reducing the gradient norms, and empirically, we show that despite the accumulation of a few past gradients, the key driver for convergence in Adam is the non-increasing step sizes. \emph{Source code will be available upon acceptance.}
  \keywords{Adam \and step size \and non-convex \and convergence \and optimisation}
\end{abstract}

\section{Introduction}
\label{sec:intro}
Optimisation problems in machine learning are often structured as minimizing a finite sum \( \min_x f(x) \), where \( f(x) = \frac{1}{k} \sum_{i=1}^{k} f_i(x) \). Each \( f_i \) typically exhibits non-convex behaviour, particularly in neural network domains. A prominent method for tackling such problems is \emph{Stochastic Gradient Descent} (SGD), where updates occur iteratively as \( x_{t+1} := x_t - \alpha \nabla \tilde{f}_{i_t}(x_t) \), with \( \alpha \) denoting the step size and \( \tilde{f}_{i_t} \) being a randomly chosen function from \( \{ f_1, f_2, ..., f_k \} \) at each iteration \( t \). SGD is favoured for training deep neural networks for its computational efficiency, especially with mini-batch training, even with large datasets~\cite{bottou2012stochastic}.

Adaptive variants of SGD, which incorporate past gradients through averaging, have gained traction due to their ability to track gradient scaling on an individual parameter basis, as highlighted by \emph{Bottou et al.}~\cite{bottou2009curiously}. These methods are favoured for their perceived ease of tuning compared to traditional SGD. Adaptive gradient methods typically update using a vector obtained by applying a linear transformation, often referred to as "diagonal pre-conditioner," to the linear combination of all previously observed gradients. This pre-conditioning is believed to enhance algorithm robustness to hyperparameter choices, making them less sensitive to initial settings.

Adagrad, introduced by \emph{Duchi et al.} in~\cite{duchi2011adaptive, mcmahan2010adaptive}, showed superior performance, especially in scenarios featuring sparse or small gradients. However, its effectiveness diminishes in situations with non-convex loss functions and dense gradients due to rapid learning rate decay. To address this, variants like RMSProp~\cite{tieleman2012rmsprop}, Adam~\cite{kingma2015adam} (\textbf{Algorithm}~\ref{alg:Adam_training}), Adadelta~\cite{zeiler2012adadelta}, and Nadam~\cite{dozat2016incorporating} have been proposed. These methods mitigate learning rate decay by employing exponential moving averages of squared past gradients, thereby limiting reliance on all past gradients during updates. While these algorithms have found success in various applications~\cite{vaswani2019fast}, they have also been observed to exhibit non-convergence in many settings~\cite{sashank2018convergence}, especially in deterministic environments where noise levels are controlled during optimisation. Moreover, there is a significant inclination to study these algorithms in deterministic settings, especially when noise levels are regulated during optimisation. This regulation is accomplished either by employing larger batches~\cite{martens2015optimizing, de2017automated, babanezhad2015stopwasting} or by integrating variance-reducing techniques~\cite{johnson2013accelerating, defazio2014saga}.

Despite their widespread adoption in solving neural network problems, adaptive gradient methods like RMSProp and Adam often lack theoretical justifications in non-convex scenarios, even when dealing with exact or deterministic gradients~\cite{bernstein2018signsgd}. Several sufficient conditions have been proposed to guarantee the global convergence of Adam, which can be further classified into the following two categories:

\begin{algorithm}
\caption{Adam Algorithm Pseudocode~\cite{kingma2015adam}}
\label{alg:Adam_training}
\KwIn{Learning rate: $\alpha \in (0,1]$, $\beta_{1}, \beta_{2} \in [0,1)$, initial starting point $\textbf{w}_{0} \in \mathbb{R}^{d}$, a constant vector $\rho\textbf{1}_{d} > \textbf{0} \in \mathbb{R}^{d}$, we assume we have access to a noisy oracle for gradients of $f:\mathbb{R}^{d}$ $\rightarrow$ $\mathbb{R}$}
\textbf{Initialization: $\textbf{m}_{0} = 0$, $\textbf{v}_{0} = 0$}\\
\For{$t$ \textbf{from} 1 \textbf{to} $T$:}{
    $\textbf{g}_{t} = \nabla f(\textbf{w}_{t})$\\
    $\textbf{m}_{t} = \beta_{1}\textbf{m}_{t-1} + (1 - \beta_{1})\textbf{g}_{t}$\\
    $\textbf{v}_{t} = \beta_{2}\textbf{v}_{t-1} + (1 - \beta_{2})\textbf{g}_{t}^{2}$\\
    $\textbf{V}_{t} = \texttt{diag}(\textbf{v}_{t})$\\
    
    $\textbf{w}_{t+1} = \textbf{w}_{t} - \alpha(\textbf{V}_{t}^{1/2} + \texttt{diag}\left(\rho\textbf{1}_{d})\right)^{-1}\textbf{m}_{t}$

}
\textbf{End}\\
\end{algorithm}


\textbf{(B1) Learning rate decay}: 
\emph{Reddi et al.}~\cite{sashank2018convergence} demonstrated that the main cause of divergences in Adam and RMSProp predominantly arises from the disparity between the two successive learning rates.
    \begin{equation}
        \Gamma_t = \frac{\sqrt{\textbf{v}_t}}{\alpha_t} - \frac{\sqrt{\textbf{v}_{t-1}}}{\alpha_{t-1}}
    \end{equation}
If the positive definiteness property of \( \Gamma_t \) is violated, Adam and RMSProp may experience divergence. Barakat and Bianchi in~\cite{pmlr-v129-barakat20a} relaxed the \( \Gamma_t > 0 \) constraint and demonstrated that Adam can converge when \( \frac{\sqrt{\textbf{v}_t}}{\alpha_t} \geq \frac{c\sqrt{\textbf{v}_{t-1}}}{\alpha_{t-1}} \) holds for all \( t \) and some positive real \( c \). Prior research efforts~\cite{shi2020rmsprop, tian2022amos, chen2018convergence, luo2019adaptive, defossez2020simple, zou2019sufficient} have focused on establishing convergence properties of optimisation algorithms such as RMSProp and Adam. These investigations typically utilize step size schedules of the form \( \alpha_{t} = \frac{\alpha}{t^a} \) for all \( t \in \{1,2,\dots,T\} \), \( \alpha \in (0,1) \), and \( a > 0 \), or other time-dependent variations like \( \alpha_t = \alpha(1-\beta_1)\sqrt{\frac{1-\beta_{2}^t}{1-\beta_2}} \). While a non-increasing (diminishing) step size is crucial for convergence, empirical analysis suggests that rapid decay of the step size can lead to sub-optimal outcomes, as discussed in the next section.

\textbf{(B2) Temporal decorrelation}: 
\emph{Zhou et al.} in~\cite{zhou2018adashift} emphasized that the divergence observed in RMSProp stems from imbalanced learning rates rather than the absence of \( \Gamma_t > 0 \), as previously suggested by \emph{Reddi et al.}~\cite{sashank2018convergence}. Leveraging this insight, \emph{Zhou et al.}~\cite{zhou2018adashift} proposed AdaShift, which incorporates a temporal decorrelation technique to address the inappropriate correlation between \( \textbf{v}_{t} \) and the current second-order moment \( \textbf{g}^{2}_{t} \). This approach requires the adaptive learning rate \( \alpha_{t} \) to be independent of \( \textbf{g}^{2}_{t} \). However, it's worth noting that the convergence analysis of AdaShift was primarily limited to RMSProp for resolving the convex counterexample presented by \emph{Reddi et al.}~\cite{sashank2018convergence}.

In contrast to the aforementioned modifications and restrictions, we introduce an alternative sufficient condition (abbreviated as \textbf{SC}) to ensure the global convergence of the original Adam. The proposed \textbf{SC} (refer to \textbf{Section}~\ref{sec:SC}) depends on the parameter \( \beta_1 \) and a constant learning rate \( \alpha \). Our \textbf{SC} doesn't necessitate rapidly decaying step sizes and positive definiteness of \( \Gamma_t \). It's easier to verify and more practical compared to (\textbf{B2}).
    
\textbf{NOTE}: A recent study by \emph{Chen et al.}~\cite{chen2022towards} proposes a similar convergence rate to ours using a constant step size approach for mini-batch Adam, where they employ \( \alpha_t = \frac{\alpha}{\sqrt{T}} \), with \( \alpha \) being any positive real number. However, our work specifies the \emph{exact} learning rate that ensures convergence with fewer assumptions. A comparison with this step size is provided in the Appendix, where we evaluate its performance against our approach for various values of \( \alpha \).

\section{Motivation \& Contributions}
\label{sec:motivation}
In practical scenarios, such as learning latent variables from vast datasets with unknown distributions, the goal is to solve the optimisation problem:
\begin{equation}
    min_{\textbf{w} \in \mathbb{R}^d} \{\mathcal{L}(\textbf{w}) = \mathbb{E}_{\zeta \sim \mathcal{P}}[\tilde{\mathcal{L}}(\zeta,\textbf{w})]\}
\end{equation}
Here \( \mathcal{L}(\textbf{w}) \) is a non-convex loss function and \( \zeta \) is a random variable with an unknown distribution \( \mathcal{P} \). Common iterative optimisation algorithms like Adam and RMSProp are often employed to solve this problem. It's observed that using a more aggressive constant step size often leads to favourable outcomes~\cite{NEURIPS2020_a9078e86, li2021contrastive}. However, determining the optimal learning rate requires an exhaustive search. Specifically, when optimizing for a parameter set \( \textbf{w} \in \mathbb{R}^d \), the goal is to find a parameter \( \textbf{w}^{*} \) aligning with a minimum in the loss landscape \( \mathcal{L}(\textbf{w}) \), where \( \mathcal{L}: \mathbb{R}^d \rightarrow \mathbb{R} \). Gradient descent-based algorithms, including those with momentum, iterati                     vely update parameters using \( \textbf{w}_{t+1} = \textbf{w}_{t} - \alpha_{t} g(\nabla_{\textbf{w}}\mathcal{L}(\textbf{w}_{t})) \), where \( \alpha_{t} \) is the step size at iteration \( t \) and \( g: \mathbb{R}^d \rightarrow \mathbb{R}^d \) is a function incorporating the gradient and momentum terms. Convergence of the parameter sequence \( \{\textbf{w}_0, \textbf{w}_1, \dots, \textbf{w}_{T-1}\} \) to \( \textbf{w}^* \) necessitates the sequence \( \{ (\textbf{w}_{t+1} - \textbf{w}_{t}) \} \) to converge to \( \boldsymbol{0} \) as \( t \rightarrow T \). This corresponds to the sequence \( \{ \alpha_{t} g(\nabla_{\textbf{w}}\mathcal{L}(\textbf{w}_{t})) \} \) converging to \( \boldsymbol{0} \in \mathbb{R}^d \). The literature often opts for a non-increasing step size \( \alpha_t \), such as \( \alpha_t \propto \frac{1}{t^a} \), where \( a > 0 \), to theoretically prove convergence. Despite non-zero gradient norms, \( \alpha_t \) can cause \( \{ \alpha_{t} g(\nabla_{\textbf{w}}\mathcal{L}(\textbf{w}_{t})) \} \) to approach \( \boldsymbol{0} \), even if \( \{ \nabla_{\textbf{w}}\mathcal{L}(\textbf{w}_{t}) \} \) does not. Although accumulating past gradients helps mitigate the decay of the learning rate, there's a risk that learning rates with rapid decay may dominate. Our empirical analysis in \textbf{Section}~\ref{sec:exp} demonstrates that Adam, with non-increasing learning rates, does not aggressively drive the gradient norm of the loss towards zero. With a fixed step size \( \alpha_t = \alpha > 0 \), if \( \textbf{w}_t \) $\rightarrow$ \( \textbf{w}^* \), then $\| \nabla_{\textbf{w}}\mathcal{L}(\textbf{w}_{t}) \|_2$ $\rightarrow$ \( \boldsymbol{0} \). This ensures convergence to a saddle point of \( \mathcal{L}(\textbf{w}) \) with a fixed step size, an assurance not available with a non-increasing and iteration-dependent step size.

Consequently, our research aims to identify the optimal constant step size, ensuring convergence in both deterministic and stochastic Adam iterations beyond a time threshold $t > T(\beta_1, \rho)$ where $T(\beta_1, \rho)$ is a natural number.

\textbf{A summary of our contributions}: In this work, we present the following primary contributions:
\begin{enumerate}
   \item We derive an exact constant step size to guarantee convergence of deterministic as well as stochastic Adam. To the best of our knowledge, this study is the first to theoretically guarantee Adam's global convergence (gradient norm of loss function converges to 0) with an exact constant step size. 
   
    \item Our study offers runtime bounds for deterministic and stochastic Adam to achieve approximate criticality with smooth non-convex functions.
    \item We introduced a simple method to estimate the Lipschitz constant\footnote{Our derived step size, vital for convergence, depends on the Lipschitz constant.} of the loss function \emph{w.r.t} the network parameters. We offer a probabilistic guarantee for the convergence of our estimated Lipschitz constant to its true value.
    \item We empirically show that even with past gradient accumulation, deterministic Adam can be affected by rapidly decaying learning rates. This indicates that these rapid decay rates play a dominant role in driving convergence.
    \item We also demonstrate empirically that with our analysed step size, Adam quickly converges to a favourable saddle point with high validation accuracy.
\end{enumerate}


\section{Convergence Guarantee For Adam With Fix Step Size}
\label{sec:converge}
Previously, it has been shown in~\cite{sashank2018convergence,de2018convergence,defossez2020simple} that deterministic RMSProp and Adam can converge under certain conditions with adaptive step size. Here, we give the first result about convergence to criticality for deterministic and stochastic Adam with constant step size, albeit under a certain technical assumption about the loss function (and hence on the noisy first-order oracle). 
\subsection{Novel Sufficient Conditions (SC) for Convergence of Adam}
\label{sec:SC}
Below are the commonly employed assumptions for analyzing the convergence of stochastic algorithms for non-convex problems:
\begin{enumerate}
    \item The minimum value of the problem $\mathcal{L}^{*} = argmin_{\textbf{w} \in \mathbb{R}^d} \mathcal{L}(\textbf{w})$ is lower bounded, \emph{i.e} $\mathcal{L}^* > -\infty$. 
    \item The loss landscape $\mathcal{L}(\textbf{w})$ should be strictly non-negative. Most commonly used losses like cross-entropy, $L_p$ reconstruction, InfoNCE loss~\cite{sohn2016improved, wu2018unsupervised} are non-negative.
    \item The stochastic gradient $\textbf{g}_{t}$ of the loss function is an unbiased estimate, \emph{i.e} $\mathbb{E}[\textbf{g}_{t}] = \nabla_{\textbf{w}} \mathcal{L}(\textbf{w}_t)$.
    \item The second-order moment of stochastic gradient $\textbf{g}_t$ is uniformly upper-bounded, \emph{i.e} $\mathbb{E}[\|\textbf{g}_t\|^2] \leq (\mathbb{E}[\|\textbf{g}_t\|])^2 \leq \gamma^2$.
\end{enumerate}
In addition, we suppose that that the parameters $\beta_1$ and learning rate $\alpha$ satisfy the following restrictions:
\begin{enumerate}
    \item The parameter $\beta_1$ satisfies $\beta_1 < \frac{\epsilon}{\epsilon + \gamma}$ for some $\epsilon > 0$.
    \item The step size $\alpha_t = \alpha \in \mathbb{R}^{+}$ remains constant throughout the analysis. 
\end{enumerate}

\begin{theorem}

\textbf{Deterministic Adam converges with proper choice of constant step size.} 
\label{theorem: 1}
Let the loss function $\mathcal{L}(\textbf{w})$ be $K-$Lipchitz\footnote{A function $f: \mathbb{R}^d \rightarrow \mathbb{R}$ is \emph{K-Lipchitz} for some $K > 0$ if it satisfies, $f(\textbf{y}) \leq f(\textbf{x}) + \nabla_{\textbf{x}} f(\textbf{x})^{T}(\textbf{y} - \textbf{x}) + \frac{K}{2}\|\textbf{y} -\textbf{x}\|_{2}^{2}$ $\forall$ $\textbf{x},\textbf{y} \in \mathbb{R}^d$} and let $\gamma < \infty$ be an upper bound on the norm of the gradient of $\mathcal{L}$. Also assume that $\mathcal{L}$ has a well-defined minimizer $\textbf{w}^{*}$ such that $\textbf{w}^{*} = argmin_{\textbf{w} \in \mathbb{R}^d}$ $\mathcal{L}(\textbf{w})$. Then the following holds for Algorithm (1):

For any $\epsilon, \rho> 0$ if we let $\alpha = \sqrt{2(\mathcal{L}(\textbf{w}_{0}) - \mathcal{L}(\textbf{w}^{*}))/K\delta^{2}T}$, then there exists a natural number $T(\beta_1,\rho)$ (depends on $\beta_1$ and $\rho$) such that $\underset{t = 1,\dots,T}{min}\|\nabla_{\textbf{w}}\mathcal{L}(\textbf{w}_{t})\|_{2} \leq \epsilon$ for some $t \geq T(\beta_1,\rho)$, where $\delta^{2} = \frac{\gamma^{2}}{\rho^{2}}$.
\end{theorem}

\begin{theorem}
\label{theorem: 2}
\textbf{Stochastic Adam converges with proper choice of constant step size} Let the loss function $\mathcal{L}(\textbf{w})$ be $K-$Lipchitz and be of the form $\mathcal{L} = \sum_{j=1}^{m}\mathcal{L}_{j}$ such that (a) each $\mathcal{L}_{j}$ is at-least once differentiable, (b) the gradients satisfy $sign(\mathcal{L}_{r}(\textbf{w}))$ \footnote{$sign: \mathbb{R}^d \rightarrow \{-1,1\}^d, sign(\textbf{z})_j = 1$ if $\textbf{z}_j \geq 0$, else -1.}  $= sign(\mathcal{L}_{s}(\textbf{w}))$ for all $r,s \in \{1,2,\dots,m\}$, (c) $\mathcal{L}$ has a well-defined minimizer $\textbf{w}^{*}$ such that $\textbf{w}^{*} = argmin_{\textbf{w} \in \mathbb{R}^d}$ $\mathcal{L}(\textbf{w})$. Let the gradient oracle, upon invocation at \( \textbf{w}_t \in \mathbb{R}^d \), randomly selects  $j_t$ from the set $\{1, 2, \ldots, m\}$ uniformly, and then provides $\nabla f_{j_{t}}(x_t) = \textbf{g}_{t}$ as the result. 

Then, for any $\epsilon, \rho > 0$ if we let $\alpha = \sqrt{2(\mathcal{L}(\textbf{w}_{0}) - \mathcal{L}(\textbf{w}^{*}))/K\delta^{2}T}$, then there exists a natural number $T(\beta_1,\rho)$ (depends on $\beta_1$ and $\rho$) such that \\$\underset{t = 1,\dots,T}{min}\mathbb{E}[\|\nabla_{\textbf{w}}\mathcal{L}(\textbf{w}_{t})\|_{2}] \leq \epsilon$ for some $t \geq T(\beta_1,\rho)$, where $\delta^{2} = \frac{\gamma^{2}}{\rho^{2}}$.

\end{theorem}
\textbf{Remark 1}: With our analysis, we showed that both deterministic and stochastic Adam with constant step size converges with rate $\mathcal{O}(\frac{1}{T^{1/4}})$. This convergence rate is similar to the convergence rate of SGD proposed by \emph{Li et al.} in~\cite{li2014efficient}. Our motivation behind these theorems was primarily to understand the conditions under which Adam can converge, especially considering the negative results presented in \emph{Reddi et al.}~\cite{sashank2018convergence}. However, we acknowledge that it is still an open problem to tighten the analysis of deterministic as well as stochastic Adam and obtain faster convergence rates than we have shown in the theorem.\\
\textbf{Remark 2}: Based on the preceding two theorems, it becomes evident that the Adam achieves convergence when equipped with an appropriately chosen constant step size $\alpha = \sqrt{2(\mathcal{L}(\textbf{w}_{0}) - \mathcal{L}(\textbf{w}^{*}))/K\delta^{2}T}$, \textbf{according to our analysis}. Subsequently, we analyse the given step size and conduct extensive experiments to ascertain that, with the designated step size, the gradient norm effectively tends towards zero. Hence, our empirical investigations validate that the chosen step size ensures convergence in practice. We compare our step size with several state-of-the-art schedulers and an array of constant step sizes in \textbf{Section}~\ref{sec:exp}.
\subsection{Analysis of Constant Step Size}
\label{sec: step}
The optimal learning rate, as per \textbf{Theorem}~\ref{theorem: 1} and \textbf{Theorem}~\ref{theorem: 2}, depends on factors like the Lipchitz constant of the loss\footnote{The Lipchitz constant of the loss is with respect to network parameters.} $(K)$, initial and final loss values, the total number of iterations/epochs $(T)$, and an additional term denoted as $\delta^{2}$. For simplicity, we omitted the $\delta^{2}$ term as it depends on an oracle of gradients and set the final loss term to zero. Thus, our approximate learning rate is now expressed as $\alpha \approx \sqrt{\frac{2\mathcal{L}(\textbf{w}_{0})}{KT}}$, making it practical while still capturing the core concept of \textbf{Theorem}~\ref{theorem: 1} and \textbf{Theorem}~\ref{theorem: 2} . 

As our step size depends on the Lipchitz constant of the loss $(K)$, we next propose a method to estimate the Lipschitz constant, enabling us to implement the step size effectively. Detailed procedure is given in the next Section.

\subsection{Approximating The Lipchitz Constant of Loss}
Determining the appropriate Lipschitz constant for a neural network is a key area of research in deep learning, with various methods proposed for estimation, as demonstrated in recent works~\cite{NEURIPS2018_d54e99a6, NEURIPS2019_95e1533e, latorre2020lipschitz, fazlyab2019efficient, gouk2021regularisation}. Estimating the learning rate prior to training requires careful estimation of the Lipschitz constant of the loss function. For this, we refer to \textbf{Theorem 3.3.6} in~\cite{federer2014geometric}.

\begin{theorem}
\label{theorem: 3}
\emph{[}Federer et al.\emph{]} If $f:\mathbb{R}^{d} \rightarrow \mathbb{R}$ is a locally Lipschitz continuous function, then $f$ is differentiable almost everywhere. Moreover, if $f$ is Lipschitz continuous, then
\begin{equation}
    L(f) = sup_{\textbf{x} \in \mathbb{R}^d}\|\nabla_{\textbf{x}}f\|_2
\end{equation}
\end{theorem}
For large networks, directly optimizing according to \textbf{Theorem}~\ref{theorem: 3} can be time consuming. Instead, we propose a novel technique to approximate the Lipschitz constant. It's important to note that this work does not aim to estimate the best Lipschitz constant for the loss or network; rather, it focuses on the convergence analysis of Adam. We introduce Algorithm~\ref{alg:lipchitz} to approximate the Lipschitz constant of the loss\footnote{\textbf{Algorithm}~\ref{alg:lipchitz} is given for the full-batch scenario. The algorithm for the mini-batch case is deferred to the Appendix due to space constraints.}, supported by mathematical proof that our estimate \emph{converges in distribution} to the original Lipschitz constant.
\begin{algorithm}
\caption{Estimating Lipchitz Constant of Loss Function}
\label{alg:lipchitz}
\KwIn{Dataset: $\mathcal{X} \sim \mathcal{P}$, Loss function: $\mathcal{L}:\mathbb{R}^{d} \rightarrow \mathbb{R}$, A network parameterized by $\textbf{w} \in \mathbb{R}^d \sim \mathcal{W}$, No. of iterations: $N$}
\KwOut{$\hat{L}$}
\textbf{Initialization: $\hat{L} = 0$}, Iteration Number $n =0$\\
\For{$n$ \textbf{from} 1 \textbf{to} $N$:}{
Randomly sample the network weights: $\textbf{w}_{n} \sim \mathcal{W}$\\

    Compute loss for the current pass: $\mathcal{L}(\textbf{w}_{n}) = \mathbb{E}_{\zeta \sim \mathcal{P}}[\tilde{\mathcal{L}}(\zeta,\textbf{w}_n)]$\\
    Compute the norm of gradient at current instant: $\|\nabla_{\textbf{w}}\mathcal{L}(\textbf{w}_{n})\|_2$\\
    Estimate the Lipchitz constant $(\hat{L})$: $\hat{L} = max(\|\nabla_{\textbf{w}}\mathcal{L}(\textbf{w}_{n})\|_2, \hat{L})$\\
}

\textbf{End}\\
\end{algorithm}

We shall now prove that our estimated Lipchitz constant $(\hat{L})$ converges to the real Lipchitz constant $(K)$ in distribution. 
\begin{theorem}
\label{theorem: 4}
    Given $\|.\|_2$ and $\nabla_{\textbf{w}}\mathcal{L}$ are continuous functions and $\|\nabla_{\textbf{w}}\mathcal{L}\|_2$ is bounded from above. If $\mathcal{L}$ is lipchitz continuous then, $K = sup_{\textbf{w} \in \mathbb{R}^d}\|\nabla_{\textbf{w}}\mathcal{L}\|_2$. Let $\hat{L}$ be a random variable defined as :
    \begin{equation}
        \hat{L} = \underset{1 \leq j \leq N}{max} \|\nabla_{\textbf{w}}\mathcal{L}(\textbf{w}_j)\|_2
    \end{equation} 
where $\textbf{w}_j$, $j \in \{1,2,\dots,N\}$ are i.i.d samples drawn from a distribution $\mathcal{W}$. Then, as $n \rightarrow \infty$, the random variable $\hat{L}$ converges in distribution to $K$. \\Mathematically:
\begin{equation}
    \lim_{{n \to \infty}} P(\hat{L} \leq k) = \begin{cases} 0 & \text{if } k < K \\ 1 & \text{if } k \geq K \end{cases}
\end{equation}
\end{theorem}
From Theorem~\ref{theorem: 4}, \( \hat{L} \xrightarrow[]{d} L(\mathcal{L}) \)\footnote{\textbf{Theorem}~\ref{theorem: 4} holds for mini-batch case too. Refer to Appendix for more insights.}. optimisation of the loss gradient can be achieved through iterative techniques like steepest descent~\cite{Boyd_Vandenberghe_2004}, backtracking~\cite{Boyd_Vandenberghe_2004}, and Wolfe's and Goldstein's conditions~\cite{wolfe1969convergence}, aiming to identify the direction (weights \(\textbf{w}\)) maximizing the gradient norm. However, non-convex loss surfaces in deep learning may lead to convergence to suboptimal local maxima. Also, losses with unbounded gradients can lead to excessively high Lipschitz constants. Our approach addresses this by leveraging the \emph{convergence in distribution} phenomenon from probability theory. We randomly select directions from a distribution \(\mathcal{W}\) and take the maximum gradient norm across all evaluated samples, theoretically converging to the true Lipschitz constant in distribution.
\section{Experimental Setup}
\label{sec:exp}
We assess Adam's performance with our selected step size through experiments on fully connected networks with ReLU activations. Additionally, we expand our analysis to include classification tasks on MNIST and CIFAR-10 datasets using LeNet~\cite{lecun1998gradient} and VGG-9~\cite{simonyan2014very}, respectively. To maintain strict experimental control, we ensure that all network layers have identical widths denoted as \( h \). For experiments, in \textbf{Section}~\ref{sec:fb_exp} and~\ref{sec:mb_exp}, we used Kaiming uniform distribution~\cite{he2015delving} for initialization. Default PyTorch parameters are chosen [\hyperlink{https://pytorch.org/docs/stable/nn.init.html}{Link}]. No regularization is applied.



We conducted classification experiments on the MNIST dataset for various network sizes, encompassing different values of network depth and $h$. CIFAR-10 experiments are carried out only on VGG-9 architecture. All experiments are implemented using PyTorch. We compare the performance of Adam using our step size with:
\begin{enumerate}
    \item A list of commonly used schedulers and several non-increasing step sizes that were utilized to demonstrate theoretical convergence for Adam in past literature. The full list is given in \textbf{Table}~\ref{tab:x}.
    \item An array of constant step sizes $\{10^{-2}, 10^{-3}, 10^{-4}, 10^{-5}\}$ which are widely used in literature.
\end{enumerate}
Each experiment was executed for 100 epochs. For various schedulers, we fine-tuned their hyperparameters to identify configurations where they exhibited the best performance and retained those values for the final comparison with our step size. The value for \( k \) in the exponential learning rate in Table~\ref{tab:x} is carefully selected to ensure that the decay rate is moderate. This choice is made to provide a fair comparison with our constant step size approach, as overly rapid decay rates may not effectively minimize the gradient norm of the loss function.




\begin{table}[tb]

  \caption{A detailed description of schedulers is provided in the following table. In this table, the initial learning rate \( \alpha_0 \) is maintained constant at \( 10^{-2} \), and the total number of iterations \( T \) is set to 100 for full-batch experiments. For mini-batch experiments, $T$ depends on the batch size.}
  \label{tab:x}
  \label{tab:headings}
  \centering
  \begin{tabular}{@{}ccc@{}}
    \toprule
    \multicolumn{1}{c}{\textbf{Schedulers}} & \multicolumn{1}{c}{\textbf{Mathematical Form}} & \multicolumn{1}{c}{\textbf{Hyperparameter Choices}}\\
    \midrule
    Linear LR & $\alpha_{t} = \alpha_{t-1}\left(1 + \frac{e_f - s_f}{s_f + e_f(T - 1)}\right)$
 & $s_f = 0.1$ \& $e_f = 1$\\
 \midrule
    Step LR & $\alpha_t = \begin{cases}
\alpha_0 & \text{for } t < 40 \\
\frac{\alpha_0}{10} & \text{for } 40 \leq t \leq 80\\
\frac{\alpha_0}{100} & \text{for } 80 < t \leq T
\end{cases}$ & -- \\
\midrule
    One Cyclic LR & As described in~\cite{smith2019super} & As given in this \href{https://pytorch.org/docs/stable/generated/torch.optim.lr_scheduler.OneCycleLR.html}{link}\\
    \midrule
    Square Root Decay & $\alpha_t = \frac{\alpha_0}{\sqrt{t}}$ & --\\
    \midrule
    Inverse Time Decay  & $\alpha_t = \frac{\alpha_0}{1 + \alpha_{0}t}$ & --\\
    \midrule
    Cosine Decay~\cite{loshchilov2016sgdr} & $\alpha_t = \alpha_{min} + \frac{1}{2}(\alpha_0 - \alpha_{min})(1 + cos\frac{\pi t}{T})$ & $\alpha_{min} = 10^{-3}$\\
    \midrule
    Exponential Decay & $\alpha_t = \alpha_{0}e^{-kt}$ & $k = 10^{-3}$\\
    \midrule
    \textbf{Ours} & $\alpha = \sqrt{\frac{2\mathcal{L}(\textbf{w}_0)}{\hat{L}.T}}$ & --\\
  \bottomrule
  \end{tabular}
\end{table}

Below is the list of experiments we have conducted:

\begin{enumerate}
    \item \textbf{Full-batch experiments (Section~\ref{sec:fb_exp})}: 
    To efficiently compare learning rate schedules (Table~\ref{tab:x}), we created smaller versions of the MNIST and CIFAR-10 datasets: mini-MNIST and mini-CIFAR-10, respectively. Mini-MNIST comprises 5500 training images and 1000 test images, selected from the original MNIST dataset. Mini-CIFAR-10 includes 500 training images and 100 test images per class, representing about 10\% of the full CIFAR-10 dataset. Despite their reduced size, both mini datasets adequately represent their respective originals for experimentation.
    \item \textbf{Mini-batch experiments (Section~\ref{sec:mb_exp})}: To see if our findings from the full-batch experiments apply to mini-batch scenarios according to \textbf{Theorem}~\ref{theorem: 2},  we conducted similar experiments using a mini-batch setup with a fixed batch size of 5,000 for CIFAR10 and MNIST both. We utilized the entire training sets of CIFAR10 and MNIST for these experiments. We tested our models on CIFAR10 using the VGG-9 architecture and on MNIST using the LeNet architecture. We also justified why such a large batch size is needed for empirical convergence.
    \item \textbf{Effect of initialization on our learning rate (Section~\ref{sec:dis_exp})}: As our learning rate $\alpha = \sqrt{\frac{2\mathcal{L}(\textbf{w}_0)}{\hat{L}T}}$ depends on the initial loss value and the estimated Lipschitz constant, we conducted experiments to demonstrate the independence of our learning rate with respect to different network initialization techniques.
    \item \textbf{Additional experiments on CNN architectures (Section~\ref{sec:cnn_exp})}: To examine the potential generalizability of our theoretical findings across architectures, we trained an image classifier on CIFAR-10 and MNIST datasets employing VGG-like networks and LeNet architectures, respectively.
\end{enumerate}

\subsection{Comparing Performances in Full-Batch Setting}
\label{sec:fb_exp}
In Figure~\ref{fig:3_1000} and \ref{fig:3_1000_step}, we illustrate the variations in gradient norm, training loss, and test loss across iterations for different schedulers. These experiments were conducted on single-layer, three-layer, and five-layer classifiers with 1000 nodes in each hidden layer, trained on mini-MNIST. Additional comparisons for various neural network architectures with different widths and depths are provided in the appendix, where similar qualitative trends can be observed.
\\
\\
\textbf{(1) Comparison against various learning rate schedulers}: We compared our step size against several well-known schedulers listed in Table~\ref{tab:x}. Except for cosine decay and one-cycle learning rate, all other schedulers are non-increasing. Our goal was to evaluate their efficacy in driving the gradient norm of the loss function towards zero.
\begin{figure}[htbp]
  \centering
  \begin{subfigure}[b]{0.31\textwidth}
    \includegraphics[width=\textwidth]{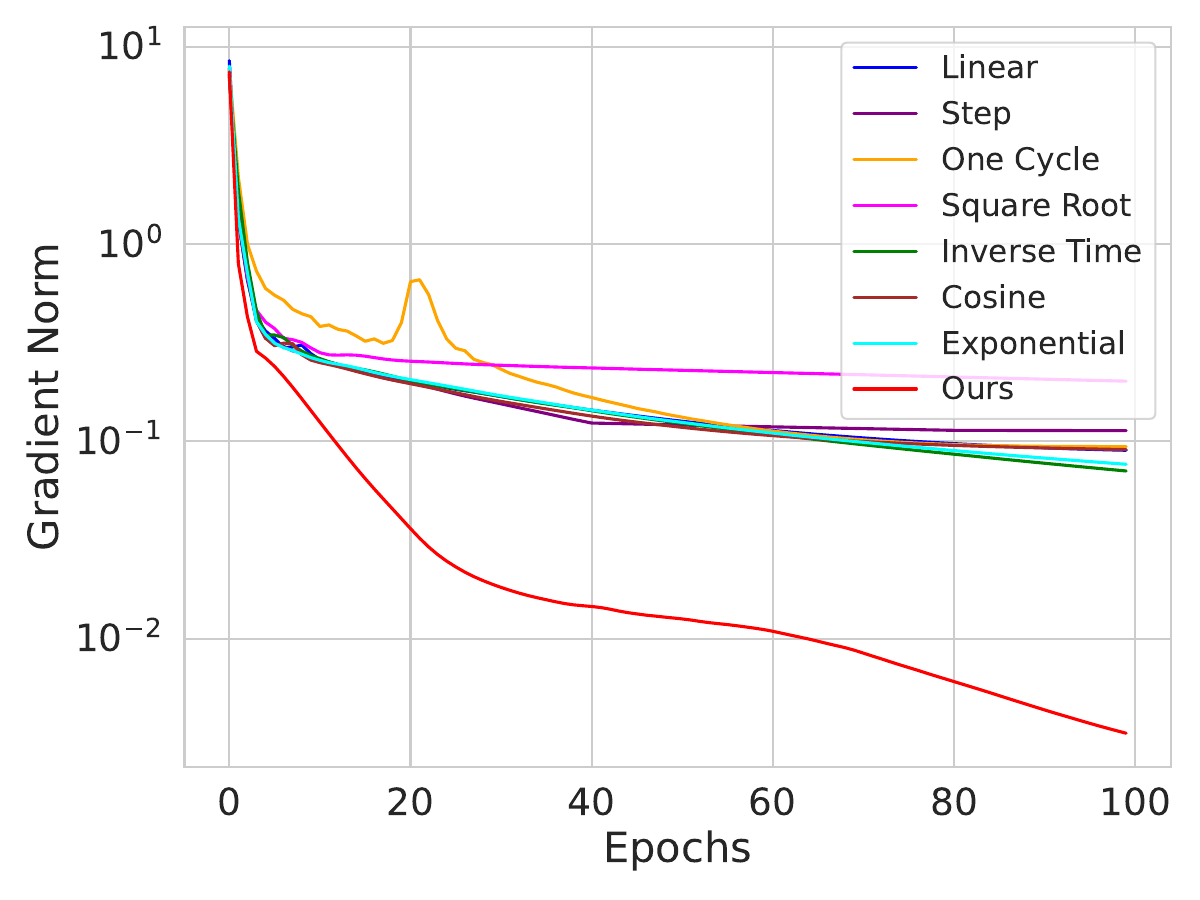}
    \caption{Gradient norm v/s Epochs}
  \end{subfigure}
  \hfill
  \begin{subfigure}[b]{0.31\textwidth}
    \includegraphics[width=\textwidth]{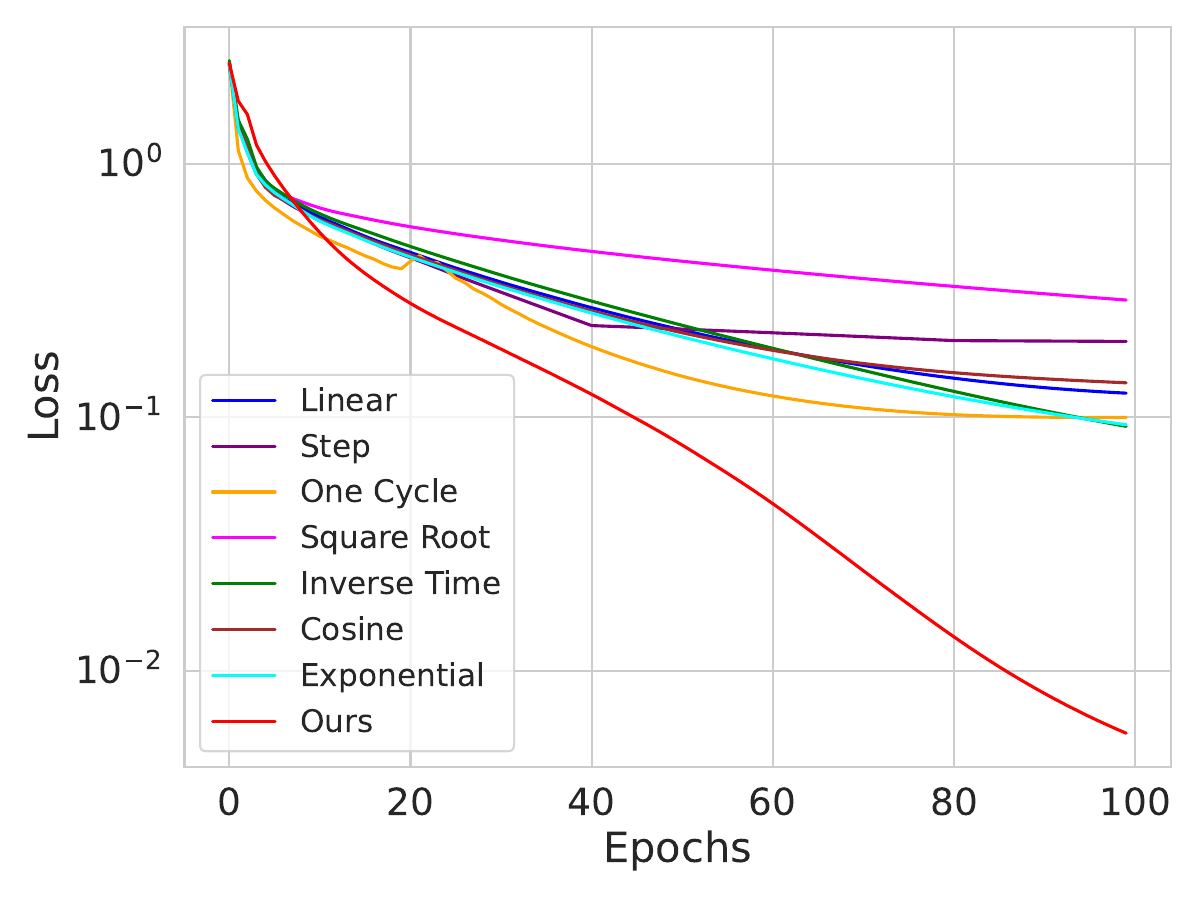}
    \caption{Training loss v/s Epochs}
  \end{subfigure}
  \hfill
  \begin{subfigure}[b]{0.31\textwidth}
    \includegraphics[width=\textwidth]{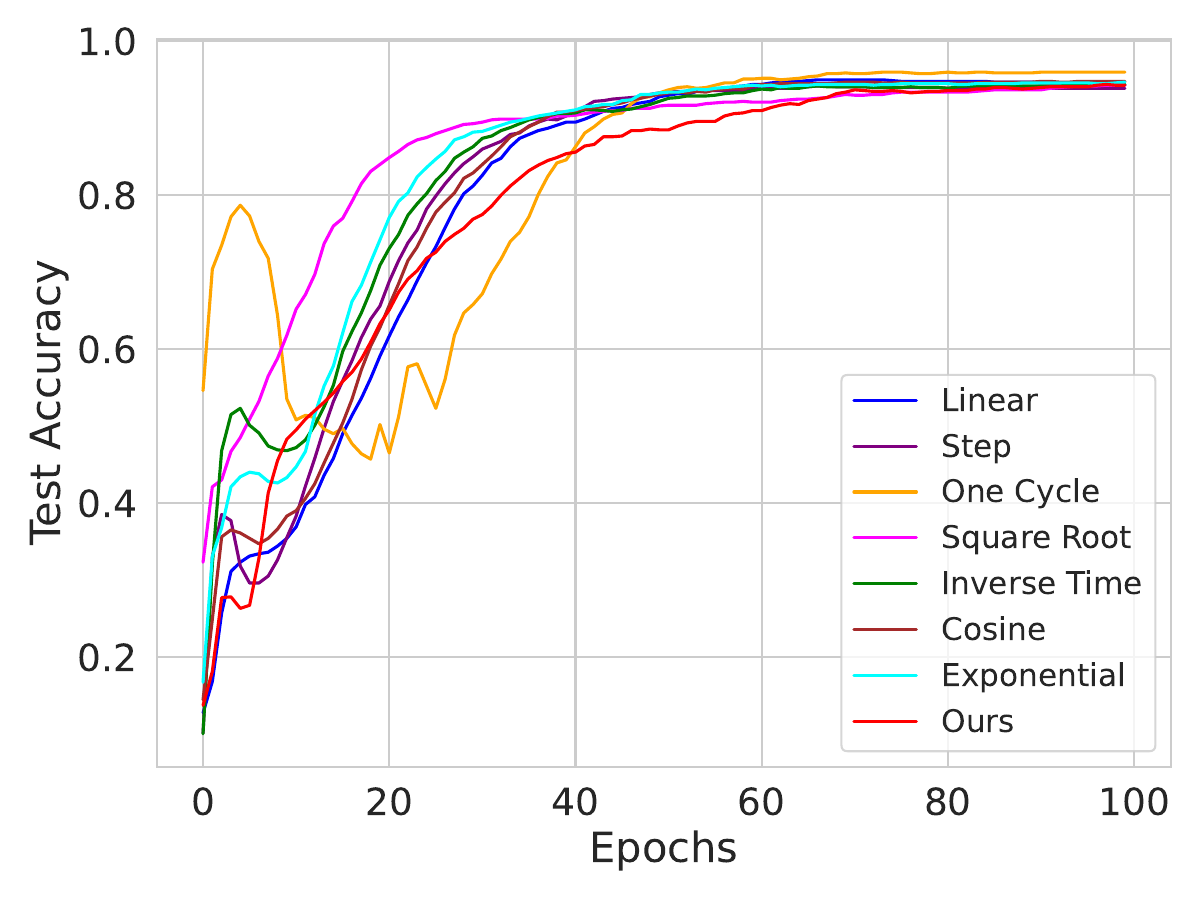}
    \caption{Validation acc. v/s Epochs}
  \end{subfigure}
  \caption{Full-batch experiments on a 3 layer network with 1000 nodes in each layer, trained on MNIST.}
\label{fig:3_1000}
\end{figure}
\\
\textbf{(2) Comparison against various fixed learning rates}: We tested our learning rate against several other fixed learning rates, $\{10^{-2}, 10^{-3}, 10^{-3}, 10^{-4}, 10^{-5}\}$. We wanted to demonstrate that instead of searching extensively for the best learning rate to train a model, one could just use our learning rate and achieve decent results.

\begin{figure}[htbp]
  \centering
  \begin{subfigure}[b]{0.31\textwidth}
    \includegraphics[width=\textwidth]{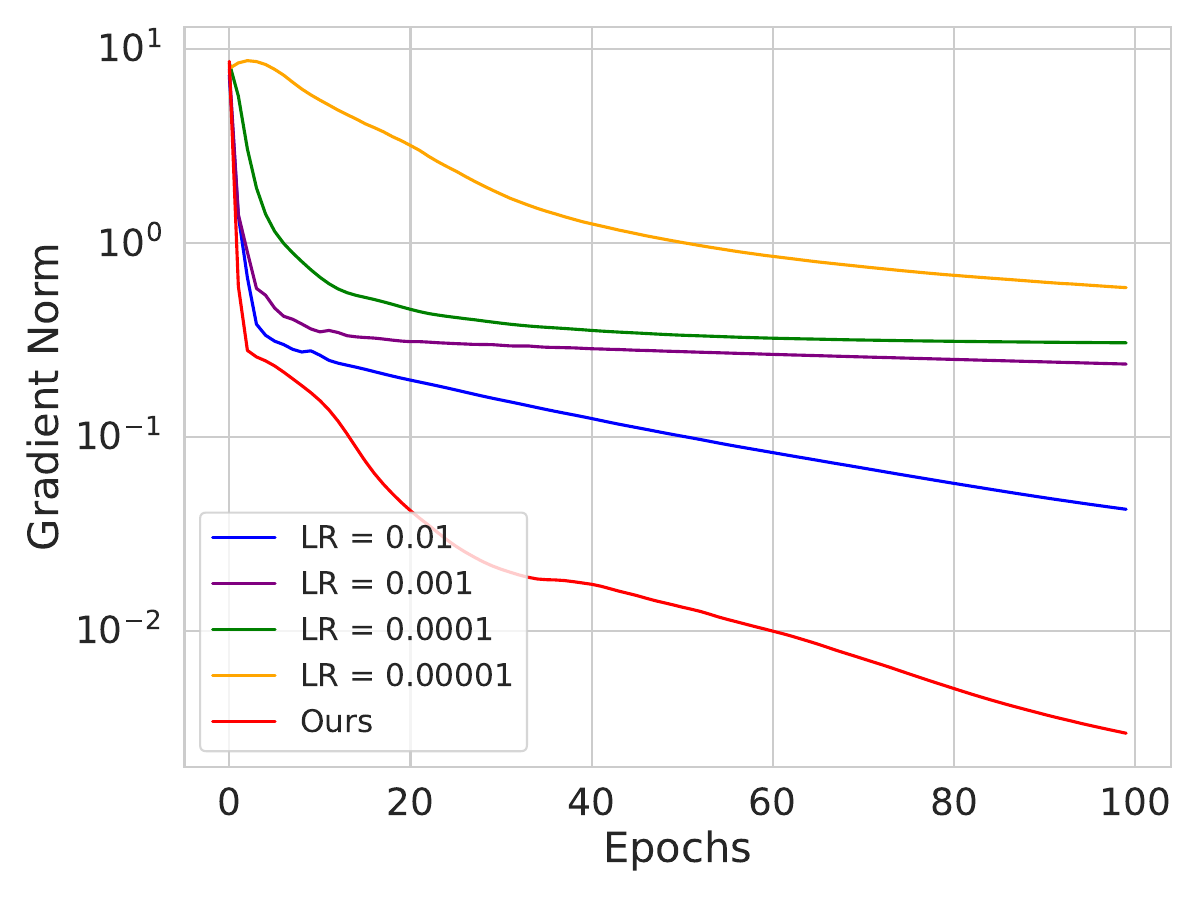}
    \caption{Gradient norm v/s Epochs}
  \end{subfigure}
  \hfill
  \begin{subfigure}[b]{0.31\textwidth}
    \includegraphics[width=\textwidth]{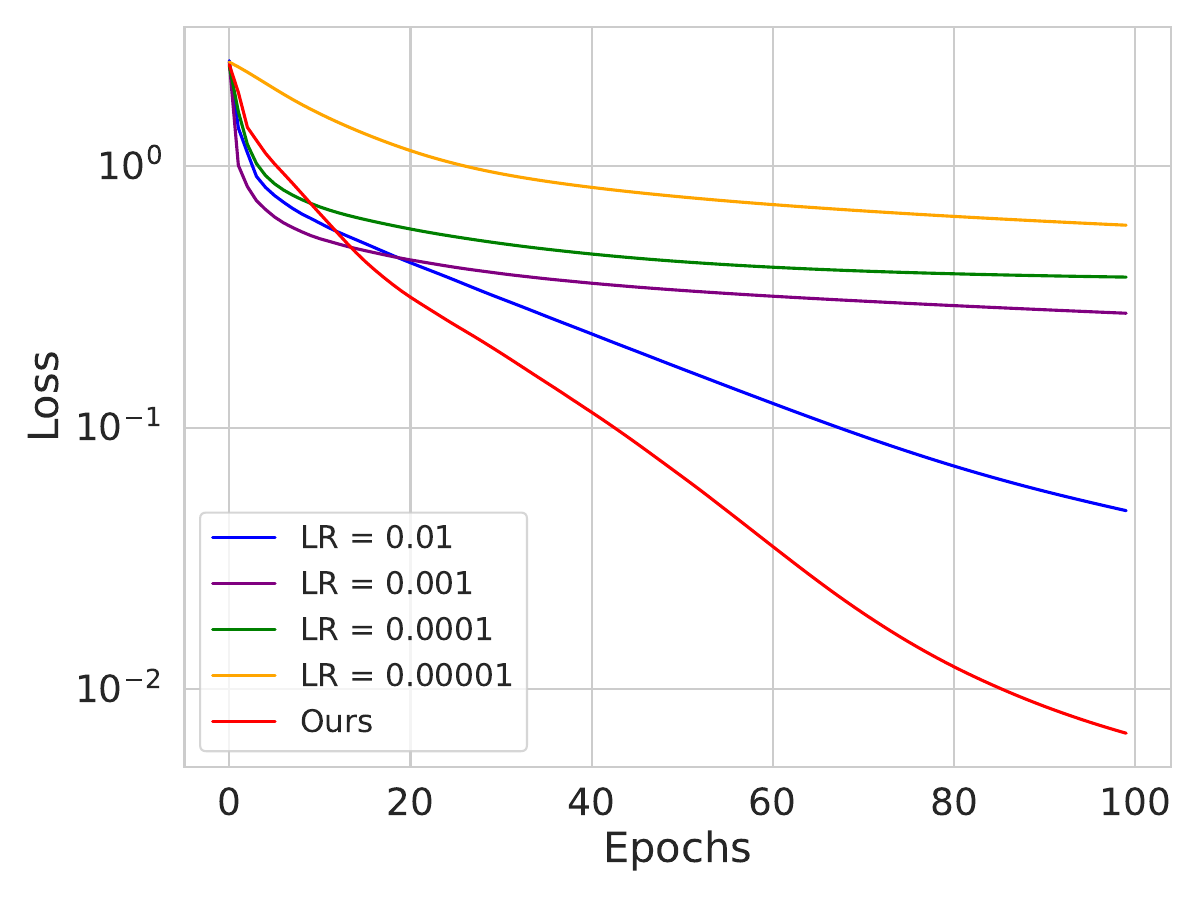}
    \caption{Training loss v/s Epochs}
  \end{subfigure}
  \hfill
  \begin{subfigure}[b]{0.31\textwidth}
    \includegraphics[width=\textwidth]{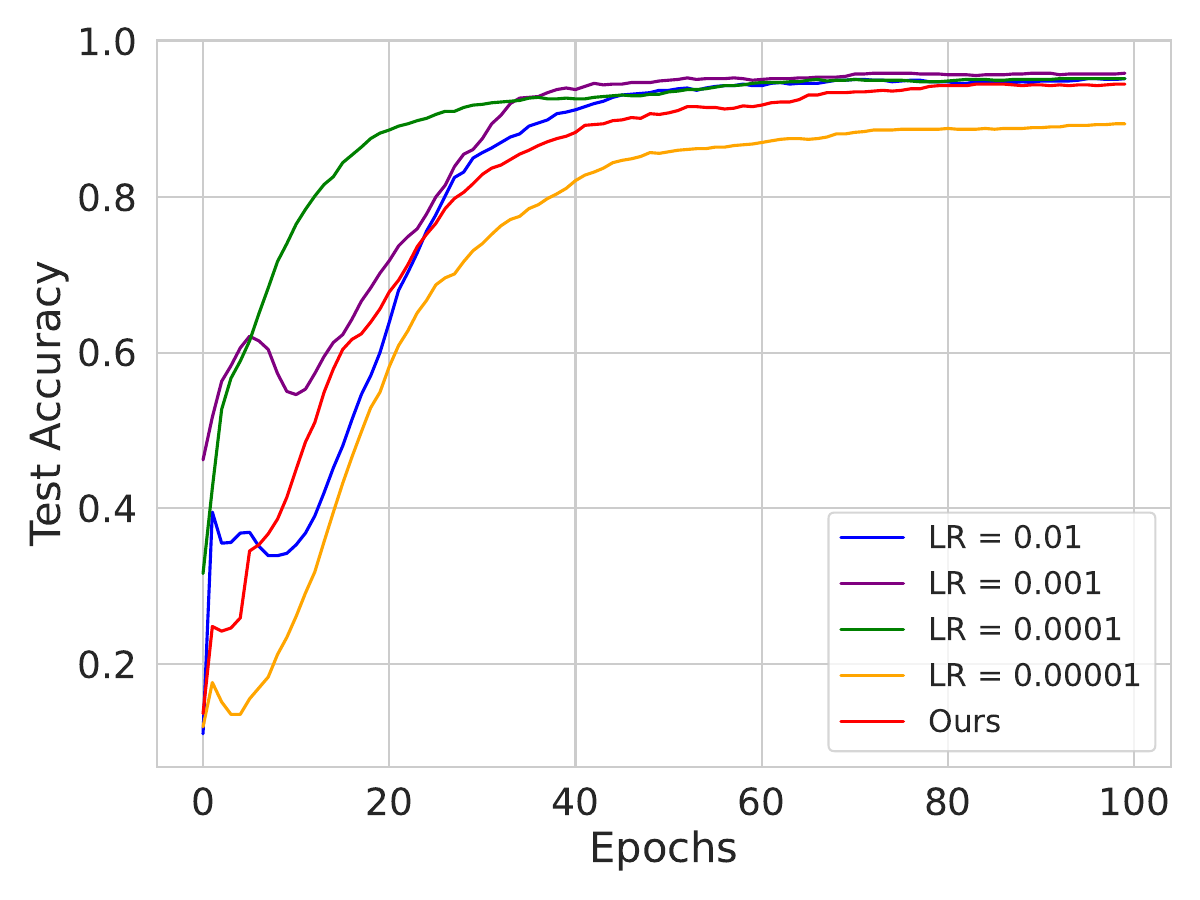}
    \caption{Validation acc. v/s Epochs}
  \end{subfigure}
  \caption{Full-batch experiments on a 3 layer network with 1000 nodes in each layer, trained on MNIST.}
\label{fig:3_1000_step}
\end{figure}

Looking at Figures~\ref{fig:3_1000} and~\ref{fig:3_1000_step}, it is clear that our suggested learning rate helps decrease the gradient norm of loss more aggressively towards zero as compared to other schedulers and non-increasing learning rates, supporting our theory. Even with the accumulation of a few past gradients, deterministic Adam proves ineffective in mitigating the rapid decay of the learning rate. This is evident from Figure~\ref{fig:3_1000}, where one can observe a notable difference in gradient norms after 100 epochs between the best-performing scheduler (Exponential in this case) and our learning rate is in order of $10^{-2}$. 

Additionally, these figures demonstrate that our learning rate not only reaches any stationary point but tends to approach a local minimum or its neighbourhood where the model achieves good validation accuracy. The plot of loss versus epoch indicates that our proposed learning rate leads to faster convergence in practice as compared to other schedulers. However, it's worth noting that achieving faster convergence in general is still a challenge. Nevertheless, using this learning rate results in quicker convergence than other schedulers and a range of fixed learning rates.

\subsection{Comparing Performances in Mini-Batch Setting}
\label{sec:mb_exp}
In this section, we replicate the same set of experiments in mini-batch setting. We choose 5,000 as batch size for both the CIFAR10 and MNIST datasets.

\begin{figure}[htbp]
  \centering
  \begin{subfigure}[b]{0.31\textwidth}
    \includegraphics[width=\textwidth]{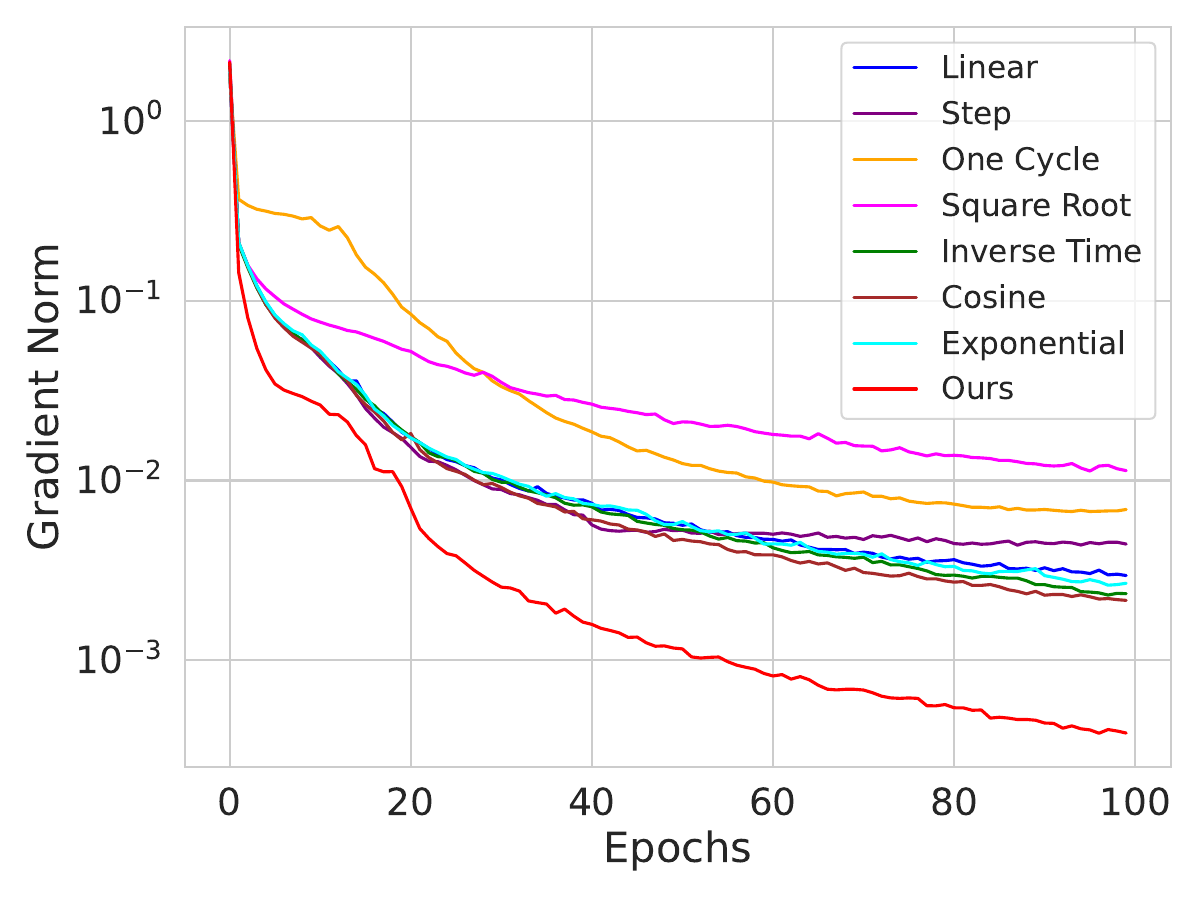}
    \caption{Gradient norm v/s Epochs}
  \end{subfigure}
  \hfill
  \begin{subfigure}[b]{0.31\textwidth}
    \includegraphics[width=\textwidth]{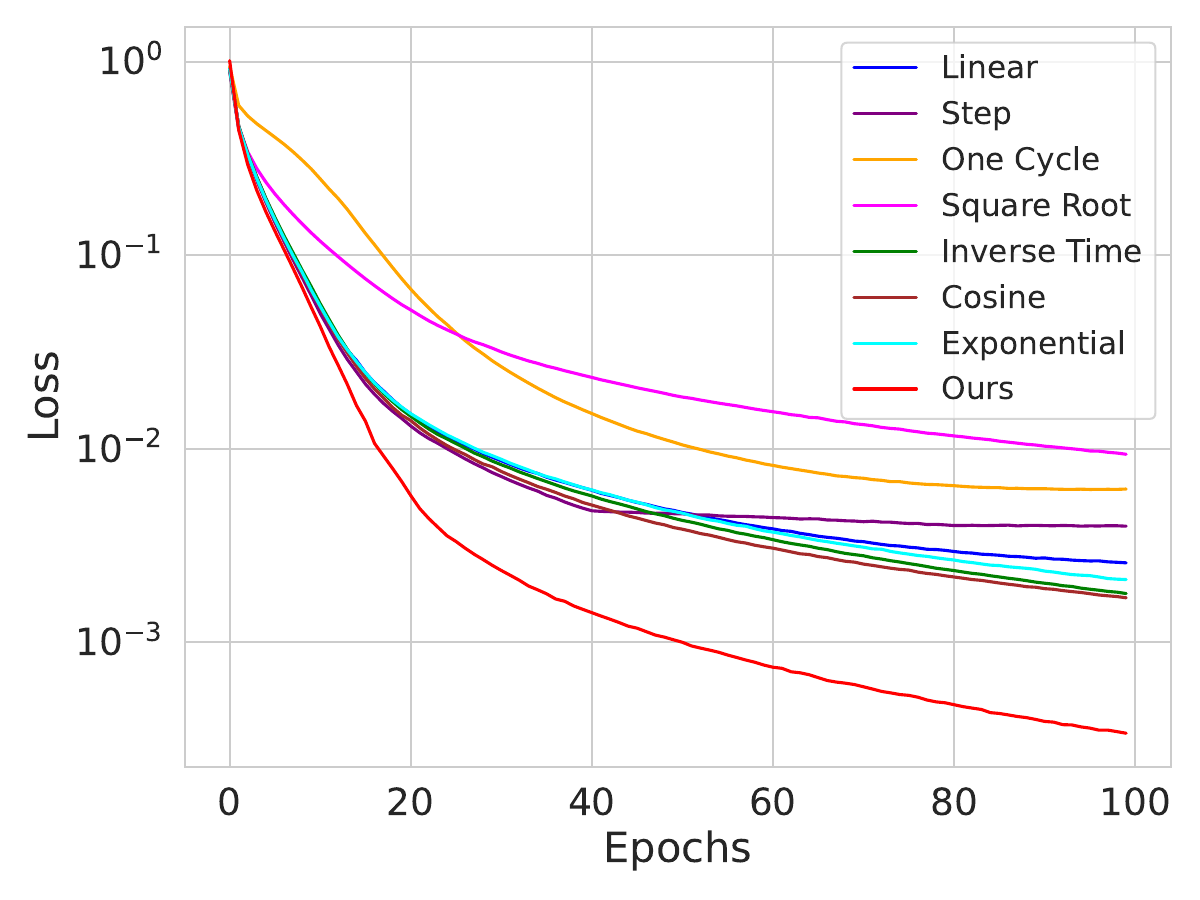}
    \caption{Training loss v/s Epochs}
  \end{subfigure}
  \hfill
  \begin{subfigure}[b]{0.31\textwidth}
    \includegraphics[width=\textwidth]{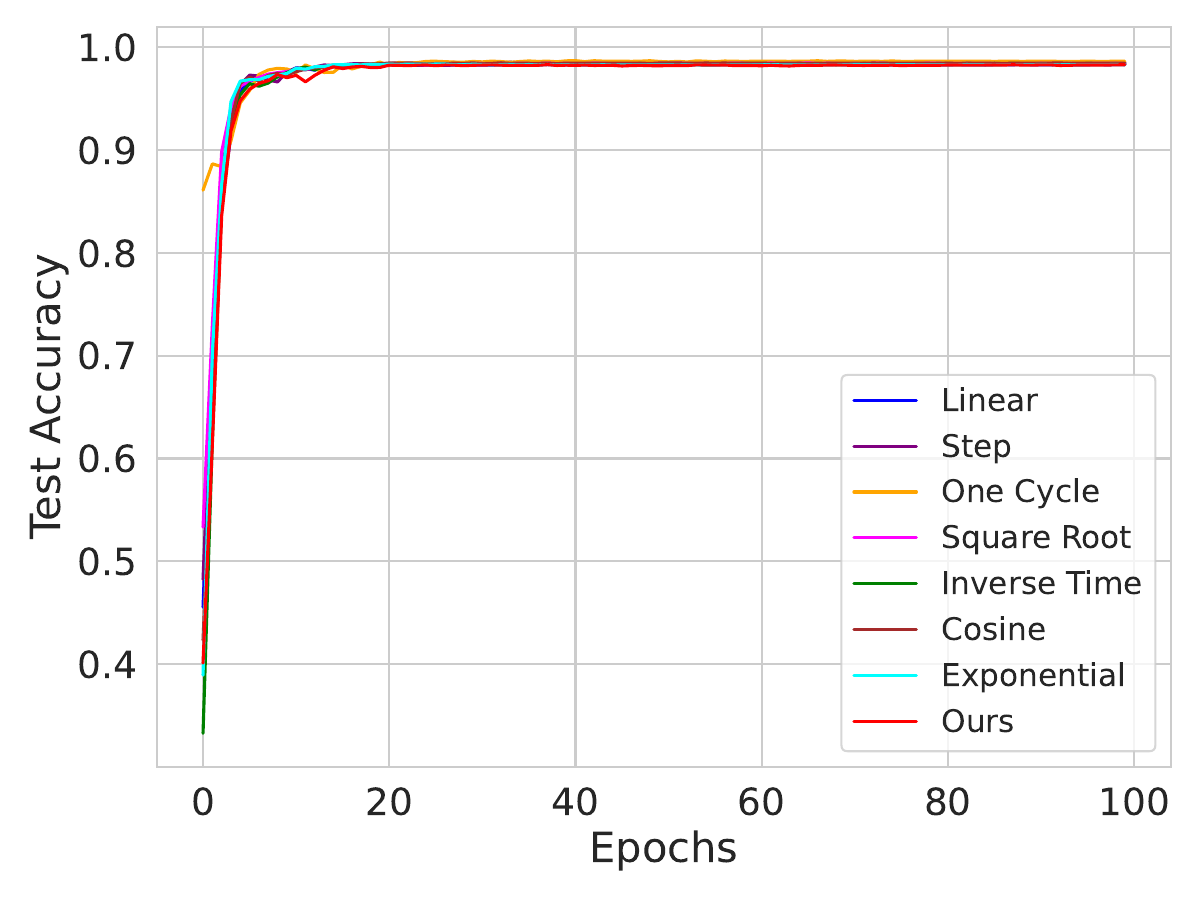}
    \caption{Validation acc. v/s Epochs}
  \end{subfigure}
  \caption{\textbf{Comparision with schedulers}: Mini-batch experiments on a 3 layer network with 1000 nodes in each layer, trained on MNIST.}
\label{fig:3_1000_mini}
\end{figure}
\begin{figure}[htbp]
  \centering
  \begin{subfigure}[b]{0.31\textwidth}
    \includegraphics[width=\textwidth]{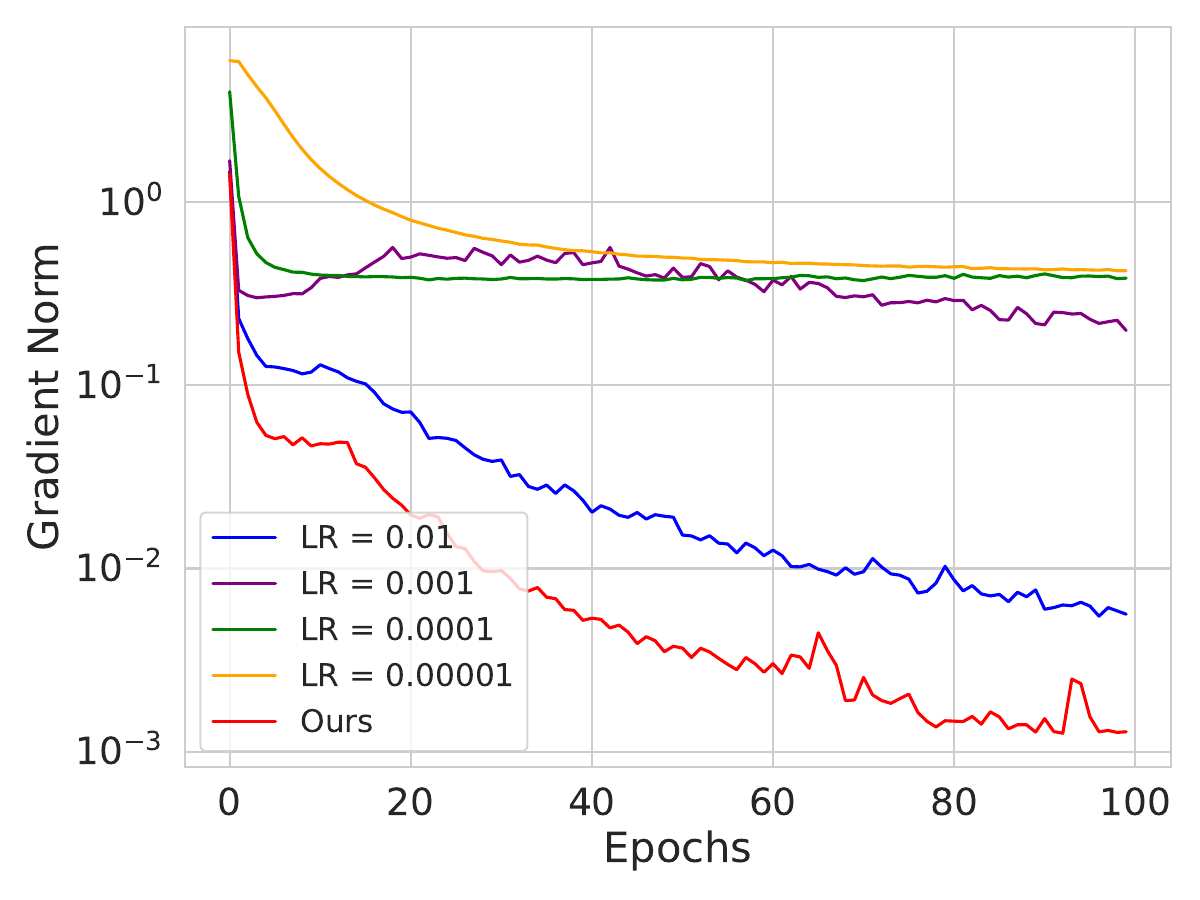}
    \caption{Gradient norm v/s Epochs}
  \end{subfigure}
  \hfill
  \begin{subfigure}[b]{0.31\textwidth}
    \includegraphics[width=\textwidth]{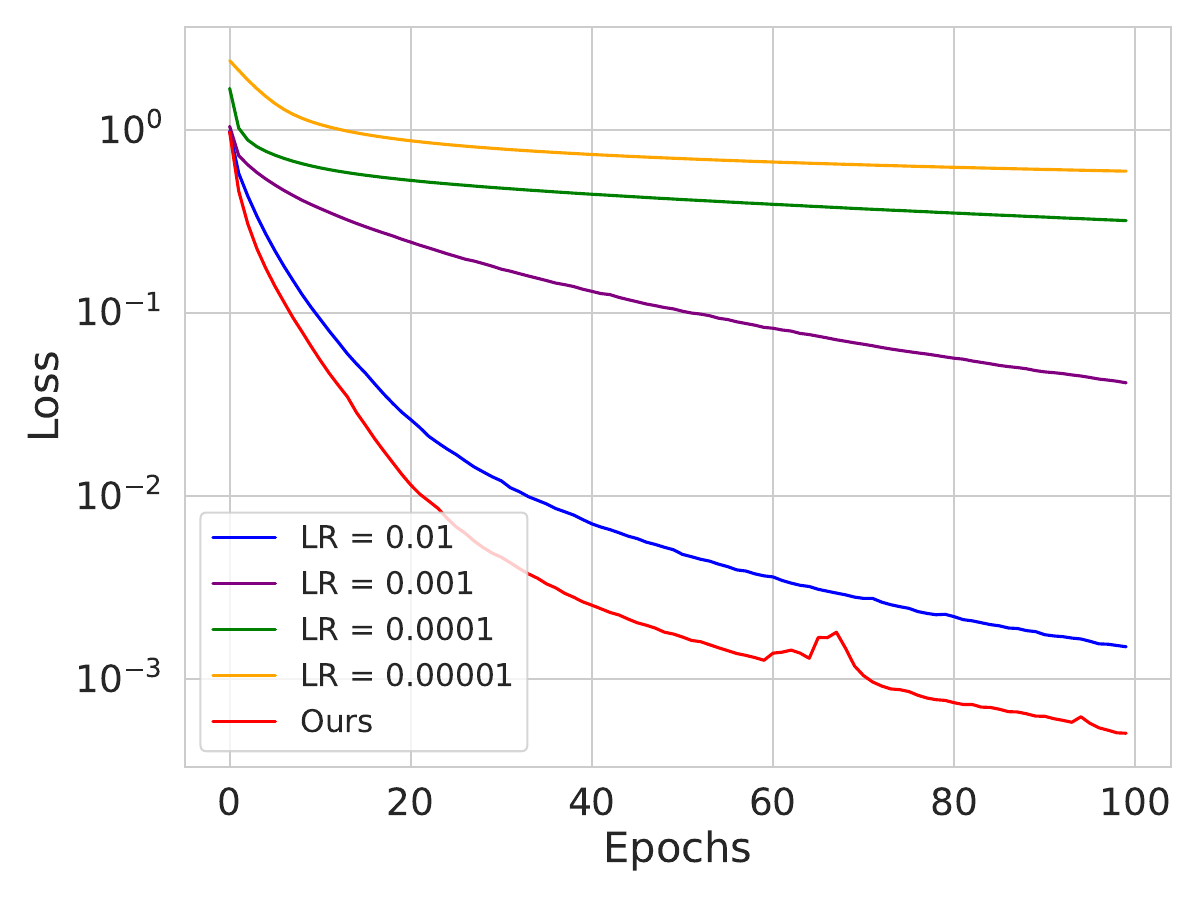}
    \caption{Training loss v/s Epochs}
  \end{subfigure}
  \hfill
  \begin{subfigure}[b]{0.31\textwidth}
    \includegraphics[width=\textwidth]{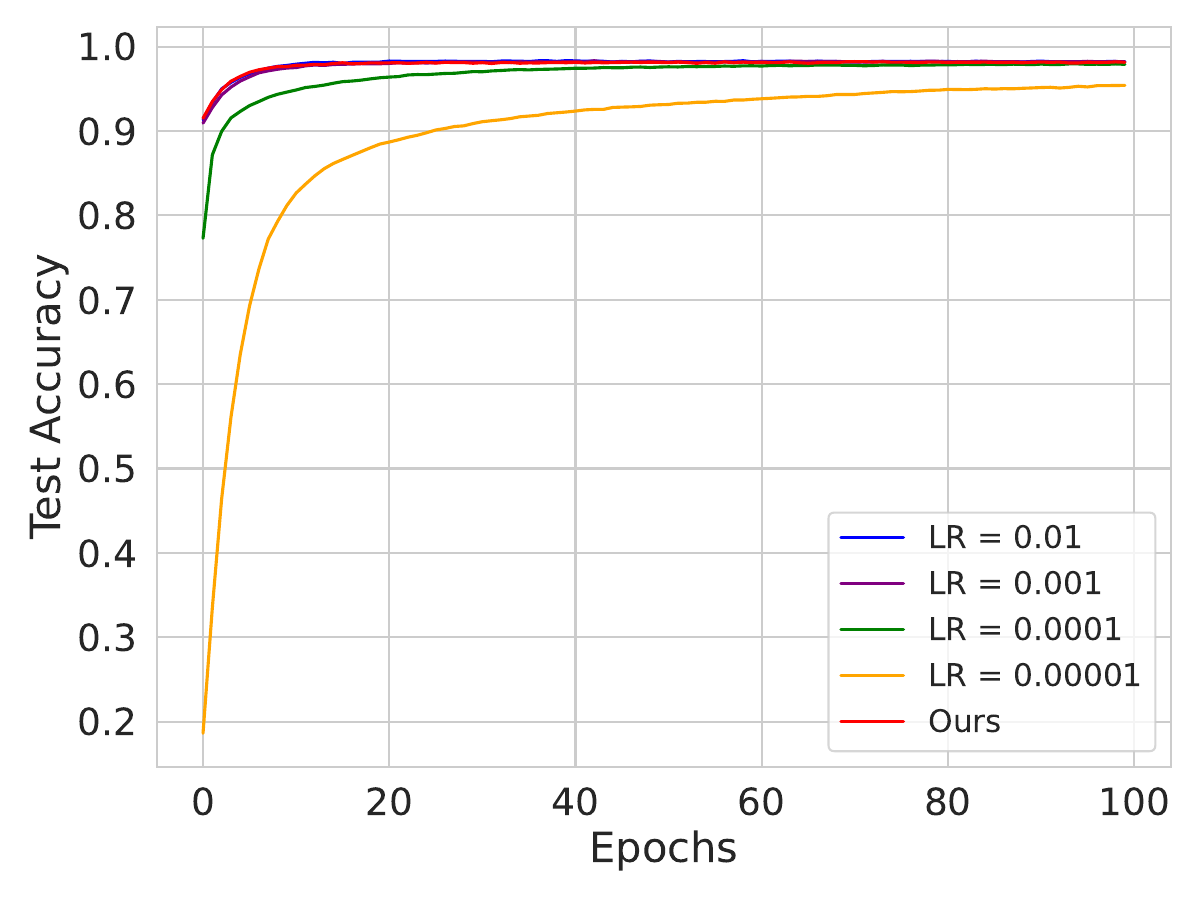}
    \caption{Validation acc. v/s Epochs}
  \end{subfigure}
  \caption{\textbf{Comparision with constant step sizes}: Mini-batch experiments on a single layer network with 1000 nodes in each layer, trained on MNIST.}
\label{fig:3_1000_mini_step}
\end{figure}
We observed that even in the mini-batch setting, our learning rate effectively reduces the gradient norm of loss more aggressively compared to other schedulers and constant step sizes. This trend aligns with the behaviour observed in the full-batch setting across all plots, including gradient norm, training loss, and validation accuracy.
\\
\textbf{Why such a large batch size is chosen?} Despite the non-convergence issue demonstrated by \emph{Reddi et al.}~\cite{sashank2018convergence}, it does not rule out convergence if the minibatch size increases over time, thereby reducing the variance of stochastic gradients. Increasing minibatch size has been shown to aid convergence in some optimisation algorithms not based on exponential moving average (EMA) methods~\cite{bernstein2018signsgd, hazan2015beyond}.

In all our minibatch experiments (including those in the appendix), we choose a large batch size because our theoretical analysis for the minibatch setting, as described in \textbf{Theorem}~\ref{theorem: 2}, assumes that the loss is $K-Lipschitz$ and is represented by $\mathcal{L} = \sum_{j=1}^{m}\mathcal{L}_{j}$. Our proposed learning rate relies on the Lipschitz constant of the loss function. To accurately estimate this constant, it is crucial to estimate the stochastic gradients with low variance, as the Lipschitz constant is essentially the supremum of the true gradient norm over the network weights (\textbf{Theorem}~\ref{theorem: 3}). This approximation of gradients affects the estimation of the Lipschitz constant. Therefore, to make our step size effective in the minibatch setting, it is crucial to estimate the gradients with low variance, which in turn provides the correct estimation of the Lipschitz constant. Our empirical analysis suggests that increasing batch size will lead to convergence, as increasing batch size decreases variance. We now show results with our learning rate with increasing minibatch size.
\begin{figure}[htbp]
  \centering
  \begin{subfigure}[b]{0.31\textwidth}
    \includegraphics[width=\textwidth]{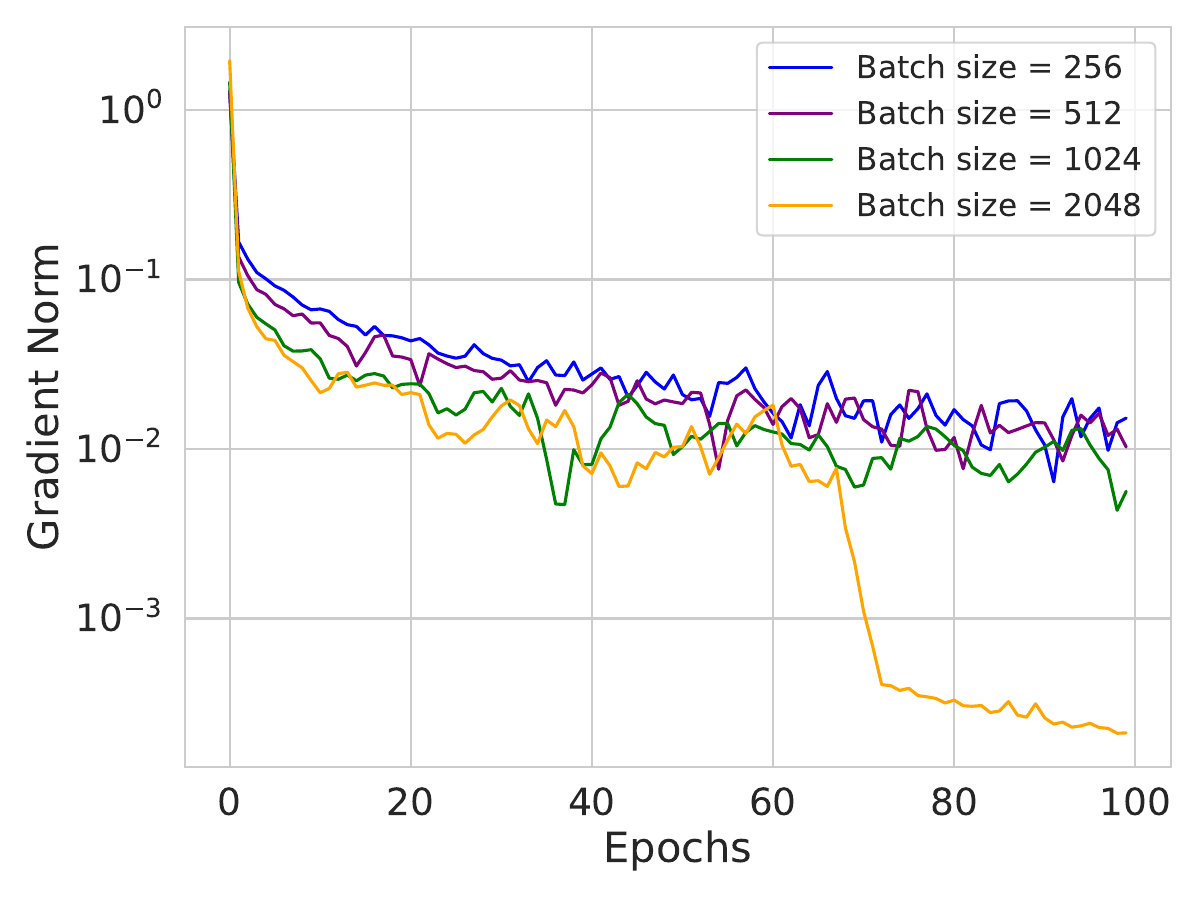}
  \end{subfigure}
  \hfill
  \begin{subfigure}[b]{0.31\textwidth}
    \includegraphics[width=\textwidth]{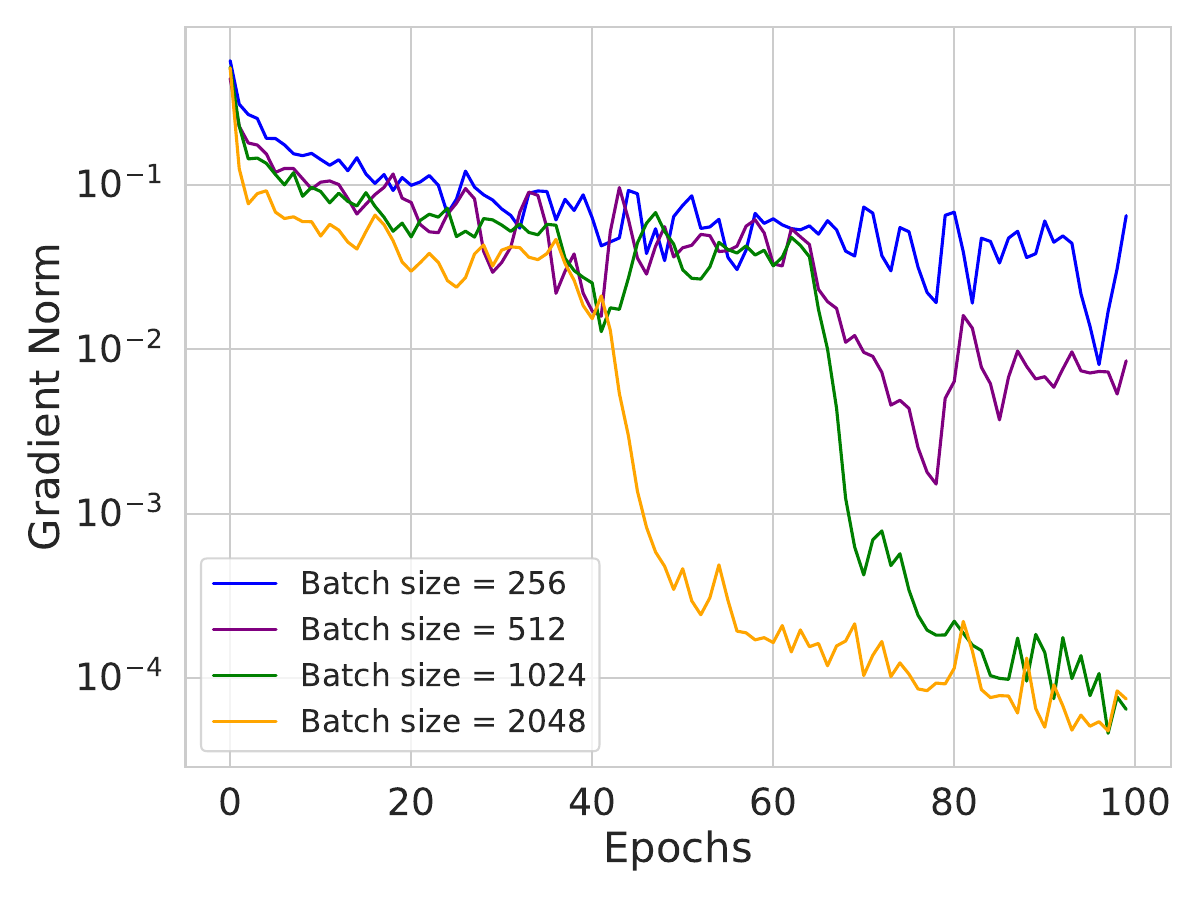}
  \end{subfigure}
  \hfill
  \begin{subfigure}[b]{0.31\textwidth}
    \includegraphics[width=\textwidth]{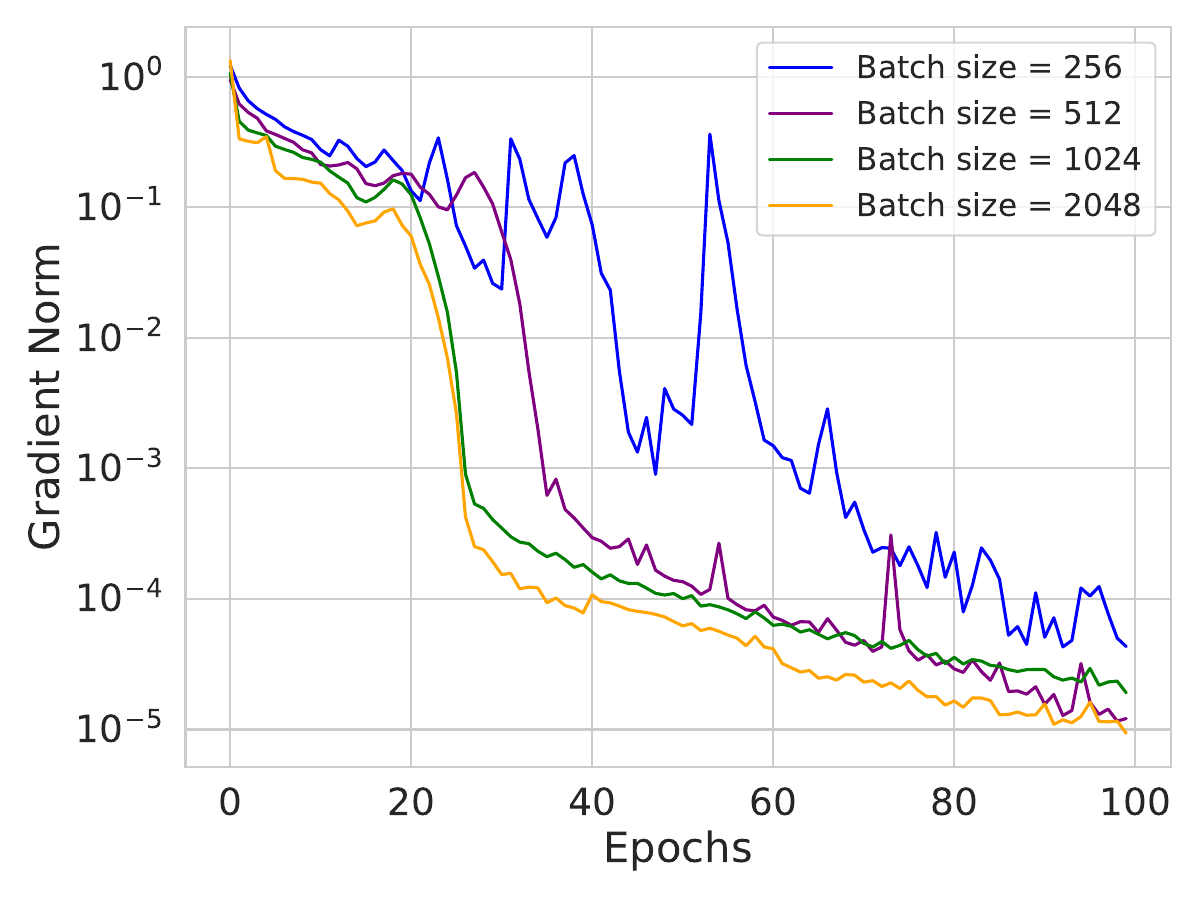}
  \end{subfigure}
  \caption{Increasing batch size experiments on \emph{Left}: a linear network with five layers and 3000 nodes in each layer, trained on MNIST, \emph{Middle}: LeNet architecture trained on MNIST and \emph{Right}: VGG-9 architecture trained on CIFAR-10.}
\label{fig:batch_size}
\end{figure}
\subsection{Effect of Initialization on Our Learning Rate}
\label{sec:dis_exp}
Our learning rate $\alpha = \sqrt{\frac{2\mathcal{L}(\textbf{w}_0)}{\hat{L}.T}}$ is dependent on the initial value of the loss $\mathcal{L}(\textbf{w}_0)$ and the estimated Lipchitz constant $\hat{L}$. Therefore, in this section, we empirically demonstrated that our learning is not sensitive to various types of initialization strategies in both full-batch and mini-batch setups. 

We employed a diverse range of weight initialization strategies, including sampling from various distributions and initialization techniques. Specifically, we randomly sampled weights from different distributions and estimated the Lipschitz constant using the same. Subsequently, we visualized the gradient norm of the networks using our learning rate under these weight initialization strategies. These strategies include sampling from the normal distribution and uniform distribution, as well as employing initialization methods such as orthogonal initialization~\cite{saxe2013exact}, Xavier normal initialization~\cite{glorot2010understanding}, and Kaiming uniform initialization~\cite{he2015delving}. The details of the parameters involved in each initialization method is differed to the appendix due to space constraints.
\begin{figure}
  \centering
  \begin{subfigure}[b]{0.31\textwidth}
    \includegraphics[width=\textwidth]{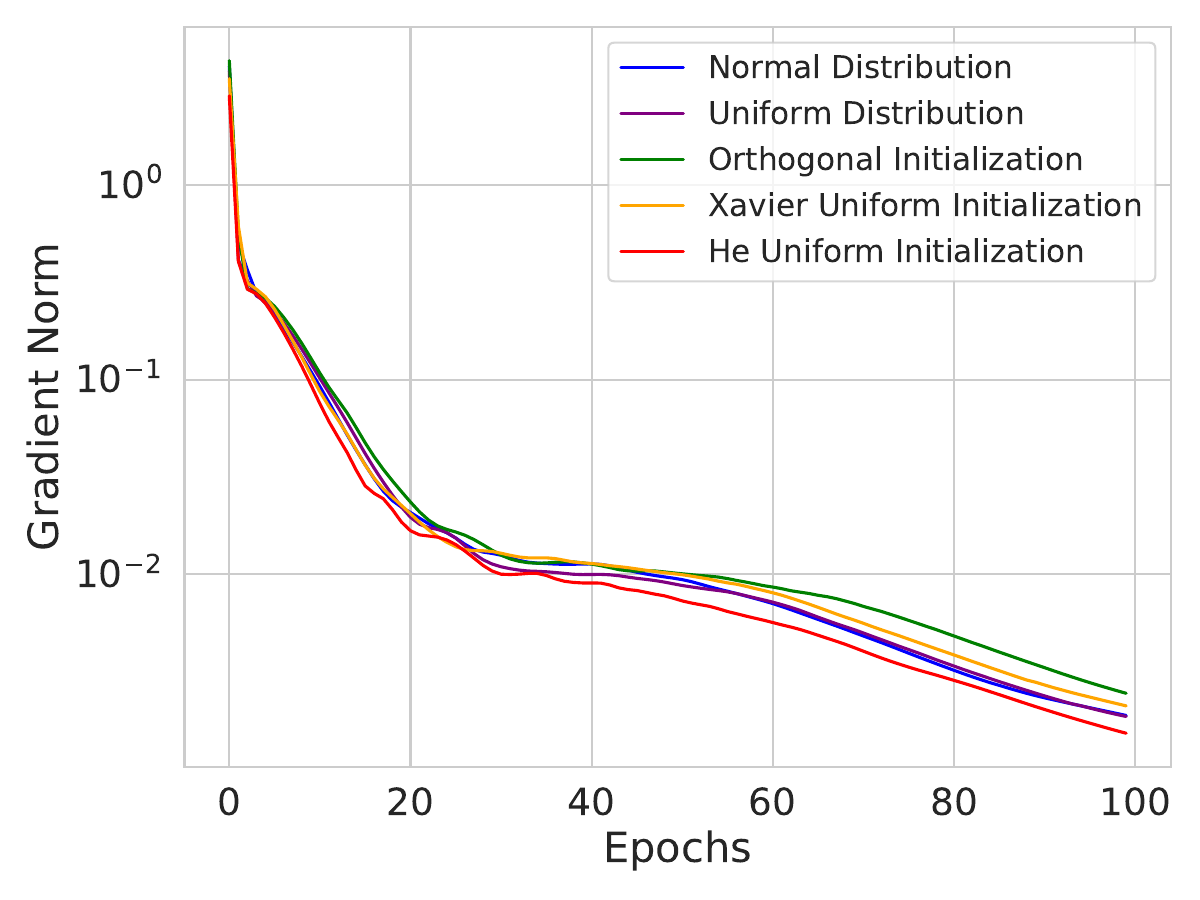}
  \end{subfigure}
  \hfill
  \begin{subfigure}[b]{0.31\textwidth}
    \includegraphics[width=\textwidth]{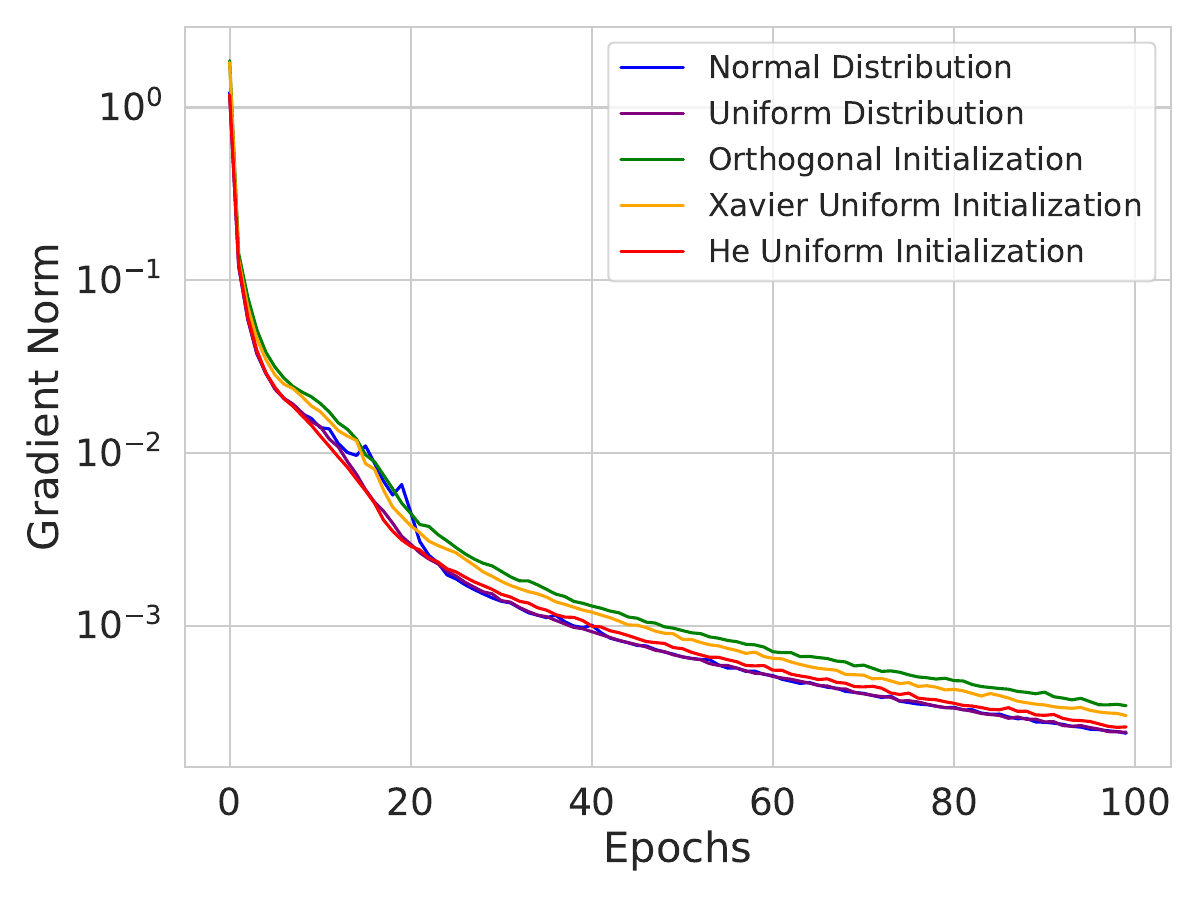}
  \end{subfigure}
  \hfill
  \begin{subfigure}[b]{0.31\textwidth}
    \includegraphics[width=\textwidth]{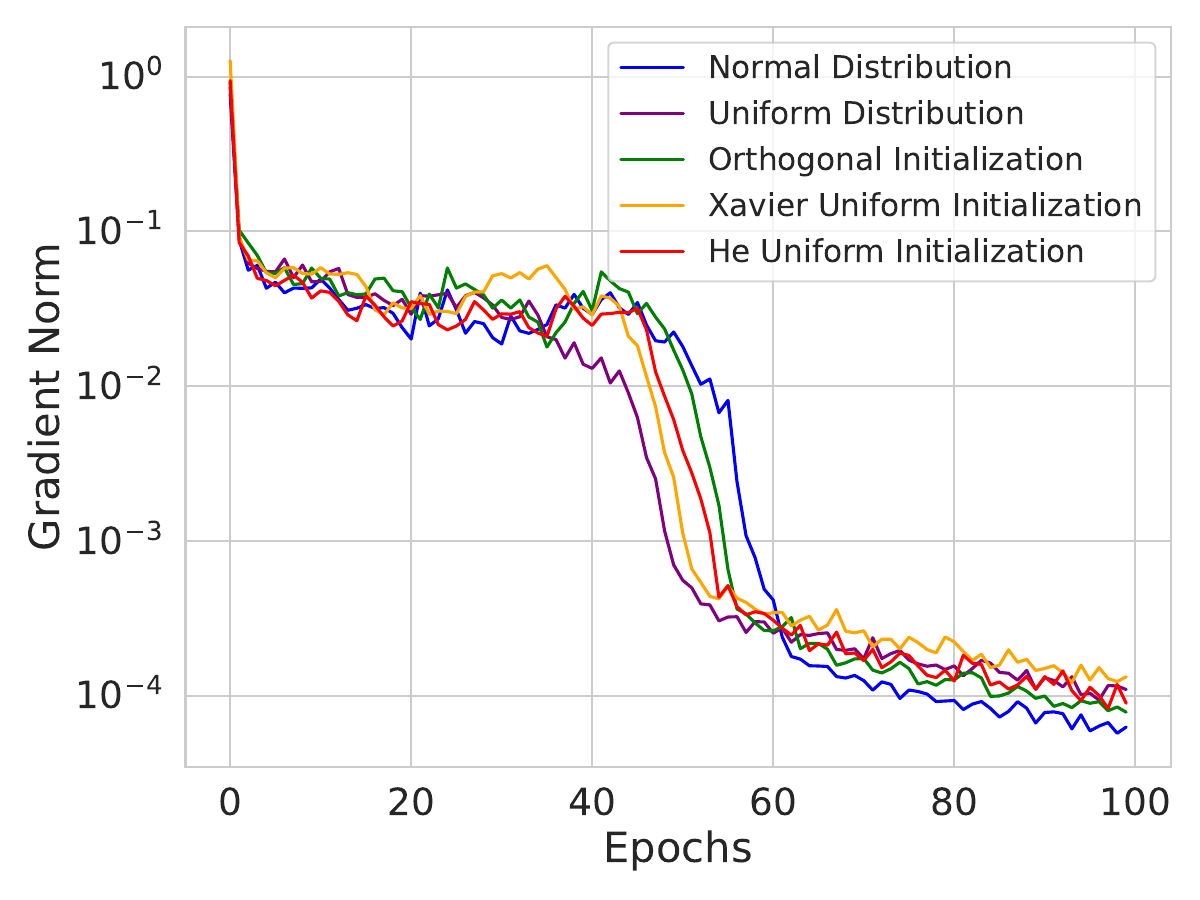}
  \end{subfigure}
  \caption{Effect of network initialization: \emph{Left}: a linear network with 3 layers and 1000 nodes in each layer with full-batch setting \emph{Middle}: a linear network with 3 layers and 3000 nodes in each layer with mini-batch size of 5,000 and \emph{Right}: LeNet architecture with mini-batch size of 5,000. All three networks are trained on MNIST.}
\label{fig:distribution}
\end{figure}

From Figure~\ref{fig:distribution}, it is apparent that our learning rate yields similar results across various initialization strategies in both full-batch and mini-batch setups across various architectures. Therefore, we can conclude that empirically, the performance of our learning rate is independent of the initialization method.

\subsection{LeNet on MNIST and VGG-9 on CIFAR-10}
\label{sec:cnn_exp}
To test whether these results might qualitatively hold for models, we trained an image classifier on CIFAR-10 and MNIST using VGG-9 and LeNet, respectively\footnote{Detailed architecture of VGG-9 and LeNet are given in Appendix due to space constraint.}. 
We use minibatches of size 5,000 for the LeNet and 2,500 for VGG-9. 

\begin{figure}
  \centering
  \begin{subfigure}[b]{0.31\textwidth}
    \includegraphics[width=\textwidth]{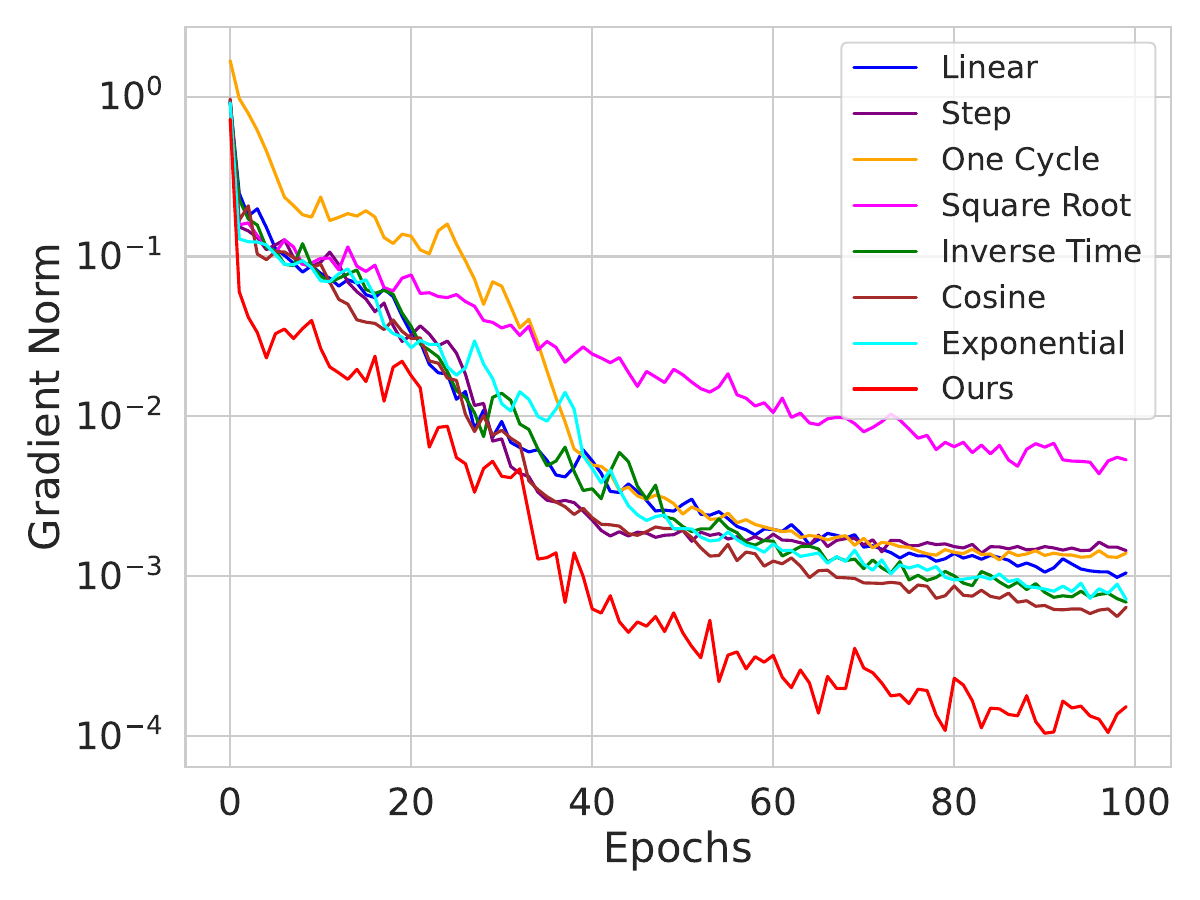}
  \end{subfigure}
  \hfill
  \begin{subfigure}[b]{0.31\textwidth}
    \includegraphics[width=\textwidth]{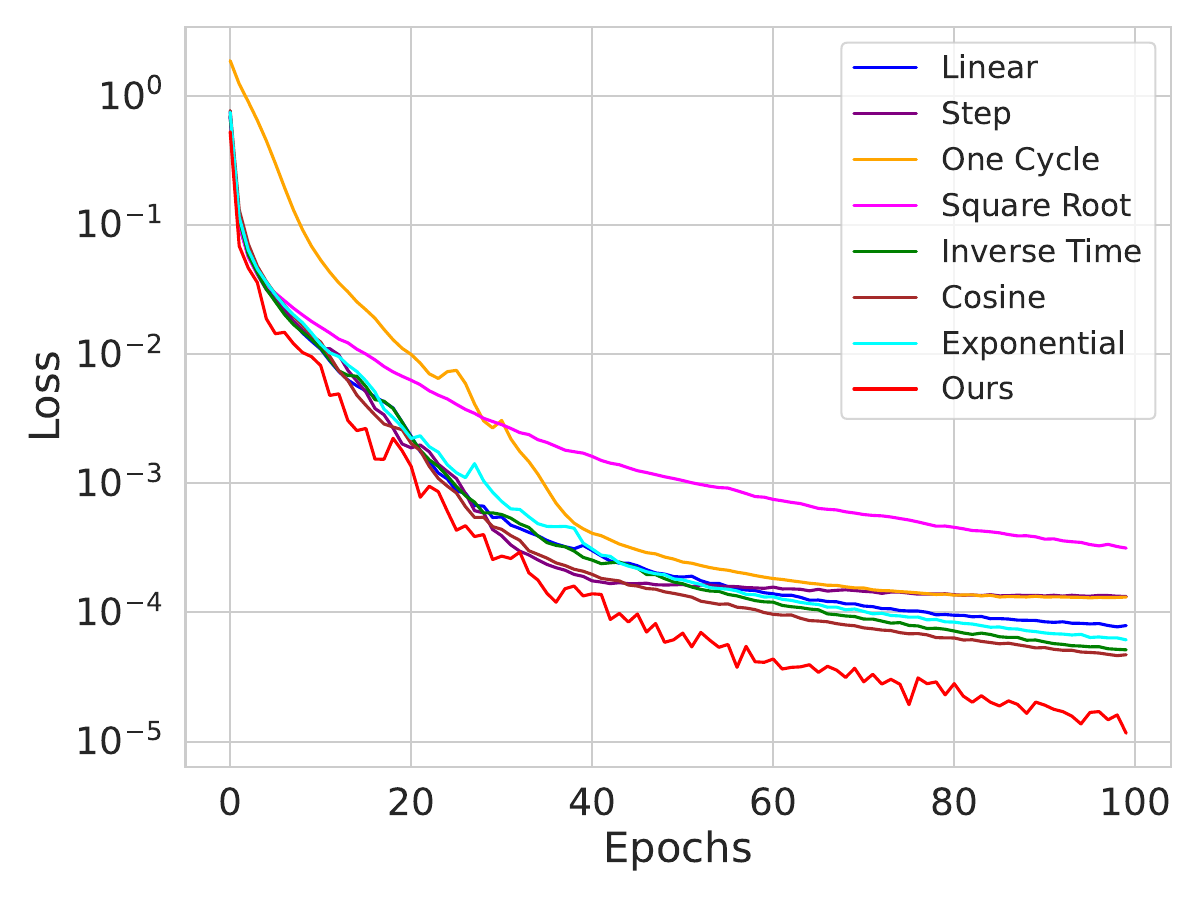}
  \end{subfigure}
  \hfill
  \begin{subfigure}[b]{0.31\textwidth}
    \includegraphics[width=\textwidth]{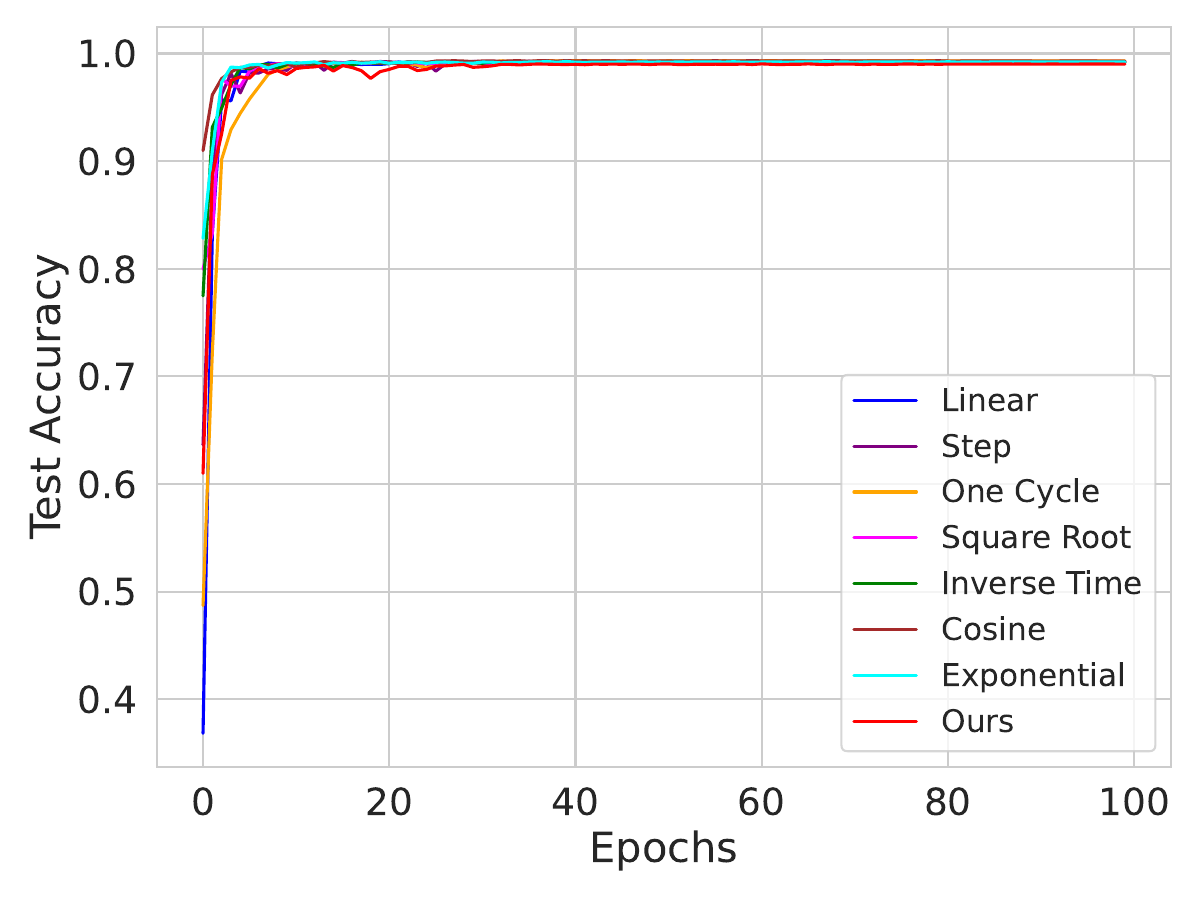}
  \end{subfigure}
  \caption{Mini-batch experiments on LeNet architecture with MNIST data.}
\label{fig:lenet}
\end{figure}
  
\begin{figure}
 \centering
  \begin{subfigure}[b]{0.31\textwidth}
    \includegraphics[width=\textwidth]{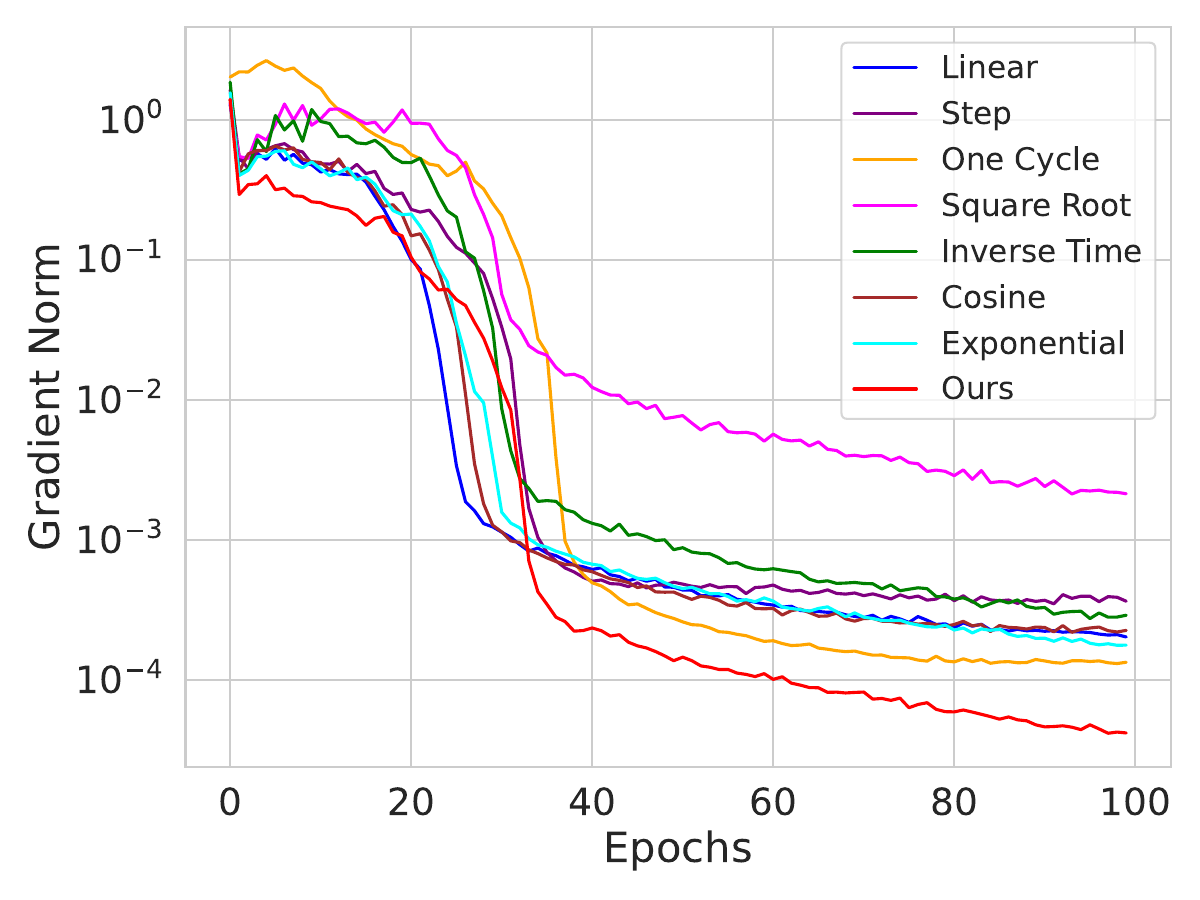}
  \end{subfigure}
  \hfill
  \begin{subfigure}[b]{0.31\textwidth}
    \includegraphics[width=\textwidth]{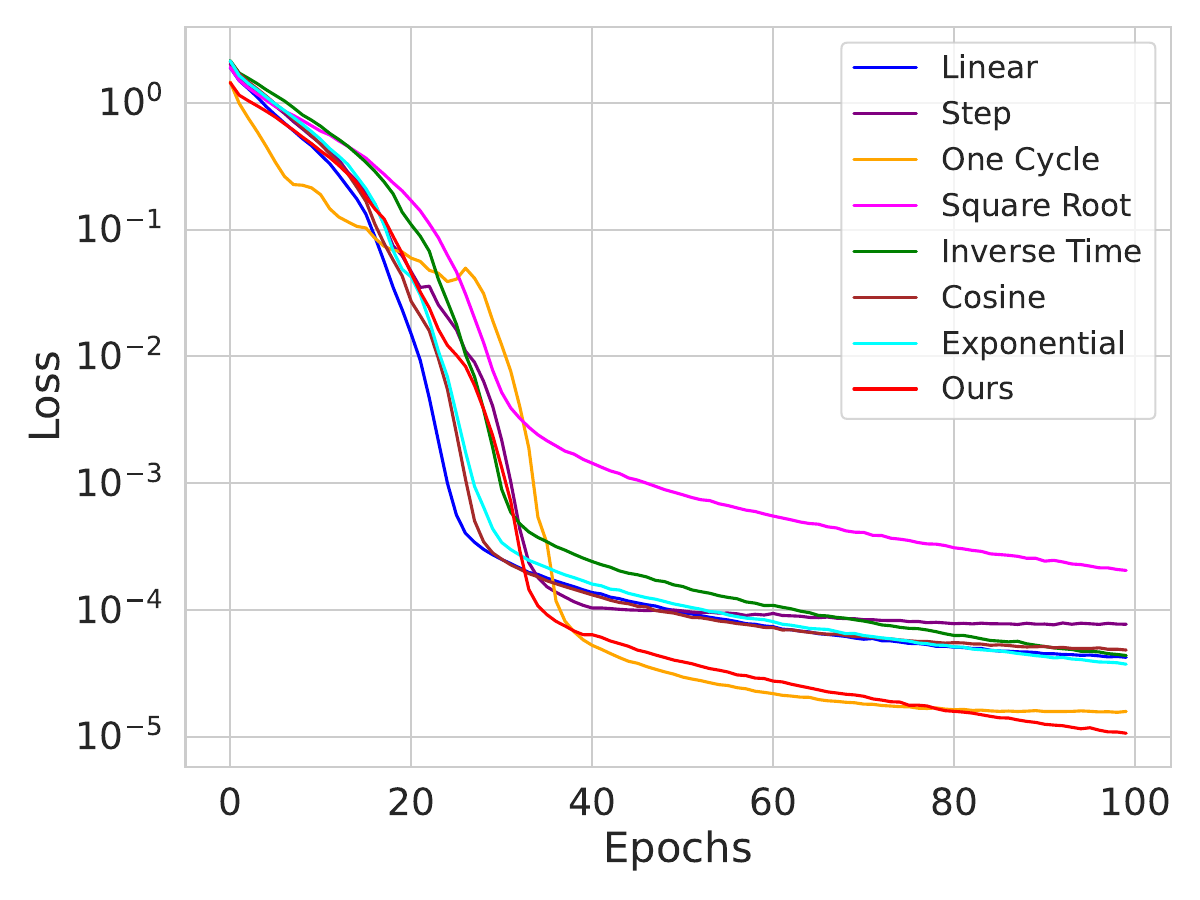}
  \end{subfigure}
  \hfill
  \begin{subfigure}[b]{0.31\textwidth}
    \includegraphics[width=\textwidth]{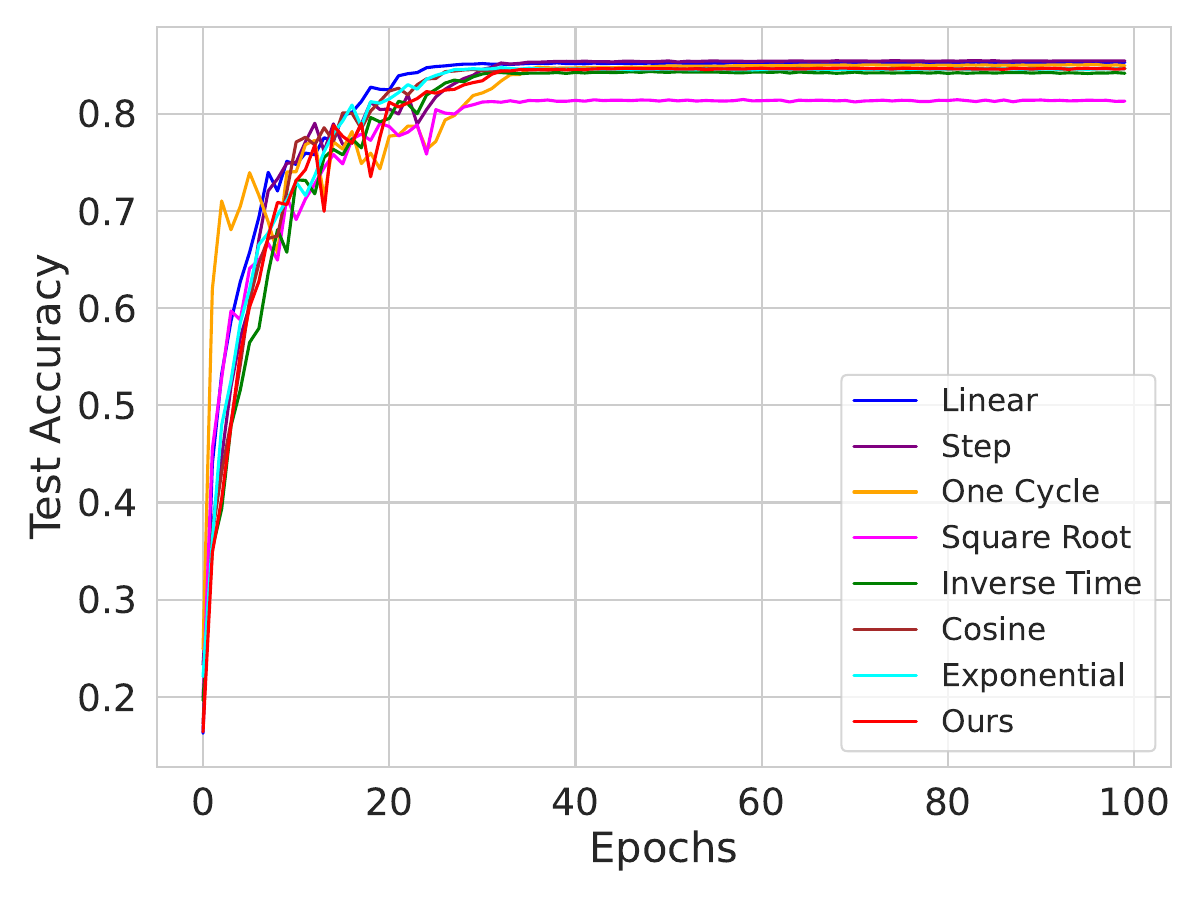}
  \end{subfigure}
  \caption{Mini-batch experiments on VGG-9 architecture with CIFAR-10 data.}
\label{fig:vgg}
\end{figure}
As illustrated in Figures~\ref{fig:lenet} and~\ref{fig:vgg}, our proposed learning rate performs well with convolutional neural networks (CNNs) as well. This training lasted for 100 epochs.

\section{Conclusion and Limitations}
This paper presents the first known theoretical guarantees for the convergence of Adam with an exact constant step size in non-convex settings. It provides insights into why using non-decreasing step sizes may yield suboptimal results and advocates for sticking to a fixed step size for convergence in Adam. \textbf{(i)} We analysed Adam's convergence properties and established a simple condition on the step size which ensures convergence in both deterministic and stochastic non-convex settings. We tag it as a sufficient condition \textbf{SC} in our text. \textbf{(ii)} We proposed a novel method for efficiently approximating the Lipschitz constant of the loss function with respect to the parameters, which is crucial for our proposed learning rate. \textbf{(iii)} Our empirical findings suggest that even with the accumulation of the past few gradients, the key driver for convergence in Adam is the non-increasing nature of step sizes. \textbf{(iv)} Finally, our theoretical claims are validated through extensive experiments on training deep neural networks on various datasets. Our experiments validate the effectiveness of the proposed step size. It drives the gradient norm towards zero more aggressively than the commonly used schedulers and a range of arbitrarily chosen constant step sizes. Our derived step size is easy to use and estimate and can be used for training a wide range of tasks.

\textbf{Limitations}: \textbf{(i)} It is still open to tightening the convergence rate of Adam. \textbf{(ii)} In Section~\ref{sec:mb_exp}, we empirically demonstrate that increasing the batch size leads to improved convergence. In our proof for the convergence of stochastic Adam, we chose to exclude the batch size factor for simplicity. However, we aim to theoretically demonstrate in the future that, according to our analysis, increasing batch size results in faster convergence.

\bibliographystyle{unsrt}
\bibliography{main}

\clearpage
\section*{Appendix}

In this supplementary section, we provide complete proof of our main results. \textbf{Section A} proves the main theorems. \textbf{Section B} describes the architectures of LeNet and VGG-9. It also contains an algorithm to estimate the Lipchitz constant of loss function \emph{w.r.t} network parameters in mini-batch setting. \textbf{Section C} details the other comprehensive experiments that were performed on various neural network architectures.

\textbf{NOTE 1}: We will denote the gradient of loss function \emph{w.r.t} model weights by $\nabla\mathcal{L}(\textbf{w})$ throughout the \textbf{Appendix}. In our main text, it is denoted by $\nabla_{\textbf{w}}\mathcal{L}(\textbf{w})$. Both symbols represent the gradient \emph{w.r.t} model weights \textbf{w}.

\textbf{NOTE 2}: From here onwards, the matrix $(\textbf{V}_{t}^{1/2} + \texttt{diag}\left(\rho\textbf{1}_{d})\right)^{-1}$ (mentioned in Algorithm~\ref{alg:Adam_training}) will be referred as $\textbf{A}_t$.

\textbf{NOTE 3}: Please note that Figure~\ref{fig:3_1000_mini_step} in the main text was intended to be conducted on mini batch setting, but it was mistakenly labelled and the plot refers to the full batch setting. The correct figure for the mini batch setting is Figure~\ref{fig:1__1000_mini_step} in this appendix. We apologize for any confusion caused by this error. Nevertheless, the conclusions and results remain unchanged, demonstrating the effectiveness of our learning rate.
\renewcommand{\thesubsection}{\Alph{subsection}}
\subsection{Proofs}
\setcounter{theorem}{0}
\begin{theorem}

\textbf{Deterministic Adam converges with proper choice of constant step size.} Let the loss function $\mathcal{L}(\textbf{w})$ be $K-$Lipchitz and let $\gamma < \infty$ be an upper bound on the norm of the gradient of $\mathcal{L}$. Also assume that $\mathcal{L}$ has a well-defined minimizer $\textbf{w}^{*}$ such that $\textbf{w}^{*} = argmin_{\textbf{w} \in \mathbb{R}^d}$ $\mathcal{L}(\textbf{w})$. Then the following holds for Algorithm (1):

For any $\epsilon, \rho> 0$ if we let $\alpha = \sqrt{2(\mathcal{L}(\textbf{w}_{0}) - \mathcal{L}(\textbf{w}^{*}))/K\delta^{2}T}$, then there exists a natural number $T(\beta_1,\rho)$ (depends on $\beta_1$ and $\rho$) such that $\underset{t = 1,\dots,T}{min}\|\nabla\mathcal{L}(\textbf{w}_{t})\|_{2} \leq \epsilon$ for some $t \geq T(\beta_1,\rho)$, where $\delta^{2} = \frac{\gamma^{2}}{\rho^{2}}$.
\end{theorem}

\begin{proof}
    We will prove this by contradiction. Let us assume that $\|\nabla\mathcal{L}(\textbf{w})\|_2 > \epsilon$ $\forall$ $t \in \{1, 2, 3, \dots\}$. Using the $K$-Smoothness of the loss function, we define the following relationship between consecutive updates:
    \setcounter{equation}{0}
    \begin{align}
    \mathcal{L}(\textbf{w}_{t+1}) - \mathcal{L}(\textbf{w}_{t})\leq & \hspace{5pt}\nabla\mathcal{L}(\textbf{w}_{t})^{T}(\textbf{w}_{t+1} - \textbf{w}_{t}) + \frac{K}{2}\|\textbf{w}_{t+1} - \textbf{w}_{t}\|_{2}^{2}\notag \\
    \leq & -\alpha\nabla\mathcal{L}(\textbf{w}_{t})^{T}(\textbf{A}_{t}\textbf{m}_{t}) + \frac{K}{2}\alpha^{2}\|\textbf{A}_{t}\textbf{m}_{t}\|_{2}^{2}\label{eq:1}
    \end{align}
    \textbf{Upperbound on} $\|\textbf{A}_{t}\textbf{m}_{t}\|_2$:\\
    We have, $\lambda_{max}(\textbf{A}_t) \leq \frac{1}{\rho + min_{j = 1,2,3,\dots,d}\sqrt{(\textbf{v}_t)_j}}$. Also, the recursion of $\textbf{v}_t$ can be solved as, $\textbf{v}_t = (1-\beta_2)\sum_{k=1}^{t}\beta_2^{t-k}\nabla\mathcal{L}(\textbf{w}_k)^2$. Now, define: \\
    $\epsilon_t = min_{k=1\dots t, j=1\dots d}\nabla\mathcal{L}(\textbf{w}_k)^2_j$, and it gives us:
    \begin{align}
        \lambda_{max}(\textbf{A}_t) \leq \frac{1}{\rho + \sqrt{(1-\beta_2^t)\epsilon_t}}\label{eq:2}
    \end{align}
    
    
    Now, to respect the inequality in Eq.(\ref{eq:2}) we must find an upper bound on $\|\textbf{m}_{t}\|_{2}^{2}$. The equation of $\textbf{m}_{t}$ without recursion is $\textbf{m}_{t} = (1 - \beta_{1})\sum_{k=1}^{t}\beta_{1}^{t-k}\nabla\mathcal{L}(\textbf{w}_{t})$. Using triangle inequality and defining $\gamma_t = max_{j \in \{1,2,3\dots\}}\|\nabla_{\textbf{w}}\mathcal{L}(\textbf{w}_{j})\|_{2}$ we have, $\|\textbf{m}_{t}\|_{2} \leq (1 - \beta_{1}^{t})\gamma_{t}$. Now, combining the estimate the $\|\textbf{m}_t\|_2$ with Eq.(\ref{eq:2}), we have:
    \begin{align}
        \|\textbf{A}_t\textbf{m}_t\|_2 \leq \frac{(1 - \beta_{1}^{t})\gamma_{t}}{\rho + \sqrt{(1-\beta_2^t)\epsilon_t}} \leq \frac{(1 - \beta_{1}^{t})\gamma_{t}}{\rho}\label{eq:3}
    \end{align}
    \textbf{Lowerbound on} $\nabla\mathcal{L}(\textbf{w}_{t})^{T}(\textbf{A}_{t}\textbf{m}_{t})$:\\
    To work with this, we define the following sequence for $j = 1, 2\dots, t$:
    \begin{align}
        P_j = \nabla\mathcal{L}(\textbf{w}_{t})^{T}(\textbf{A}_{t}\textbf{m}_{j})\notag
    \end{align}
    Now, we obtain the following sequence by substituting the update rule for $\textbf{m}_t$:
    \begin{align}
        P_j - \beta_1 P_{j-1} = & \hspace{5pt} \nabla\mathcal{L}(\textbf{w}_{t})^{T}\textbf{A}_{t}(\textbf{m}_j - \beta_1\textbf{m}_{j-1})\notag \\
        = & \hspace{5pt}(1 - \beta_{1})\nabla\mathcal{L}(\textbf{w}_{t})^{T}(\textbf{A}_{t}\nabla\mathcal{L}(\textbf{w}_{j}))\notag
    \end{align}
    At $j = t$, we have:
    \begin{align}
    P_{t} - \beta_{1}P_{t-1} \geq & \hspace{5pt}(1 - \beta_{1})\|\nabla\mathcal{L}(\textbf{w}_{t})\|_{2}^{2}\lambda_{min}(\textbf{A}_{t})\notag
    \end{align}
    Let us define $\gamma_{t-1} = \underset{1 \leq j \leq t-1}{max}\|\nabla\mathcal{L}(\textbf{w}_{j})\|_{2}$, and $\forall j \in \{1, 2, \dots t-1\}$, this gives us:
    \begin{align}
        P_{j} - \beta_{1}P_{j-1} \geq & \hspace{5pt}-(1 - \beta_{1})\|\nabla\mathcal{L}(\textbf{w}_{t})\|_{2} \gamma_{t-1}\lambda_{max}(\textbf{A}_{t})\notag
    \end{align}
    Now, we define the following identity:
    \begin{align}
        P_t - \beta_1 P_0 = (P_t - \beta_1 P_{t-1}) + \beta_1(P_{t-1} - \beta_1 P_{t-2}) + \dots + \beta_1^{t-1}(P_{1} - \beta_1 P_{0})\label{eq:4}
    \end{align}
    Now, we use the lowerbounds obtained on $P_j - \beta_1 P_{j-1}$ $\forall$ $j \in \{1,2\dots t-1\}$ and $P_t - \beta_1 P_{t-1}$ to lowerbound Eq.(\ref{eq:4}) as:
    \begin{align}
    P_{t} - \beta_{1}^{t}P_{0}\geq & \hspace{5pt}(1 - \beta_{1})\|\nabla\mathcal{L}(\textbf{w}_{t})\|_{2}^{2}\lambda_{min}(\textbf{A}_{t})\notag\\
    - & \hspace{5pt}(1 - \beta_{1})\|\nabla\mathcal{L}(\textbf{w}_{t})\|_{2} \gamma_{t-1}\lambda_{max}(\textbf{A}_{t})\sum_{j=0}^{t-1}\beta_{1}^{j}\notag\\
    \geq & \hspace{5pt}(1 - \beta_{1})\|\nabla\mathcal{L}(\textbf{w}_{t})\|_{2}^{2}\lambda_{min}(\textbf{A}_{t})\notag\\
    - & \hspace{5pt}(\beta_{1} - \beta_{1}^{t})\|\nabla\mathcal{L}(\textbf{w}_{t})\|_{2} \gamma_{t-1}\lambda_{max}(\textbf{A}_{t})\label{eq:5}
    \end{align}
    We have, $\lambda_{min}(\textbf{A}_t) \leq \frac{1}{\rho + \sqrt{max_{j=1\dots d}(\textbf{v}_t)_j}}$. We recall that the recursion of $\textbf{v}_t$ can be solved as $\textbf{v}_t = (1-\beta_2)\sum_{k=1}^{t}\beta_2^{t-k}\nabla\mathcal{L}(\textbf{w}_k)^2$ and we define\\ $\gamma_t = max_{j \in \{1,2,3\dots\}}\|\nabla_{\textbf{w}}\mathcal{L}(\textbf{w}_{j})\|_{2}$ to get:
    \begin{align}
        \lambda_{min}(\textbf{A}_t) \geq \frac{1}{\rho + \sqrt{(1-\beta_2^t)\gamma_t^2}}\label{eq:6}
    \end{align}
    We now combine Eq.(\ref{eq:6}) and Eq.(\ref{eq:2}) and the known value of $P_0 = 0$ (from initial condition) to get from the Eq.(\ref{eq:5}):
    \begin{align}
        P_t \geq & \hspace{5pt} -(\beta_1 - \beta_1^t)\|\nabla\mathcal{L}(\textbf{w}_{t})\|_{2}\gamma_{t-1}\frac{1}{\rho + \sqrt{(1-\beta_2^t)\epsilon_t}} \notag\\
        + & \hspace{5pt} (1 - \beta_{1})\|\nabla\mathcal{L}(\textbf{w}_{t})\|_{2}^{2}\frac{1}{\rho + \sqrt{(1-\beta_2^t)\gamma_t^2}}\notag\\
        \geq & \hspace{5pt} \|\nabla\mathcal{L}(\textbf{w}_{t})\|_{2}^{2} \left\{\frac{(1-\beta_1)}{\rho + \sqrt{(1-\beta_2^t)\gamma_t^2}} - \frac{(\beta_1 - \beta_1^t)\gamma}{\rho\|\nabla\mathcal{L}(\textbf{w}_{t})\|_2}\right\}\label{eq:7}
    \end{align}
    In Eq.(\ref{eq:7}), we have set $\epsilon_t = 0$ and $\gamma_{t-1} = \gamma_t = \gamma$. Next we analyse the following part of the lowerbound obtained above:
   {\tiny
\begin{align}
        \frac{(1-\beta_1)}{\rho + \sqrt{(1-\beta_2^t)\gamma^2}} - \frac{(\beta_1 - \beta_1^t)\gamma}{\rho\|\nabla\mathcal{L}(\textbf{w}_{t})\|_2}= & \hspace{5pt}\frac{\rho\|\nabla\mathcal{L}(\textbf{w}_{t})\|_2(1-\beta_1) - (\beta_1 - \beta_1^t)\gamma(\rho + \gamma\sqrt{(1-\beta_2^t)})}{\rho\|\nabla\mathcal{L}(\textbf{w}_{t})\|_2(\rho + \gamma\sqrt{(1-\beta_2^t)})}\notag \\
        = & \hspace{5pt} \gamma(\beta_1 - \beta_1^t)\frac{\rho\left\{\frac{\|\nabla\mathcal{L}(\textbf{w}_{t})\|_2(1-\beta_1)}{(\beta_1 - \beta_1^t)\gamma} - 1\right\}-\gamma\sqrt{(1-\beta_2^t)}}{\rho\|\nabla\mathcal{L}(\textbf{w}_{t})\|_2(\rho + \gamma\sqrt{(1-\beta_2^t)})}\notag \\
        = & \hspace{5pt} \gamma(\beta_1 - \beta_1^t)\left\{\frac{\|\nabla\mathcal{L}(\textbf{w}_{t})\|_2(1-\beta_1)}{(\beta_1 - \beta_1^t)\gamma} - 1\right\}\frac{\rho - \left\{\frac{\gamma\sqrt{(1-\beta_2^t)}}{-1 + \frac{\|\nabla\mathcal{L}(\textbf{w}_{t})\|_2(1-\beta_1)}{(\beta_1 - \beta_1^t)\gamma}}\right\}}{\rho\|\nabla\mathcal{L}(\textbf{w}_{t})\|_2(\rho + \gamma\sqrt{(1-\beta_2^t)})}\notag
    \end{align}
}

Now, we recall that we are working under the assumption that $\|\nabla\mathcal{L}(\textbf{w}_t)\|_2 > \epsilon$. Also, by definition $\beta_1 \in (0,1) \implies \beta_1 - \beta_1^t \in (0,\beta_1)$. This further implies that, $\frac{\|\nabla\mathcal{L}(\textbf{w}_{t})\|_2(1-\beta_1)}{(\beta_1 - \beta_1^t)\gamma} > \frac{\epsilon(1-\beta_1)}{\beta_1\gamma} > 1$. These inequalities now allow us to define a constant, $\frac{\epsilon(1-\beta_1)}{\beta_1\gamma} - 1 = \phi_1 > 0$ \emph{s.t} $\frac{\|\nabla\mathcal{L}(\textbf{w}_{t})\|_2(1-\beta_1)}{(\beta_1 - \beta_1^t)\gamma} - 1 > \phi_1$.

Now, according to our definition of $\rho > 0$, it allows us to define another constant $\phi_2 > 0$ to obtain:
\begin{align}
    \left\{\frac{\gamma\sqrt{(1-\beta_2^t)}}{-1 + \frac{\|\nabla\mathcal{L}(\textbf{w}_{t})\|_2(1-\beta_1)}{(\beta_1 - \beta_1^t)\gamma}}\right\} < \frac{\gamma}{\phi_1} = \rho - \phi_2\notag
\end{align}

Putting the above back into the lowerbound of $P_t$ in Eq.(\ref{eq:7}), we have:
\begin{align}
    P_t \geq \|\nabla\mathcal{L}(\textbf{w}_{t})\|_{2}^{2}\left\{\frac{\gamma(\beta_1 - \beta_1^2)\phi_1\phi_2}{\rho\gamma(\rho+\gamma)}\right\} = c\|\nabla\mathcal{L}(\textbf{w}_{t})\|_{2}^{2}\label{eq:8}
\end{align}

Here $ c = \frac{\gamma(\beta_1 - \beta_1^2)\phi_1\phi_2}{\rho\gamma(\rho+\gamma)} > 0$ is a constant. Now, we substitute Eq.(\ref{eq:8}) and Eq.(\ref{eq:3}) in Eq.(\ref{eq:1}). We get:
\begin{align}
    \mathcal{L}(\textbf{w}_{t+1}) - \mathcal{L}(\textbf{w}_{t})\leq & \hspace{5pt} -\alpha c\|\nabla\mathcal{L}(\textbf{w}_{t})\|_{2}^{2} + \frac{K}{2}\alpha^{2}\frac{(1 - \beta_{1}^{t})^2\gamma_{t}^2}{\rho^2}\label{eq:9}
\end{align}

In the last term of Eq.(\ref{eq:9}), we can have $\frac{(1 - \beta_{1}^{t})^2\gamma_{t}^2}{\rho^2} \leq \frac{\gamma_t^2}{\rho^2}$. We put $\gamma_{t-1} = \gamma_t = \gamma$ and finally we have $\frac{(1 - \beta_{1}^{t})^2\gamma_{t}^2}{\rho^2} \leq \frac{\gamma_t^2}{\rho^2} = \frac{\gamma^2}{\rho^2} = \delta^2$.

\begin{align}
    \mathcal{L}(\textbf{w}_{t+1}) - \mathcal{L}(\textbf{w}_{t})\leq & \hspace{5pt} -\alpha c\|\nabla\mathcal{L}(\textbf{w}_{t})\|_{2}^{2} + \frac{K}{2}\alpha^{2}\delta^2\notag\\
    \alpha c\|\nabla\mathcal{L}(\textbf{w}_{t})\|_{2}^{2} \leq & \hspace{5pt} \mathcal{L}(\textbf{w}_{t}) - \mathcal{L}(\textbf{w}_{t+1}) + \frac{K}{2}\alpha^{2}\delta^{2}\notag\\
    \|\nabla\mathcal{L}(\textbf{w}_{t})\|_{2}^{2} \leq & \hspace{5pt} \frac{\mathcal{L}(\textbf{w}_{t}) - \mathcal{L}(\textbf{w}_{t+1})}{\alpha c} + \frac{K\alpha\delta^{2}}{2c}\label{eq:10}
\end{align}

From Eq.(\ref{eq:10}), we have the following inequalities:
\begin{align}
\left\{
\begin{aligned}
    \|\nabla\mathcal{L}(\textbf{w}_{0})\|_{2}^{2} \leq & \hspace{5pt} \frac{\mathcal{L}(\textbf{w}_{0}) - \mathcal{L}(\textbf{w}_{1})}{\alpha c} + \frac{K\alpha\delta^{2}}{2c} & \notag\\
    \|\nabla\mathcal{L}(\textbf{w}_{1})\|_{2}^{2} \leq & \hspace{5pt} \frac{\mathcal{L}(\textbf{w}_{1}) - \mathcal{L}(\textbf{W}_{2})}{\alpha c} + \frac{K\alpha\delta^{2}}{2c} & \notag\\
    & \vdots \\
    \|\nabla\mathcal{L}(\textbf{w}_{T-1})\|_{2}^{2} \leq & \hspace{5pt} \frac{\mathcal{L}(\textbf{w}_{T-1}) - \mathcal{L}(\textbf{w}_{t})}{\alpha c} + \frac{K\alpha\delta^{2}}{2c} &\\
\end{aligned}
\right.
\end{align}

Summing up all the inequalities presented above, we obtain:
\begin{align*}
    \sum_{t=0}^{T-1} \|\nabla\mathcal{L}(\textbf{w}_{t})\|_{2}^{2} \leq & \hspace{5pt} \frac{\mathcal{L}(\textbf{w}_{0}) - \mathcal{L}(\textbf{w}_{t})}{\alpha c} + \frac{K\alpha\delta^{2}T}{2c}
\end{align*}
The inequality remains valid if we substitute $\|\nabla\mathcal{L}(\textbf{w}_{t})\|_{2}^{2}$ with $\underset{0\leq t \leq T-1}{min}\|\nabla\mathcal{L}(\textbf{w}_{t})\|_{2}^{2}$ within the summation on the left-hand side (LHS).
\begin{align}
    \underset{0 \leq t \leq T-1}{min}\|\nabla\mathcal{L}(\textbf{w}_{t})\|_{2}^{2}T \leq & \hspace{5pt} \frac{\mathcal{L}(\textbf{w}_{0}) - \mathcal{L}(\textbf{w}^{*})}{\alpha c} + \frac{K\alpha\delta^{2}T}{2c}\notag\\
    \underset{0 \leq t \leq T-1}{min}\|\nabla\mathcal{L}(\textbf{w}_{t})\|_{2}^{2} \leq & \hspace{5pt} \frac{\mathcal{L}(\textbf{w}_{0}) - \mathcal{L}(\textbf{w}^{*})}{\alpha c T} + \frac{K\alpha\delta^{2}}{2c}\notag\\
    \underset{0 \leq t \leq T-1}{min}\|\nabla\mathcal{L}(\textbf{w}_{t})\|_{2}^{2} \leq & \hspace{1pt} \frac{1}{\sqrt{T}}\left(\frac{\mathcal{L}(\textbf{w}_{0}) - \mathcal{L}(\textbf{w}^{*})}{cb} + \frac{K\delta^{2}b}{2c}\right)\notag
\end{align}
where $b = \alpha\sqrt{T}$. We set $b = \sqrt{2(\mathcal{L}(\textbf{w}_{0}) - \mathcal{L}(\textbf{w}^{*})\delta^{2})/K\delta^{2}}$, and we have:
\begin{align}
    \underset{0 \leq t \leq T-1}{\min} \|\nabla\mathcal{L}(\textbf{w}_{t})\|_{2} \leq \left(\frac{2K\delta^{2}}{T}(\mathcal{L}(\textbf{w}_{0})-\mathcal{L}(\textbf{w}^{*}))\right)^{\frac{1}{4}}\notag
\end{align}

When $T \geq \left(\frac{2K\delta^{2}}{\epsilon^{4}}(\mathcal{L}(\textbf{w}_{0})-\mathcal{L}(\textbf{w}^{*}))\right)$, we will have\\ $\underset{0 \leq t \leq T-1}{\min} \|\nabla\mathcal{L}(\textbf{w}_{t})\|_{2} \leq \epsilon$ which will contradict the assumption, \emph{i.e.} $\|\nabla\mathcal{L}(\textbf{w}_{t})\|_{2} > \epsilon$ for all $t \in \{1, 2, \dots\}$. Hence, completing the proof.
\end{proof}

\begin{theorem}
\textbf{Stochastic Adam converges with proper choice of constant step size.} Let the loss function $\mathcal{L}(\textbf{w})$ be $K-$Lipchitz and be of the form $\mathcal{L} = \sum_{j=1}^{m}\mathcal{L}_{j}$ such that (a) each $\mathcal{L}_{j}$ is at-least once differentiable, (b) the gradients satisfy $sign(\mathcal{L}_{r}(\textbf{w}))$ $= sign(\mathcal{L}_{s}(\textbf{w}))$ for all $r,s \in \{1,2,\dots,m\}$, (c) $\mathcal{L}$ has a well-defined minimizer $\textbf{w}^{*}$ such that $\textbf{w}^{*} = argmin_{\textbf{w} \in \mathbb{R}^d}$ $\mathcal{L}(\textbf{w})$. Let the gradient oracle, upon invocation at \( \textbf{w}_t \in \mathbb{R}^d \), randomly selects  $j_t$ from the set $\{1, 2, \ldots, m\}$ uniformly, and then provides $\nabla f_{j_{t}}(x_t) = \textbf{g}_{t}$ as the result. 

Then, for any $\epsilon, \rho > 0$ if we let $\alpha = \sqrt{2(\mathcal{L}(\textbf{w}_{0}) - \mathcal{L}(\textbf{w}^{*}))/K\delta^{2}T}$, then there exists a natural number $T(\beta_1,\rho)$ (depends on $\beta_1$ and $\rho$) such that \\$\underset{t = 1,\dots,T}{min}\mathbb{E}[\|\nabla\mathcal{L}(\textbf{w}_{t})\|_{2}] \leq \epsilon$ for some $t \geq T(\beta_1,\rho)$, where $\delta^{2} = \frac{\gamma^{2}}{\rho^{2}}$.

\end{theorem}

\begin{proof}
     We will prove this by contradiction. Let us assume that $\|\nabla\mathcal{L}(\textbf{w})\|_2 > \epsilon$ $\forall$ $t \in \{1, 2, 3, \dots\}$. We define $\gamma_t = max_{k=1\dots t}\|\nabla\mathcal{L}_{j_k}(\textbf{w}_k)\|$ and we solve the recursion for $\textbf{v}_t$, $\textbf{v}_t = (1-\beta_2)\sum_{k=1}^{t}\beta_2^{t-k}\textbf{g}_k^2$. We then write the following bounds:
{\scriptsize
\begin{align}
        \lambda_{max}(\textbf{A}_t) \geq \frac{1}{\sqrt{max_{j=1\dots d}(\textbf{v}_t)_j}} \geq \frac{1}{\sqrt{max_{j=1\dots d}((1-\beta_2)\sum_{k=1}^t \beta_2^{t-k}(\textbf{g}_k^2)_j})} \geq \frac{1}{\rho + \sqrt{(1-\beta_2^t)\gamma_t^2}}\label{eq:11}
\end{align}
}
Now, we define $\epsilon_t = min_{k=1\dots t, j=1\dots d}\nabla\mathcal{L}_{j_k}(\textbf{w}_k)^2_j$, and we get the following bound:
\begin{align}
        \lambda_{max}(\textbf{A}_t) \leq \frac{1}{\rho + min_{j = 1,2,3,\dots,d}\sqrt{(\textbf{v}_t)_j}} \leq \frac{1}{\rho + \sqrt{(1-\beta_2^t)\epsilon_t}}\label{eq:12}
    \end{align}
Now, we assume that the gradient of each $\mathcal{L}_j$ is upperbounded by a positive constant $\gamma$. Invoking this bound, we replace the eigenvalue bounds in the above equation with worst-case estimates, $\theta_{max}$ and $\theta_{min}$, defined as:
\begin{align}
    \lambda_{min}(\textbf{A}_t) \geq & \hspace{5pt} \frac{1}{\rho + \gamma} = \theta_{min}\notag\\
    \lambda_{max}(\textbf{A}_t) \leq & \hspace{5pt} \frac{1}{\rho + \sqrt{(1-\beta_2)}} = \theta_{max}\notag
\end{align}
We note that, the updates of stochastic Adam are $\textbf{w}_{t+1} = \textbf{w}_t - \alpha\textbf{A}_t\textbf{m}_t$. In case of stochastic Adam, $\textbf{m}_t = \beta_{1}\textbf{m}_{t-1} + (1 - \beta_{1})\textbf{g}_{t}$, here $\textbf{g}_t$ is the stochastic gradient at $t^{th}$ iterate. Let, $B_t = \{\textbf{w}_1, \textbf{w}_2, \textbf{w}_3 \dots \textbf{w}_t\}$ be a set of random variables corresponding to the first $t$ iterates. The assumption we have about the stochastic oracle gives us the following relation $\mathbb{E}[\textbf{g}_t] = \nabla\mathcal{L}(\textbf{w}_t)$ and we have $\mathbb{E}[\|\textbf{g}_t\|^2] \leq \gamma^2$ (from \textbf{SC}). Now, we will use these stochastic oracle's properties and take conditional expectation over $\textbf{g}_t$ \emph{w.r.t} $B_t$ in the $K-$ smoothness property of the loss function, we get:
{\small
\begin{align}
    \mathbb{E}[\mathcal{L}(\textbf{w}_{t+1})|B_t] \leq & \hspace{5pt} \mathbb{E}[\mathcal{L}(\textbf{w}_t)|B_t] - \alpha\mathbb{E}[\langle\mathcal{L}(\textbf{w}_t),\textbf{w}_{t+1} - \textbf{w}_t\rangle|B_t] + \frac{K}{2}\mathbb{E}[\|\textbf{w}_{t+1} - \textbf{w}_t\|_2^2|B_t]\notag\\
    \leq & \hspace{5pt} \mathbb{E}[\mathcal{L}(\textbf{w}_t)|B_t] - \alpha\mathbb{E}[\langle\mathcal{L}(\textbf{w}_t),\textbf{A}_t\textbf{m}_t\rangle|B_t] + \frac{K\alpha^2}{2}\mathbb{E}[\|\textbf{A}_t\textbf{m}_t\|_2^2|B_t]\label{eq:13}
\end{align}
}
From here, we will separately analyse the first and last term of the RHS of the above equation. 

\textbf{Upperbound on} $\mathbb{E}[\|\textbf{A}_t\textbf{m}_t\|_2|B_t]$:\\
The equation of $\textbf{m}_{t}$ without recursion is $\textbf{m}_{t} = (1 - \beta_{1})\sum_{k=1}^{t}\beta_{1}^{t-k}\textbf{g}_k$. We have:
\begin{align}
    \mathbb{E}[\|\textbf{A}_t\textbf{m}_t\|_2|B_t] = & \hspace{5pt} \mathbb{E}\left[\left\|\textbf{A}_t(1 - \beta_{1})\sum_{k=1}^{t}\beta_{1}^{t-k}\textbf{g}_k\right\|_2^2\Bigg|B_t\right]\notag\\
    \leq & \hspace{5pt} \theta_{max}^2(1-\beta_1)^2\sum_{k=1}^{t}\beta_1^{t-k}\mathbb{E}[\|\textbf{g}_k\|_2^2|B_t]\notag\\
    \leq & \hspace{5pt}\theta_{max}^2(1-\beta_1)^2\gamma^2\sum_{k=1}^{t}\beta_1^{t-k}\notag\\
    \leq & \hspace{5pt}\theta_{max}^2(1-\beta_1^t)^2\gamma^2\notag\\
    \leq & \hspace{5pt}\frac{(1-\beta_1^t)^2\gamma^2}{(\rho + \sqrt{(1-\beta_2)})^2}\notag\\
    \leq & \hspace{5pt} \frac{\gamma^2}{\rho^2} = \delta^2\label{eq:14}
\end{align}
\\
\textbf{Lowerbound on} $\mathbb{E}[\langle\mathcal{L}(\textbf{w}_t),\textbf{A}_t\textbf{m}_t\rangle|B_t]$:\\
To work with this, we define the following sequence for $j = 1, 2\dots, t$:
    \begin{align}
        P_j = \mathbb{E}[\langle\nabla\mathcal{L}(\textbf{w}_{t}), \textbf{A}_{t}\textbf{m}_{j}\rangle|B_t]\notag
    \end{align}
    
    Now, we obtain the following sequence by substituting the update rule for $\textbf{m}_t$:
    \begin{align}
        P_j - \beta_1 P_{j-1} = & \hspace{5pt} \mathbb{E}[\langle\nabla\mathcal{L}(\textbf{w}_{t}), \textbf{A}_{t}(\textbf{m}_j - \beta_1\textbf{m}_{j-1})\rangle|B_t]\notag \\
        = & \hspace{5pt}(1 - \beta_{1})\mathbb{E}[\langle\nabla\mathcal{L}(\textbf{w}_{t}), \textbf{A}_{t}\textbf{g}_j\rangle|B_t]\notag
    \end{align}
    
     At $j = t$, we have:
    \begin{align}
    P_{t} - \beta_{1}P_{t-1} \geq & \hspace{5pt}(1 - \beta_{1})\mathbb{E}[\langle\nabla\mathcal{L}(\textbf{w}_{t}), \textbf{A}_{t}\textbf{g}_t\rangle|B_t]\notag
    \end{align}
    
   We have, $\mathbb{E}[\langle\nabla\mathcal{L}(\textbf{w}_{t}), \textbf{A}_{t}\textbf{g}_t\rangle|B_t] \geq \theta_{min}\|\nabla\mathcal{L}(\textbf{w}_t)\|_2^2$. This bound is analyzed in \textbf{Lemma 1} after the end of this proof. Hence:
   \begin{align}
    P_{t} - \beta_{1}P_{t-1} \geq & \hspace{5pt}(1 - \beta_{1})\theta_{min}\|\nabla\mathcal{L}(\textbf{w}_t)\|_2^2\notag
    \end{align}

    At $j < t$, we have:\\
    \begin{align}
        P_{j} - \beta_{1}P_{j-1} = & \hspace{5pt}(1-\beta_1)\mathbb{E}[\langle\nabla\mathcal{L}(\textbf{w}_{t}), \textbf{A}_{t}\textbf{g}_j\rangle|B_t]\notag\\
        = & \hspace{5pt}(1-\beta_1)\mathbb{E}[\|\nabla\mathcal{L}(\textbf{w}_{t})\|_2\| \textbf{A}_{t}\textbf{g}_j\|_2cos\psi|B_t]\notag
    \end{align}
    
    As, $-1 \leq cos\psi \leq 1$, we obtain the following inequality:
    \begin{align}
        P_{j} - \beta_{1}P_{j-1} \geq & \hspace{5pt}-(1-\beta_1)\mathbb{E}[\|\nabla\mathcal{L}(\textbf{w}_{t})\|_2\| \textbf{A}_{t}\textbf{g}_j\|_2|B_t]\notag\\
        \geq & \hspace{5pt}-(1-\beta_1)\|\nabla\mathcal{L}(\textbf{w}_{t})\|_2\theta_{max}\mathbb{E}[\|\textbf{g}_j\|_2|B_t]\notag
    \end{align}
    
    Using \textbf{SC}, $\mathbb{E}[\|\textbf{g}_t\|_2^2] \leq (\mathbb{E}[\|\textbf{g}_t\|_2])^2 \leq \gamma^2 \implies \mathbb{E}[\|\textbf{g}_t\|_2] \leq \gamma$ $\forall t$. We get:
    \begin{align}
         P_{j} - \beta_{1}P_{j-1} \geq & \hspace{5pt} -(1-\beta_1)\|\nabla\mathcal{L}(\textbf{w}_{t})\|_2\theta_{max}\gamma\notag
    \end{align}

     Now, we define the following identity:
    \begin{align}
        P_t - \beta_1 P_0 = (P_t - \beta_1 P_{t-1}) + \beta_1(P_{t-1} - \beta_1 P_{t-2}) + \dots + \beta_1^{t-1}(P_{1} - \beta_1 P_{0})\notag
    \end{align}
    
    Now, we use the lowerbounds obtained on $P_j - \beta_1 P_{j-1}$ $\forall$ $j \in \{1,2\dots t-1\}$ and $P_t - \beta_1 P_{t-1}$ to lowerbound $P_t - \beta_1 P_0$ as:
    {\scriptsize
    \begin{align}
    P_{t} - \beta_{1}^{t}P_{0}\geq & \hspace{5pt}(1 - \beta_{1})\|\nabla\mathcal{L}(\textbf{w}_{t})\|_{2}^{2}\theta_{min}-(1 - \beta_{1})\|\nabla\mathcal{L}(\textbf{w}_{t})\|_{2} \gamma\theta_{max}\sum_{j=0}^{t-1}\beta_{1}^{j}\notag\\
    \geq & \hspace{5pt}(1 - \beta_{1})\|\nabla\mathcal{L}(\textbf{w}_{t})\|_{2}^{2}\theta_{min}-(\beta_{1} - \beta_{1}^{t})\|\nabla\mathcal{L}(\textbf{w}_{t})\|_{2} \gamma\theta_{max}\notag\\
     \geq & \hspace{5pt}(1 - \beta_{1})\|\nabla\mathcal{L}(\textbf{w}_{t})\|_{2}^{2}\frac{1}{\rho+\gamma} - (\beta_{1} - \beta_{1}^{t})\|\nabla\mathcal{L}(\textbf{w}_{t})\|_{2}\frac{\gamma}{\rho + \sqrt{(1-\beta_1)}}\notag\\
     \geq & \hspace{5pt}\|\nabla\mathcal{L}(\textbf{w}_{t})\|_{2}^{2}\left\{\frac{(1-\beta_1)}{\rho + \gamma} - \frac{(\beta_1 - \beta_1^t)\gamma}{\|\nabla\mathcal{L}(\textbf{w}_{t})\|_{2}(\rho + \sqrt{(1-\beta_1})}\right\}\notag\\
     \geq & \hspace{5pt}\|\nabla\mathcal{L}(\textbf{w}_{t})\|_{2}^{2}\left\{\frac{(1-\beta_1)}{\rho + \gamma} - \frac{(\beta_1 - \beta_1^t)\gamma}{\|\nabla\mathcal{L}(\textbf{w}_{t})\|_{2}\rho}\right\}\notag\\
     \geq & \hspace{5pt}\|\nabla\mathcal{L}(\textbf{w}_{t})\|_{2}^2\left\{\frac{(1-\beta_1)\|\nabla\mathcal{L}(\textbf{w}_{t})\|_{2}\rho - (\beta_1 - \beta_1^t)\gamma(\rho + \gamma)}{(\rho + \gamma)\rho\|\nabla\mathcal{L}(\textbf{w}_{t})\|_{2}}\right\}\notag\\
     \geq & \hspace{5pt}\|\nabla\mathcal{L}(\textbf{w}_{t})\|_{2}^2(\beta_1 - \beta_1^t)\gamma\frac{\rho\left\{\frac{(1-\beta_1)\|\nabla\mathcal{L}(\textbf{w}_{t})\|_{2}}{(\beta_1 - \beta_1^t)\gamma} - 1\right\} - \gamma}{(\rho + \gamma)\rho\|\nabla\mathcal{L}(\textbf{w}_{t})\|_{2}}\notag\\
     \geq & \hspace{5pt}\|\nabla\mathcal{L}(\textbf{w}_{t})\|_{2}^2(\beta_1 - \beta_1^t)\gamma\left\{\frac{(1-\beta_1)\|\nabla\mathcal{L}(\textbf{w}_{t})\|_{2}}{(\beta_1 - \beta_1^t)\gamma} - 1\right\}\frac{\rho - \left\{\frac{\gamma}{-1 + \frac{(1-\beta_1)\|\nabla\mathcal{L}(\textbf{w}_{t})\|_{2}}{(\beta_1 - \beta_1^t)\gamma}}\right\}}{(\rho + \gamma)\rho\|\nabla\mathcal{L}(\textbf{w}_{t})\|_{2}}\notag
    \end{align}
    }

Now, we recall that we are working under the assumption that $\|\nabla\mathcal{L}(\textbf{w}_t)\|_2 > \epsilon$. Also, by definition $\beta_1 \in (0,1) \implies \beta_1 - \beta_1^t \in (0,\beta_1)$. This further implies that, $\frac{\|\nabla\mathcal{L}(\textbf{w}_{t})\|_2(1-\beta_1)}{(\beta_1 - \beta_1^t)\gamma} > \frac{\epsilon(1-\beta_1)}{\beta_1\gamma} > 1$. These inequalities now allow us to define a constant, $\frac{\epsilon(1-\beta_1)}{\beta_1\gamma} - 1 = \phi_1 > 0$ \emph{s.t} $\frac{\|\nabla\mathcal{L}(\textbf{w}_{t})\|_2(1-\beta_1)}{(\beta_1 - \beta_1^t)\gamma} - 1 > \phi_1$.

Now, according to our definition of $\rho > 0$, it allows us to define another constant $\phi_2 > 0$ to obtain:
\begin{align}
    \left\{\frac{\gamma}{-1 + \frac{\|\nabla\mathcal{L}(\textbf{w}_{t})\|_2(1-\beta_1)}{(\beta_1 - \beta_1^t)\gamma}}\right\} < \frac{\gamma}{\phi_1} = \rho - \phi_2\notag
\end{align}

Putting the above back into the lowerbound of $P_t - \beta_1^t P_0$, we have:
\begin{align}
    P_t - \beta_1^t P_0 \geq \|\nabla\mathcal{L}(\textbf{w}_{t})\|_{2}^{2}\left\{\frac{\gamma(\beta_1 - \beta_1^2)\phi_1\phi_2}{\rho\gamma(\rho+\gamma)}\right\} = c\|\nabla\mathcal{L}(\textbf{w}_{t})\|_{2}^{2}\notag
\end{align}

Here $ c = \frac{\gamma(\beta_1 - \beta_1^2)\phi_1\phi_2}{\rho\gamma(\rho+\gamma)} > 0$ is a constant. We use the initial condition and put $P_0 = 0$, we get:
\begin{align}
    P_t \geq \|\nabla\mathcal{L}(\textbf{w}_{t})\|_{2}^{2}\left\{\frac{\gamma(\beta_1 - \beta_1^2)\phi_1\phi_2}{\rho\gamma(\rho+\gamma)}\right\} = c\|\nabla\mathcal{L}(\textbf{w}_{t})\|_{2}^{2}\notag
\end{align}

Next, we proceed by putting the lowerbound of $P_t$ and Eq.(\ref{eq:14}) both in Eq.(\ref{eq:13}). We get:
\begin{align}
    \mathbb{E}[\mathcal{L}(\textbf{w}_{t+1})|B_t] \leq & \hspace{5pt} \mathbb{E}[\mathcal{L}(\textbf{w}_t)|B_t] -\alpha c\|\nabla\mathcal{L}(\textbf{w}_{t})\|_{2}^{2} + \frac{K}{2}\alpha^{2}\delta^2\notag\\
    \alpha c\mathbb{E}[\|\nabla\mathcal{L}(\textbf{w}_{t})\|_{2}^{2}] \leq & \hspace{5pt} \mathbb{E}[\mathcal{L}(\textbf{w}_{t}) - \mathcal{L}(\textbf{w}_{t+1})] + \frac{K}{2}\alpha^{2}\delta^{2}\notag\\
    \mathbb{E}[\|\nabla\mathcal{L}(\textbf{w}_{t})\|_{2}^{2}] \leq & \hspace{5pt} \frac{\mathbb{E}[\mathcal{L}(\textbf{w}_{t}) - \mathcal{L}(\textbf{w}_{t+1})]}{\alpha c} + \frac{K\alpha\delta^{2}}{2c}\label{eq:15}
\end{align}

From Eq.(\ref{eq:15}), we have the following inequalities:
\begin{align}
\left\{
\begin{aligned}
    \mathbb{E}[\|\nabla\mathcal{L}(\textbf{w}_{0})\|_{2}^{2}] \leq & \hspace{5pt} \frac{\mathbb{E}[\mathcal{L}(\textbf{w}_{0}) - \mathcal{L}(\textbf{w}_{1})]}{\alpha c} + \frac{K\alpha\delta^{2}}{2c} & \notag\\
    \mathbb{E}[\|\nabla\mathcal{L}(\textbf{w}_{1})\|_{2}^{2}] \leq & \hspace{5pt} \frac{\mathbb{E}[\mathcal{L}(\textbf{w}_{1}) - \mathcal{L}(\textbf{W}_{2})]}{\alpha c} + \frac{K\alpha\delta^{2}}{2c} & \notag\\
    & \vdots \\
    \mathbb{E}[\|\nabla\mathcal{L}(\textbf{w}_{T-1})\|_{2}^{2}] \leq & \hspace{5pt} \frac{\mathbb{E}[\mathcal{L}(\textbf{w}_{T-1}) - \mathcal{L}(\textbf{w}_{t})]}{\alpha c} + \frac{K\alpha\delta^{2}}{2c} &\\
\end{aligned}
\right.
\end{align}

Summing up all the inequalities presented above, we obtain:
\begin{align*}
    \sum_{t=0}^{T-1} \mathbb{E}[\|\nabla\mathcal{L}(\textbf{w}_{t})\|_{2}^{2}] \leq & \hspace{5pt} \frac{\mathbb{E}[\mathcal{L}(\textbf{w}_{0}) - \mathcal{L}(\textbf{w}_{t})]}{\alpha c} + \frac{K\alpha\delta^{2}T}{2c}
\end{align*}
The inequality remains valid if we substitute $\mathbb{E}[\|\nabla\mathcal{L}(\textbf{w}_{t})\|_{2}^{2}]$ with \\$\underset{0\leq t \leq T-1}{min}\mathbb{E}[\|\nabla\mathcal{L}(\textbf{w}_{t})\|_{2}^{2}]$ within the summation on the left-hand side (LHS).
\begin{align}
    \underset{0\leq t \leq T-1}{min}\mathbb{E}[\|\nabla\mathcal{L}(\textbf{w}_{t})\|_{2}^{2}]T \leq & \hspace{5pt} \frac{\mathcal{L}(\textbf{w}_{0}) - \mathcal{L}(\textbf{w}^{*})}{\alpha c} + \frac{K\alpha\delta^{2}T}{2c}\notag\\
    \underset{0\leq t \leq T-1}{min}\mathbb{E}[\|\nabla\mathcal{L}(\textbf{w}_{t})\|_{2}^{2}] \leq & \hspace{5pt} \frac{\mathcal{L}(\textbf{w}_{0}) - \mathcal{L}(\textbf{w}^{*})}{\alpha c T} + \frac{K\alpha\delta^{2}}{2c}\notag\\
    \underset{0\leq t \leq T-1}{min}\mathbb{E}[\|\nabla\mathcal{L}(\textbf{w}_{t})\|_{2}^{2}] \leq & \hspace{1pt} \frac{1}{\sqrt{T}}\left(\frac{\mathcal{L}(\textbf{w}_{0}) - \mathcal{L}(\textbf{w}^{*})}{cb} + \frac{K\delta^{2}b}{2c}\right)\notag
\end{align}
where $b = \alpha\sqrt{T}$. We set $b = \sqrt{2(\mathcal{L}(\textbf{w}_{0}) - \mathcal{L}(\textbf{w}^{*})\delta^{2})/K\delta^{2}}$, and we have:
\begin{align}
    \underset{0\leq t \leq T-1}{min}\mathbb{E}[\|\nabla\mathcal{L}(\textbf{w}_{t})\|_{2}^{2}] \leq \left(\frac{2K\delta^{2}}{T}(\mathcal{L}(\textbf{w}_{0})-\mathcal{L}(\textbf{w}^{*}))\right)^{\frac{1}{4}}\notag
\end{align}

When $T \geq \left(\frac{2K\delta^{2}}{\epsilon^{4}}(\mathcal{L}(\textbf{w}_{0})-\mathcal{L}(\textbf{w}^{*}))\right)$, we will have\\ $\underset{0\leq t \leq T-1}{min}\mathbb{E}[\|\nabla\mathcal{L}(\textbf{w}_{t})\|_{2}^{2}] \leq \epsilon$ which will contradict the assumption, \emph{i.e.} $\|\nabla\mathcal{L}(\textbf{w}_{t})\|_{2} > \epsilon$ for all $t \in \{1, 2, \dots\}$. Hence, completing the proof.
\end{proof}
\textbf{Lemma 1.} At any instant $t$, the following holds:\
\begin{align}
    \mathbb{E}[\langle\nabla\mathcal{L}(\textbf{w}_{t}), \textbf{A}_{t}\textbf{g}_t\rangle|B_t] \geq \theta_{min}\|\nabla\mathcal{L}(\textbf{w})_t\|_2^2\notag
\end{align}
\begin{proof}
    \begin{align}
        \mathbb{E}[\langle\nabla\mathcal{L}(\textbf{w}_{t}), \textbf{A}_{t}\textbf{g}_t\rangle|B_t] = & \hspace{5pt} \mathbb{E}\left[\sum_{j=1}^d \nabla_j\mathcal{L}(\textbf{w}_t)(\textbf{A}_t)_{jj}(\textbf{g}_t)_j|B_t\right]\notag\\
        = & \hspace{5pt} \sum_{j=1}^d \nabla_j\mathcal{L}(\textbf{w}_t)\mathbb{E}\left[(\textbf{A}_t)_{jj}(\textbf{g}_t)_j|B_t\right]\label{eq:16}
    \end{align}

    Now, we will introduce a new variable $\psi_{ri} = [\nabla_q\mathcal{L}(\textbf{w}_t)]_j$. Here, $r$ is the index for training set, $r \in \{1,\dots,k\}$. Also, condition to $B_t$, $\psi_{ri} 's$ are constants. This implies that $\nabla_j\mathcal{L}(\textbf{w}_t) = \frac{1}{k}\sum_{r=1}^k\psi_{rj}$. Further, we recall that $\mathbb{E}[(\textbf{g}_t)_j] = \nabla_j\mathcal{L}(\textbf{w}_t)$, and the expectation is taken over at $t^{th}$ update. Moreover, our implementation of the oracle is analogous to conducting random sampling with uniform distribution, where each \( (g_t)_j \) is drawn independently from the set \( \{\psi_{rj}\}_{r=1,\ldots,k} \).

    Given we have $\textbf{v}_{t} = \beta_{2}\textbf{v}_{t-1} + (1 - \beta_{2})\textbf{g}_{t}^{2}$, this implies $(\textbf{A}_t)_{jj} = \frac{1}{\sqrt{(1-\beta_2)(\textbf{g}_t)_j^2 + l_j}}$. We have defined $l_j = (1-\beta_2)\sum_{k=1}^{t-k}(\textbf{g}_k)_j^2 + \rho$. Note that, condition on $B_t$, $l_j$ is constant. This results in a specific expression for the required expectation over the $t$-th oracle call as:
    \begin{align}
        \mathbb{E}\left[(\textbf{A}_t)_{jj}(\textbf{g}_t)_j|B_t\right] = & \hspace{5pt} \mathbb{E}\left[(\textbf{A}_t)_{jj}(\textbf{g}_t)_j|B_t\right]\notag\\
        = & \hspace{5pt} \mathbb{E}_{(\textbf{g}_t)_j \sim \{\psi_{rj}\}_{r=1\dots k}}\left[\frac{(\textbf{g}_t)_j}{\sqrt{(1-\beta_2)(\textbf{g}_t)_j^2 + l_j}}\Bigg|B_t\right]\notag\\
        = & \hspace{5pt} \frac{1}{k}\sum_{r=1}^k\frac{\psi_{rj}}{\sqrt{(1-\beta_2)\psi_{rj}^2 + l_j}}\notag
    \end{align}

    Substituting the above along with the definition of the constants $\psi_{rj}$ into Eq.(\ref{eq:16}), we get:
    \begin{align}
        \mathbb{E}[\langle\nabla\mathcal{L}(\textbf{w}_{t}), \textbf{A}_{t}\textbf{g}_t\rangle|B_t] = & \hspace{5pt} \sum_{j=1}^d\left[\frac{1}{k}\sum_{r=1}^k\psi_{rj}\right]\left[\frac{1}{k}\sum_{r=1}^k\frac{\psi_{rj}}{\sqrt{(1-\beta_2)\psi_{rj}^2 + l_j}}\right]\notag
    \end{align}

    Now, we define two vectors $\textbf{y}_j$ and $\textbf{z}_j$ $\in \mathbb{R}^k$ such that $(\textbf{y}_j)_r = \psi_{rj}$ and $(\textbf{z}_j)_r = \frac{1}{\sqrt{(1-\beta_2)\psi_{rj}^2 + l_j}}$. Plugging this in the above expression, we get:
    \begin{align}
        \mathbb{E}[\langle\nabla\mathcal{L}(\textbf{w}_{t}), \textbf{A}_{t}\textbf{g}_t\rangle|B_t] = & \hspace{5pt} \frac{1}{k^2}\sum_{j=1}^d(\textbf{y}_j^T\textbf{1}_k)(\textbf{z}_j^T\textbf{y}_j)\notag\\
        = & \hspace{5pt} \frac{1}{k^2}\sum_{j=1}^d\textbf{y}_j^T(\textbf{1}_k\textbf{z}_j^T)\textbf{y}_j\notag
    \end{align}

    Note, with this substitution, the RHS of Lemma can be rewritten as:
    \begin{align}
        \theta_{min}\|\nabla\mathcal{L}(\textbf{w}_t)\|_2 = & \hspace{5pt} \theta_{min}\sum_{j=1}^d\left[\frac{1}{k}\sum_{r=1}^k\nabla_r\mathcal{L}(\textbf{w}_t)\right]^2\notag\\
        = & \hspace{5pt} \frac{\theta_{min}}{k^2}\sum_{j=1}^d(\textbf{y}_j^T\textbf{1}_k)^2\notag\\
         = & \hspace{5pt} \frac{\theta_{min}}{k^2}\sum_{j=1}^d\textbf{y}_j^T\textbf{1}_k\textbf{1}_k^T\textbf{y}_j\notag
    \end{align}
    
    Our claim for Lemma 1 can be proved straightforwardly if we show that \\$\frac{1}{k^2}\sum_{j=1}^d\textbf{y}_j^T(\textbf{1}_k\textbf{z}_j^T)\textbf{y}_j - \frac{\theta_{min}}{k^2}\sum_{j=1}^d\textbf{y}_j^T\textbf{1}_k\textbf{1}_k^T\textbf{y}_j \geq 0$ holds for all $j$. This can be further simplified as:
    \begin{align}
        \frac{1}{k^2}\sum_{j=1}^d\textbf{y}_j^T(\textbf{1}_k\textbf{z}_j^T)\textbf{y}_j - \frac{\theta_{min}}{k^2}\sum_{j=1}^d\textbf{y}_j^T\textbf{1}_k\textbf{1}_k^T\textbf{y}_j = \frac{1}{k^2}\textbf{y}_j^T(\textbf{1}_k(\textbf{z}_j - \theta_{min}\textbf{1}_k)^T)\textbf{y}_j\notag
    \end{align}

    To further simplify, we define $\textbf{u}_j \in \mathbb{R}^k$ as:
    \begin{align}
        (\textbf{u}_j)_r = (\textbf{z}_j)_r - \theta_{min} = \frac{1}{\sqrt{(1-\beta_2)(\textbf{g}_t)_j^2 + l_j}} - \theta_{min}\notag
    \end{align}

    Our objective now simplifies to show that $\frac{1}{k^2}\mathbf{y}_j^T(\mathbf{1}_k\mathbf{u}_j^T)\mathbf{y}_j \geq 0$. Before that, we find a bound on $l_j$ by recalling the definition of $\gamma$. From which it follows that $\psi_{rj}^2 \leq \gamma^2$.

\begin{align}
    l_j \leq & \hspace{5pt} (1-\beta_2)\sum_{k=1}^{t-k}(\textbf{g}_k)_j^2 + \rho\notag\\
    \leq & \hspace{5pt} (1-\beta_2)\sum_{k=1}^{t-k}\gamma^2 + \rho\notag\\
    \leq & \hspace{5pt} \gamma^2(\beta_2 - \beta_2^{t-1}) + \rho\notag
\end{align}

The above bound on $l_j$ implies:
\begin{align}
    \sqrt{(1-\beta_2)\psi_{rj}^2 + l_j} \leq & \hspace{5pt} \sqrt{(1-\beta_2)\gamma^2 + \rho + \gamma^2(\beta_2 - \beta_2^{t-1})}\notag\\
    \leq & \hspace{5pt} \sqrt{\gamma^2(1-\beta_2^{t-1}) + \rho}\notag
\end{align}

This further implies that :\\$-\theta_{min} + \frac{1}{\sqrt{(1-\beta_2)\psi_{rj}^2 + l_j}} \geq -\theta_{min} + \frac{1}{\sqrt{\gamma^2(1-\beta_2^{t-1}) + \rho}} = -\frac{1}{\rho + \gamma} + \frac{1}{\sqrt{\gamma^2(1-\beta_2^{t-1}) + \rho}} \geq 0$

The inequality follows because $beta_2 \in (0,1]$. Putting this all together:
\begin{align}
    (\textbf{y}_j^T\textbf{1}_k)(\textbf{u}_j^T\textbf{y}_j) = & \hspace{5pt} \left[\sum_{r=1}^k\psi_{rj}\right]\left[\sum_{r=1}^k\left\{-\theta_{min}\psi_{rj}+\frac{\psi_{rj}}{\sqrt{(1-\beta_2)\psi_{rj}^2 + l_j}}\right\}\right]\notag\\
    = & \hspace{5pt} \sum_{r=1}^k\sum_{s=1}^k\left[-\theta_{min}\psi_{rj}\psi_{sj} + \frac{\psi_{rj}\psi_{sj}}{\sqrt{(1-\beta_2)\psi_{rj}^2 + l_j}}\right]\notag\\
    = & \hspace{5pt} \sum_{r=1}^k\sum_{s=1}^k\psi_{rj}\psi_{sj}\left[-\theta_{min} + \frac{1}{\sqrt{(1-\beta_2)\psi_{rj}^2 + l_j}}\right]\notag
\end{align}

Now our assumption that for all $\textbf{w}$, $sign(\nabla_r\mathcal{L}(\textbf{w})) = sign(\nabla_s\mathcal{L}(\textbf{w}))$ for all $r,s \in \{1\dots k\}$ leads to the conclusion that the term $\psi_{rj}\psi_{sj} \geq 0$. Additionally, we have previously established that $\left[-\theta_{\min} + \frac{1}{\sqrt{(1-\beta_2)\psi_{rj}^2 + l_j}}\right] \geq 0$. Therefore, we have demonstrated that $(\mathbf{y}_j^T\mathbf{1}_k)(\mathbf{u}_j^T\mathbf{y}_j) \geq 0$, which completes the proof.
\end{proof}
\begin{theorem}
    Given $\|.\|_2$ and $\nabla\mathcal{L}$ are continuous functions and $\|\nabla\mathcal{L}\|_2$ is bounded from above. If $\mathcal{L}$ is Lipchitz continuous then, $K = sup_{\textbf{w} \in \mathbb{R}^d}\|\nabla\mathcal{L}\|_2$. Let $\hat{L}$ be a random variable defined as :
    \begin{equation*}
        \hat{L} = \underset{1 \leq j \leq N}{max} \|\nabla\mathcal{L}(\textbf{w}_j)\|_2
    \end{equation*} 
where $\textbf{w}_j$, $j \in \{1,2,\dots,N\}$ are i.i.d samples drawn from a distribution $\mathcal{W}$. Then, as $N \rightarrow \infty$, the random variable $\hat{L}$ converges in distribution to $K$. \\Mathematically:
\begin{equation*}
    \lim_{{n \to \infty}} P(\hat{L} \leq k) = \begin{cases} 0 & \text{if } k < K \\ 1 & \text{if } k \geq K \end{cases}
\end{equation*}
\end{theorem}
\begin{proof}
    Given $\textbf{w}_j$ $\forall j \in \{1\dots N\}$ are i.i.d random variables drawn from a probability distribution $\mathcal{W}$, it implies that $\|\nabla\mathcal{L}(\textbf{w}_j)\|_2$ $\forall j \in \{1\dots N\}$ are also random variables. Say it follows a distribution $\mathcal{G}$. Here, we now proceed by making two important assumptions required for this analysis: \textbf{(i)} $\|\nabla\mathcal{L}(\textbf{w}_j)\|_2$ $\forall j \in \{1\dots N\}$ are also i.i.d and \textbf{(ii)} $\|\nabla\mathcal{L}(\textbf{w})\|_2$ is always upper bounded by a finite constant $\forall \textbf{w} \in \mathbb{R}^d$. 
    
    The distribution of $\hat{L}$ is given by the maximum of $N$ i.i.d samples from $\mathcal{G}$:
    \begin{align}
        P(\hat{L} \leq k) = P\left\{\underset{1 \leq j \leq N}{max} \|\nabla\mathcal{L}(\textbf{w}_j)\|_2 \leq k\right\}\notag
    \end{align}

This is equivalent to the joint probability that each $\|\nabla\mathcal{L}(\textbf{w}_j)\|_2$ is less than or equal to $k$:
\begin{align}
    P(\hat{L} \leq k) = P(\|\nabla\mathcal{L}(\textbf{w}_1)\|_2 \leq k, \|\nabla\mathcal{L}(\textbf{w}_2)\|_2 \leq k \dots \|\nabla\mathcal{L}(\textbf{w}_N)\|_2 \leq k)\label{eq:17}
\end{align}

Since, we assumed that $\|\nabla\mathcal{L}(\textbf{w}_j)\|_2$ $\forall j \in \{1\dots N\}$ are also i.i.d's, Eq.(\ref{eq:17}) can be expressed as product of individual probabilities:
\begin{align}
     P(\hat{L} \leq k) = \prod_{j=1}^N P(\|\nabla\mathcal{L}(\textbf{w}_j)\|_2 \leq k)\notag
\end{align}

As $N$ tends to infinity, this product of probabilities converges to zero unless $k$ is greater than or equal to $K$, the maximum value of $\|\nabla\mathcal{L}\|_2$. Hence:
\begin{equation*}
    \lim_{{N \to \infty}} P(\hat{L} \leq k) = \begin{cases} 0 & \text{if } k < K \\ 1 & \text{if } k \geq K \end{cases}
\end{equation*}

This is because, with an increasing number of samples, the probability that at least one of the samples $\|\nabla\mathcal{L}(\textbf{w}_j)\|_2$ exceeds or equals $K$ becomes almost certain. Since $K$ is the maximum value of $\|\nabla\mathcal{L}(\textbf{w})\|_2$, and  $\hat{L}$ represents the maximum of $N$ samples, it becomes highly likely that $\hat{L}$ is greater than or equal to any given $k \geq K$ as $N$ approaches infinity. In summary, the asymptotic behaviour suggests that as the sample size grows indefinitely, the probability of $\hat{L}$ exceeding or equaling any value greater than or equal to 
$K$ tends to 1.
\end{proof}
\subsection{Model Architectures and Algorithms}
In this section, we will provide a brief overview of the LeNet and VGG-9 architectures. Additionally, we will present an algorithm for computing the Lipschitz constant of the loss with respect to model weights in the mini-batch setting.\\

\textbf{1. LeNet Architecture (For MNIST)}
\begin{table}[]
    \centering
    \caption{LeNet Architecture}
    \label{tab:lenet}
    \begin{tabular}{|c|c|c|c|}
        \hline
        \textbf{Layers} & \textbf{Kernel Size} & \textbf{Input Size} & \textbf{Output Size} \\
        \hline
        Conv1 & $5 \times 5$ & $28 \times 28 \times 1$ & $26 \times 26 \times 6$ \\
        \hline
        AvgPooling1 & $2 \times 2$ & $26 \times 26 \times 6$ & $13 \times 13 \times 6$ \\
        \hline
        Conv2 & $5 \times 5$ & $13 \times 13 \times 6$ & $11 \times 11 \times 16$ \\
        \hline
        AvgPooling2 & $2 \times 2$ & $11 \times 11 \times 16$ & $5 \times 5 \times 16$ \\
        \hline
        Flatten & - & $5 \times 5 \times 16$ & $400$ \\
        \hline
        FC1 & - & $400$ & $120$ \\
        \hline
        FC2 & - & $120$ & $84$ \\
        \hline
        FC3 & - & $84$ & $10$ \\
        \hline
    \end{tabular}
\end{table}

Here, each Conv layer and FC layer includes batch normalisation and ReLU activation. Stride=1 and padding=1 are used for each conv layer and stride=2 and padding=0 are used for each avgpool layer. No batch normalisation and ReLU activation are used after FC3 layer.\\
\newpage
\textbf{2. VGG-9 Architecture (For CIFAR-10)}
\begin{table}[]
    \centering
    \caption{VGG-9 Architecture}
    \label{tab:vgg9}
    \begin{tabular}{|c|c|c|c|}
        \hline
        \textbf{Layers} & \textbf{Kernel Size} & \textbf{Input Size} & \textbf{Output Size} \\
        \hline
        Conv1 & $3 \times 3$ & $32 \times 32 \times 3$ & $32 \times 32 \times 64$ \\
        \hline
        Conv2 & $3 \times 3$ & $32 \times 32 \times 64$ & $32 \times 32 \times 64$ \\
        \hline
        MaxPooling1 & $2 \times 2$ & $32 \times 32 \times 64$ & $16 \times 16 \times 64$ \\
        \hline
        Conv3 & $3 \times 3$ & $16 \times 16 \times 64$ & $16 \times 16 \times 128$ \\
        \hline
        Conv4 & $3 \times 3$ & $16 \times 16 \times 128$ & $16 \times 16 \times 128$ \\
        \hline
        MaxPooling2 & $2 \times 2$ & $16 \times 16 \times 128$ & $8 \times 8 \times 128$ \\
        \hline
        Conv5 & $3 \times 3$ & $8 \times 8 \times 128$ & $8 \times 8 \times 256$ \\
        \hline
        Conv6 & $3 \times 3$ & $8 \times 8 \times 256$ & $8 \times 8 \times 256$ \\
        \hline
        Conv7 & $3 \times 3$ & $8 \times 8 \times 256$ & $8 \times 8 \times 256$ \\
        \hline
        MaxPooling3 & $2 \times 2$ & $8 \times 8 \times 256$ & $4 \times 4 \times 256$ \\
        \hline
        FC1 & - & $4 \times 4 \times 256$ & $256$ \\
        \hline
        FC2 & - & $256$ & $10$ \\
        \hline
    \end{tabular}
\end{table}

\textbf{3. Estimation of Lipchitz Constant in Mini-Batch Setting}\\

We present Algorithm~\ref{alg:lipchitz_mini} for approximating the Lipschitz constant of the loss in the mini-batch setting. Since this algorithm computes the Lipschitz constant by taking the maximum over the set of gradients evaluated on sampled weights, \textbf{Theorem}~\ref{theorem: 4} applies. In the mini-batch setting, $\Omega_j$ represents the probability distribution from which the $j^{th}$ batch of data is drawn. \textbf{Note:} Reviewer is advised not to confuse $\mathcal{L}^j$, loss for $j^{th}$ batch with $\mathcal{L}_j$ which is mentioned in \textbf{Theorem}~\ref{theorem: 2}.

\begin{algorithm}
\caption{Estimating Lipchitz Constant of Loss Function}
\label{alg:lipchitz_mini}
\KwIn{Dataset: $\mathcal{X}$ with batches $\{\mathcal{B}_{j}\}_{j=1}^{b} \sim \{\Omega_{j}\}_{j=1}^{b}$, Loss function: $\mathcal{L}:\mathbb{R}^{d} \rightarrow \mathbb{R}$, A network parameterized by $\textbf{w} \in \mathbb{R}^d \sim \mathcal{W}$, No. of iterations: $N$}
\KwOut{$\hat{L}$}
\textbf{Initialization: $\hat{L} = 0$}, $\emph{temp} =0 $, Iteration Number $n =0$\\
\For{$n$ \textbf{from} 1 \textbf{to} $N$:}{
Randomly sample the network weights: $\textbf{w}_{n} \sim \mathcal{W}$\\
\For{$j$ \textbf{from} 1 \textbf{to} $b$:}{
    Compute loss for current $j^{th}$ batch: $\mathcal{L}^{j}(\textbf{w}_{n}) = \mathbb{E}_{\mathcal{B}_j \sim \Omega_{j}}[\tilde{\mathcal{L}^{j}}(\mathcal{B}_j,\textbf{w}_n)]$\\
    Compute the norm of gradient at current instant: $\|\nabla\mathcal{L}^{j}(\textbf{w}_{n})\|_2$\\
    \emph{temp} = \emph{max}($\|\nabla\mathcal{L}^{j}(\textbf{w}_{n})\|_2$, \emph{temp})\\
    
}
Estimate the Lipchitz constant $(\hat{L})$: $\hat{L} = max(\emph{temp}, \hat{L})$\\}
\textbf{End}\\
\end{algorithm}
The gradient for each full pass of data is represented by the maximum gradient among the batches. This is done to respect \textbf{Theorem}~\ref{theorem: 3}.
\subsection{Additional Experiments}
We show training loss, test loss and gradient norm results for a variety of
additional network architectures. Across almost all network architectures, our main results remain
consistent.\\

\setcounter{figure}{0}
\textbf{C.1 Full-batch experiments, comparison against various learning rate schedulers.}

This section serves as a continuation for Section~\ref{sec:fb_exp}.
\begin{figure}[]
  \centering
  \begin{subfigure}[b]{0.31\textwidth}
    \includegraphics[width=\textwidth]{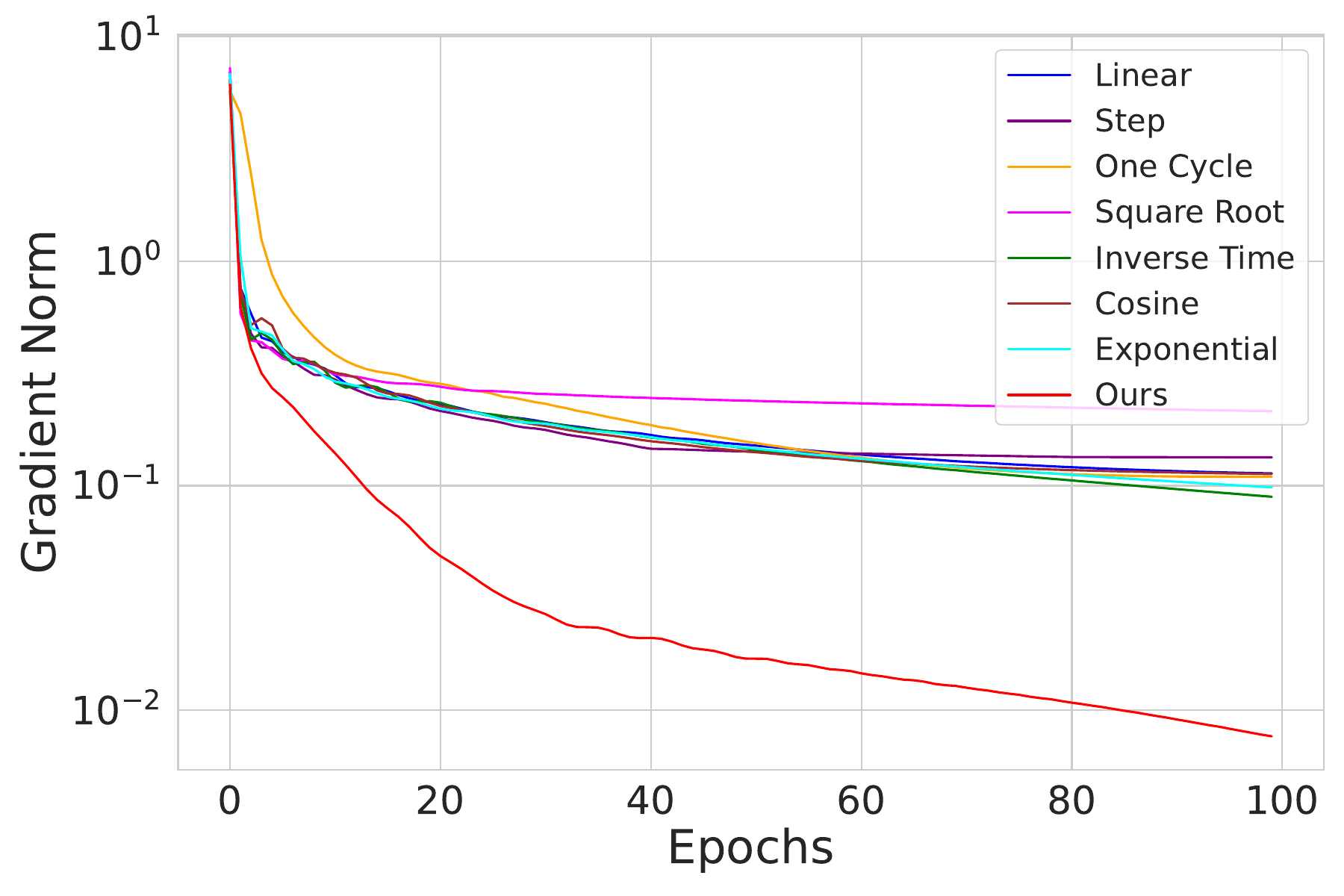}
    \caption{Gradient norm v/s Epochs}
  \end{subfigure}
  \hfill
  \begin{subfigure}[b]{0.31\textwidth}
    \includegraphics[width=\textwidth]{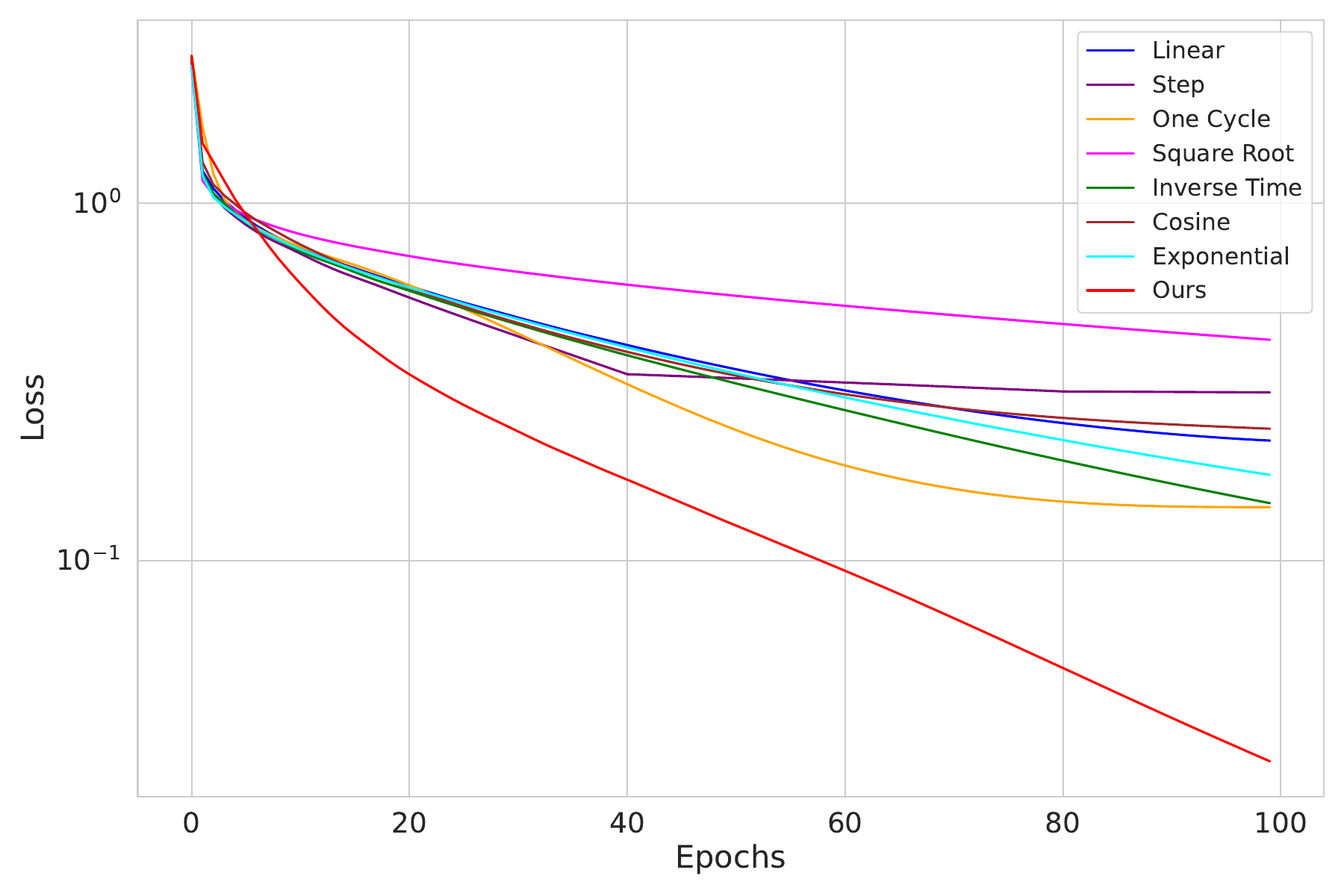}
    \caption{Training loss v/s Epochs}
  \end{subfigure}
  \hfill
  \begin{subfigure}[b]{0.31\textwidth}
    \includegraphics[width=\textwidth]{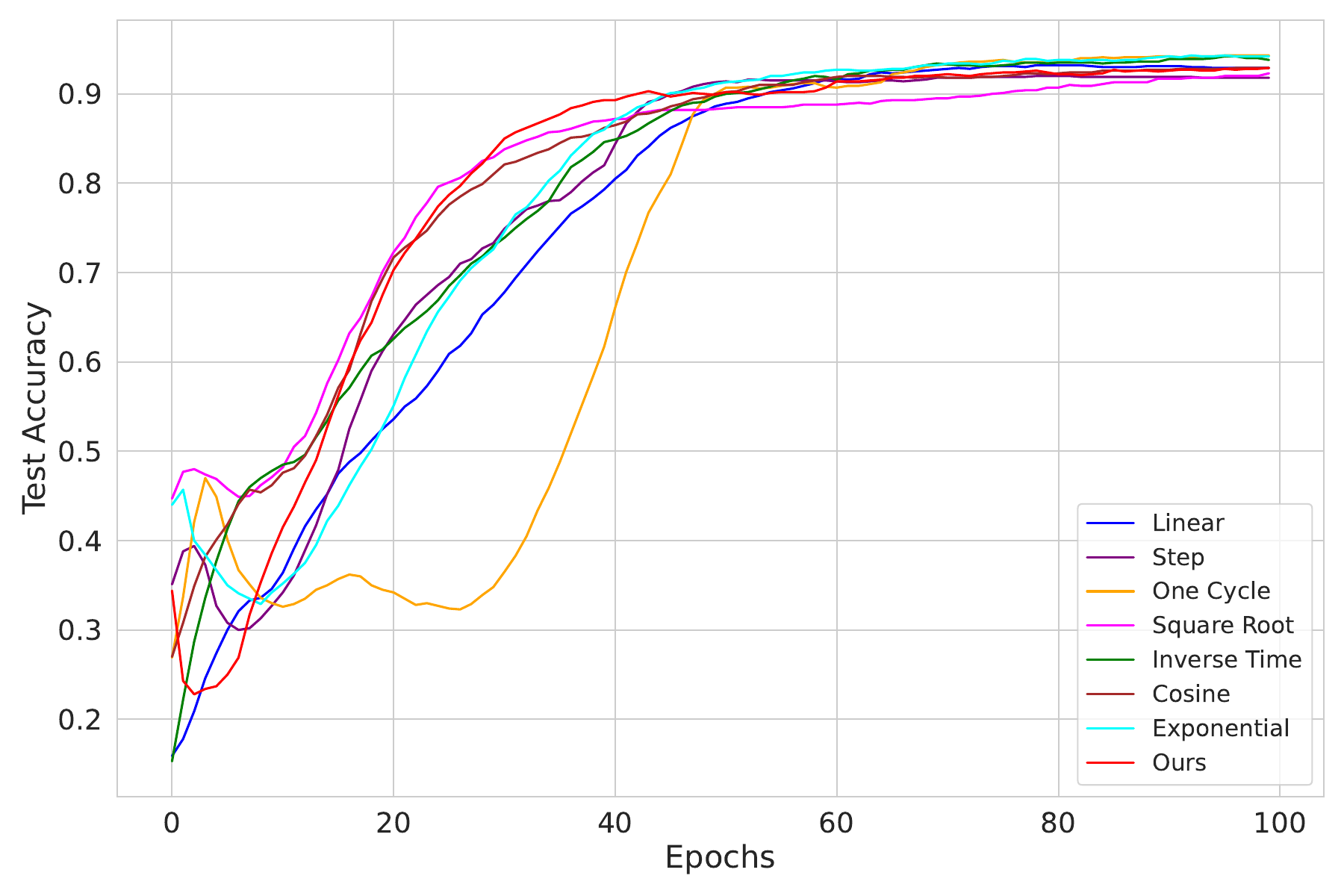}
    \caption{Validation acc. v/s Epochs}
  \end{subfigure}
  \caption{Full-batch experiments on a single layer network with 300 nodes in each layer, trained on MNIST.}
\label{fig:1_300}
\end{figure}
\begin{figure}[]
  \centering
  \begin{subfigure}[b]{0.31\textwidth}
    \includegraphics[width=\textwidth]{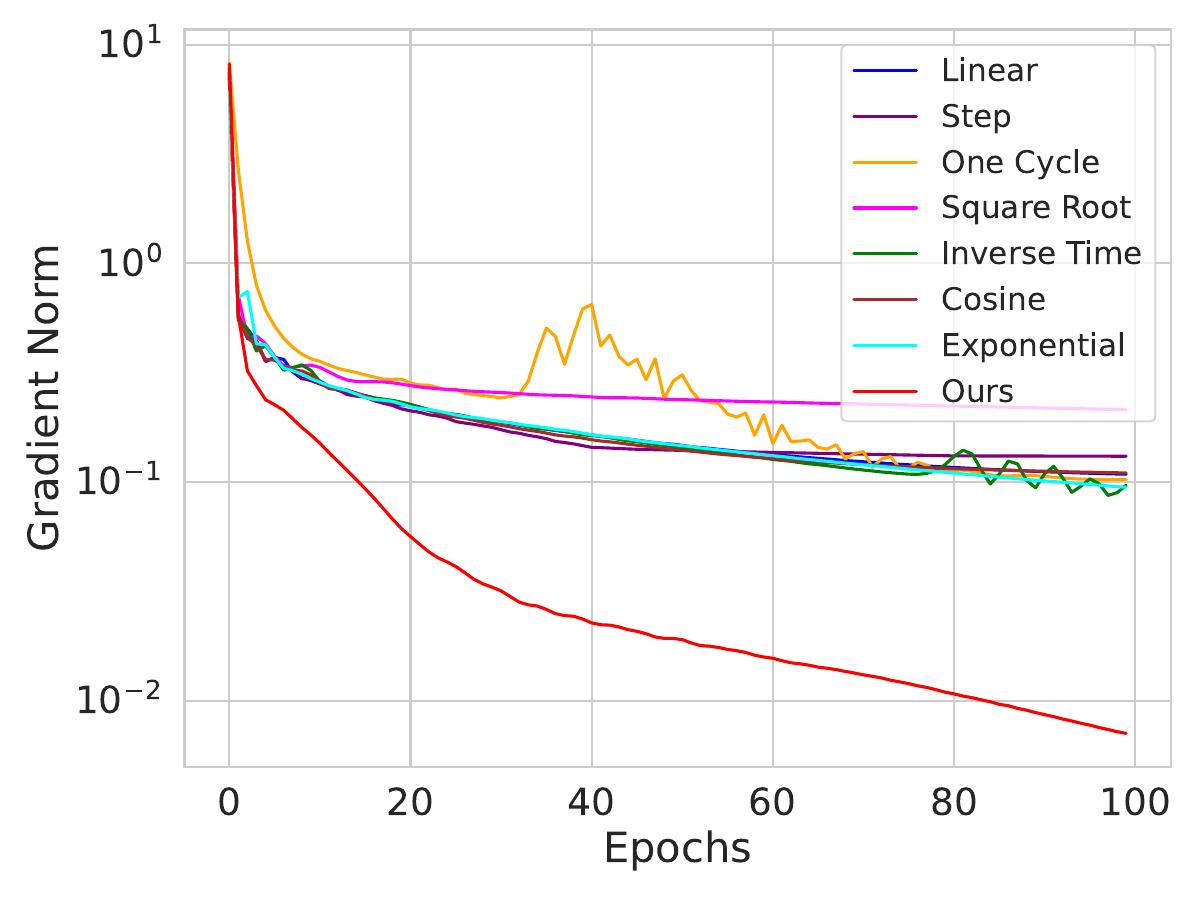}
    \caption{Gradient norm v/s Epochs}
  \end{subfigure}
  \hfill
  \begin{subfigure}[b]{0.31\textwidth}
    \includegraphics[width=\textwidth]{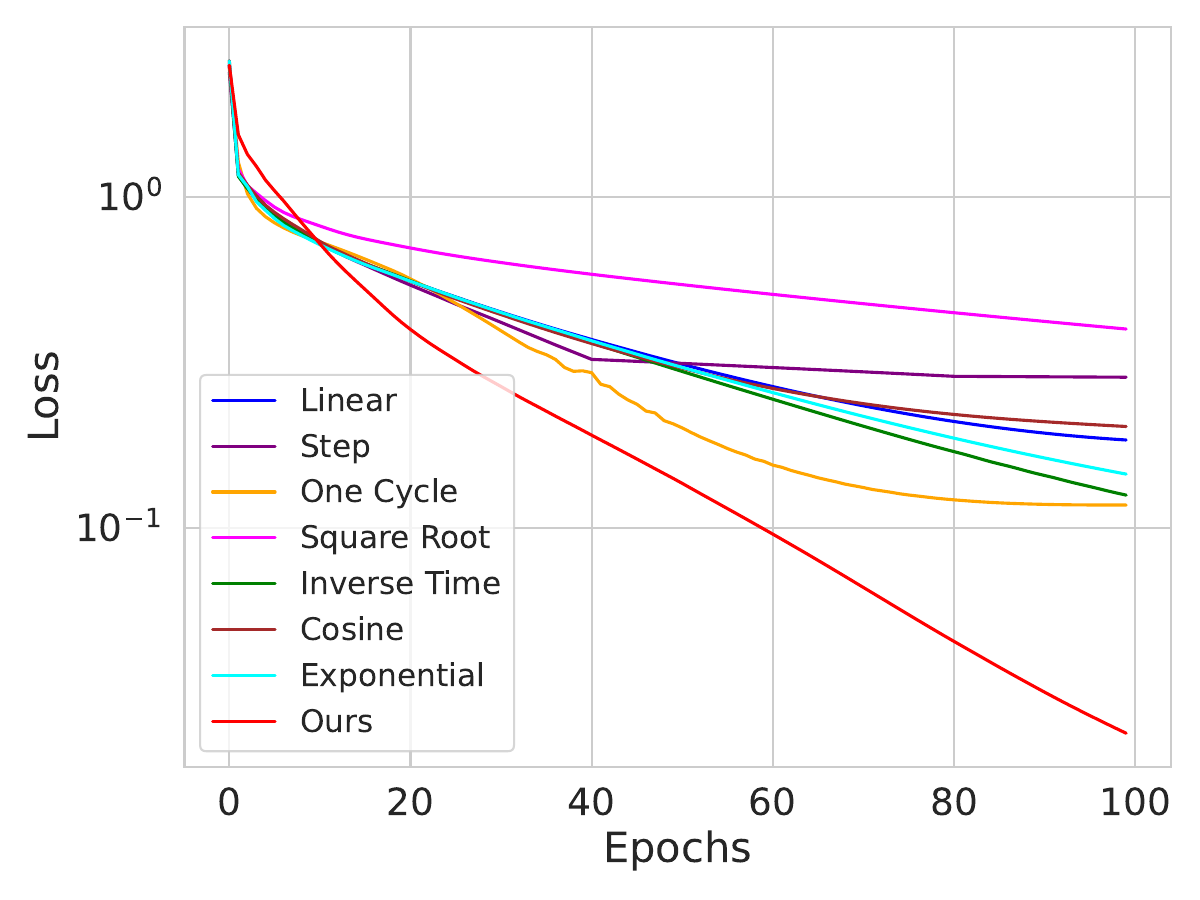}
    \caption{Training loss v/s Epochs}
  \end{subfigure}
  \hfill
  \begin{subfigure}[b]{0.31\textwidth}
    \includegraphics[width=\textwidth]{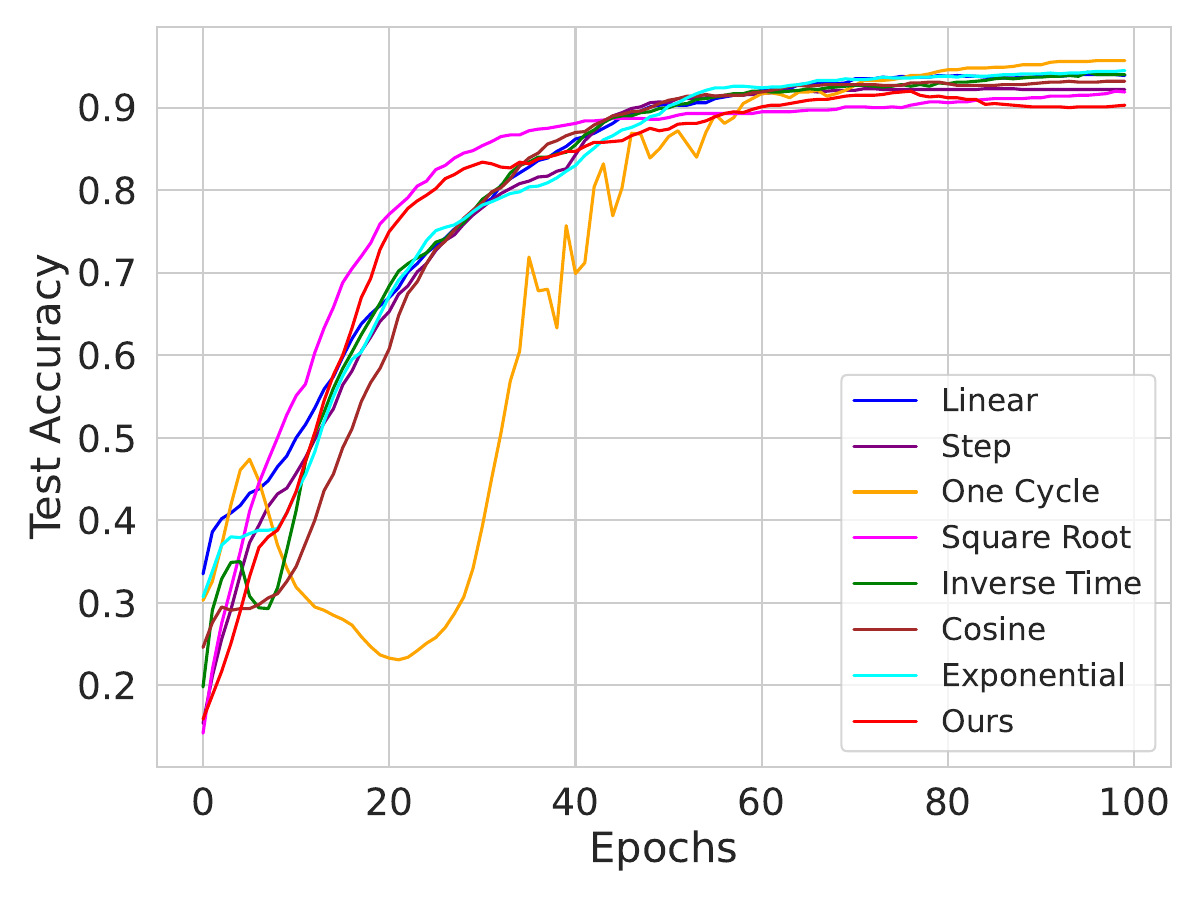}
    \caption{Validation acc. v/s Epochs}
  \end{subfigure}
  \caption{Full-batch experiments on a single layer network with 1000 nodes in each layer, trained on MNIST.}
\label{fig:1_1000}
\end{figure}
\begin{figure}[]
  \centering
  \begin{subfigure}[b]{0.31\textwidth}
    \includegraphics[width=\textwidth]{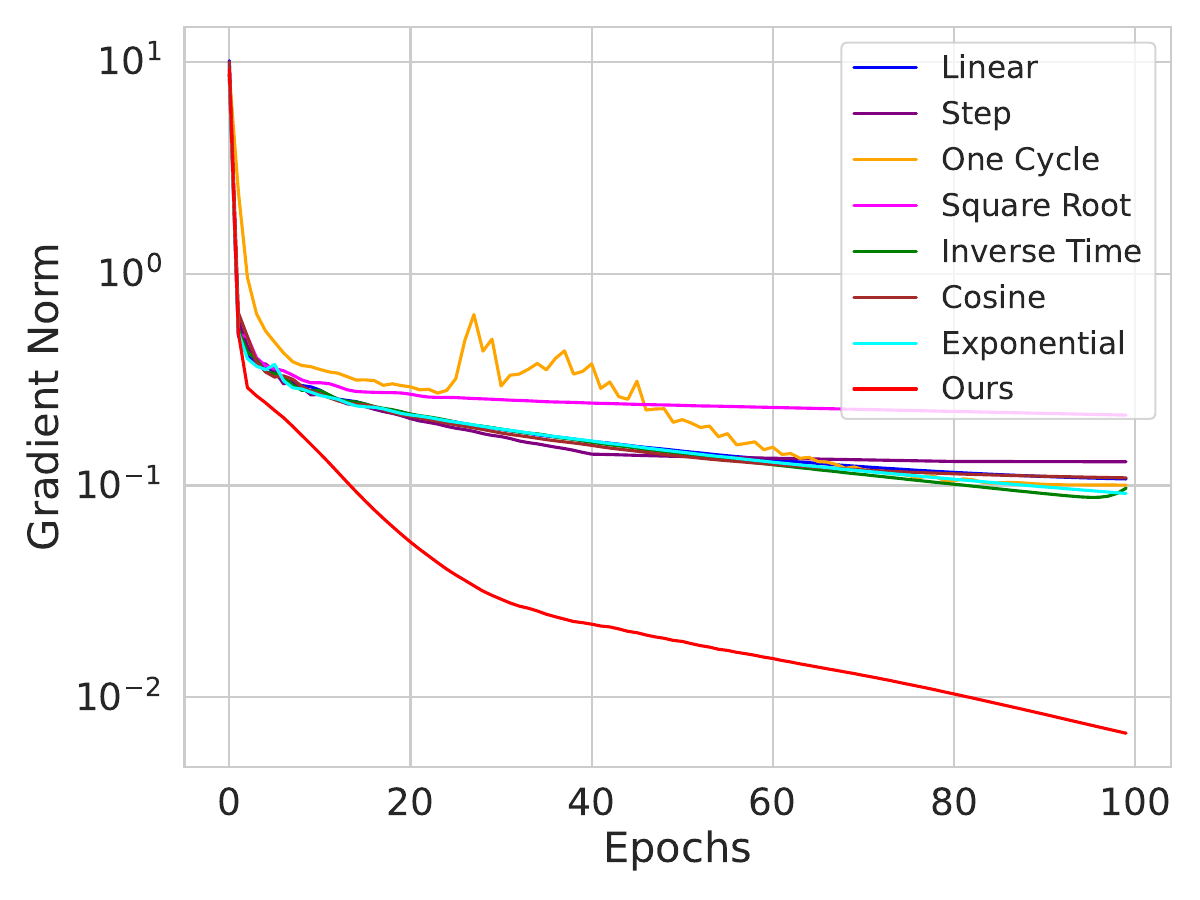}
    \caption{Gradient norm v/s Epochs}
  \end{subfigure}
  \hfill
  \begin{subfigure}[b]{0.31\textwidth}
    \includegraphics[width=\textwidth]{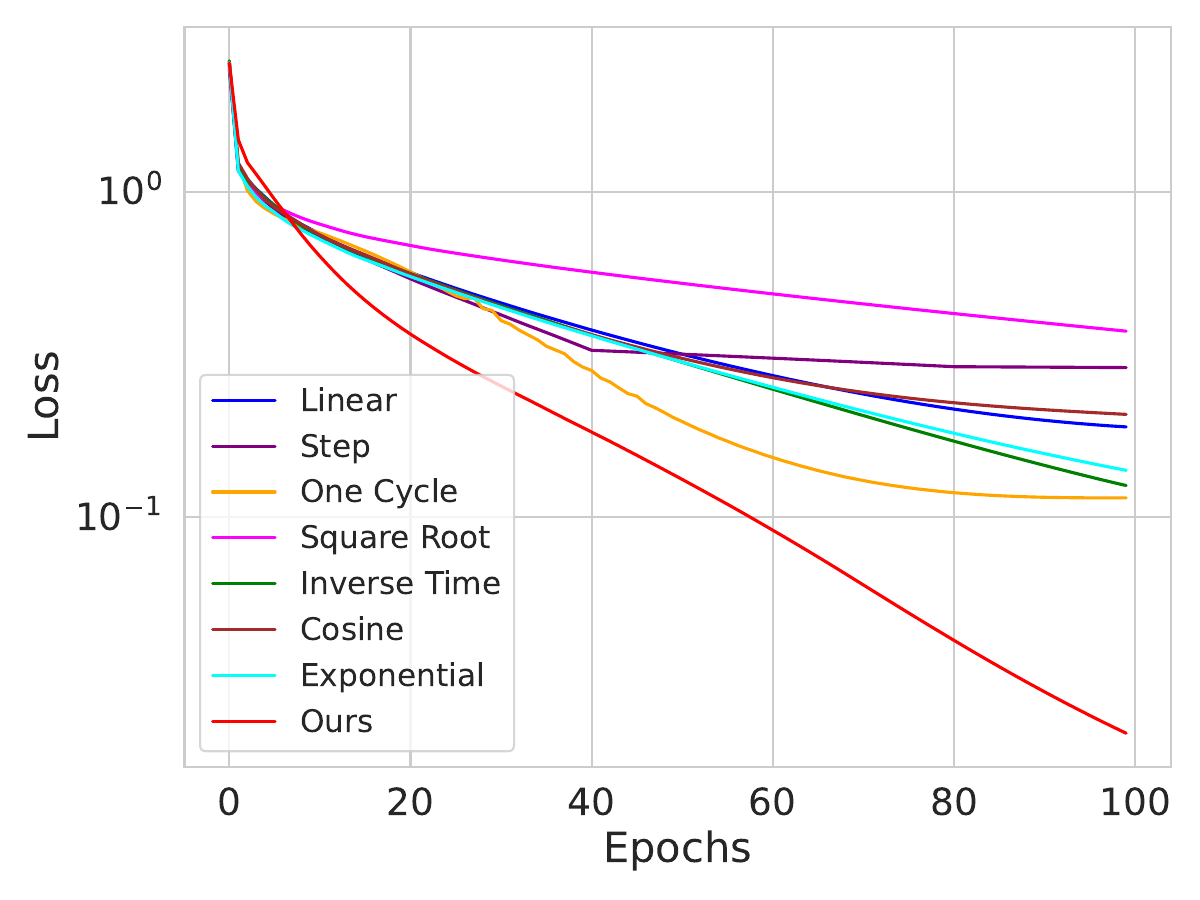}
    \caption{Training loss v/s Epochs}
  \end{subfigure}
  \hfill
  \begin{subfigure}[b]{0.31\textwidth}
    \includegraphics[width=\textwidth]{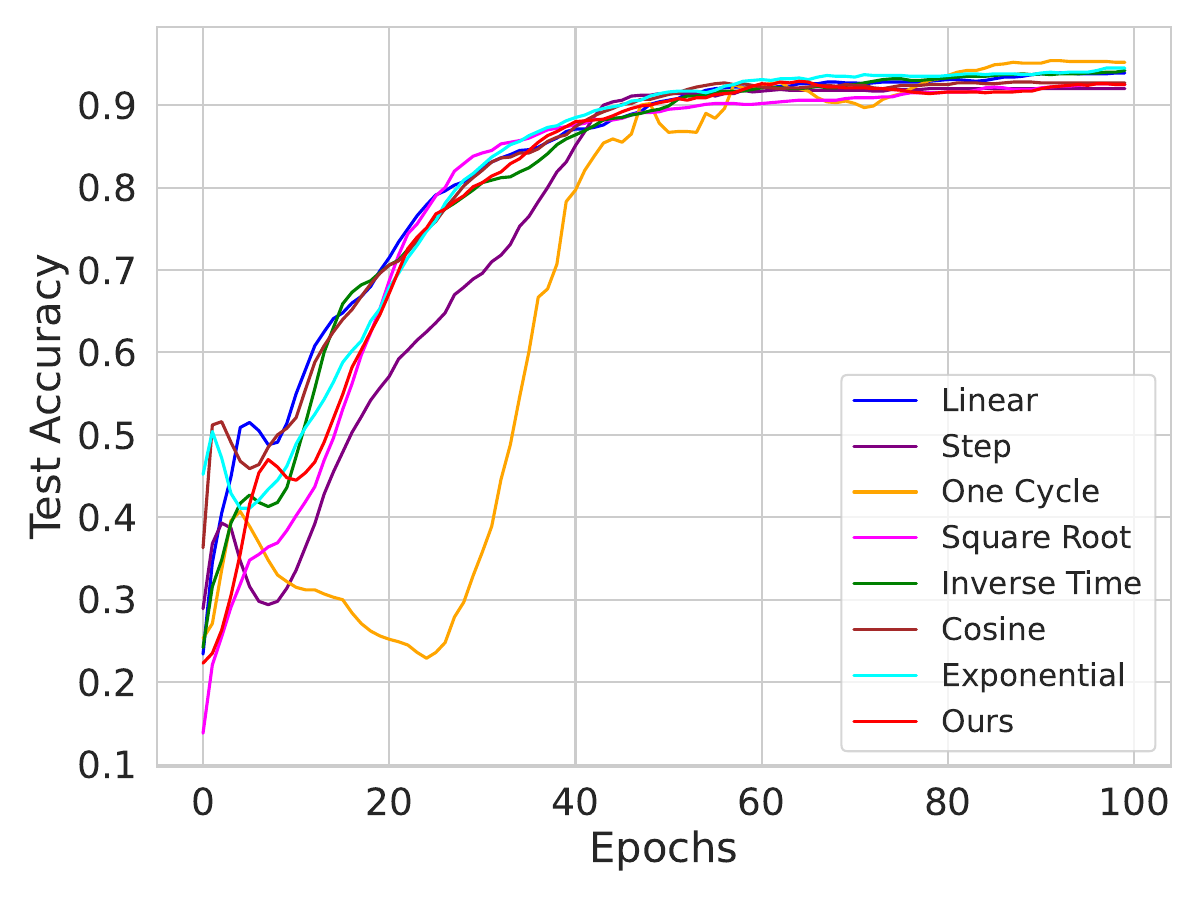}
    \caption{Validation acc. v/s Epochs}
  \end{subfigure}
  \caption{Full-batch experiments on a single layer network with 3000 nodes in each layer, trained on MNIST.}
\label{fig:1_3000}
\end{figure}
\begin{figure}[]
  \centering
  \begin{subfigure}[b]{0.31\textwidth}
    \includegraphics[width=\textwidth]{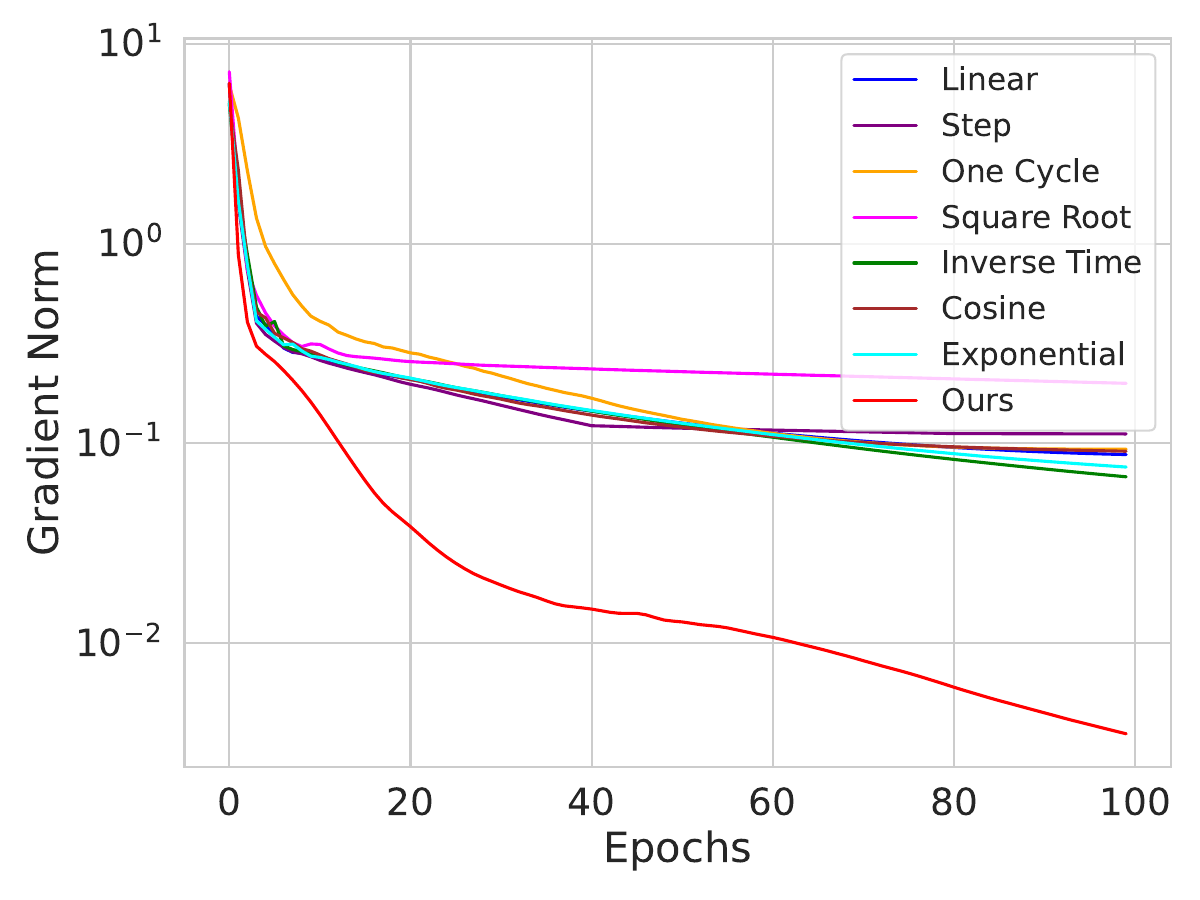}
    \caption{Gradient norm v/s Epochs}
  \end{subfigure}
  \hfill
  \begin{subfigure}[b]{0.31\textwidth}
    \includegraphics[width=\textwidth]{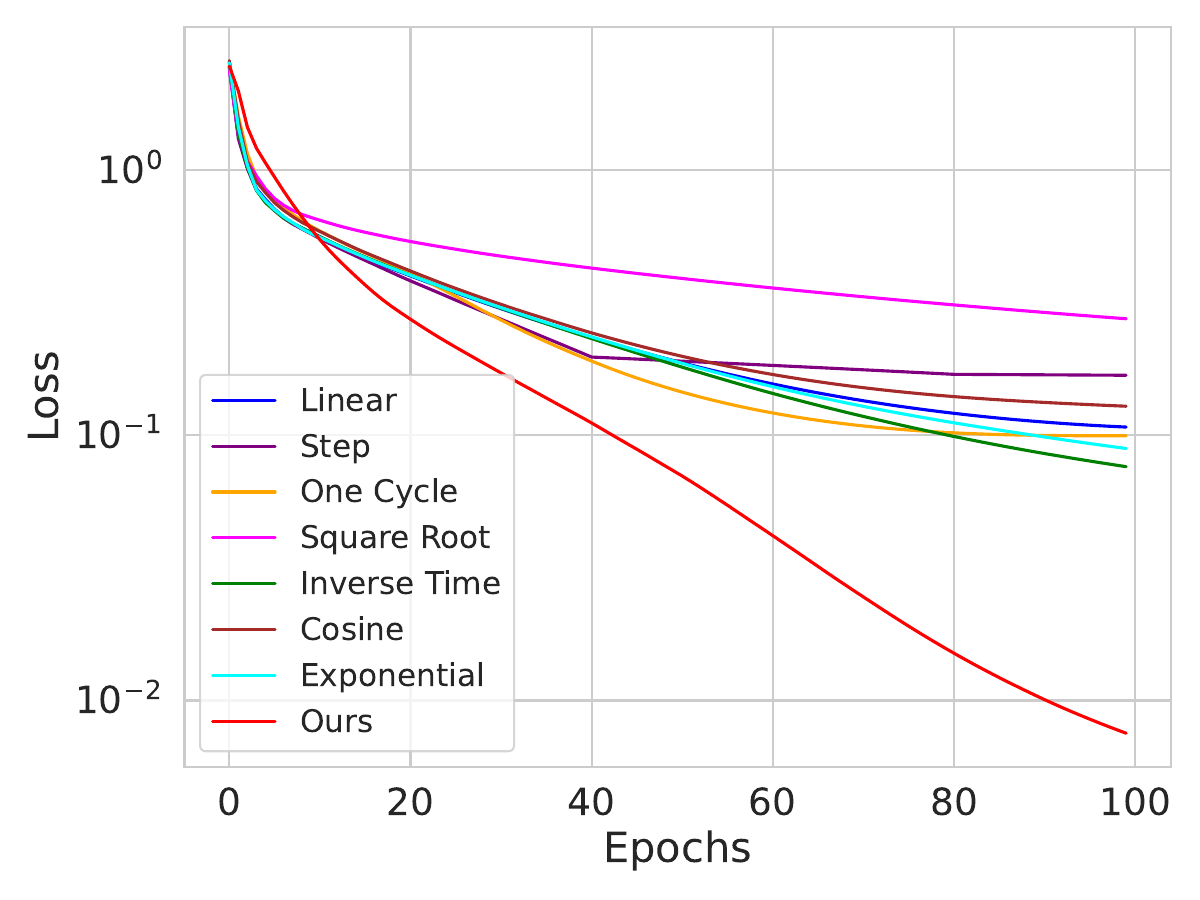}
    \caption{Training loss v/s Epochs}
  \end{subfigure}
  \hfill
  \begin{subfigure}[b]{0.31\textwidth}
    \includegraphics[width=\textwidth]{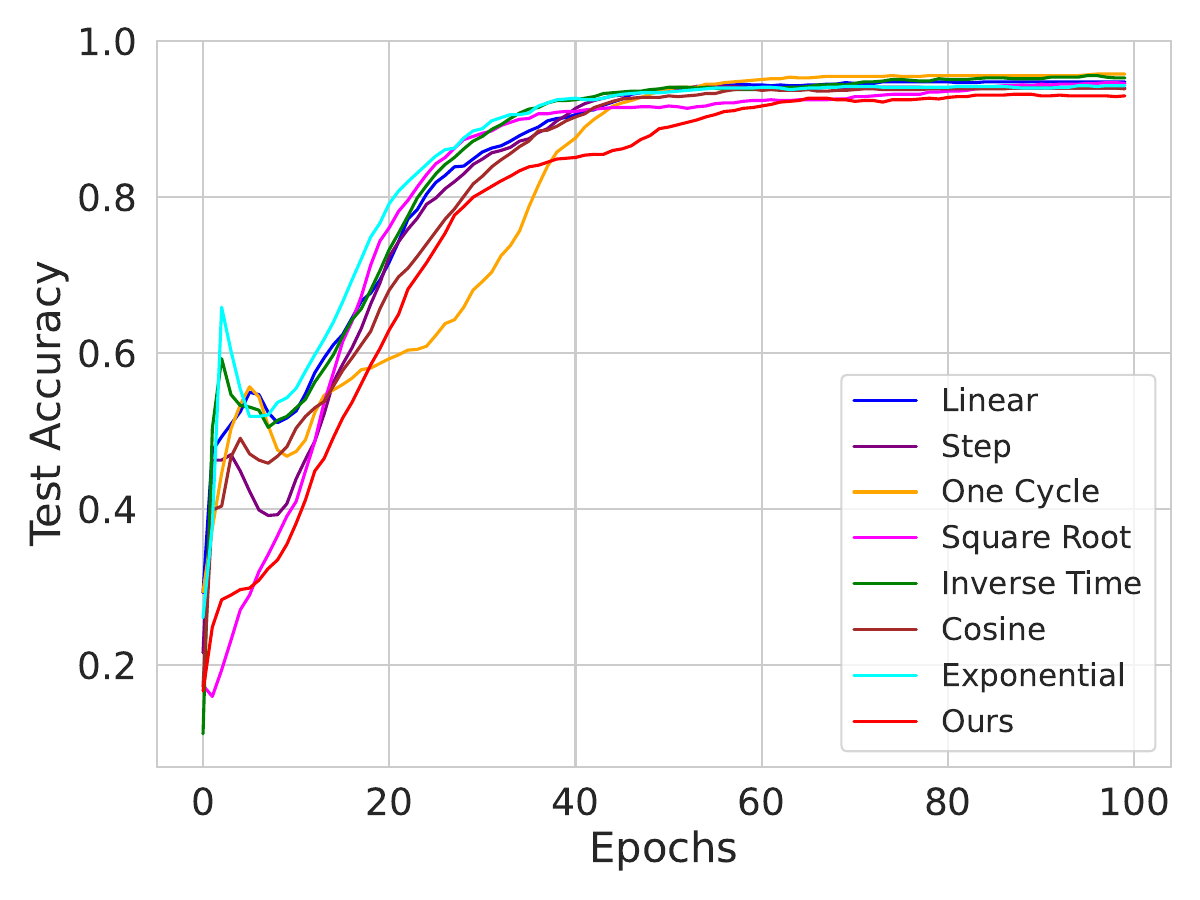}
    \caption{Validation acc. v/s Epochs}
  \end{subfigure}
  \caption{Full-batch experiments on a 3 layer network with 300 nodes in each layer, trained on MNIST.}
\label{fig:3__300}
\end{figure}
\begin{figure}[]
  \centering
  \begin{subfigure}[b]{0.31\textwidth}
    \includegraphics[width=\textwidth]{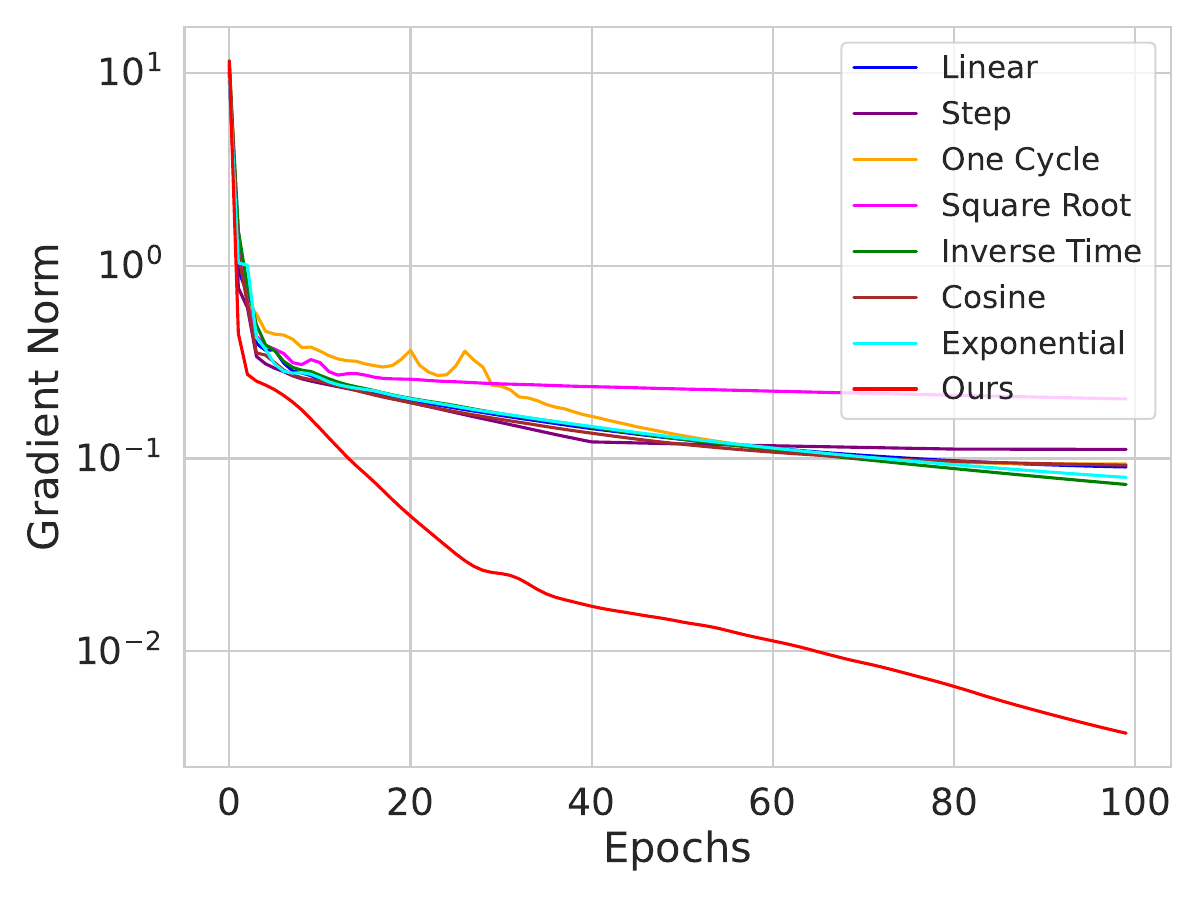}
    \caption{Gradient norm v/s Epochs}
  \end{subfigure}
  \hfill
  \begin{subfigure}[b]{0.31\textwidth}
    \includegraphics[width=\textwidth]{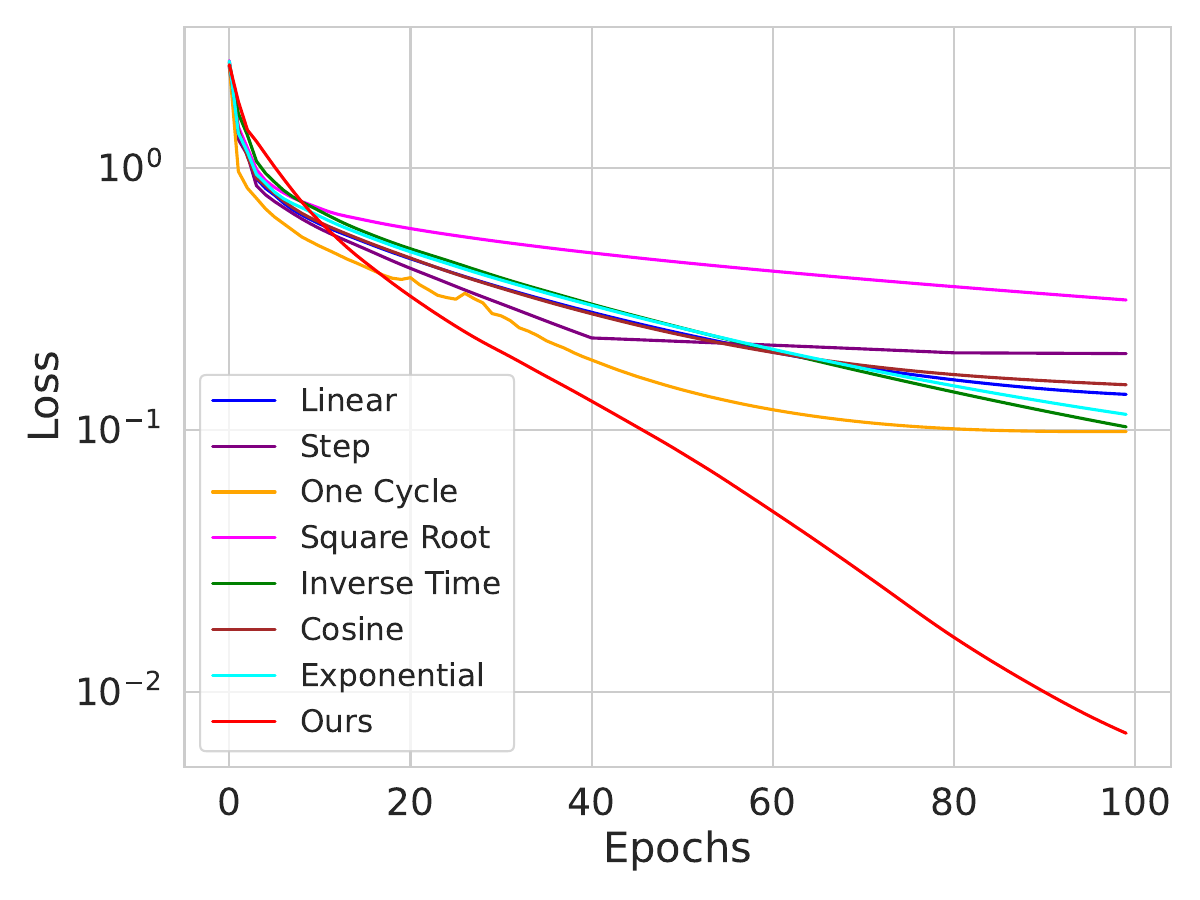}
    \caption{Training loss v/s Epochs}
  \end{subfigure}
  \hfill
  \begin{subfigure}[b]{0.31\textwidth}
    \includegraphics[width=\textwidth]{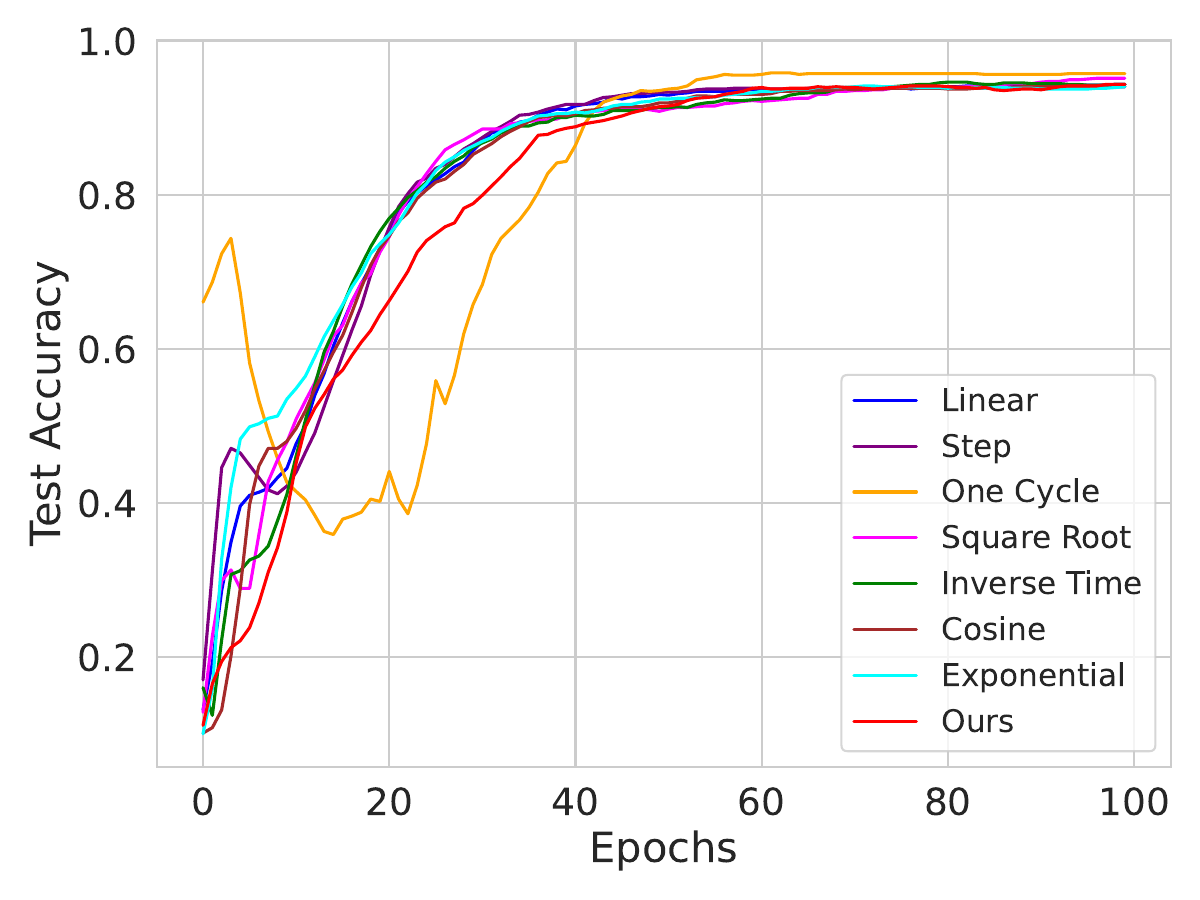}
    \caption{Validation acc. v/s Epochs}
  \end{subfigure}
  \caption{Full-batch experiments on a 3 layer network with 3000 nodes in each layer, trained on MNIST.}
\label{fig:3__3000}
\end{figure}
\begin{figure}[]
  \centering
  \begin{subfigure}[b]{0.31\textwidth}
    \includegraphics[width=\textwidth]{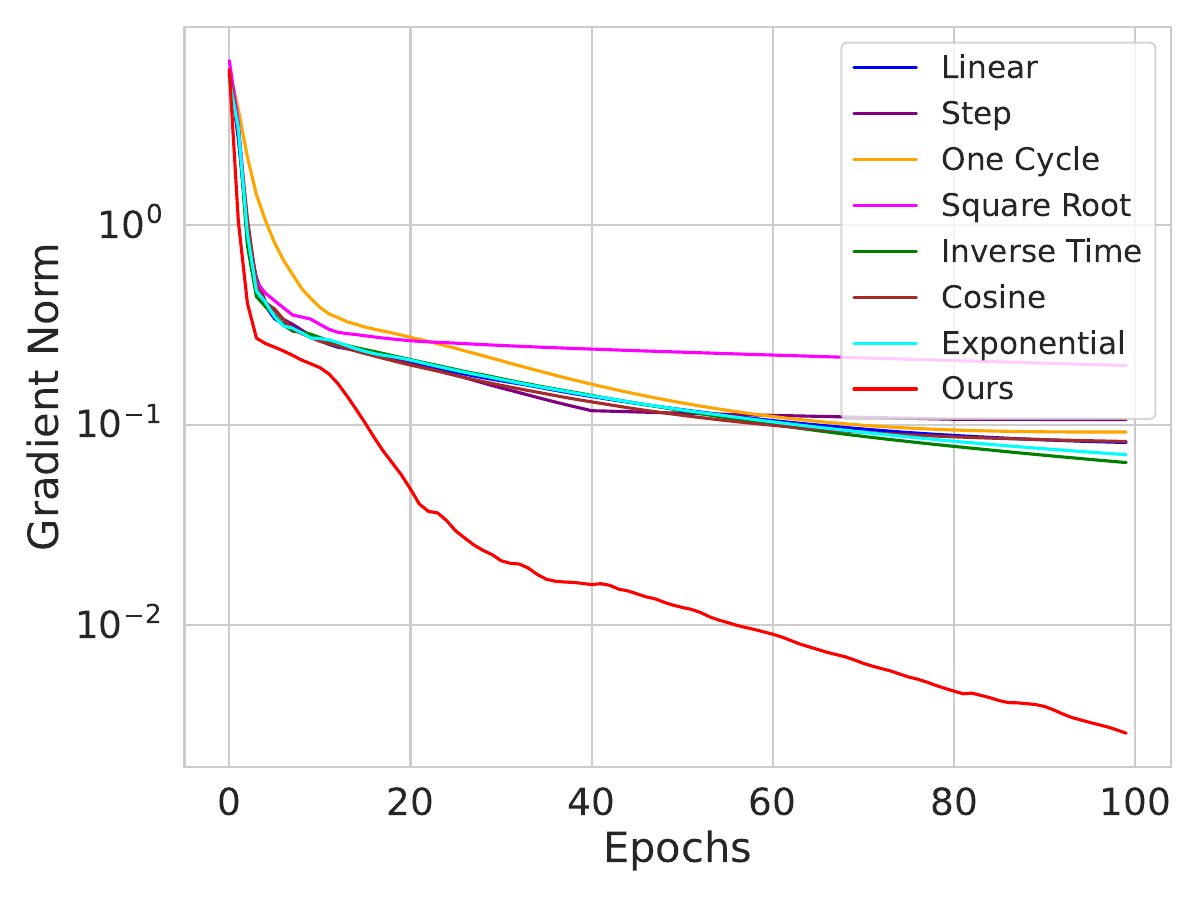}
    \caption{Gradient norm v/s Epochs}
  \end{subfigure}
  \hfill
  \begin{subfigure}[b]{0.31\textwidth}
    \includegraphics[width=\textwidth]{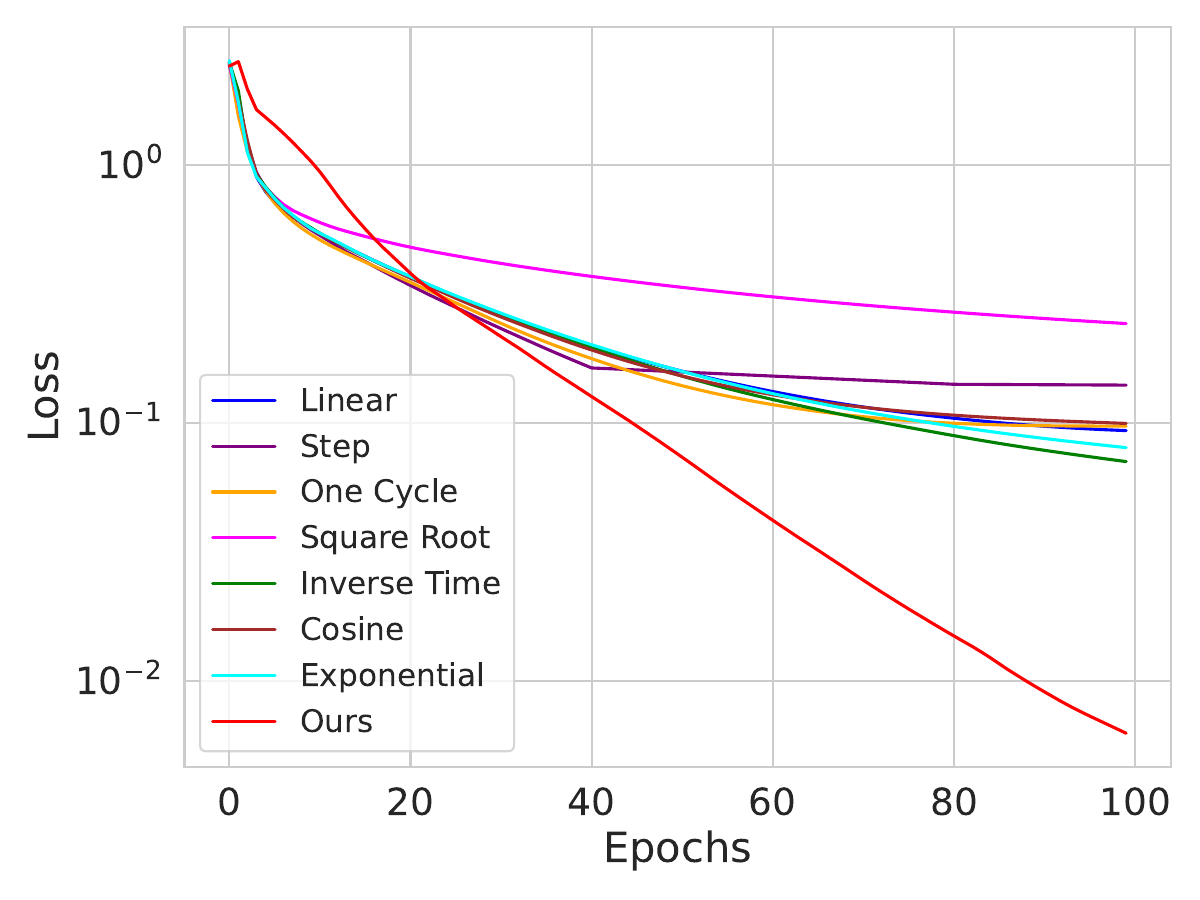}
    \caption{Training loss v/s Epochs}
  \end{subfigure}
  \hfill
  \begin{subfigure}[b]{0.31\textwidth}
    \includegraphics[width=\textwidth]{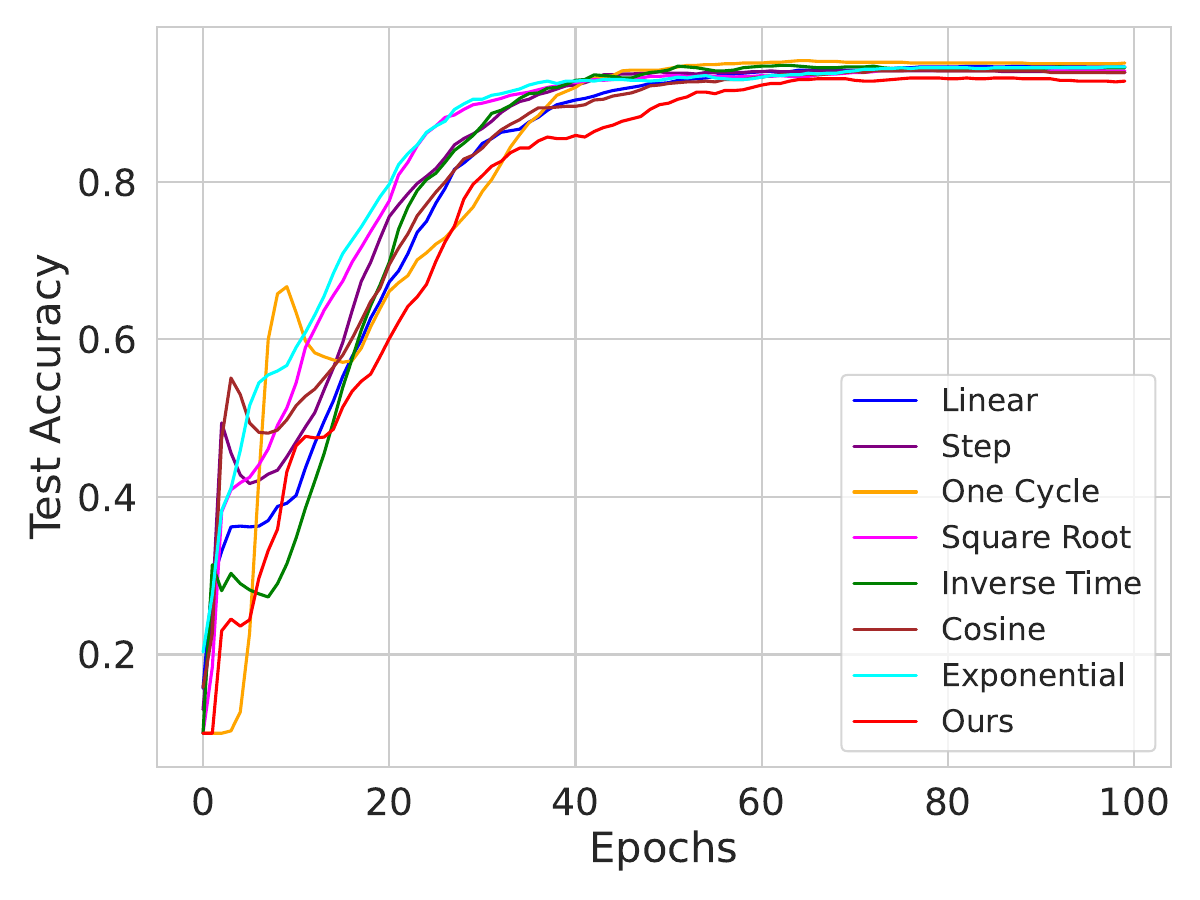}
    \caption{Validation acc. v/s Epochs}
  \end{subfigure}
  \caption{Full-batch experiments on a 5 layer network with 300 nodes in each layer, trained on MNIST.}
\label{fig:5__300}
\end{figure}
\begin{figure}[]
  \centering
  \begin{subfigure}[b]{0.31\textwidth}
    \includegraphics[width=\textwidth]{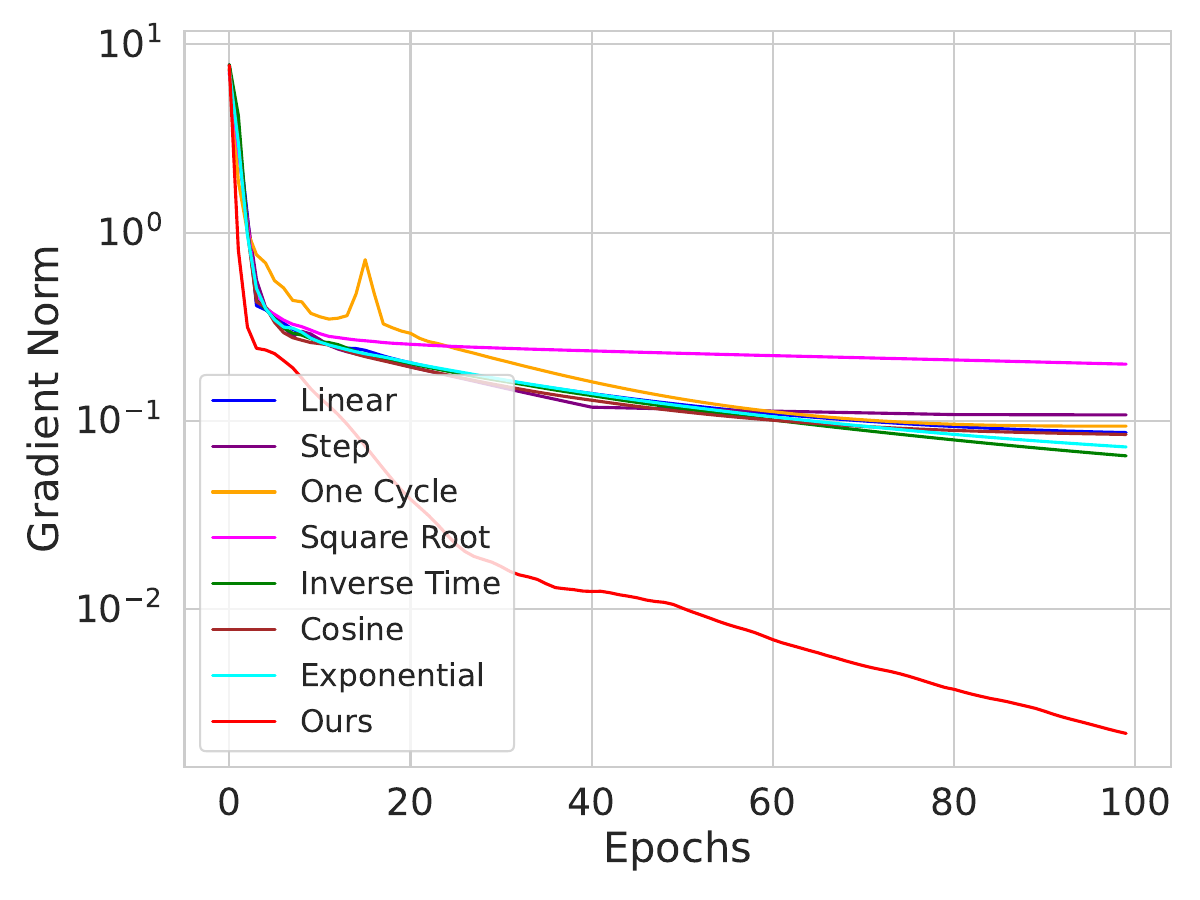}
    \caption{Gradient norm v/s Epochs}
  \end{subfigure}
  \hfill
  \begin{subfigure}[b]{0.31\textwidth}
    \includegraphics[width=\textwidth]{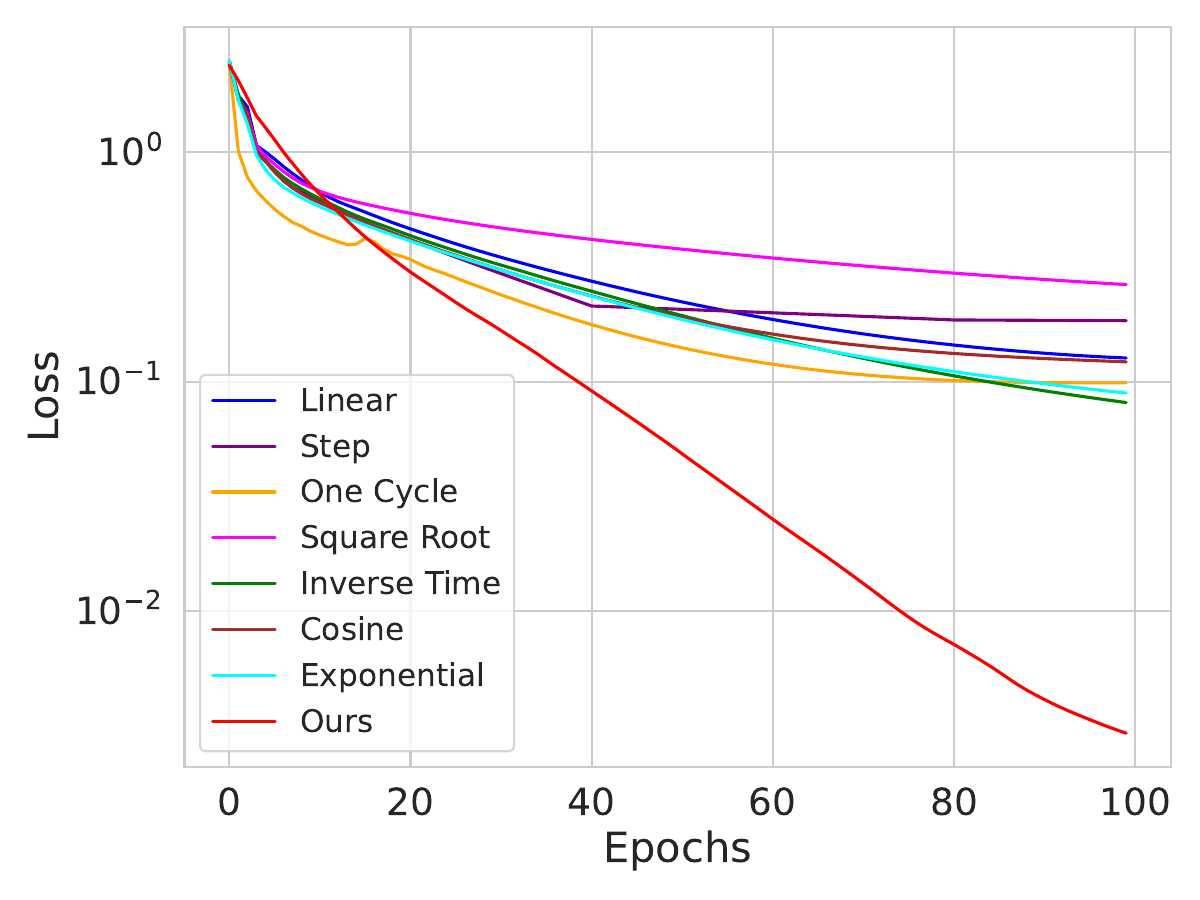}
    \caption{Training loss v/s Epochs}
  \end{subfigure}
  \hfill
  \begin{subfigure}[b]{0.31\textwidth}
    \includegraphics[width=\textwidth]{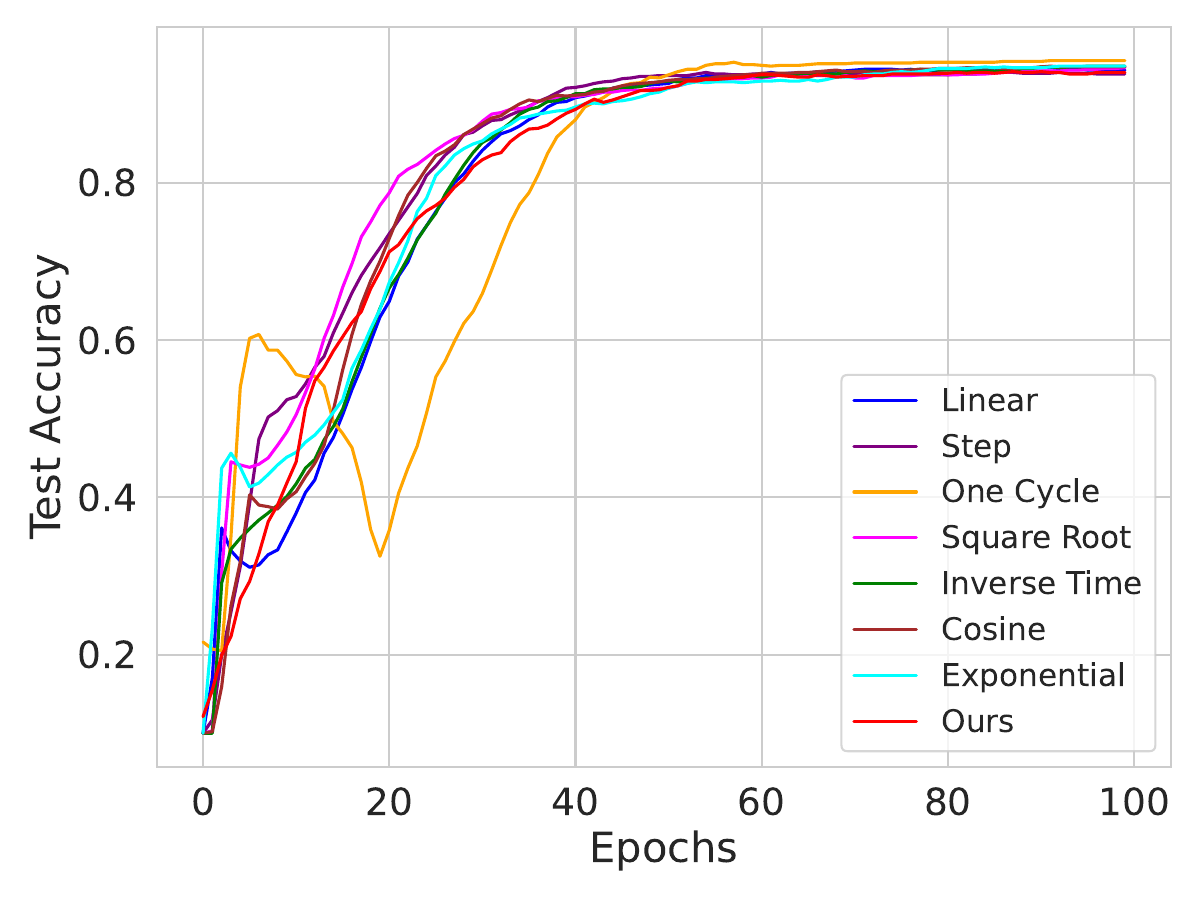}
    \caption{Validation acc. v/s Epochs}
  \end{subfigure}
  \caption{Full-batch experiments on a 5 layer network with 1000 nodes in each layer, trained on MNIST.}
\label{fig:5__1000}
\end{figure}
\begin{figure}[]
  \centering
  \begin{subfigure}[b]{0.31\textwidth}
    \includegraphics[width=\textwidth]{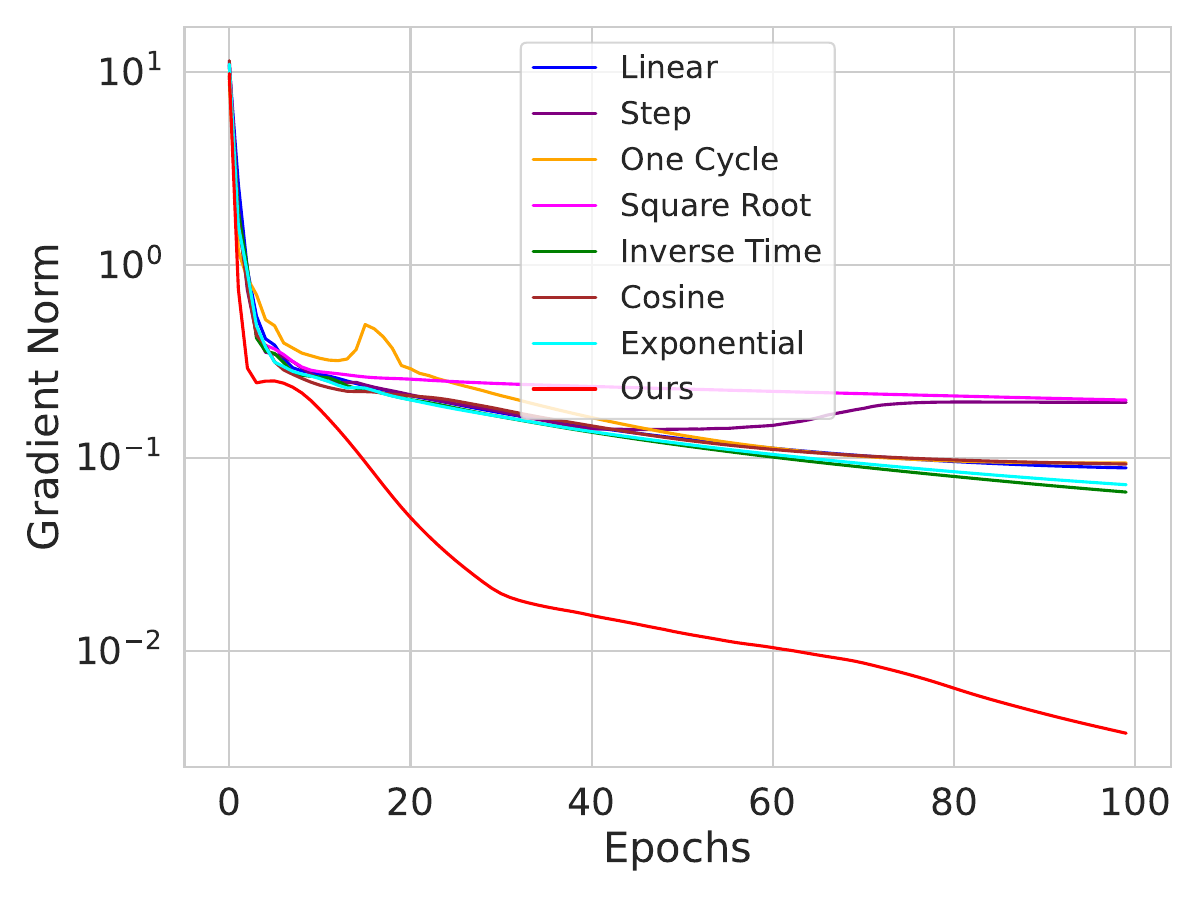}
    \caption{Gradient norm v/s Epochs}
  \end{subfigure}
  \hfill
  \begin{subfigure}[b]{0.31\textwidth}
    \includegraphics[width=\textwidth]{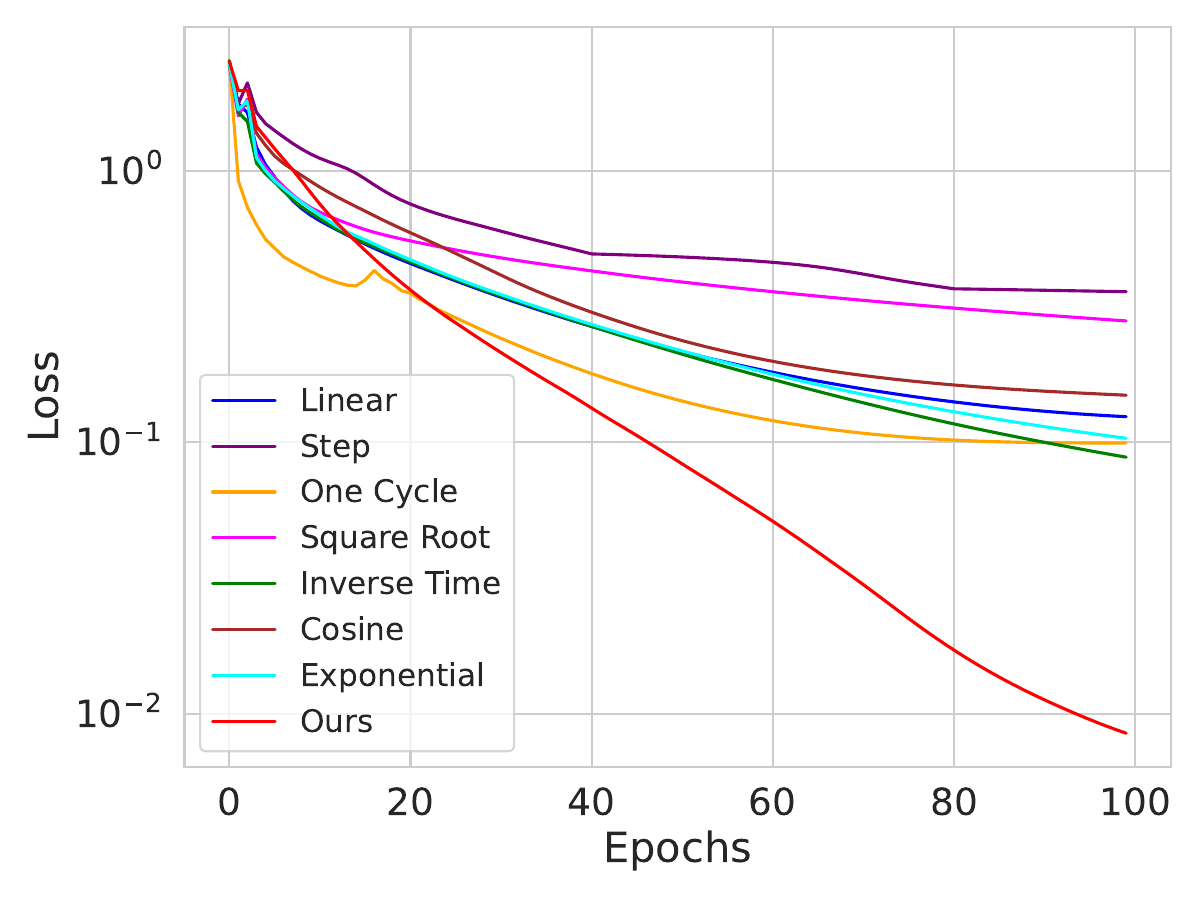}
    \caption{Training loss v/s Epochs}
  \end{subfigure}
  \hfill
  \begin{subfigure}[b]{0.31\textwidth}
    \includegraphics[width=\textwidth]{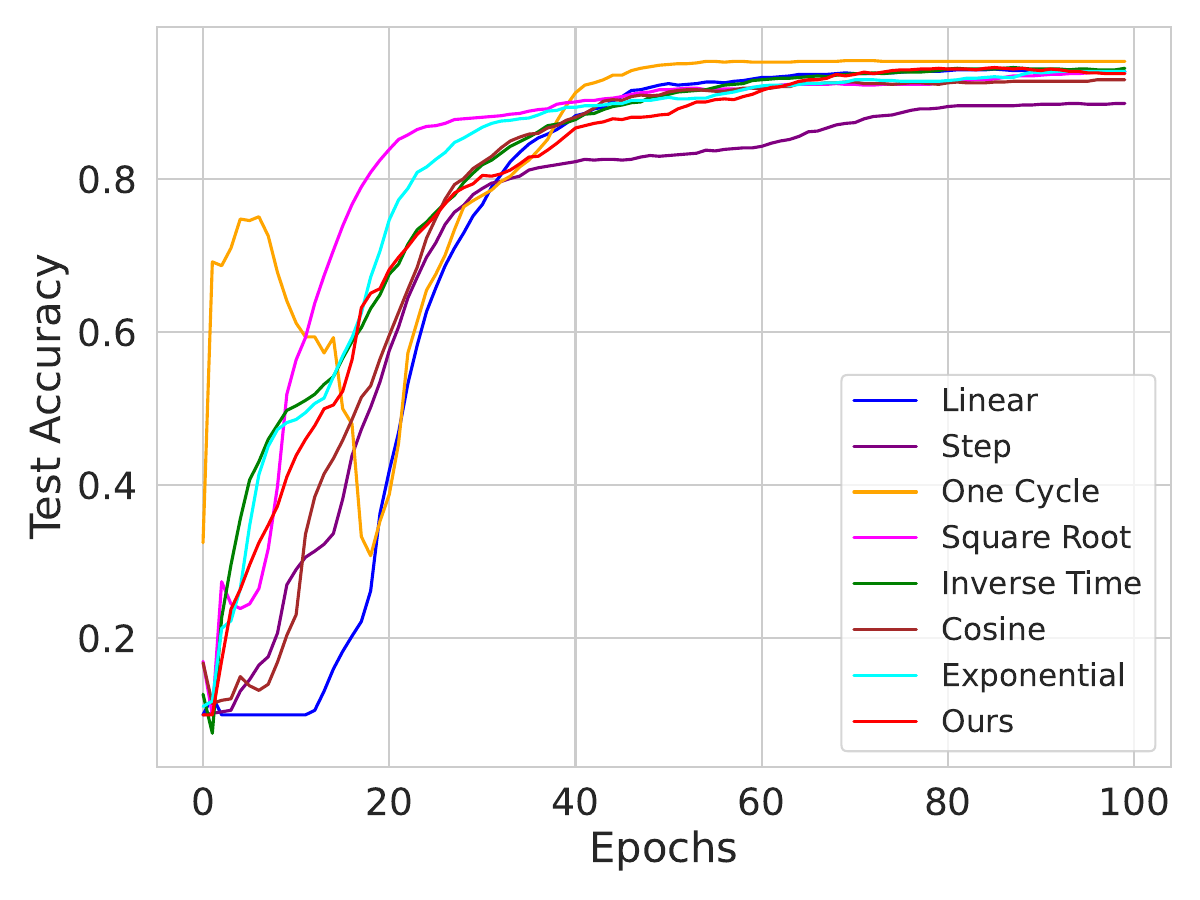}
    \caption{Validation acc. v/s Epochs}
  \end{subfigure}
  \caption{Full-batch experiments on a 5 layer network with 3000 nodes in each layer, trained on MNIST.}
\label{fig:5__3000}
\end{figure}
\newpage
Based on the preceding figures, it's evident that our chosen learning rate effectively reduces the gradient norm, leading to a commendable validation accuracy. In the case of full-batch setting, the outcomes remain consistent across various linear layer architectures.
\\

\textbf{C.2 Full-batch experiments, comparison against various fixed learning rates.}

This section serves as a continuation for Section~\ref{sec:fb_exp}.
\begin{figure}[htbp]
  \centering
  \begin{subfigure}[b]{0.31\textwidth}
    \includegraphics[width=\textwidth]{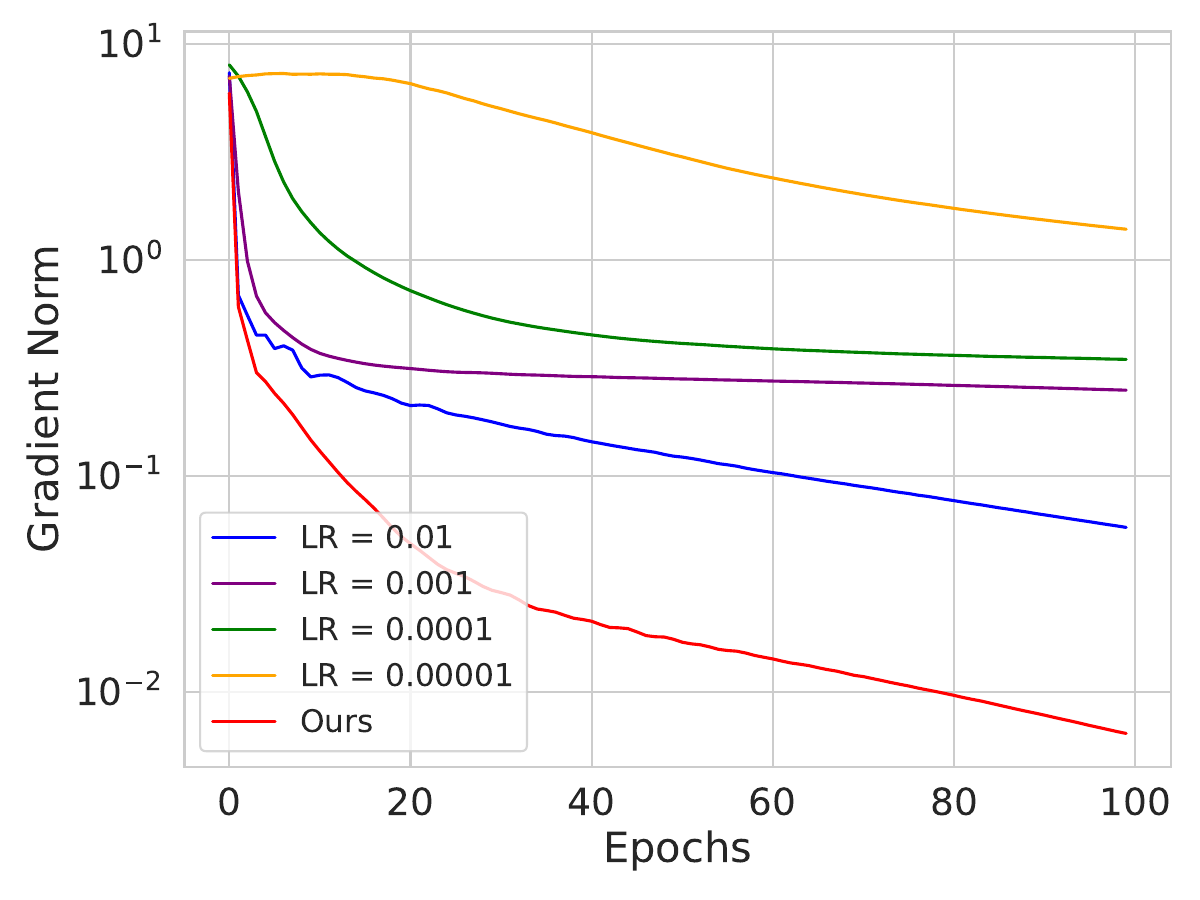}
    \caption{Gradient norm v/s Epochs}
  \end{subfigure}
  \hfill
  \begin{subfigure}[b]{0.31\textwidth}
    \includegraphics[width=\textwidth]{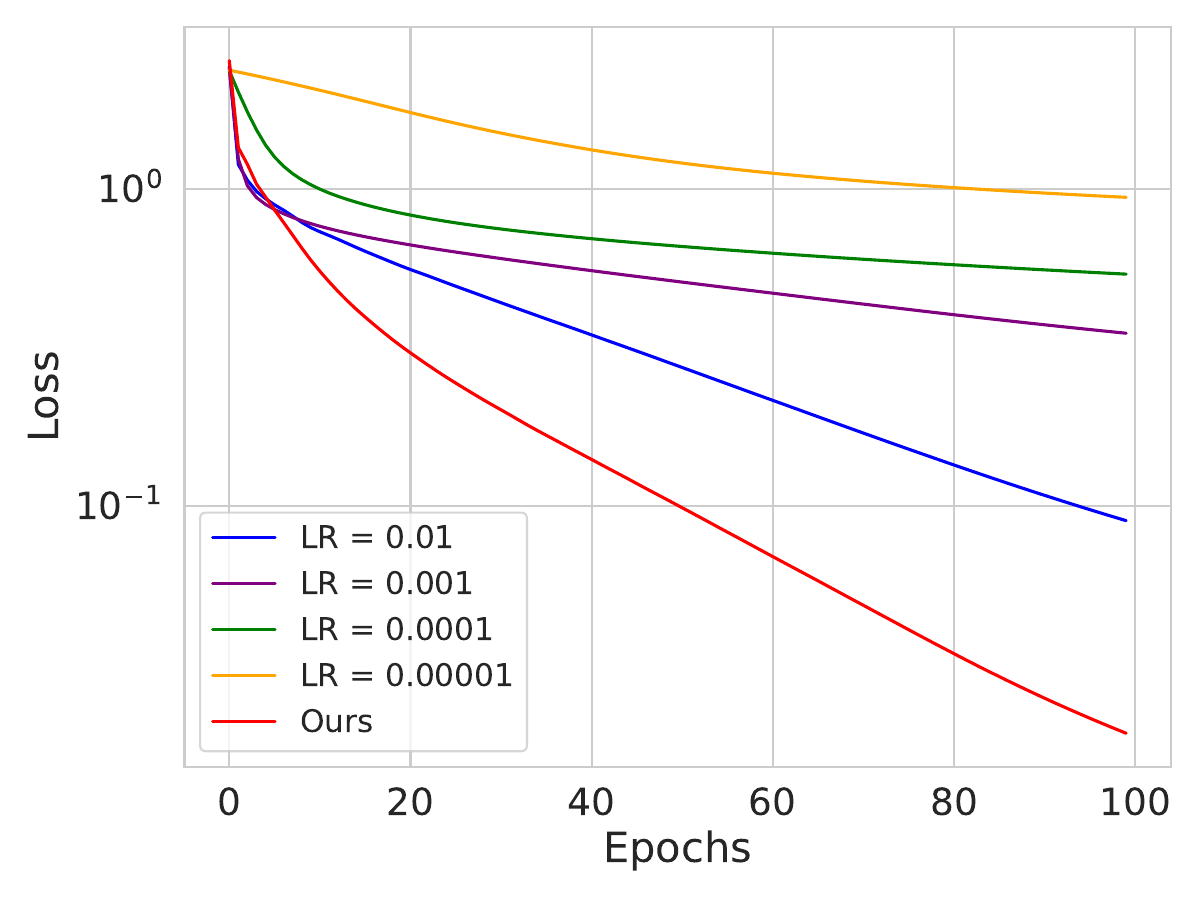}
    \caption{Training loss v/s Epochs}
  \end{subfigure}
  \hfill
  \begin{subfigure}[b]{0.31\textwidth}
    \includegraphics[width=\textwidth]{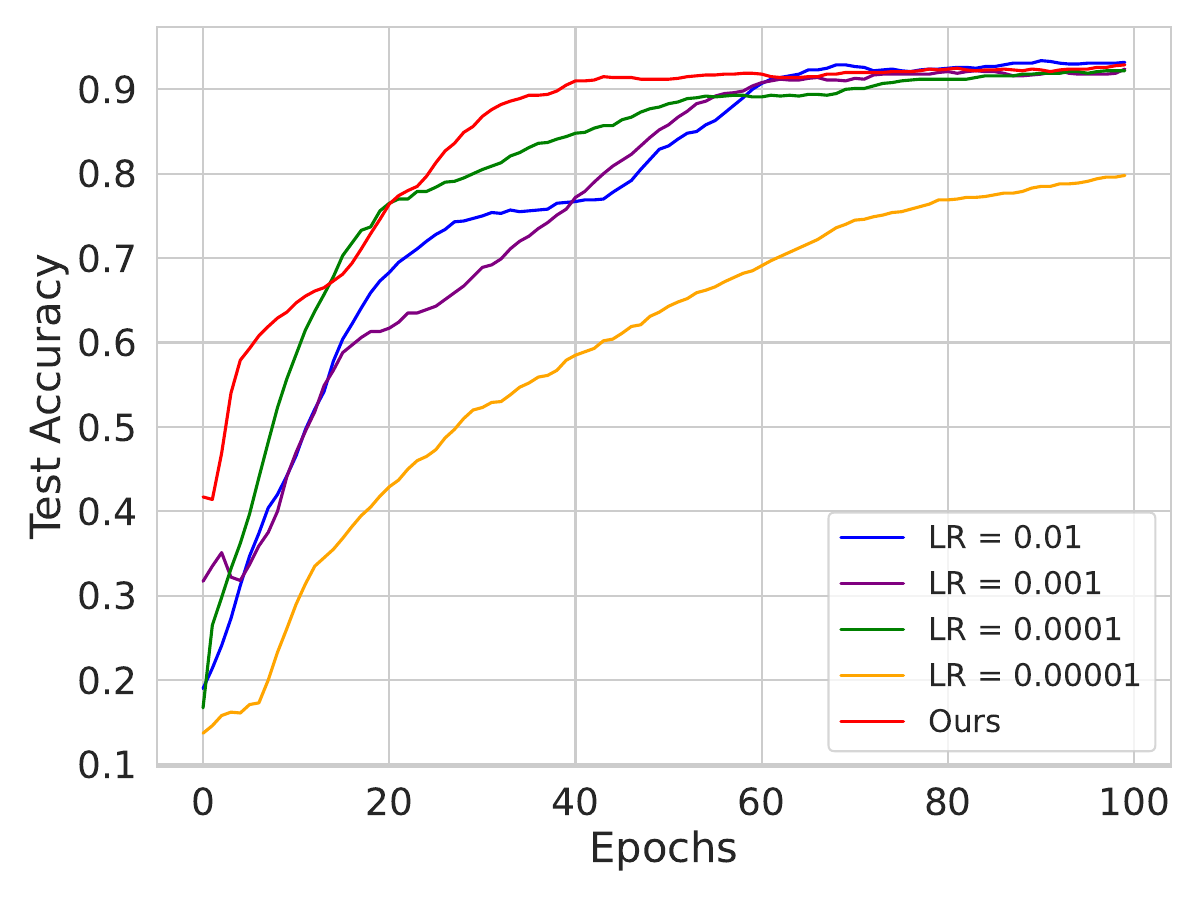}
    \caption{Validation acc. v/s Epochs}
  \end{subfigure}
  \caption{Full-batch experiments on a single layer network with 300 nodes in each layer, trained on MNIST.}
\label{fig:1__300_step}
\end{figure}
\begin{figure}[htbp]
  \centering
  \begin{subfigure}[b]{0.31\textwidth}
    \includegraphics[width=\textwidth]{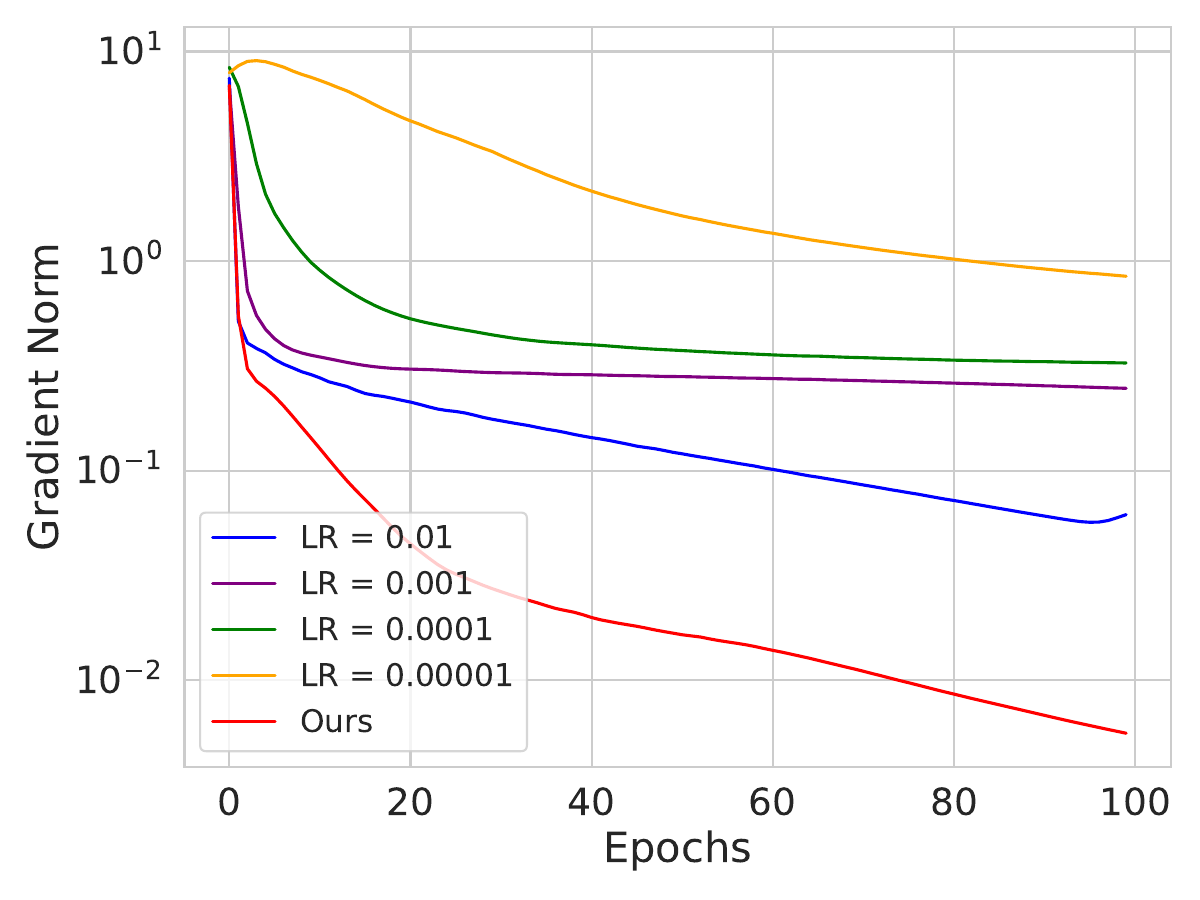}
    \caption{Gradient norm v/s Epochs}
  \end{subfigure}
  \hfill
  \begin{subfigure}[b]{0.31\textwidth}
    \includegraphics[width=\textwidth]{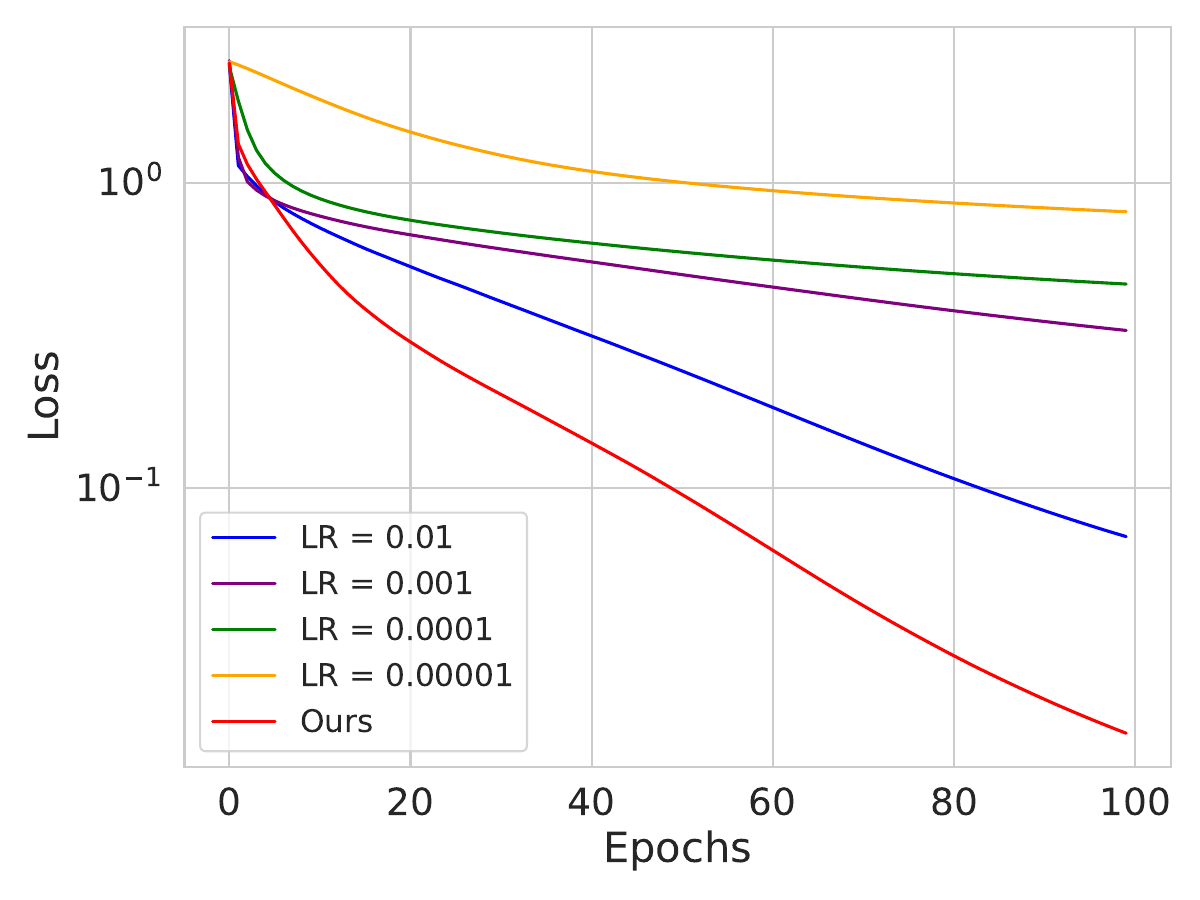}
    \caption{Training loss v/s Epochs}
  \end{subfigure}
  \hfill
  \begin{subfigure}[b]{0.31\textwidth}
    \includegraphics[width=\textwidth]{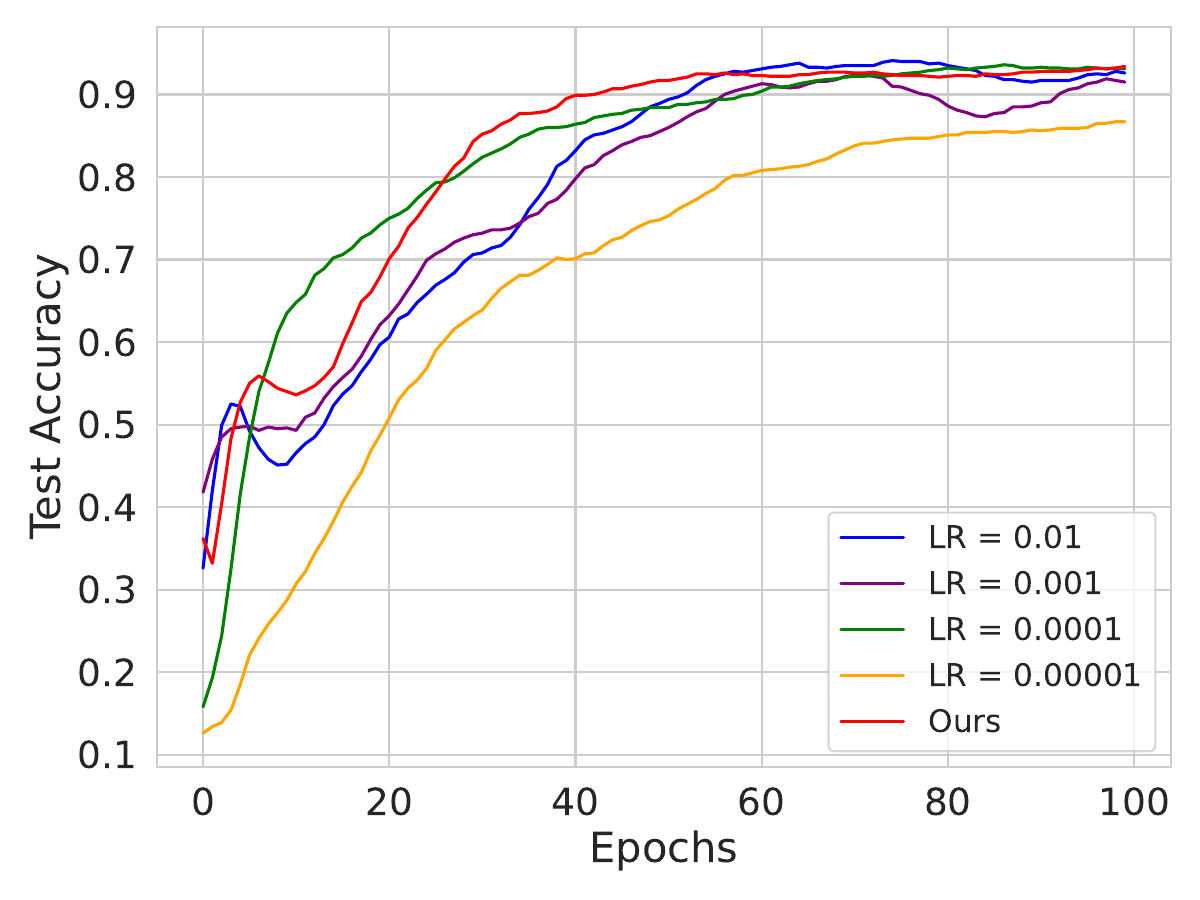}
    \caption{Validation acc. v/s Epochs}
  \end{subfigure}
  \caption{Full-batch experiments on a single layer network with 1000 nodes in each layer, trained on MNIST.}
\label{fig:1__1000_step}
\end{figure}
\begin{figure}[htbp]
  \centering
  \begin{subfigure}[b]{0.31\textwidth}
    \includegraphics[width=\textwidth]{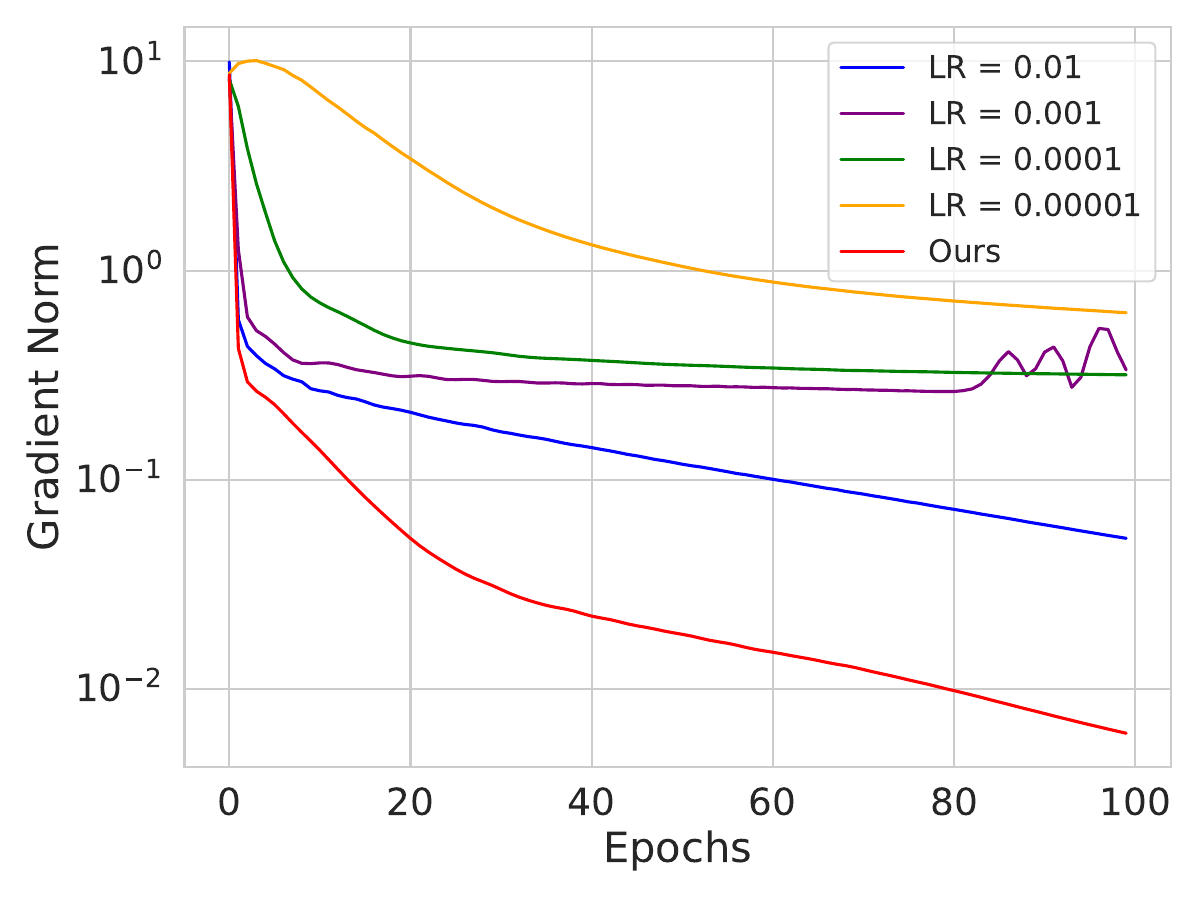}
    \caption{Gradient norm v/s Epochs}
  \end{subfigure}
  \hfill
  \begin{subfigure}[b]{0.31\textwidth}
    \includegraphics[width=\textwidth]{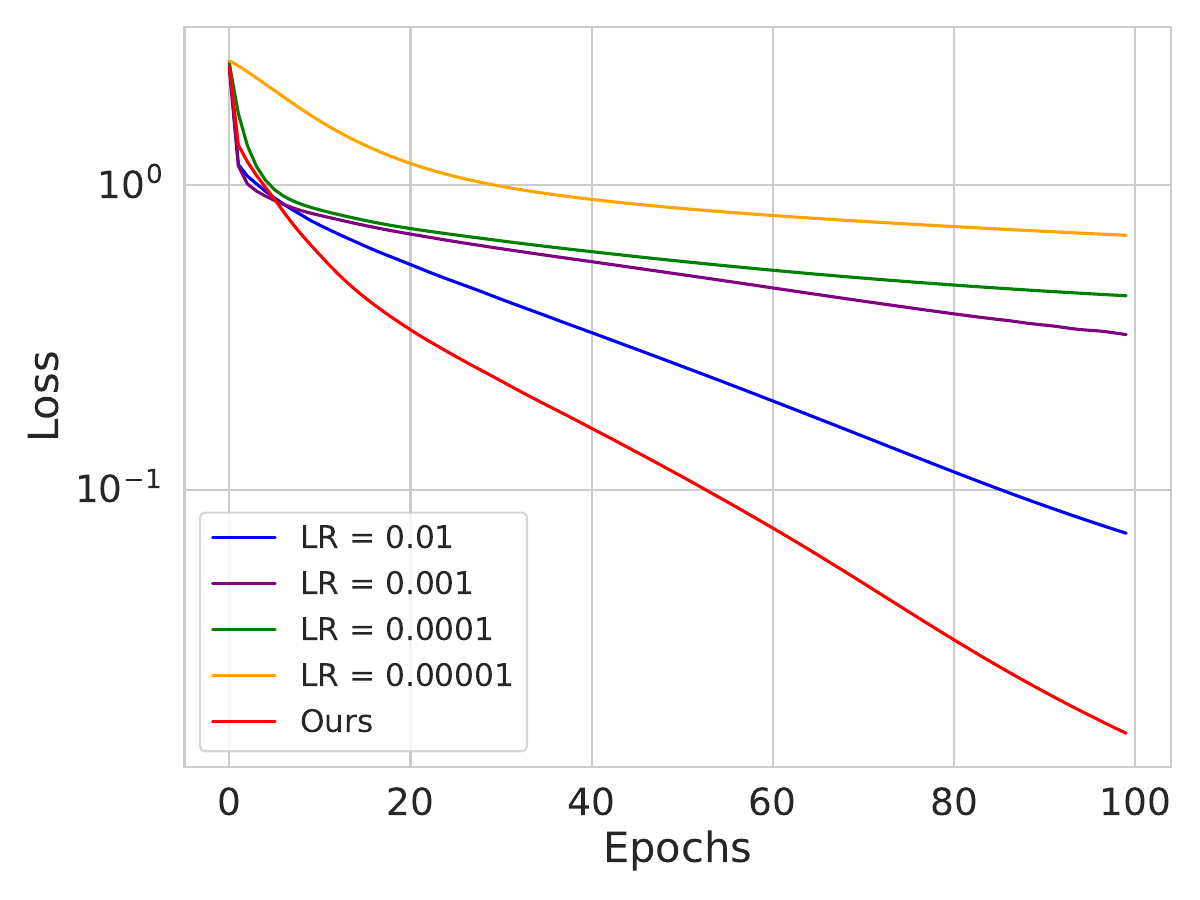}
    \caption{Training loss v/s Epochs}
  \end{subfigure}
  \hfill
  \begin{subfigure}[b]{0.31\textwidth}
    \includegraphics[width=\textwidth]{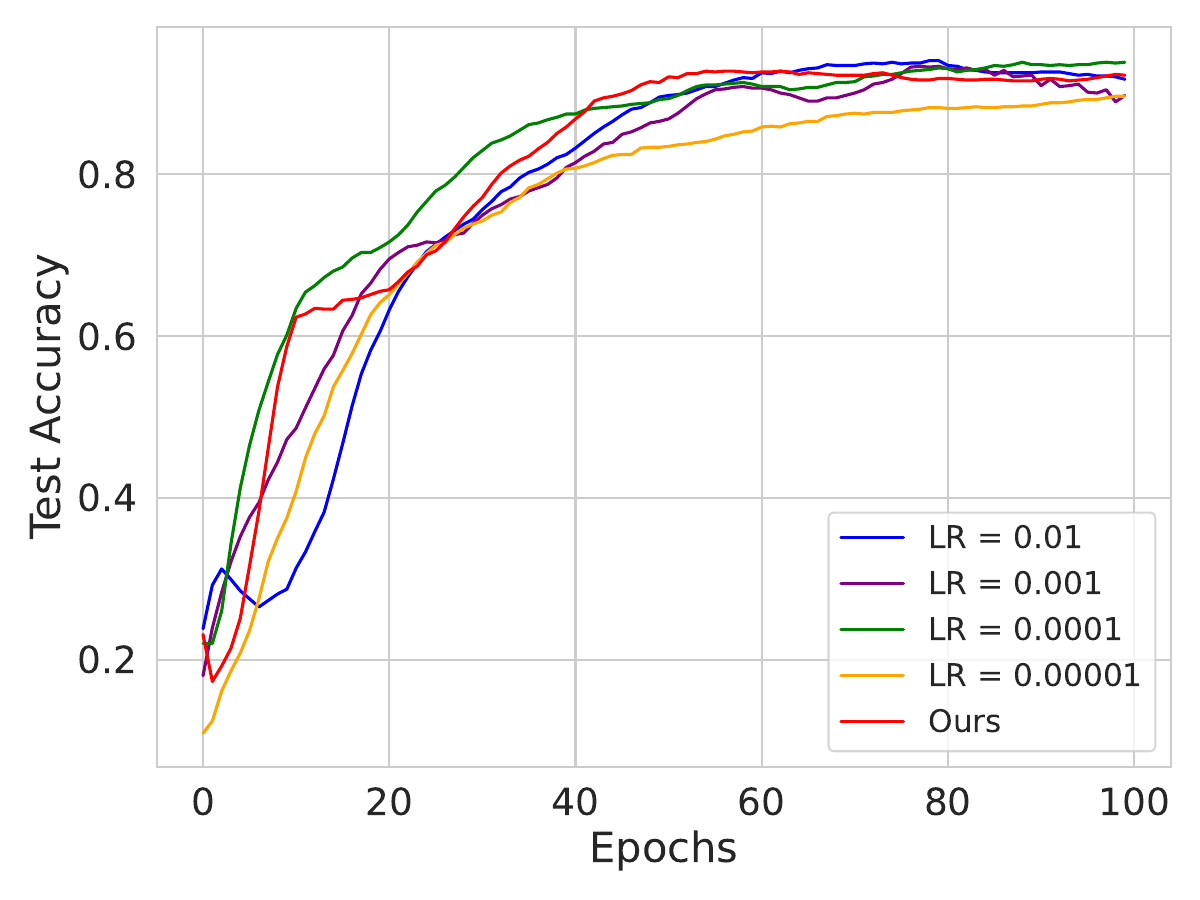}
    \caption{Validation acc. v/s Epochs}
  \end{subfigure}
  \caption{Full-batch experiments on a single layer network with 3000 nodes in each layer, trained on MNIST.}
\label{fig:1__3000_step}
\end{figure}
\begin{figure}[htbp]
  \centering
  \begin{subfigure}[b]{0.31\textwidth}
    \includegraphics[width=\textwidth]{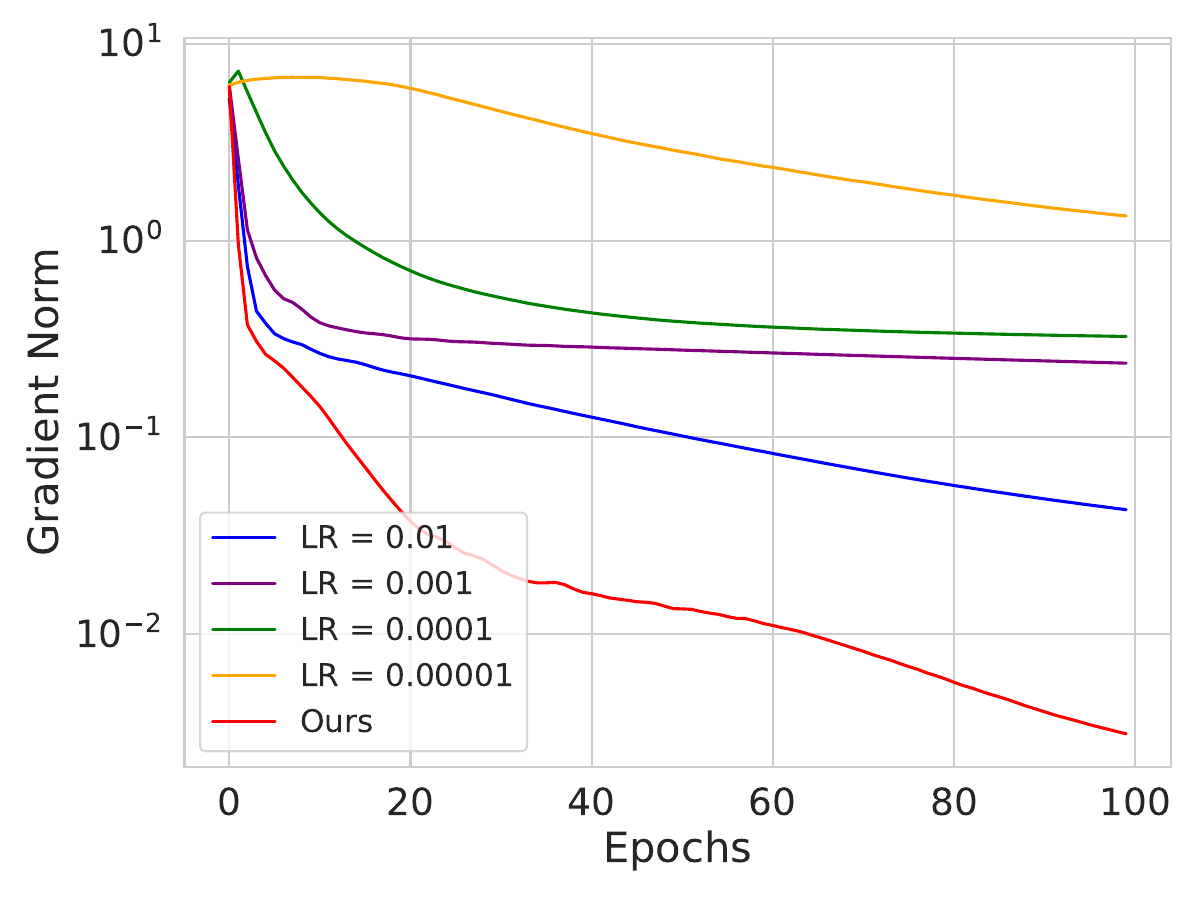}
    \caption{Gradient norm v/s Epochs}
  \end{subfigure}
  \hfill
  \begin{subfigure}[b]{0.31\textwidth}
    \includegraphics[width=\textwidth]{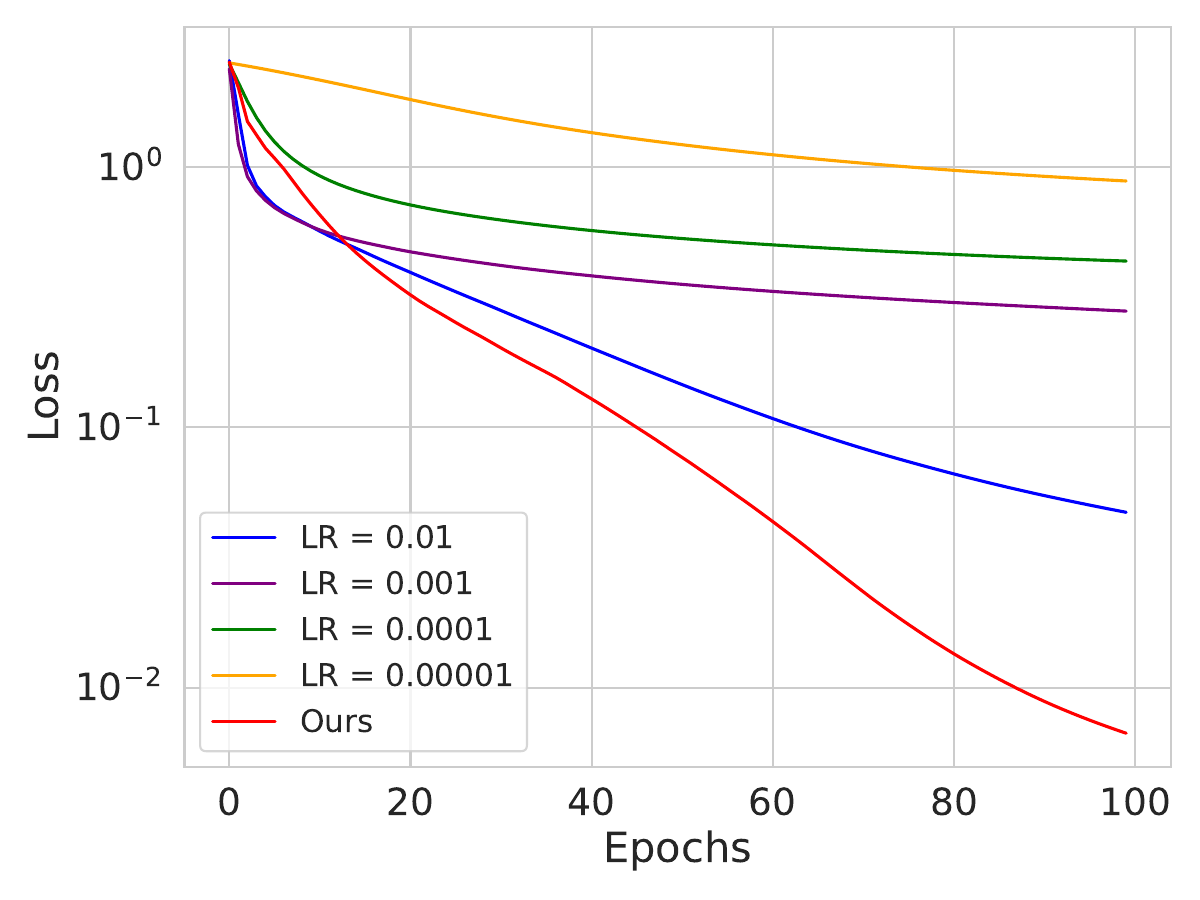}
    \caption{Training loss v/s Epochs}
  \end{subfigure}
  \hfill
  \begin{subfigure}[b]{0.31\textwidth}
    \includegraphics[width=\textwidth]{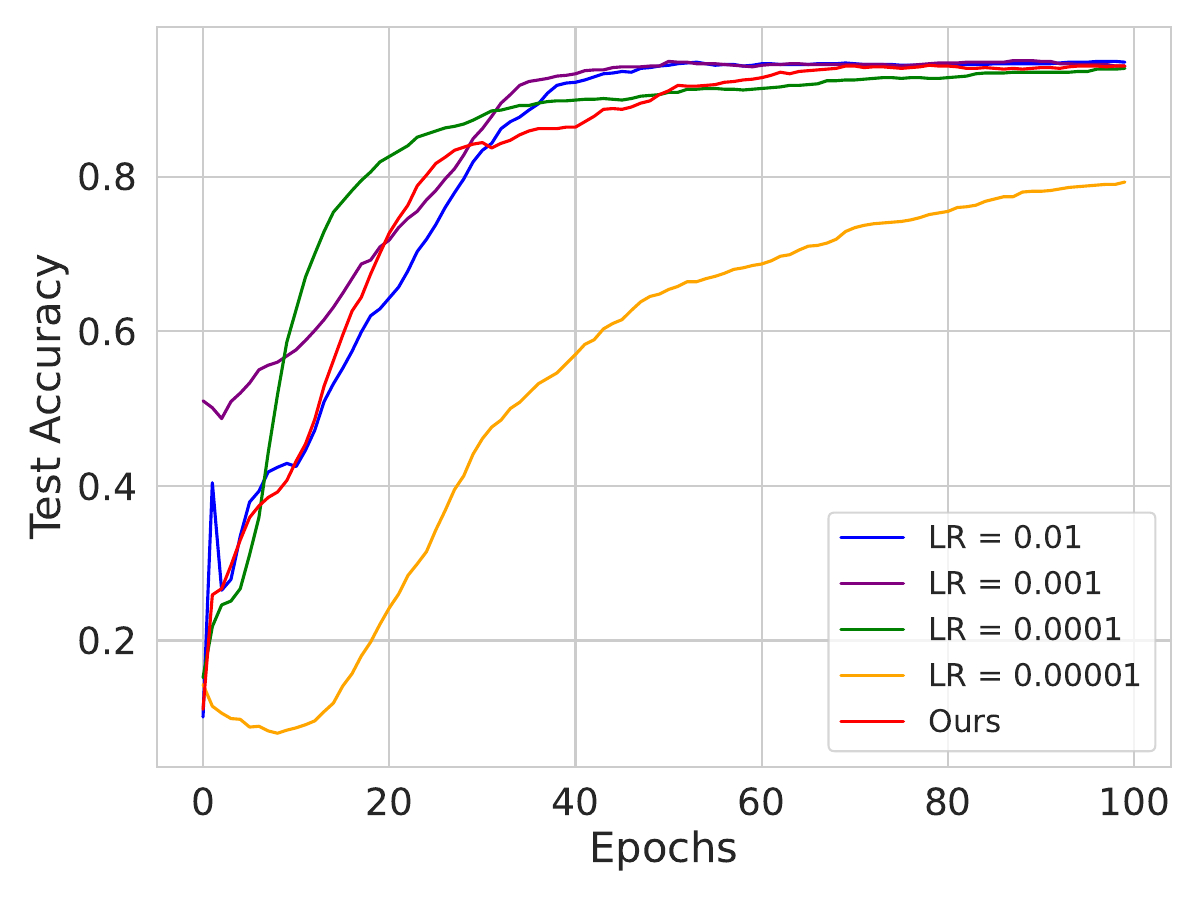}
    \caption{Validation acc. v/s Epochs}
  \end{subfigure}
  \caption{Full-batch experiments on a 3 layer network with 300 nodes in each layer, trained on MNIST.}
\label{fig:3__300_step}
\end{figure}
\begin{figure}[htbp]
  \centering
  \begin{subfigure}[b]{0.31\textwidth}
    \includegraphics[width=\textwidth]{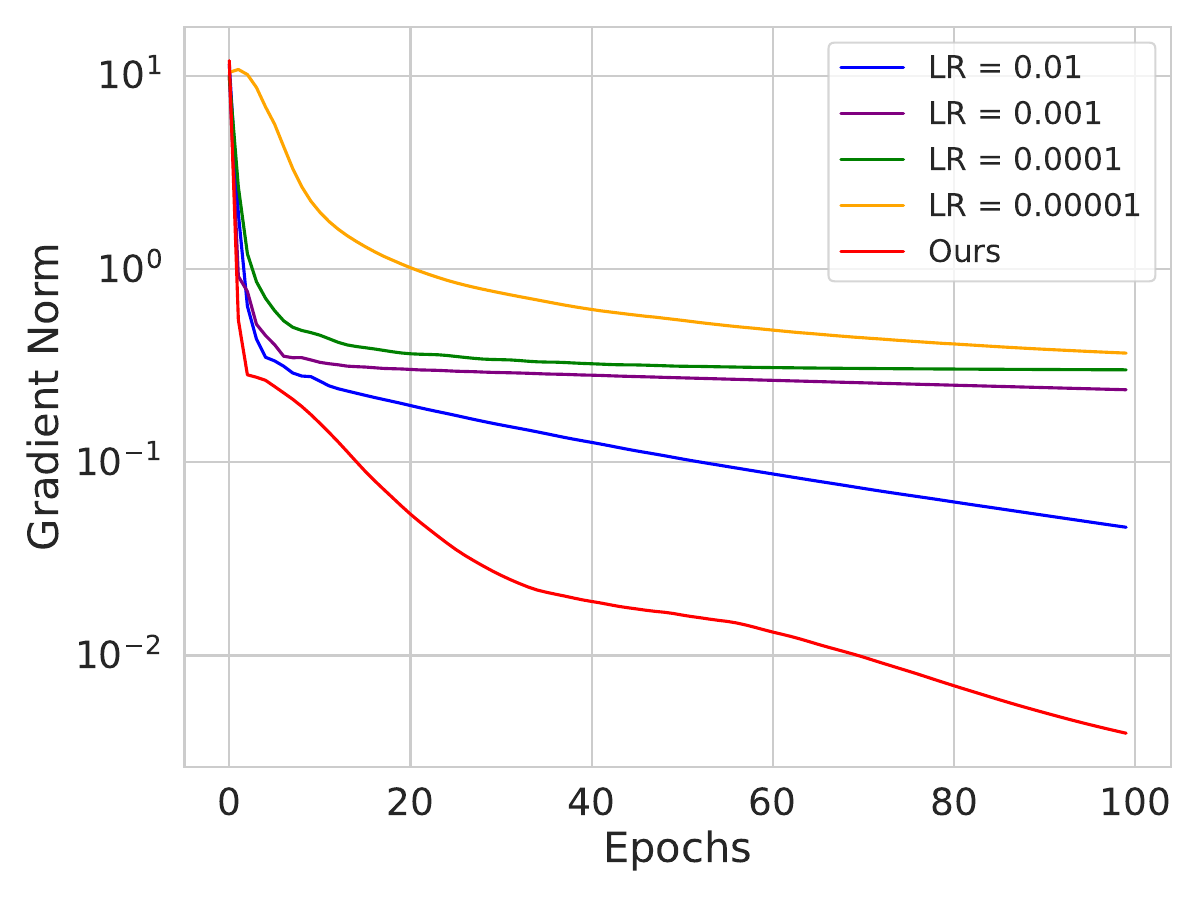}
    \caption{Gradient norm v/s Epochs}
  \end{subfigure}
  \hfill
  \begin{subfigure}[b]{0.31\textwidth}
    \includegraphics[width=\textwidth]{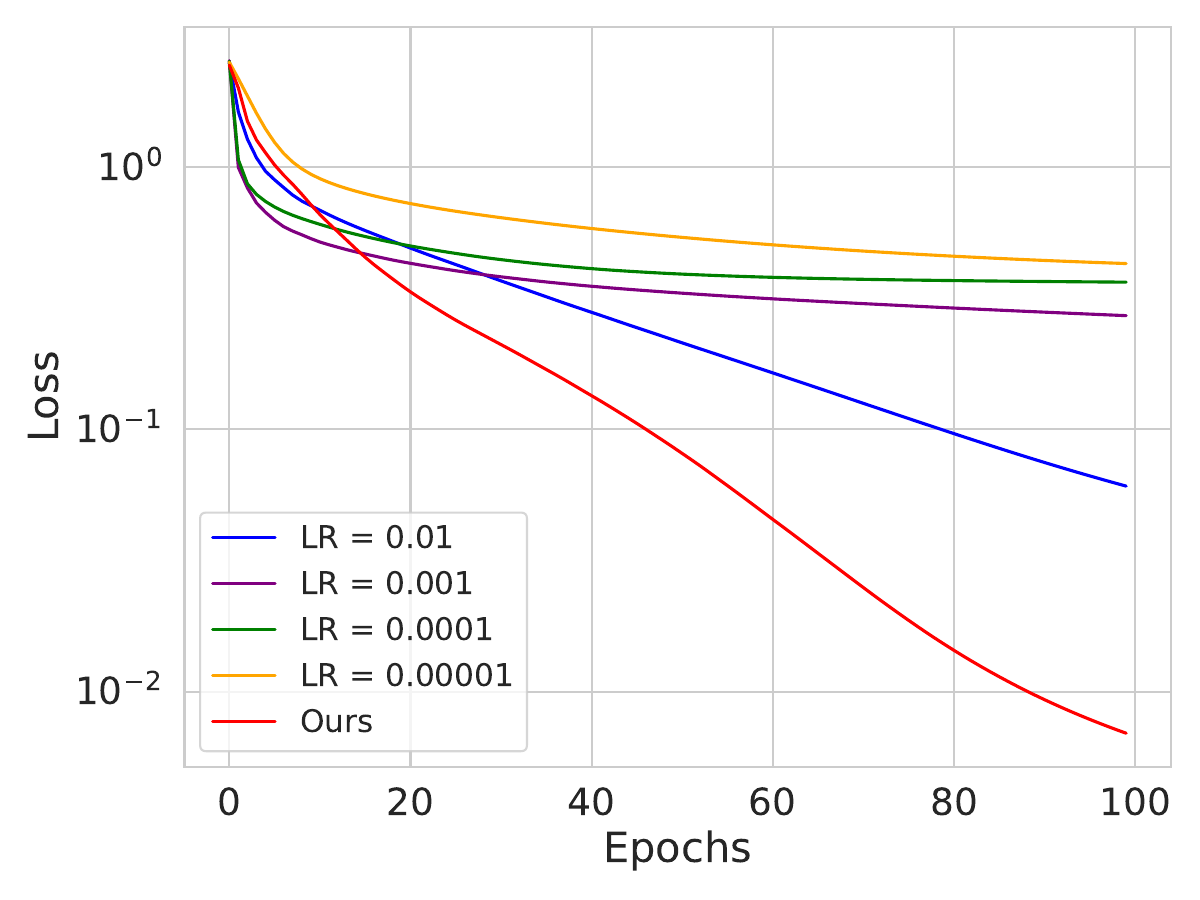}
    \caption{Training loss v/s Epochs}
  \end{subfigure}
  \hfill
  \begin{subfigure}[b]{0.31\textwidth}
    \includegraphics[width=\textwidth]{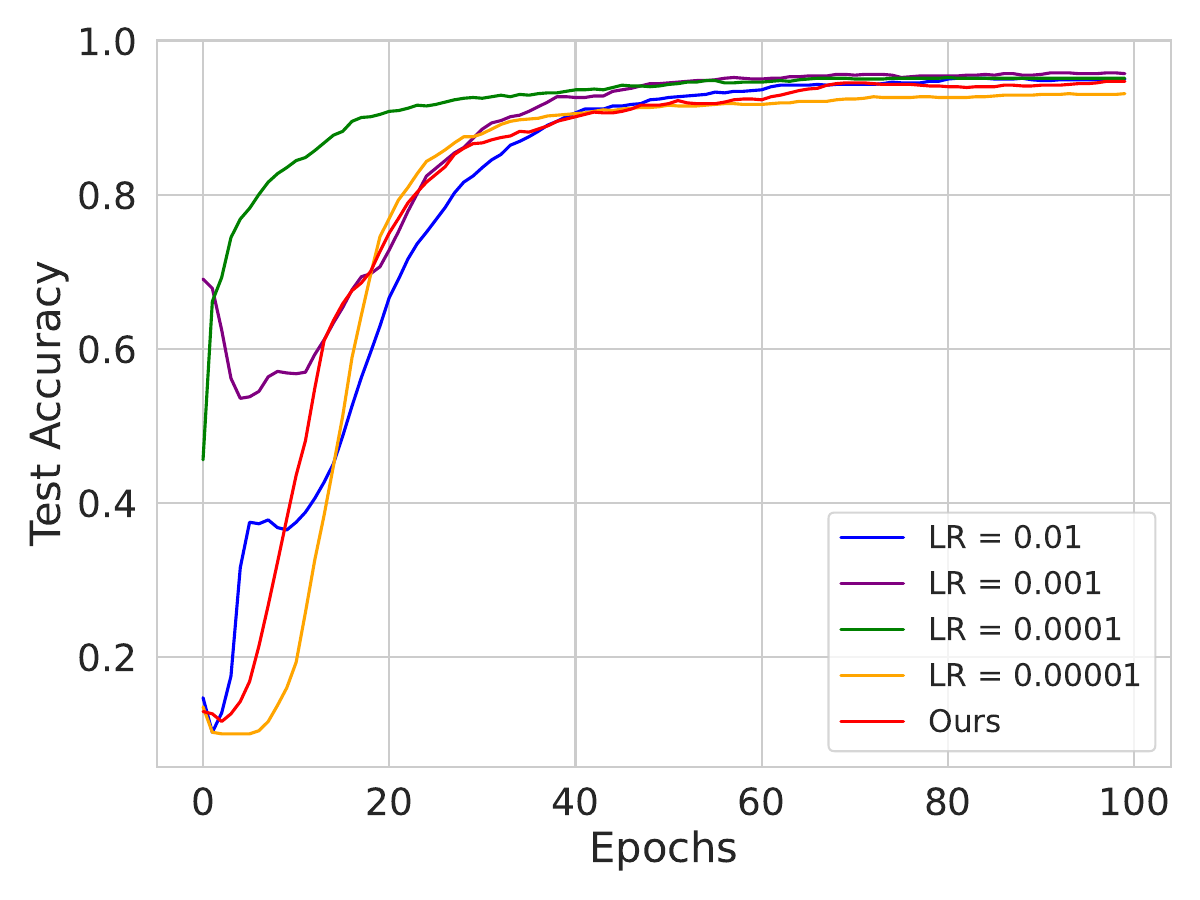}
    \caption{Validation acc. v/s Epochs}
  \end{subfigure}
  \caption{Full-batch experiments on a 3 layer network with 3000 nodes in each layer, trained on MNIST.}
\label{fig:3__3000_step}
\end{figure}
\begin{figure}[htbp]
  \centering
  \begin{subfigure}[b]{0.31\textwidth}
    \includegraphics[width=\textwidth]{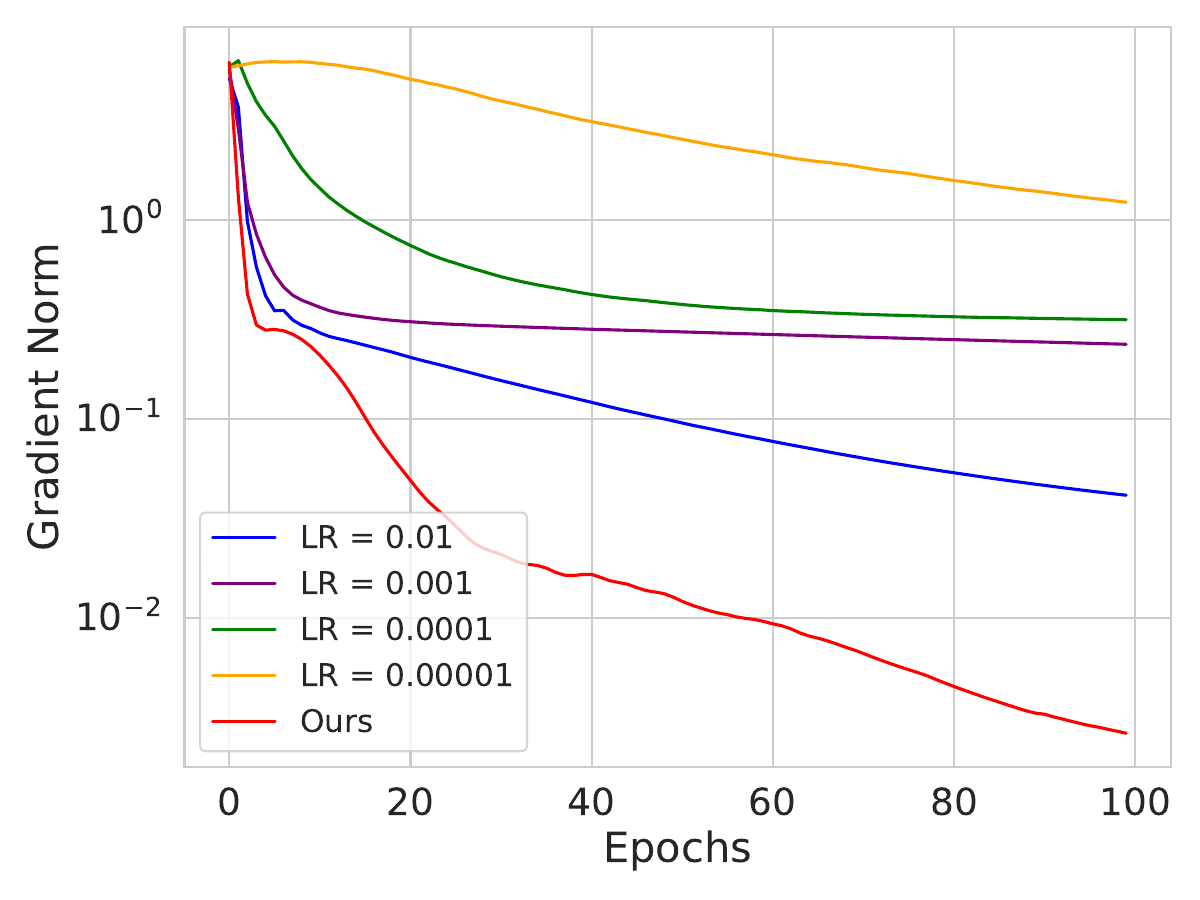}
    \caption{Gradient norm v/s Epochs}
  \end{subfigure}
  \hfill
  \begin{subfigure}[b]{0.31\textwidth}
    \includegraphics[width=\textwidth]{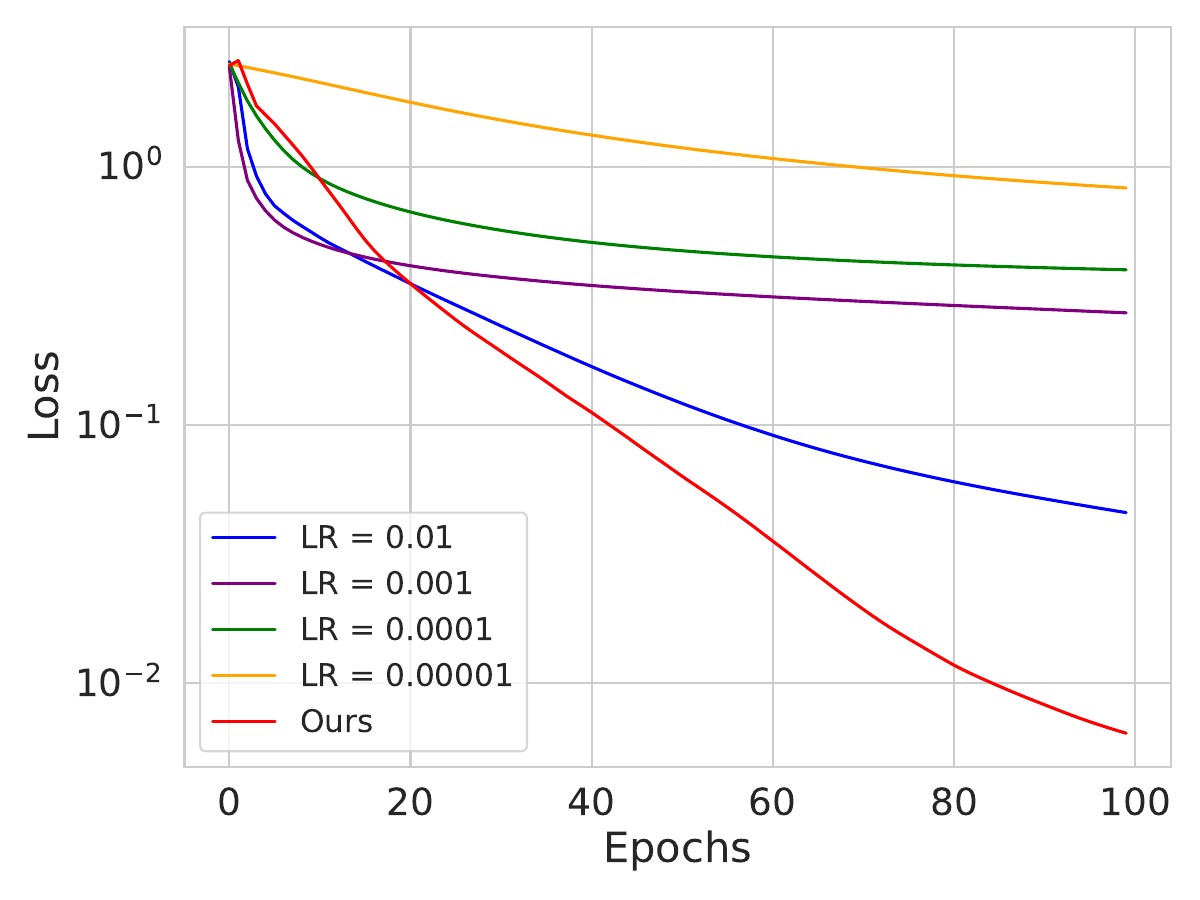}
    \caption{Training loss v/s Epochs}
  \end{subfigure}
  \hfill
  \begin{subfigure}[b]{0.31\textwidth}
    \includegraphics[width=\textwidth]{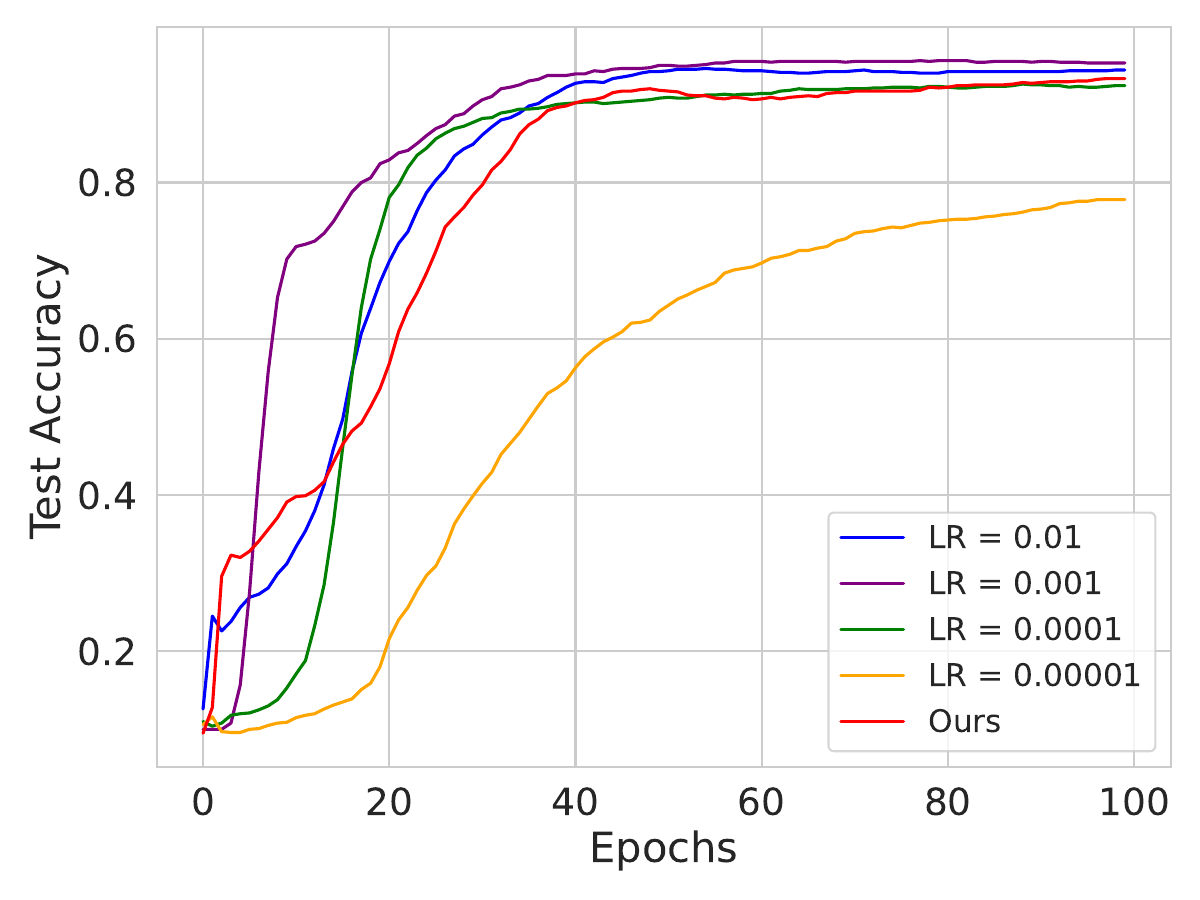}
    \caption{Validation acc. v/s Epochs}
  \end{subfigure}
  \caption{Full-batch experiments on a 5 layer network with 300 nodes in each layer, trained on MNIST.}
\label{fig:5__300_step}
\end{figure}
\begin{figure}[htbp]
  \centering
  \begin{subfigure}[b]{0.31\textwidth}
    \includegraphics[width=\textwidth]{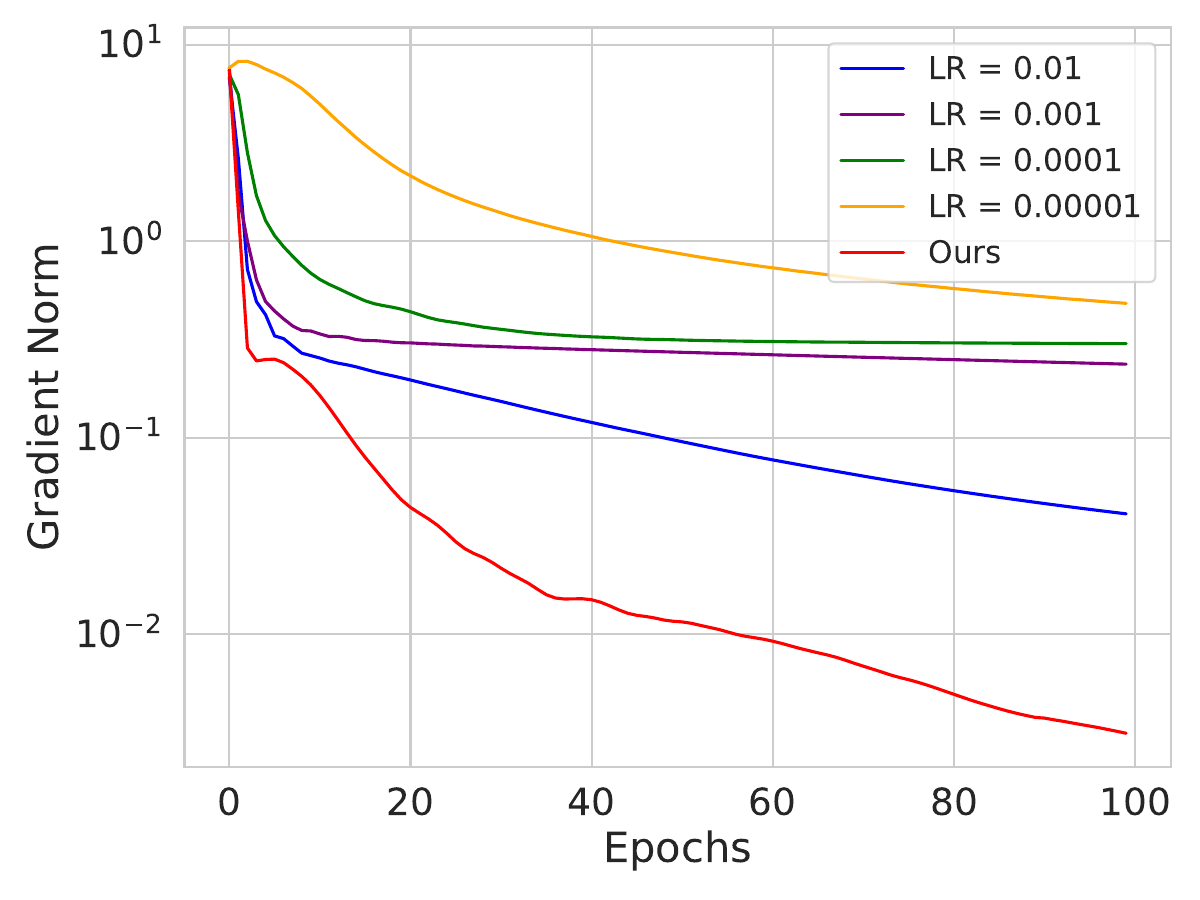}
    \caption{Gradient norm v/s Epochs}
  \end{subfigure}
  \hfill
  \begin{subfigure}[b]{0.31\textwidth}
    \includegraphics[width=\textwidth]{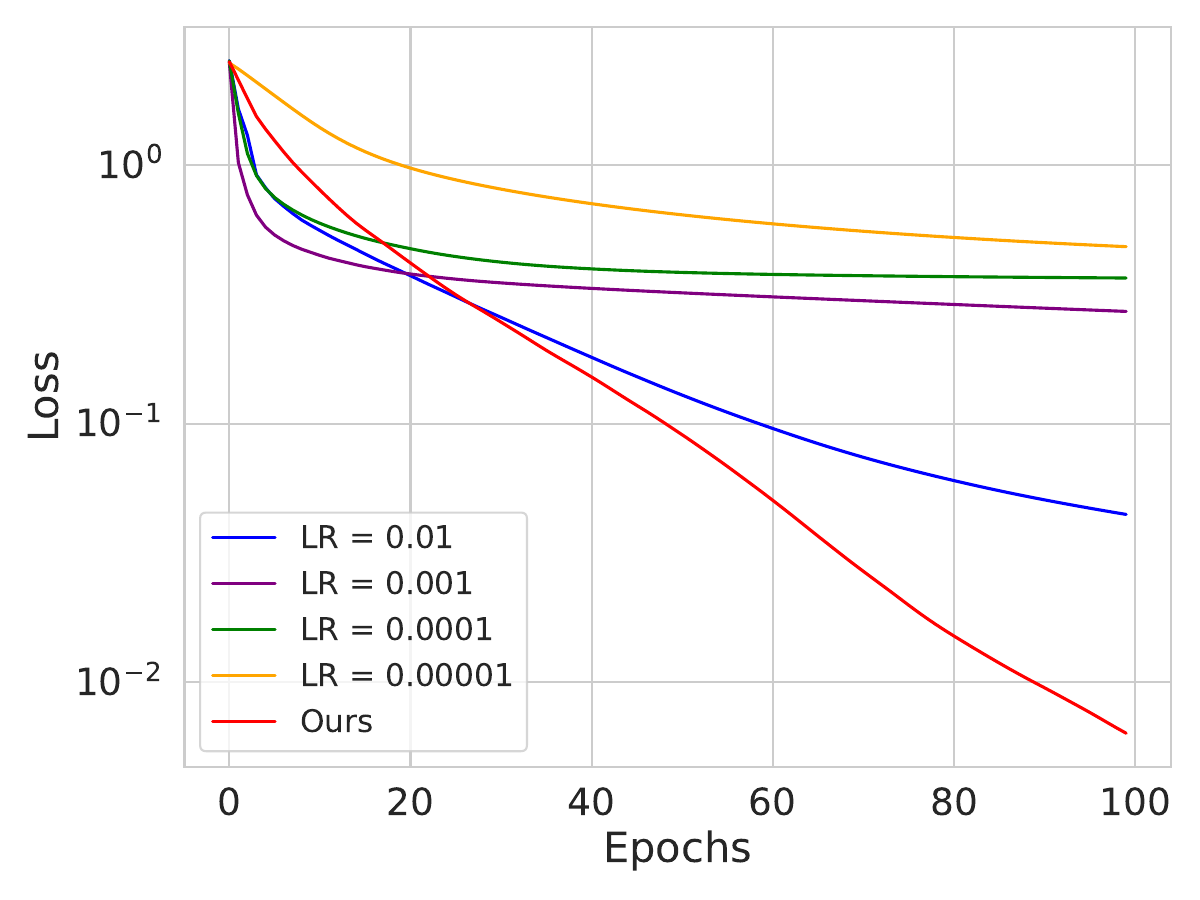}
    \caption{Training loss v/s Epochs}
  \end{subfigure}
  \hfill
  \begin{subfigure}[b]{0.31\textwidth}
    \includegraphics[width=\textwidth]{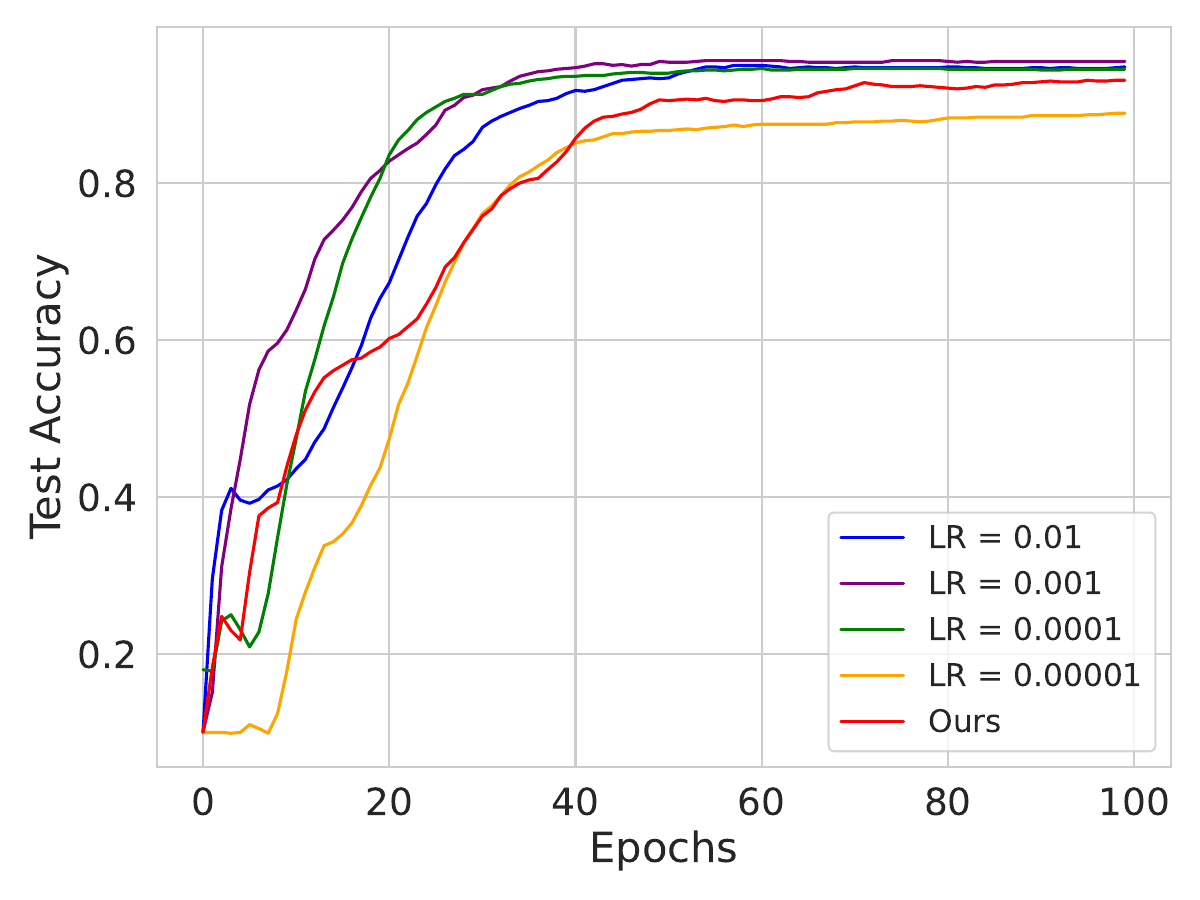}
    \caption{Validation acc. v/s Epochs}
  \end{subfigure}
  \caption{Full-batch experiments on a 5 layer network with 1000 nodes in each layer, trained on MNIST.}
\label{fig:5__1000_step}
\end{figure}
\begin{figure}[htbp]
  \centering
  \begin{subfigure}[b]{0.31\textwidth}
    \includegraphics[width=\textwidth]{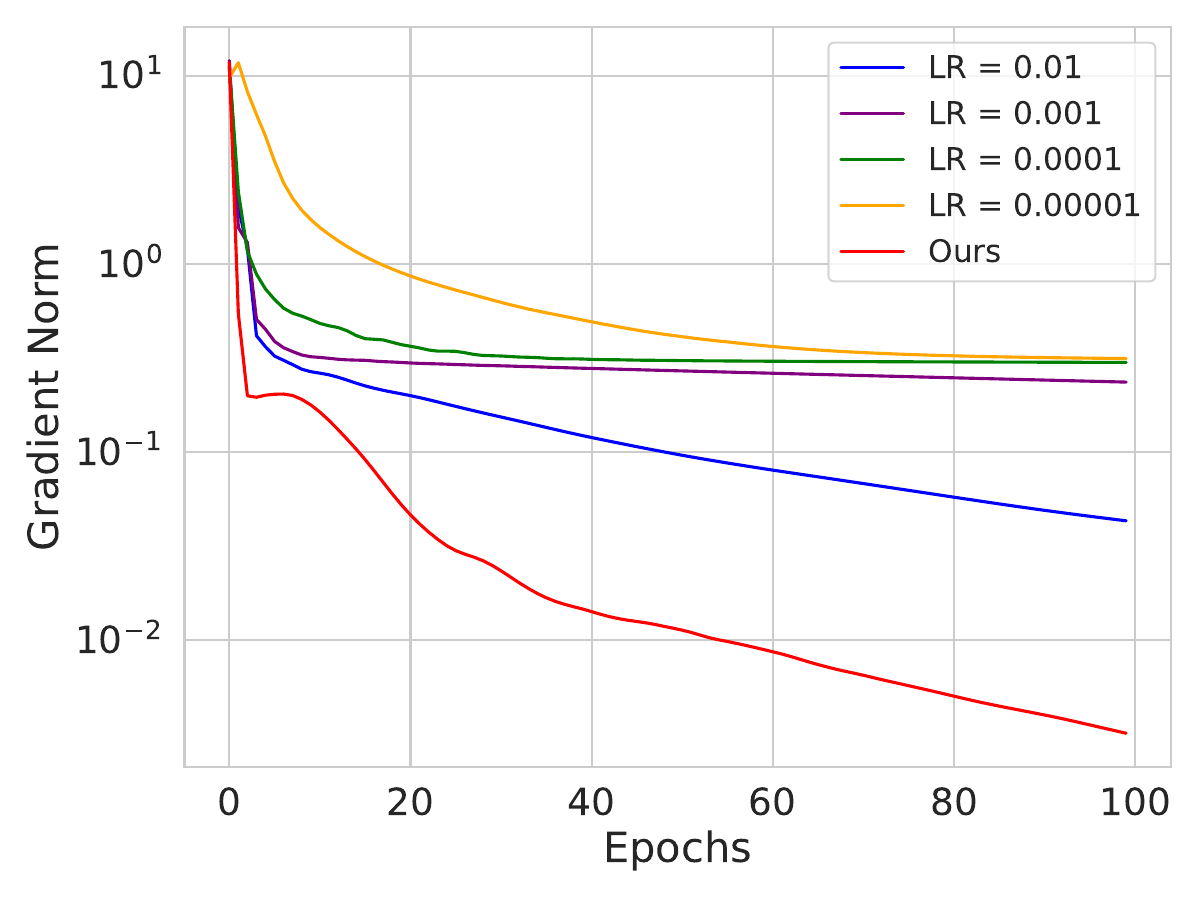}
    \caption{Gradient norm v/s Epochs}
  \end{subfigure}
  \hfill
  \begin{subfigure}[b]{0.31\textwidth}
    \includegraphics[width=\textwidth]{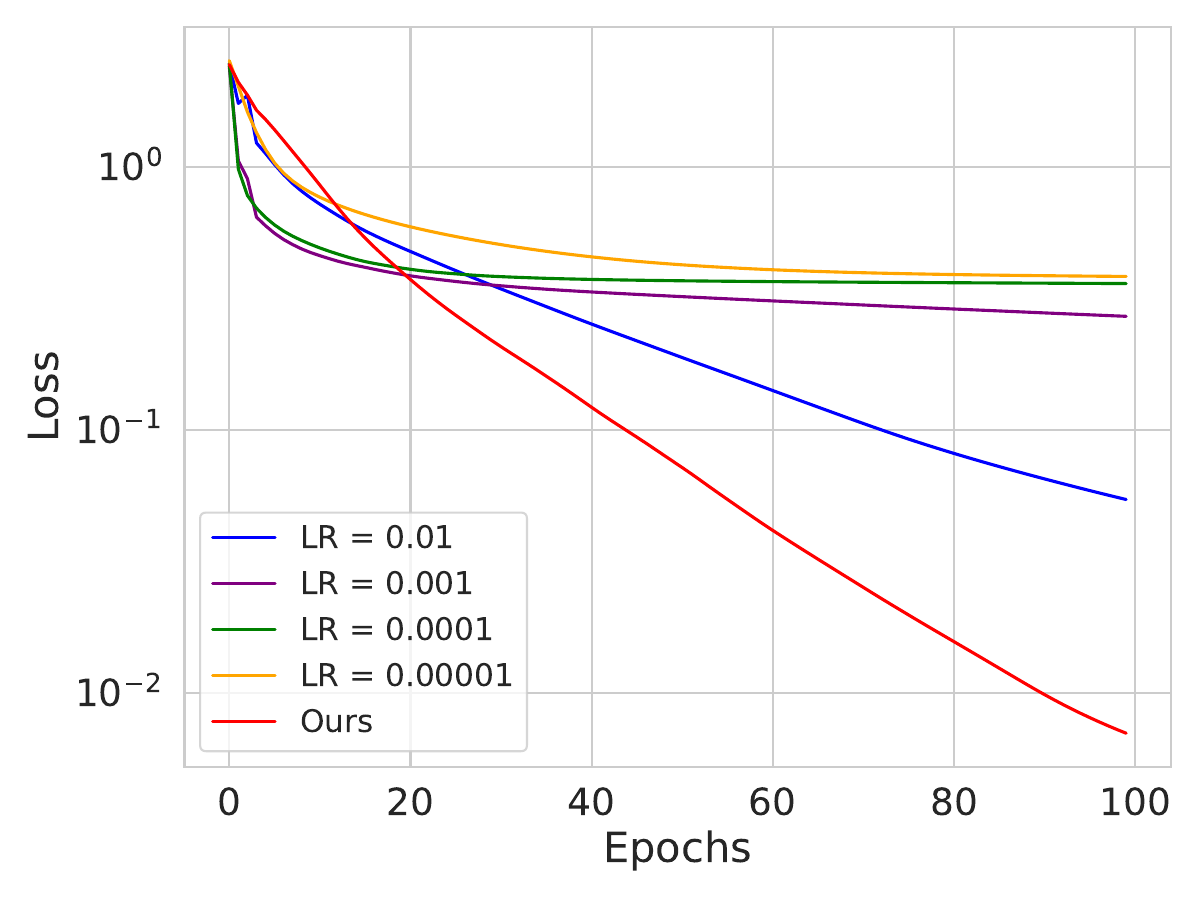}
    \caption{Training loss v/s Epochs}
  \end{subfigure}
  \hfill
  \begin{subfigure}[b]{0.31\textwidth}
    \includegraphics[width=\textwidth]{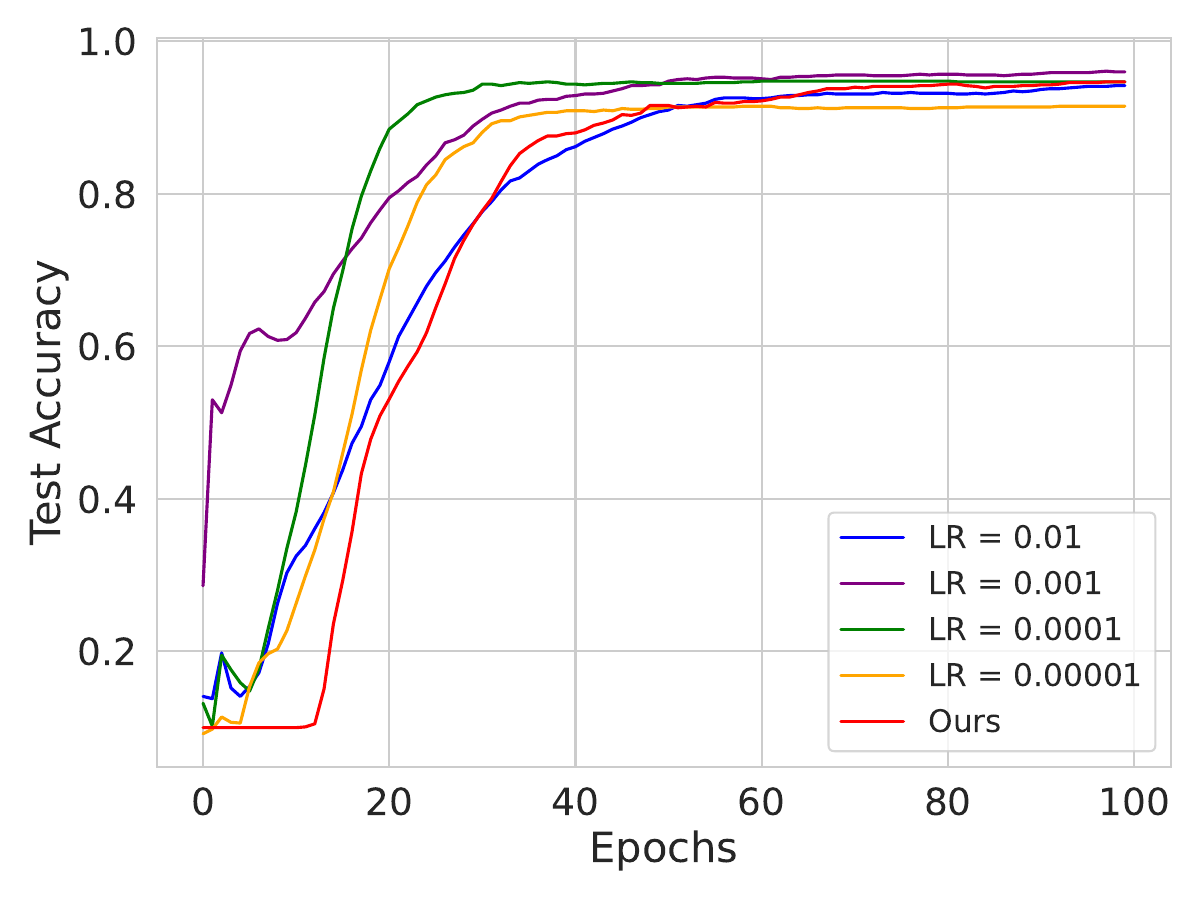}
    \caption{Validation acc. v/s Epochs}
  \end{subfigure}
  \caption{Full-batch experiments on a 5 layer network with 3000 nodes in each layer, trained on MNIST.}
\label{fig:5__3000_step}
\end{figure}
\\
Based on the preceding figures, it's evident that our chosen learning rate effectively reduces the gradient norm as compared to the series of constant learning rates, leading to a commendable validation accuracy. The outcomes remain consistent across various fully connected layer architectures. Hence, rather than performing a line search over several constant learning rates, one can adapt our learning rate to get better results.
\\

\textbf{C.3 Mini-Batch experiments on fully connected networks.}\\

In this section, we conduct a series of experiments using different depths and widths of fully connected layers, in mini-batch setting. We choose 5,000 as the batch size for MNIST. This section serves as a continuation for Section~\ref{sec:mb_exp}.

\begin{figure}[htbp]
  \centering
  \begin{subfigure}[b]{0.31\textwidth}
    \includegraphics[width=\textwidth]{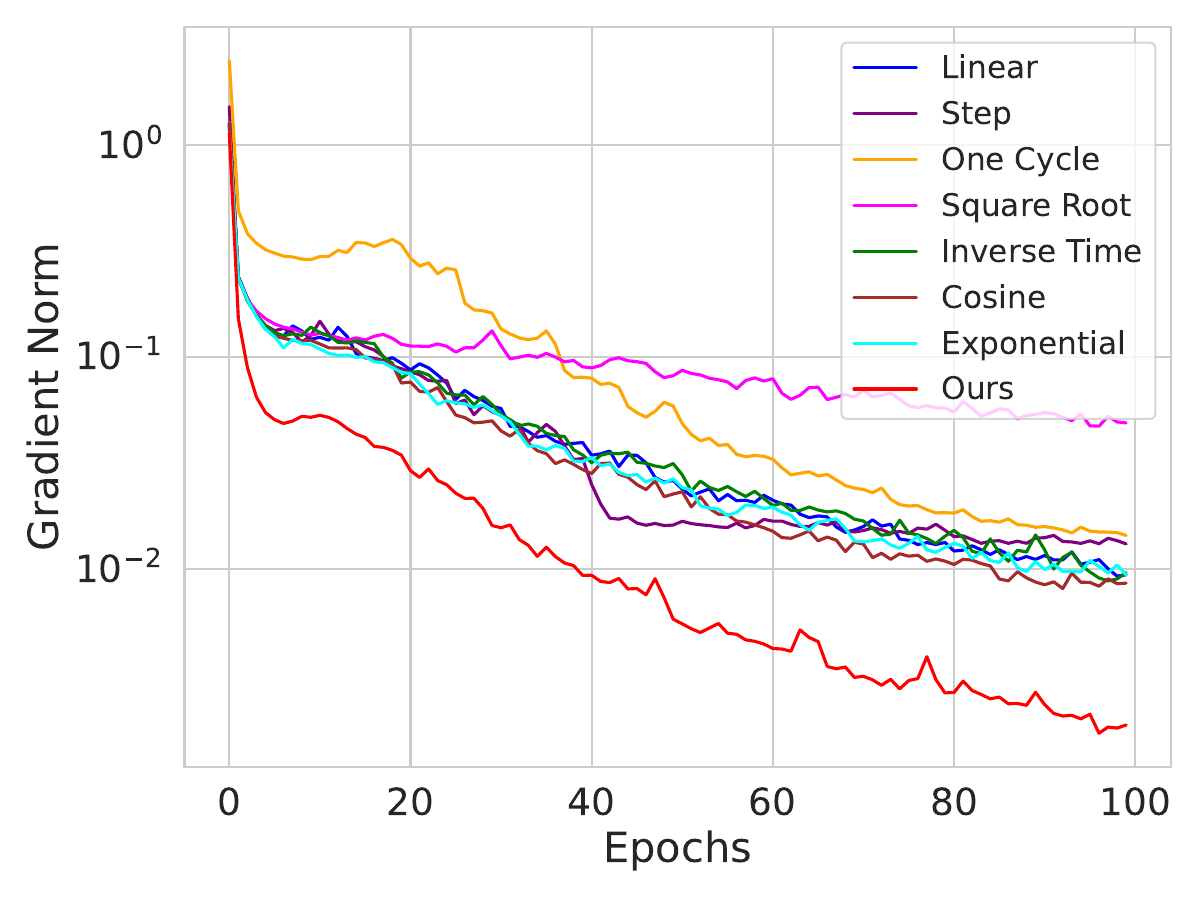}
    \caption{Gradient norm v/s Epochs}
  \end{subfigure}
  \hfill
  \begin{subfigure}[b]{0.31\textwidth}
    \includegraphics[width=\textwidth]{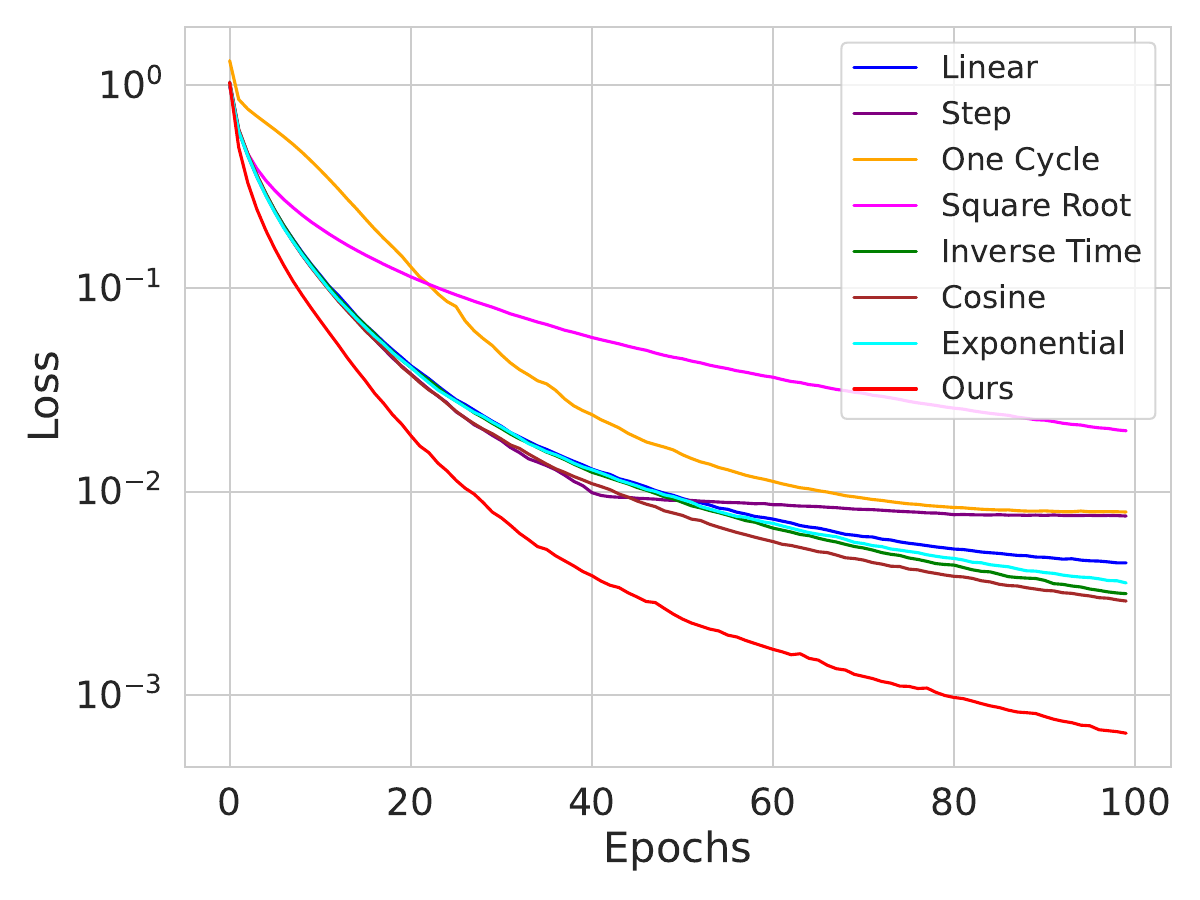}
    \caption{Training loss v/s Epochs}
  \end{subfigure}
  \hfill
  \begin{subfigure}[b]{0.31\textwidth}
    \includegraphics[width=\textwidth]{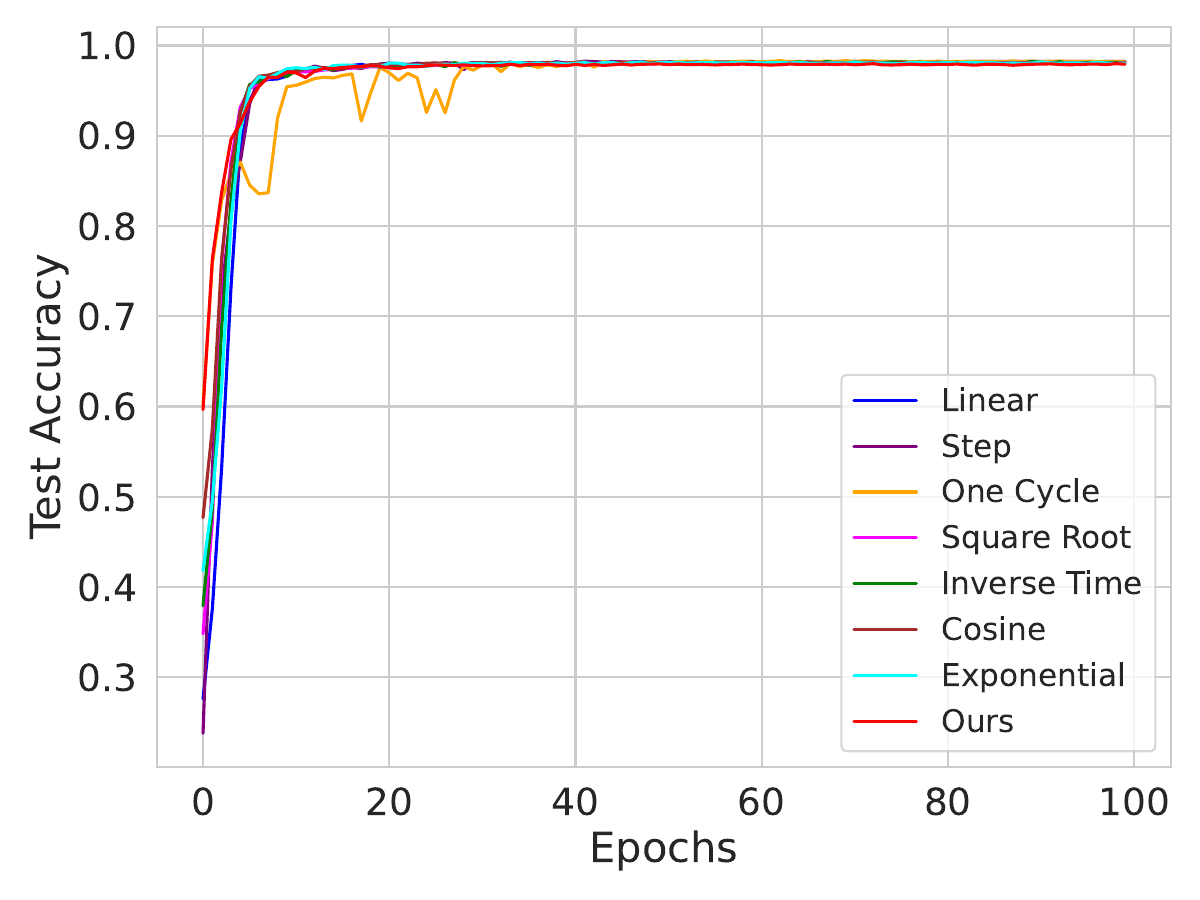}
    \caption{Validation acc. v/s Epochs}
  \end{subfigure}
  \caption{\textbf{Comparision with schedulers}: Mini-batch experiments on a single layer network with 300 nodes in each layer, trained on MNIST.}
\label{fig:1__300_mini}
\end{figure}
\begin{figure}[htbp]
  \centering
  \begin{subfigure}[b]{0.31\textwidth}
    \includegraphics[width=\textwidth]{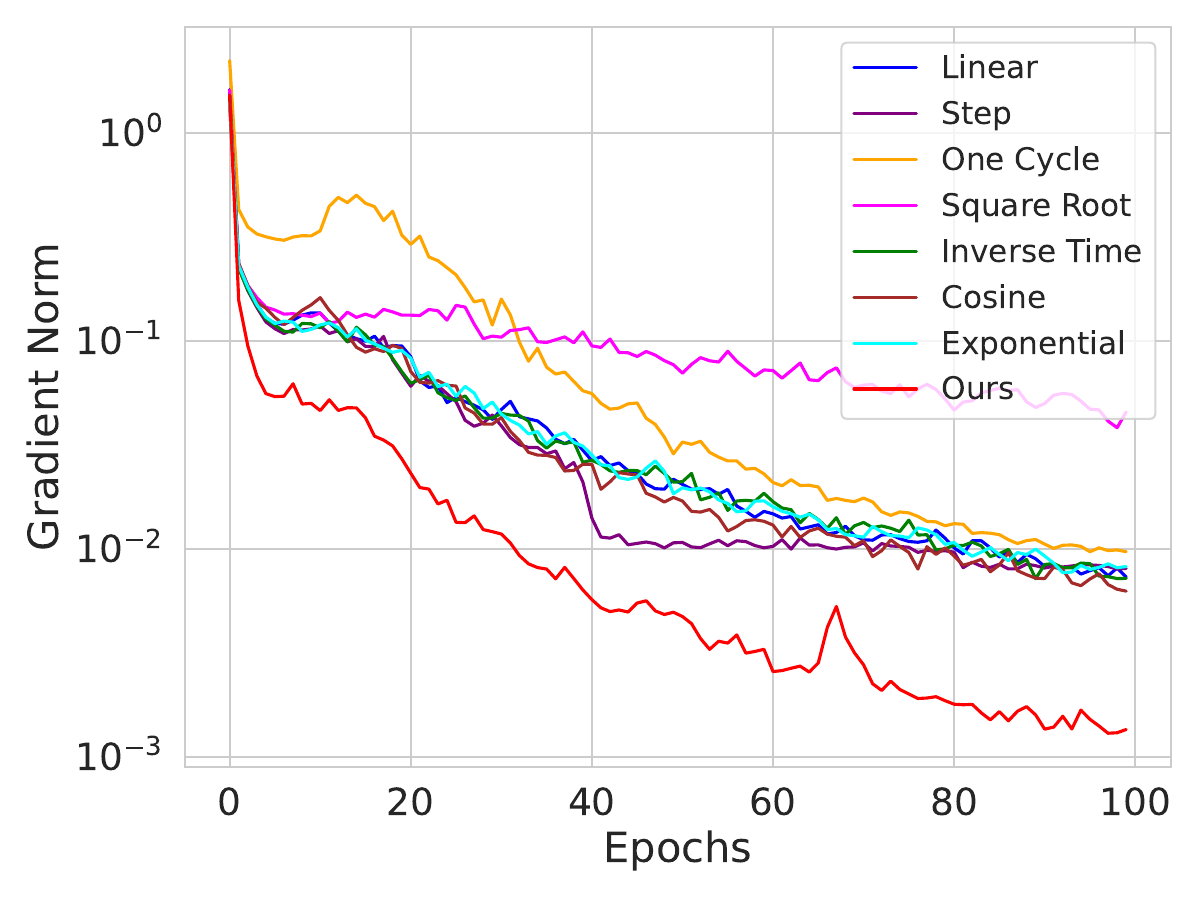}
    \caption{Gradient norm v/s Epochs}
  \end{subfigure}
  \hfill
  \begin{subfigure}[b]{0.31\textwidth}
    \includegraphics[width=\textwidth]{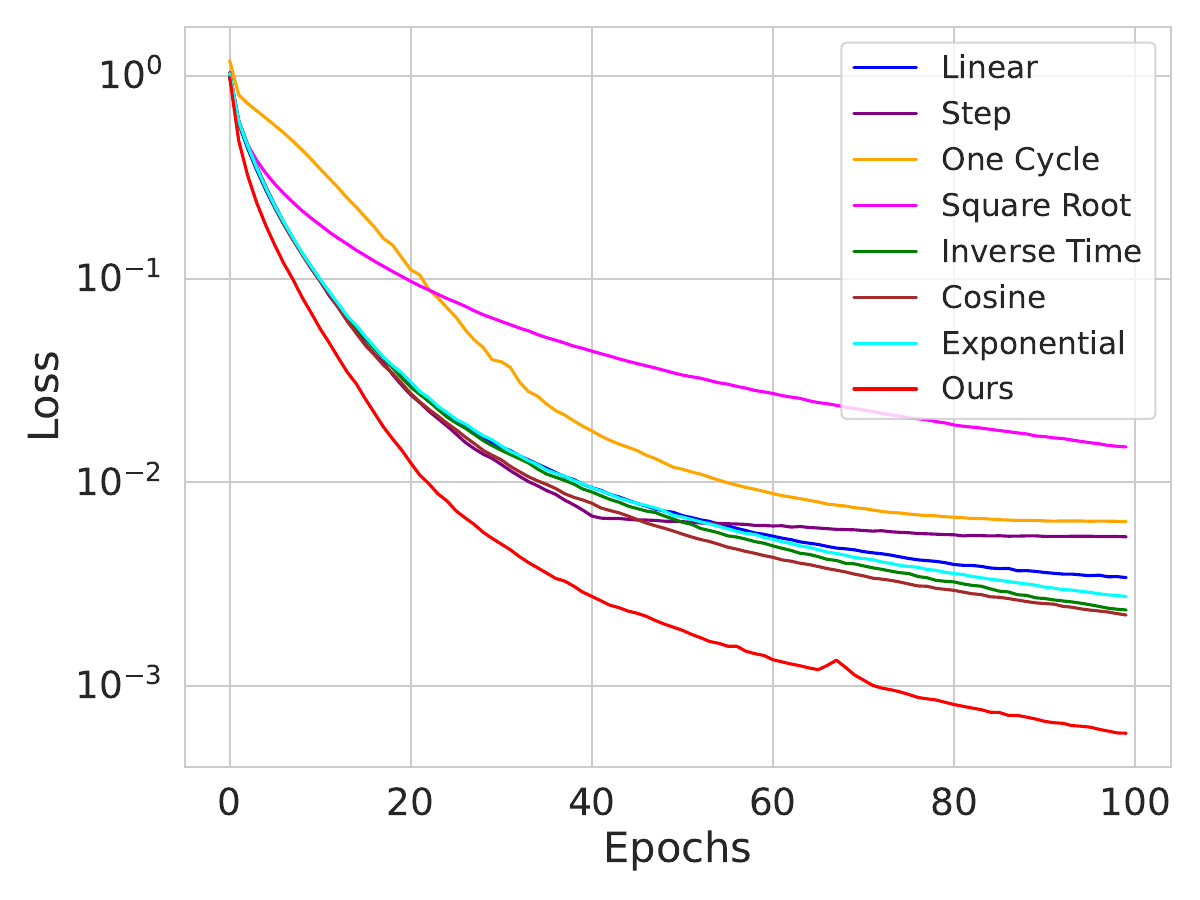}
    \caption{Training loss v/s Epochs}
  \end{subfigure}
  \hfill
  \begin{subfigure}[b]{0.31\textwidth}
    \includegraphics[width=\textwidth]{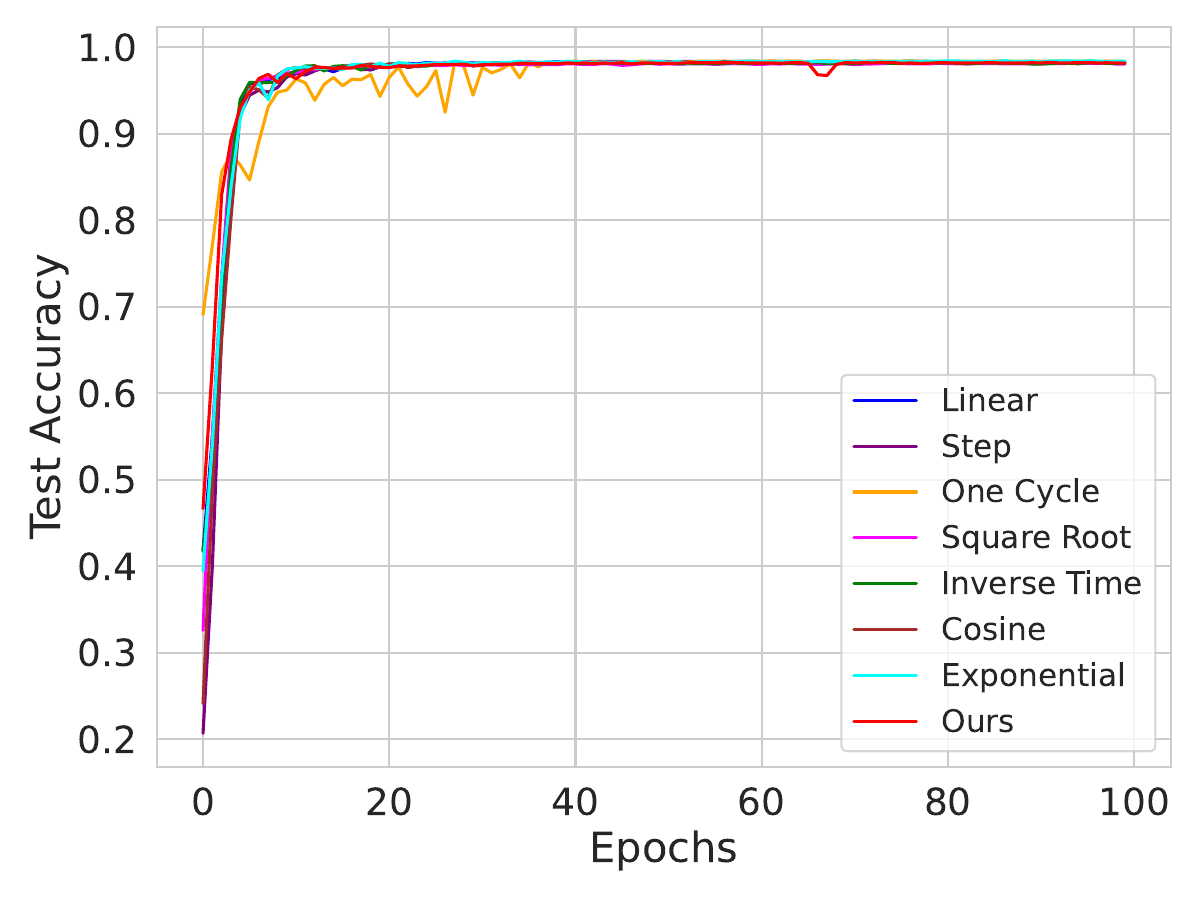}
    \caption{Validation acc. v/s Epochs}
  \end{subfigure}
  \caption{\textbf{Comparision with schedulers}: Mini-batch experiments on a single layer network with 1000 nodes in each layer, trained on MNIST.}
\label{fig:1__1000_mini}
\end{figure}
\begin{figure}[htbp]
  \centering
  \begin{subfigure}[b]{0.31\textwidth}
    \includegraphics[width=\textwidth]{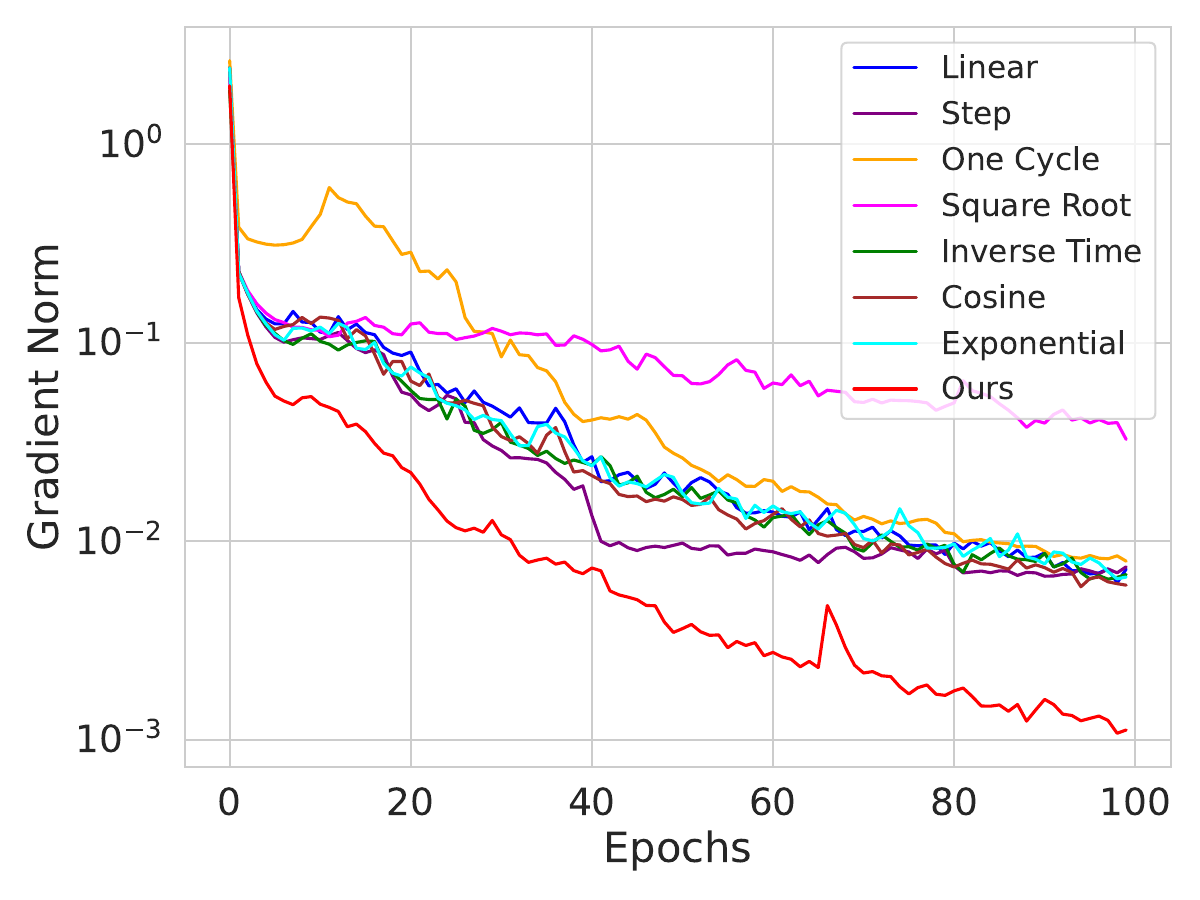}
    \caption{Gradient norm v/s Epochs}
  \end{subfigure}
  \hfill
  \begin{subfigure}[b]{0.31\textwidth}
    \includegraphics[width=\textwidth]{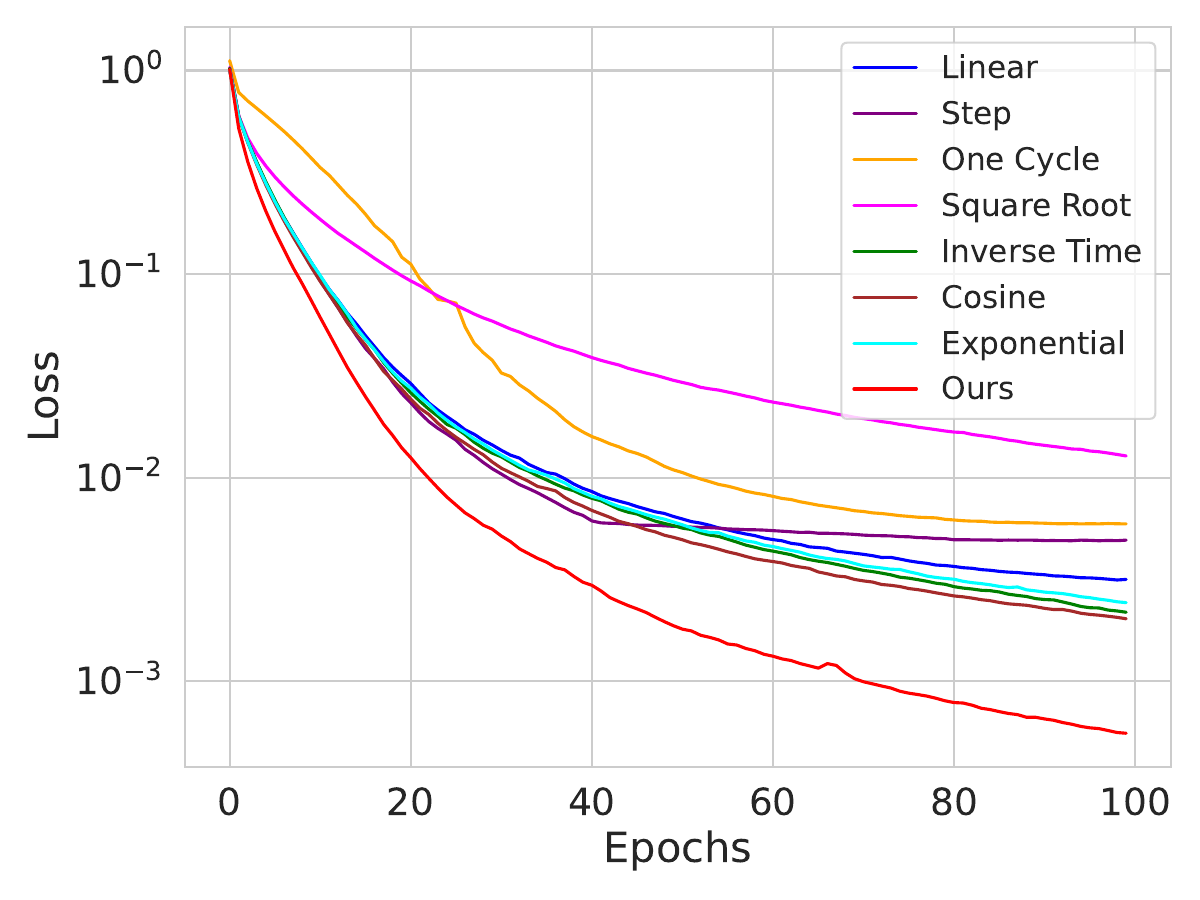}
    \caption{Training loss v/s Epochs}
  \end{subfigure}
  \hfill
  \begin{subfigure}[b]{0.31\textwidth}
    \includegraphics[width=\textwidth]{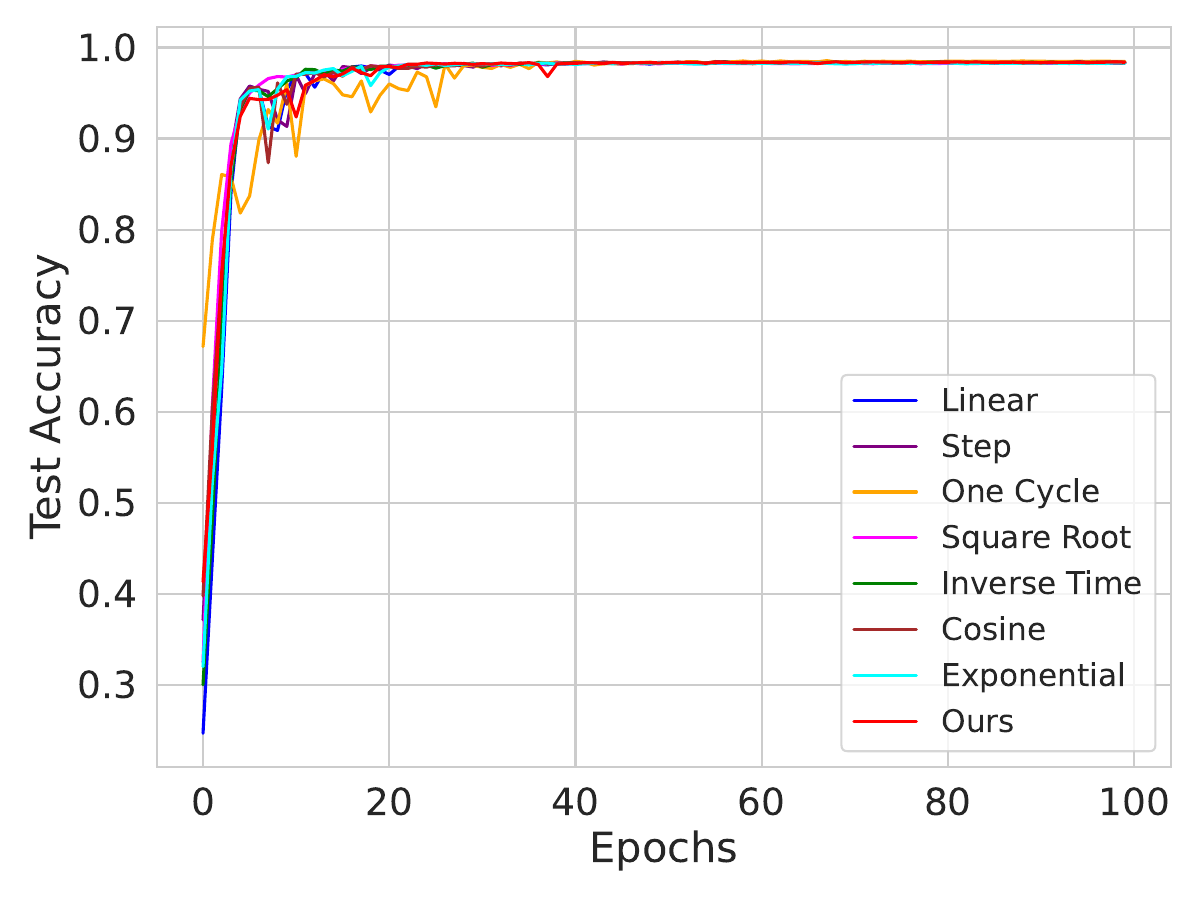}
    \caption{Validation acc. v/s Epochs}
  \end{subfigure}
  \caption{\textbf{Comparision with schedulers}: Mini-batch experiments on a single layer network with 3000 nodes in each layer, trained on MNIST.}
\label{fig:1__3000_mini}
\end{figure}
\begin{figure}[htbp]
  \centering
  \begin{subfigure}[b]{0.31\textwidth}
    \includegraphics[width=\textwidth]{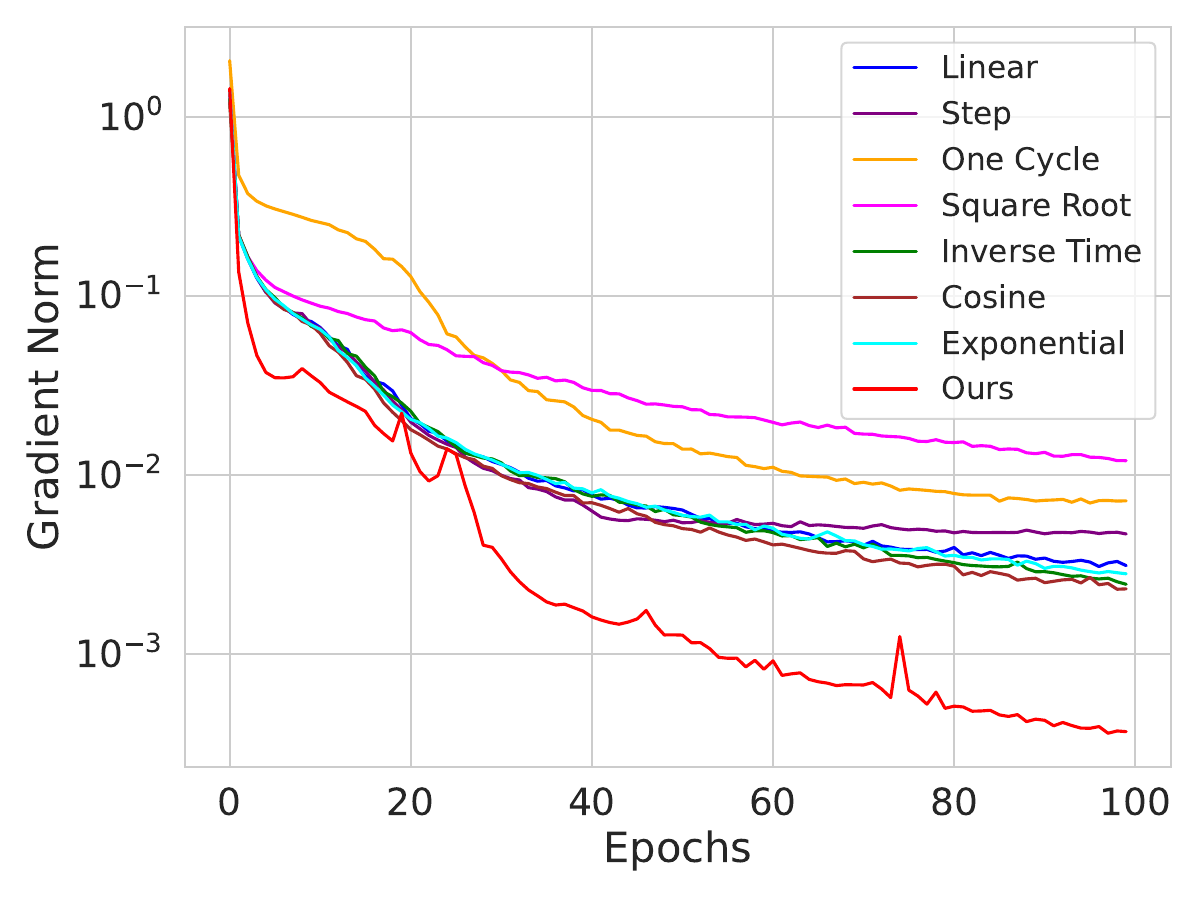}
    \caption{Gradient norm v/s Epochs}
  \end{subfigure}
  \hfill
  \begin{subfigure}[b]{0.31\textwidth}
    \includegraphics[width=\textwidth]{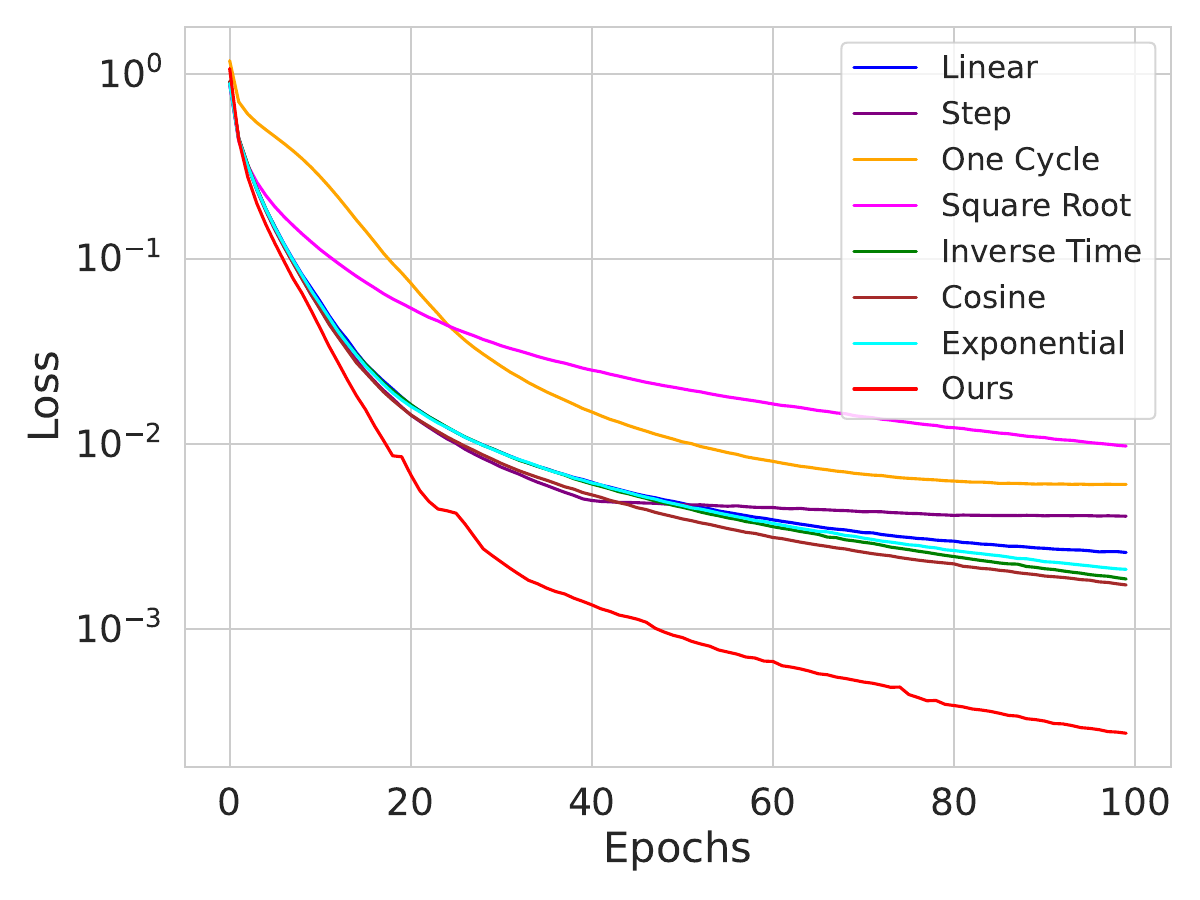}
    \caption{Training loss v/s Epochs}
  \end{subfigure}
  \hfill
  \begin{subfigure}[b]{0.31\textwidth}
    \includegraphics[width=\textwidth]{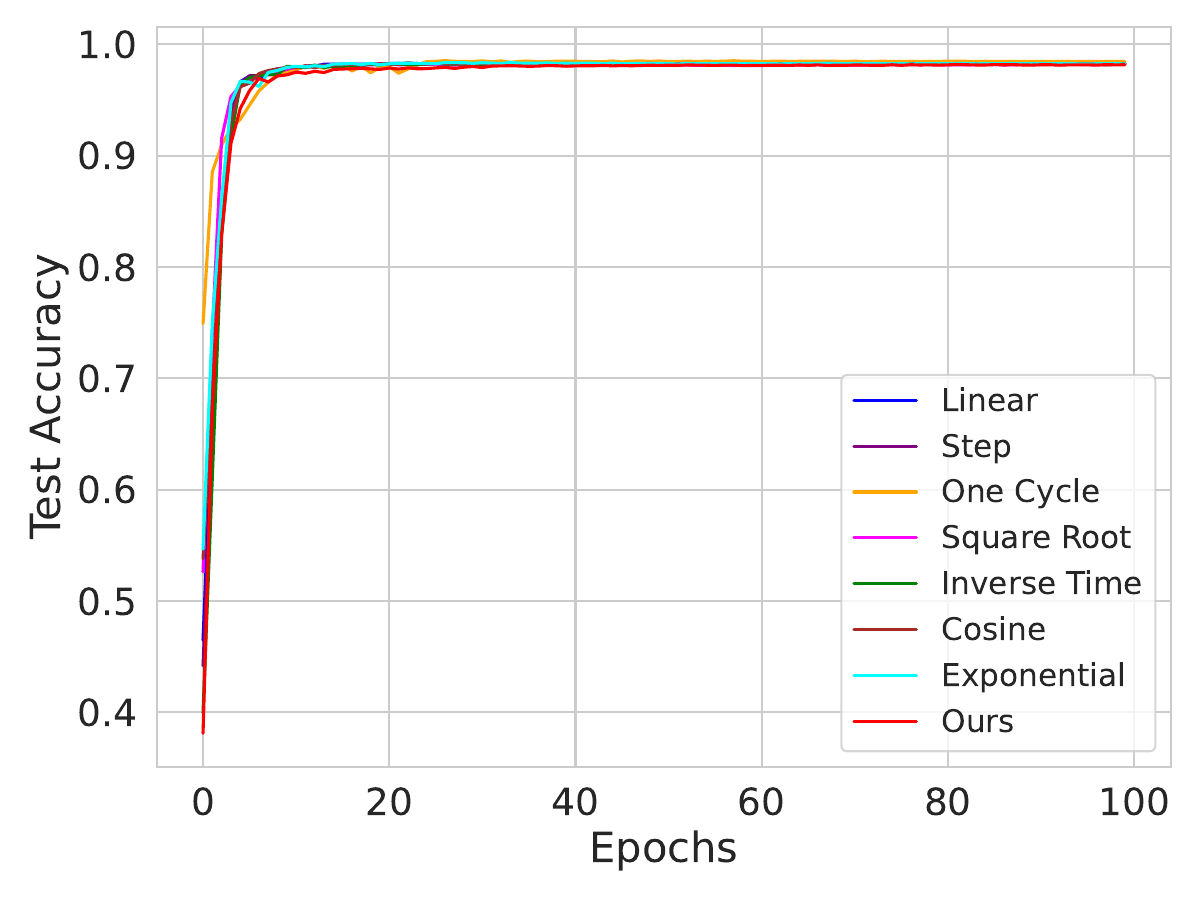}
    \caption{Validation acc. v/s Epochs}
  \end{subfigure}
  \caption{\textbf{Comparision with schedulers}: Mini-batch experiments on a 3 layer network with 300 nodes in each layer, trained on MNIST.}
\label{fig:3__300_mini}
\end{figure}
\begin{figure}[htbp]
  \centering
  \begin{subfigure}[b]{0.31\textwidth}
    \includegraphics[width=\textwidth]{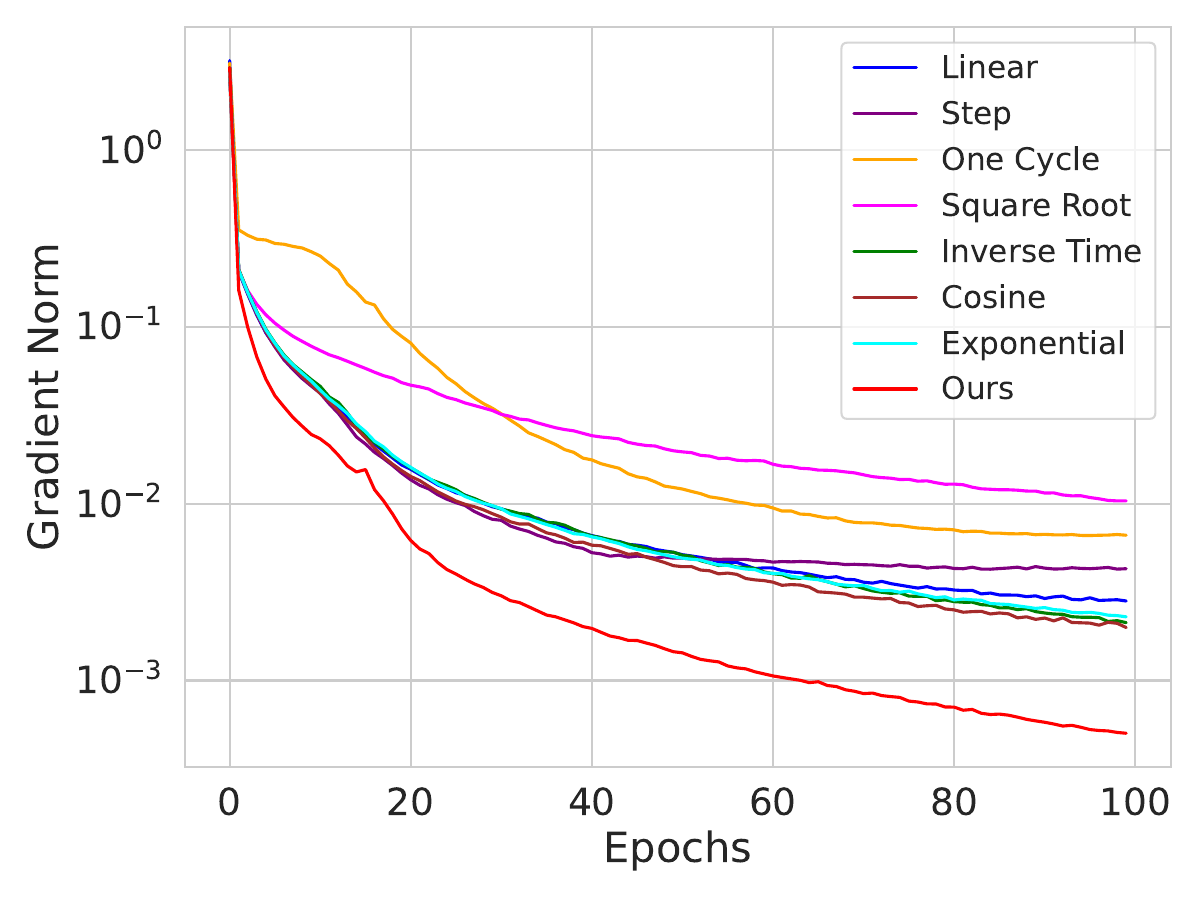}
    \caption{Gradient norm v/s Epochs}
  \end{subfigure}
  \hfill
  \begin{subfigure}[b]{0.31\textwidth}
    \includegraphics[width=\textwidth]{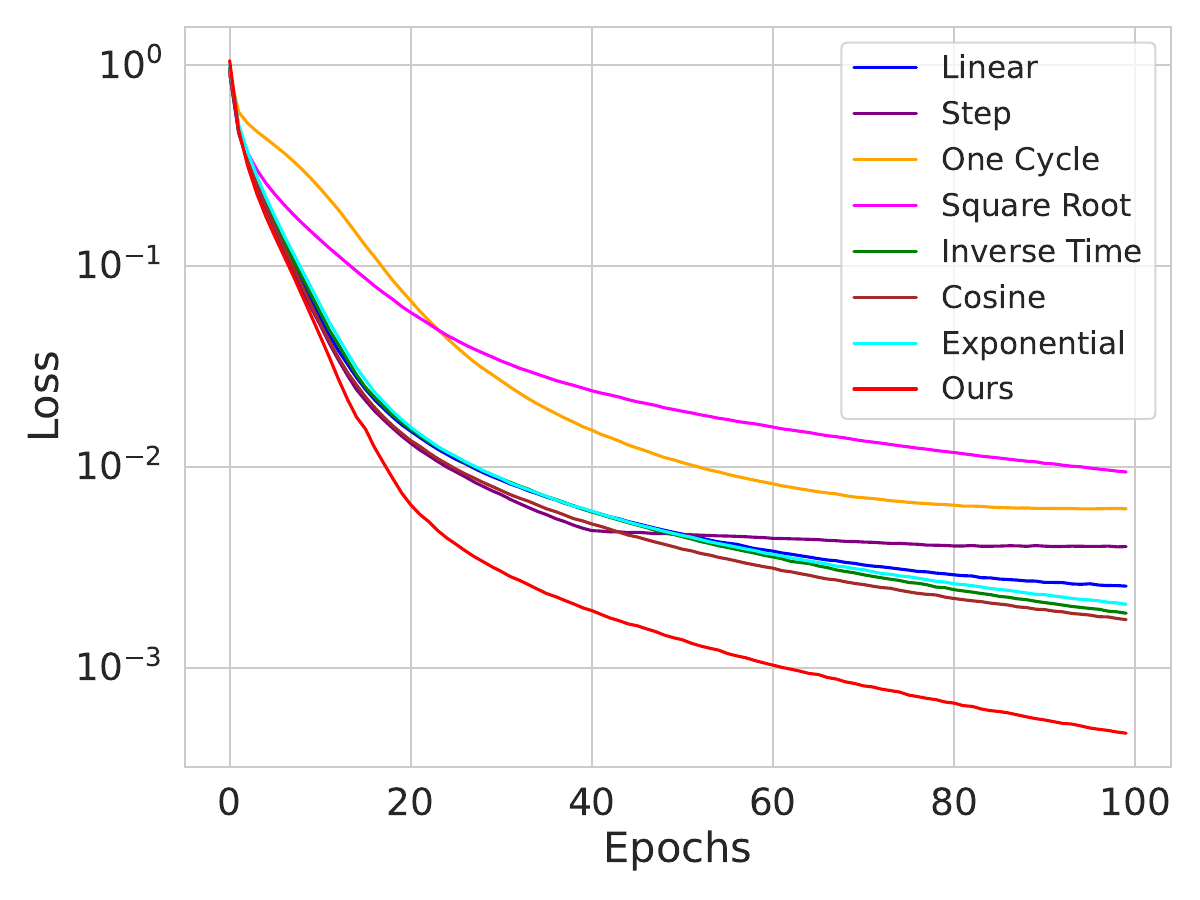}
    \caption{Training loss v/s Epochs}
  \end{subfigure}
  \hfill
  \begin{subfigure}[b]{0.31\textwidth}
    \includegraphics[width=\textwidth]{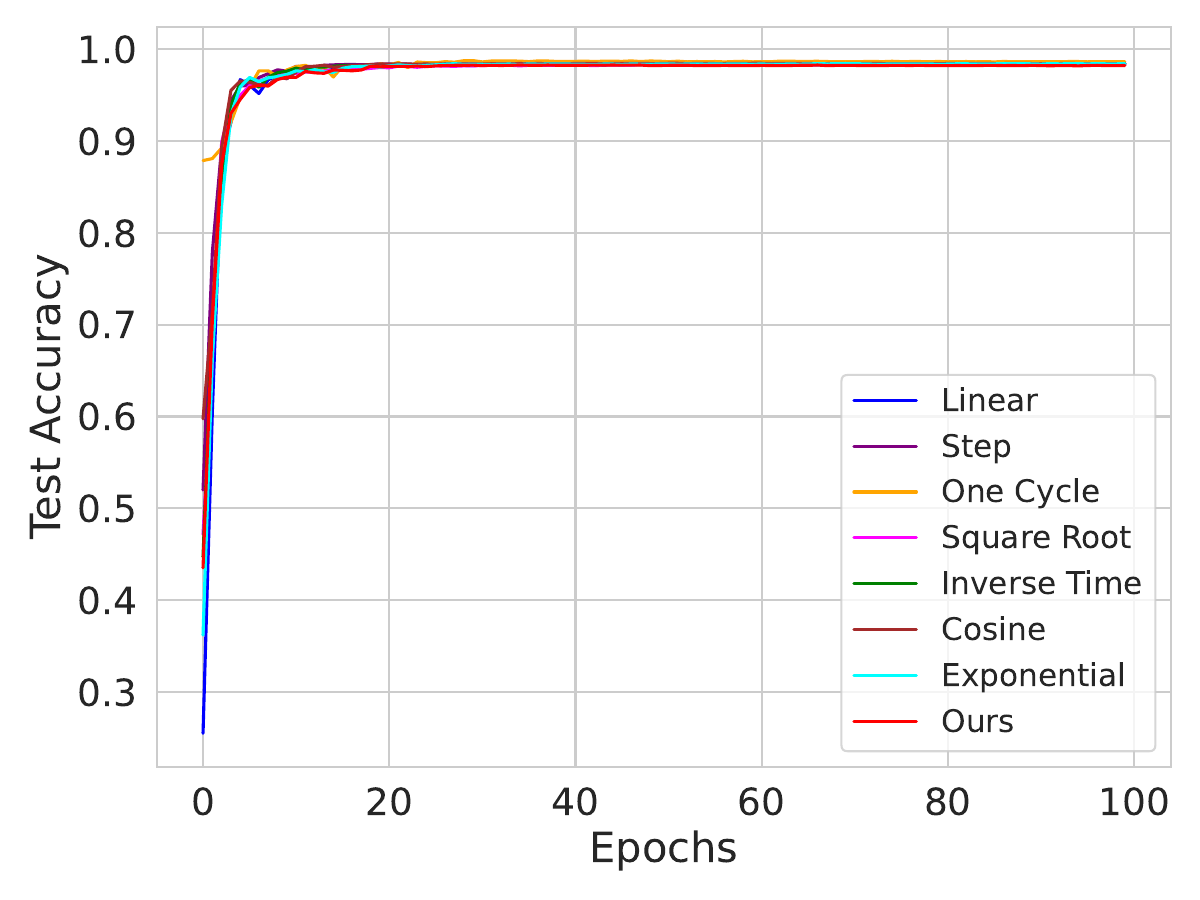}
    \caption{Validation acc. v/s Epochs}
  \end{subfigure}
  \caption{\textbf{Comparision with schedulers}: Mini-batch experiments on a 3 layer network with 3000 nodes in each layer, trained on MNIST.}
\label{fig:3__3000_mini}
\end{figure}
\begin{figure}[htbp]
  \centering
  \begin{subfigure}[b]{0.31\textwidth}
    \includegraphics[width=\textwidth]{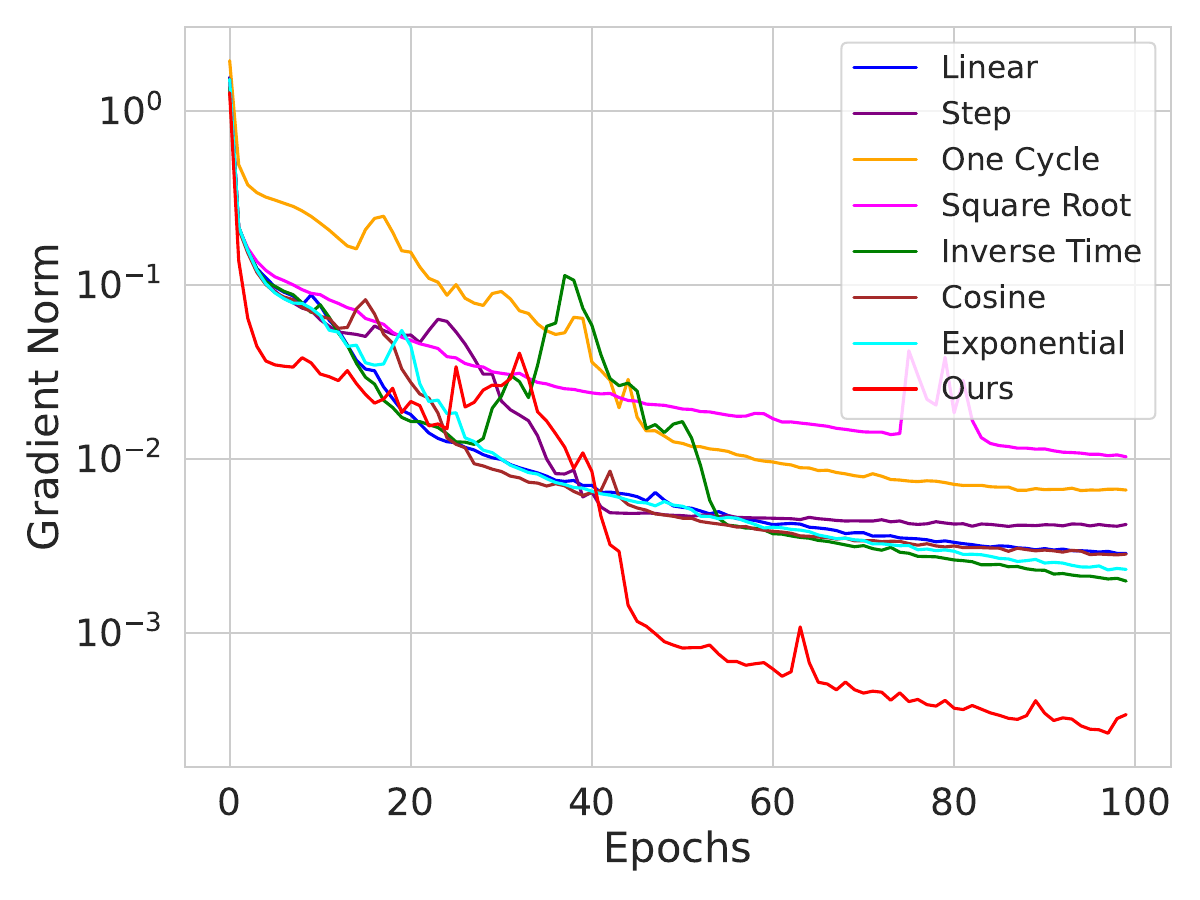}
    \caption{Gradient norm v/s Epochs}
  \end{subfigure}
  \hfill
  \begin{subfigure}[b]{0.31\textwidth}
    \includegraphics[width=\textwidth]{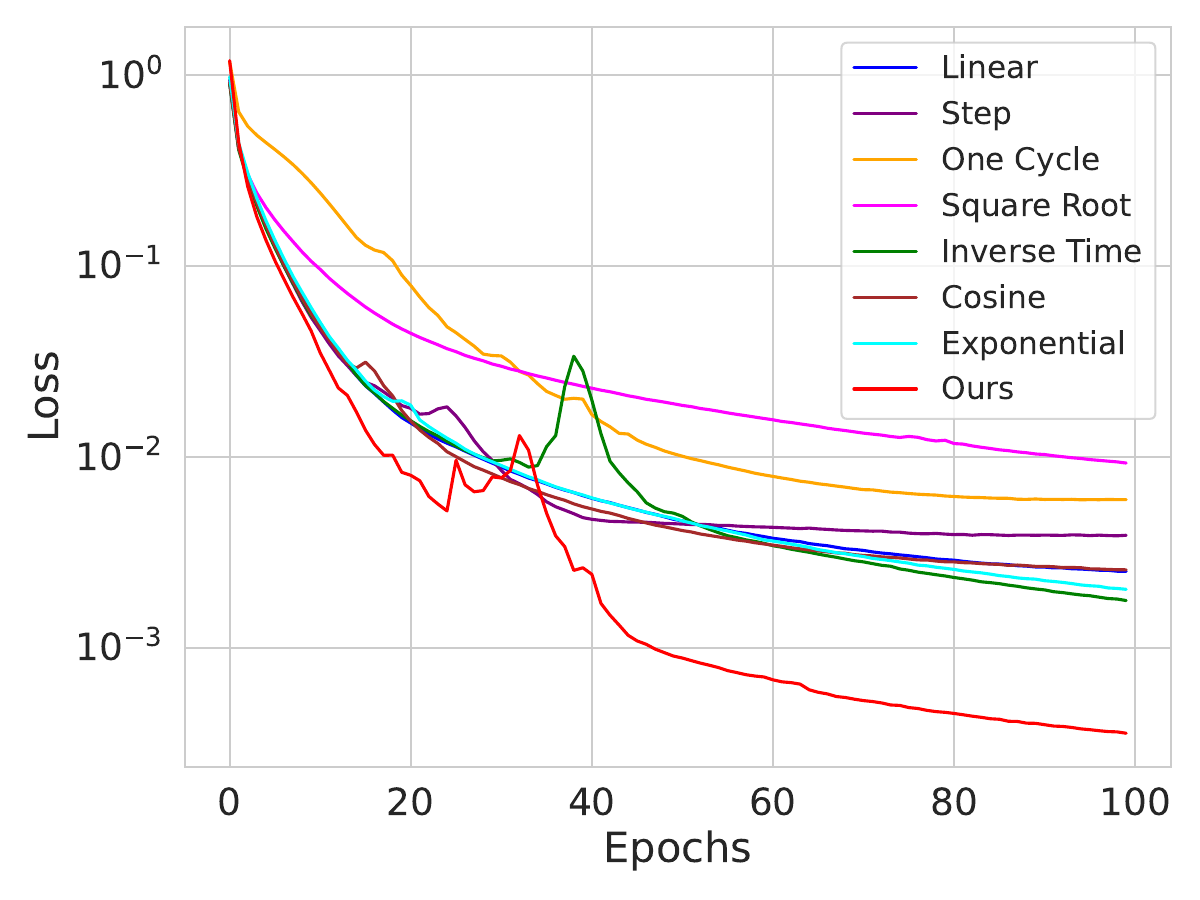}
    \caption{Training loss v/s Epochs}
  \end{subfigure}
  \hfill
  \begin{subfigure}[b]{0.31\textwidth}
    \includegraphics[width=\textwidth]{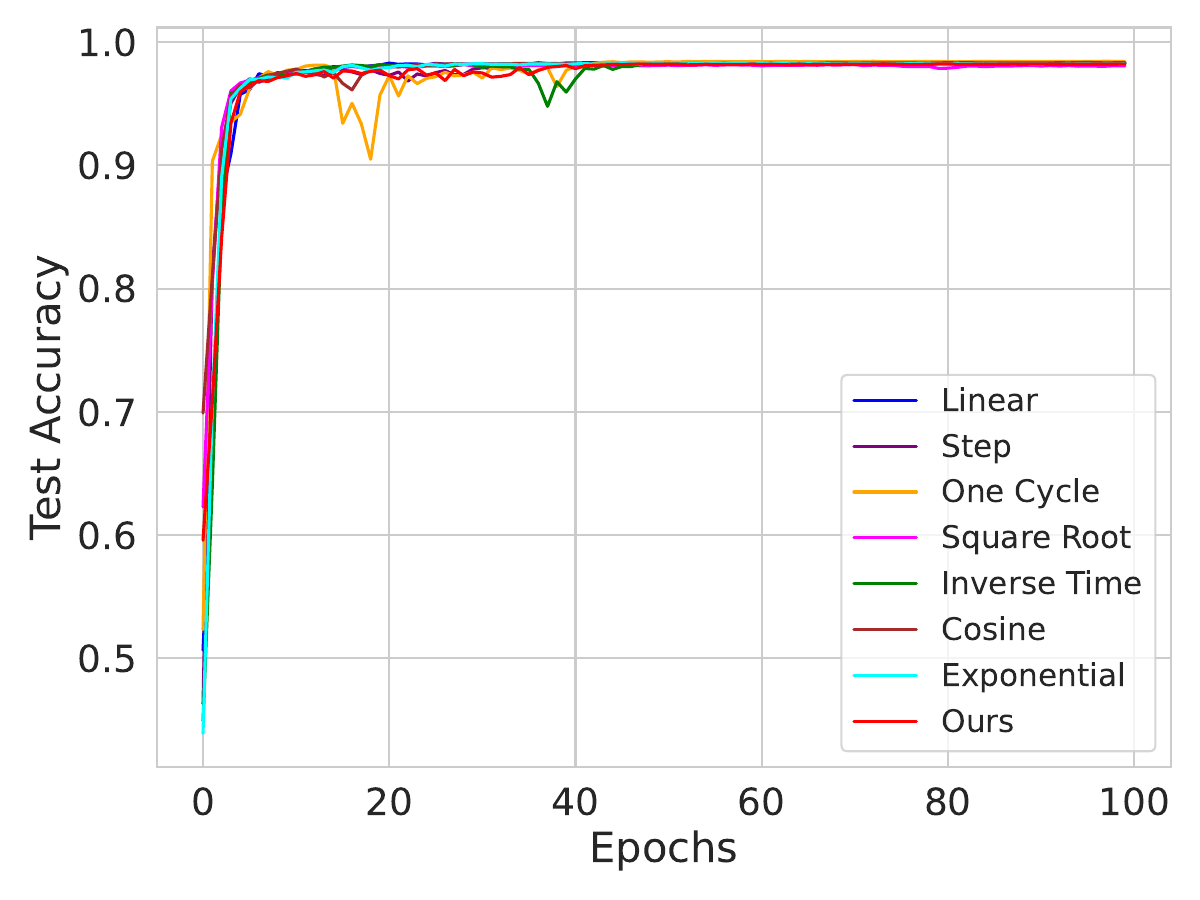}
    \caption{Validation acc. v/s Epochs}
  \end{subfigure}
  \caption{\textbf{Comparision with schedulers}: Mini-batch experiments on a 5 layer network with 300 nodes in each layer, trained on MNIST.}
\label{fig:5__300_mini}
\end{figure}
\begin{figure}[htbp]
  \centering
  \begin{subfigure}[b]{0.31\textwidth}
    \includegraphics[width=\textwidth]{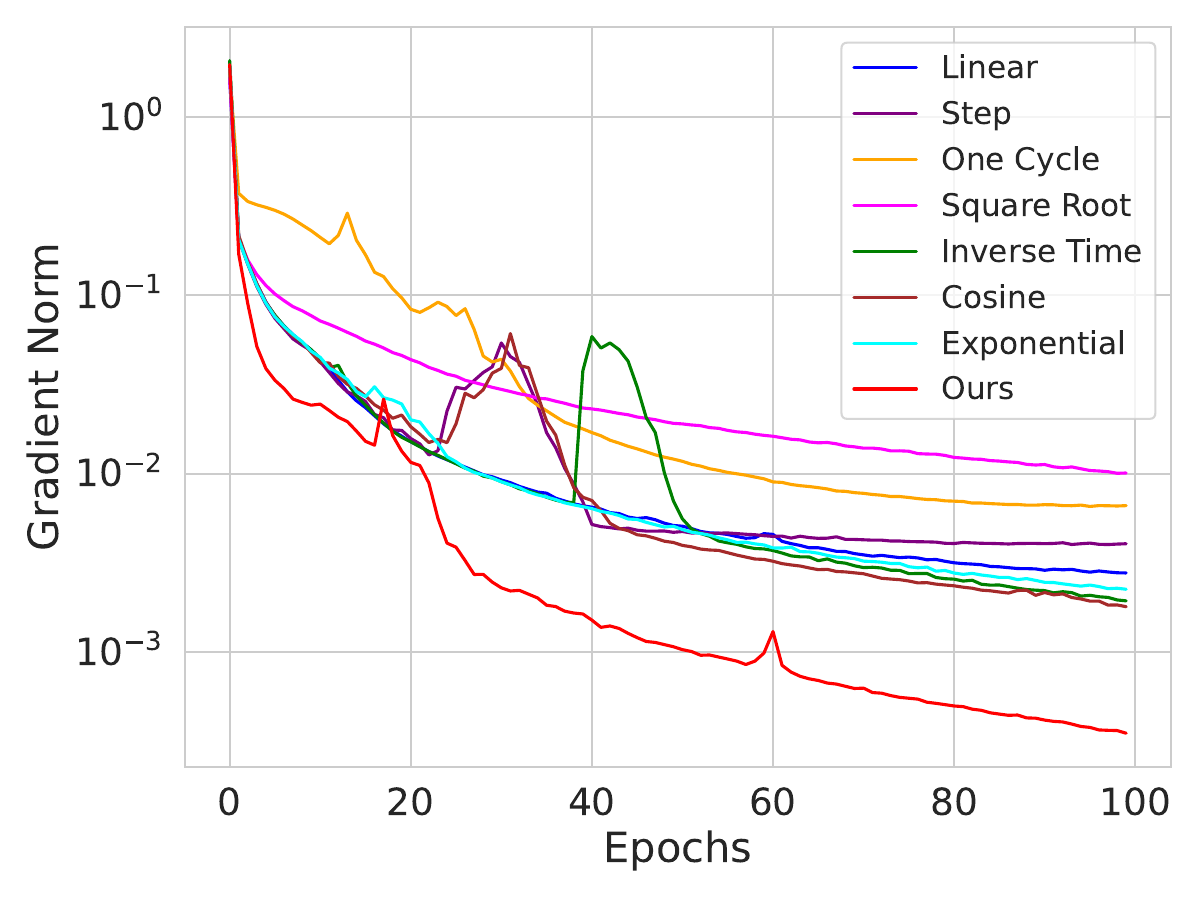}
    \caption{Gradient norm v/s Epochs}
  \end{subfigure}
  \hfill
  \begin{subfigure}[b]{0.31\textwidth}
    \includegraphics[width=\textwidth]{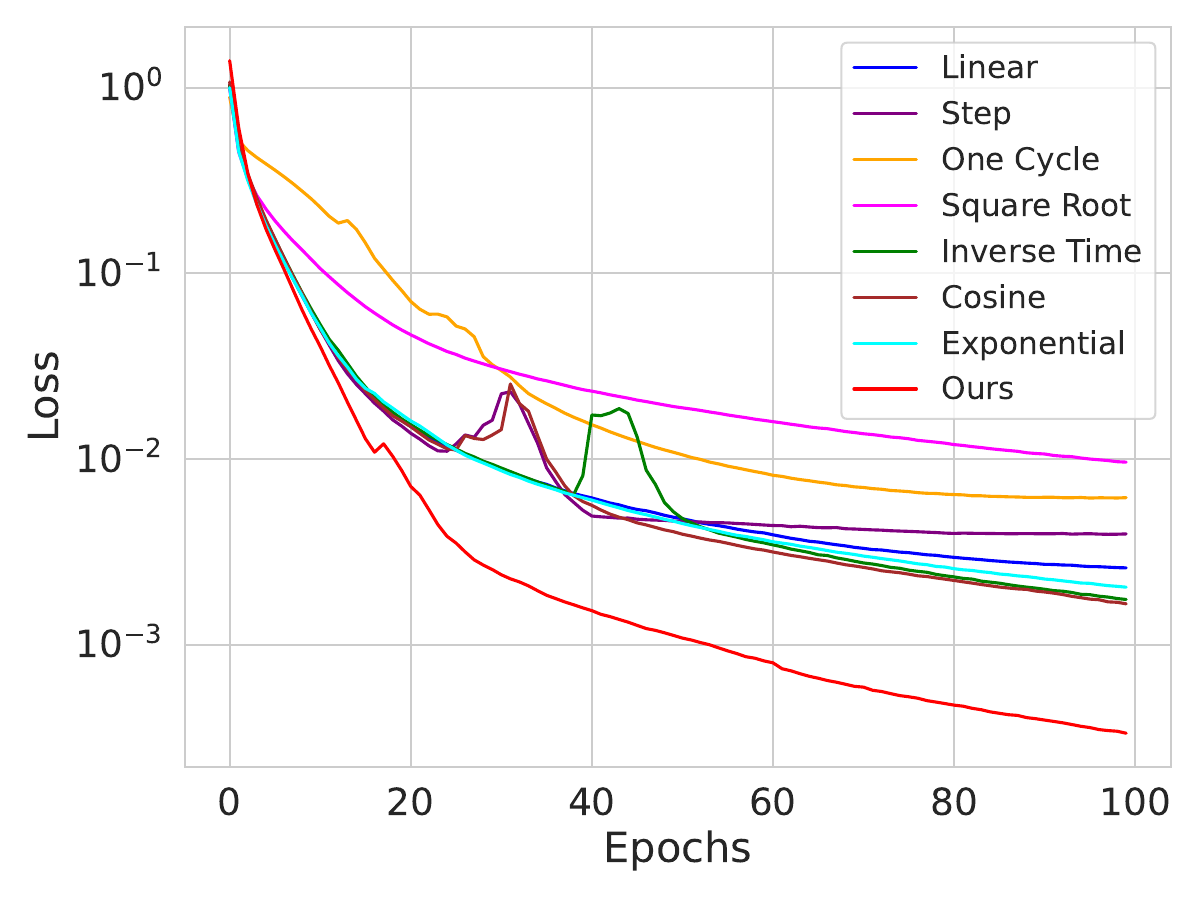}
    \caption{Training loss v/s Epochs}
  \end{subfigure}
  \hfill
  \begin{subfigure}[b]{0.31\textwidth}
    \includegraphics[width=\textwidth]{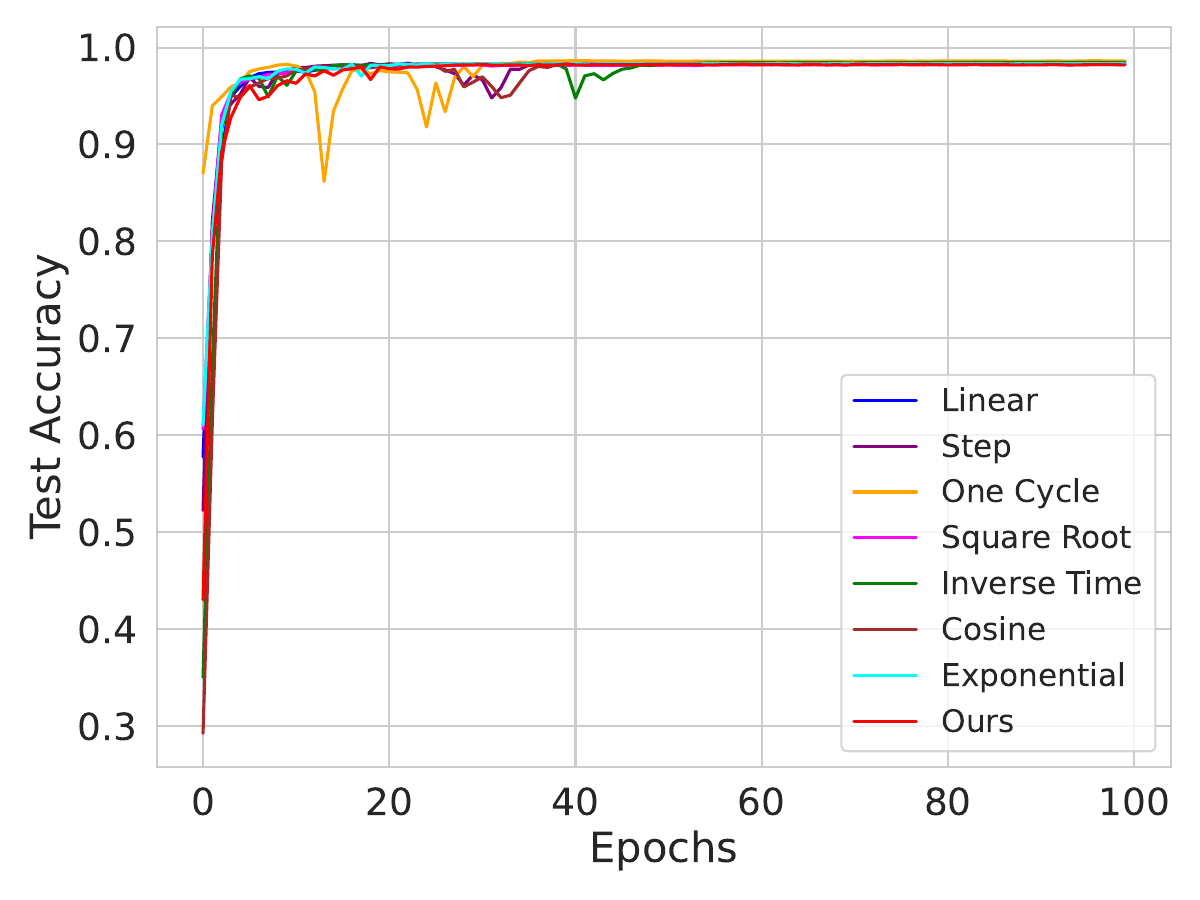}
    \caption{Validation acc. v/s Epochs}
  \end{subfigure}
  \caption{\textbf{Comparision with schedulers}: Mini-batch experiments on a 5 layer network with 1000 nodes in each layer, trained on MNIST.}
\label{fig:5__1000_mini}
\end{figure}
\begin{figure}[htbp]
  \centering
  \begin{subfigure}[b]{0.31\textwidth}
    \includegraphics[width=\textwidth]{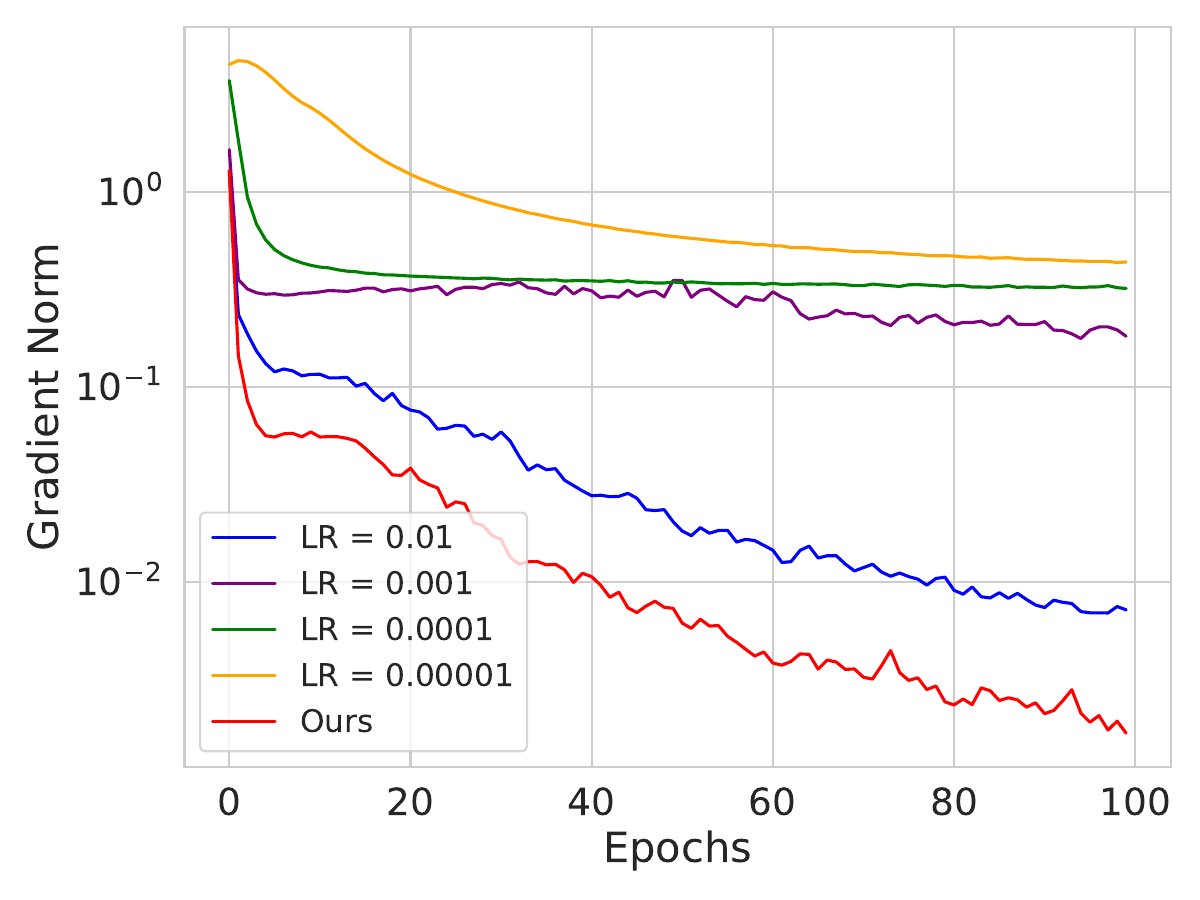}
    \caption{Gradient norm v/s Epochs}
  \end{subfigure}
  \hfill
  \begin{subfigure}[b]{0.31\textwidth}
    \includegraphics[width=\textwidth]{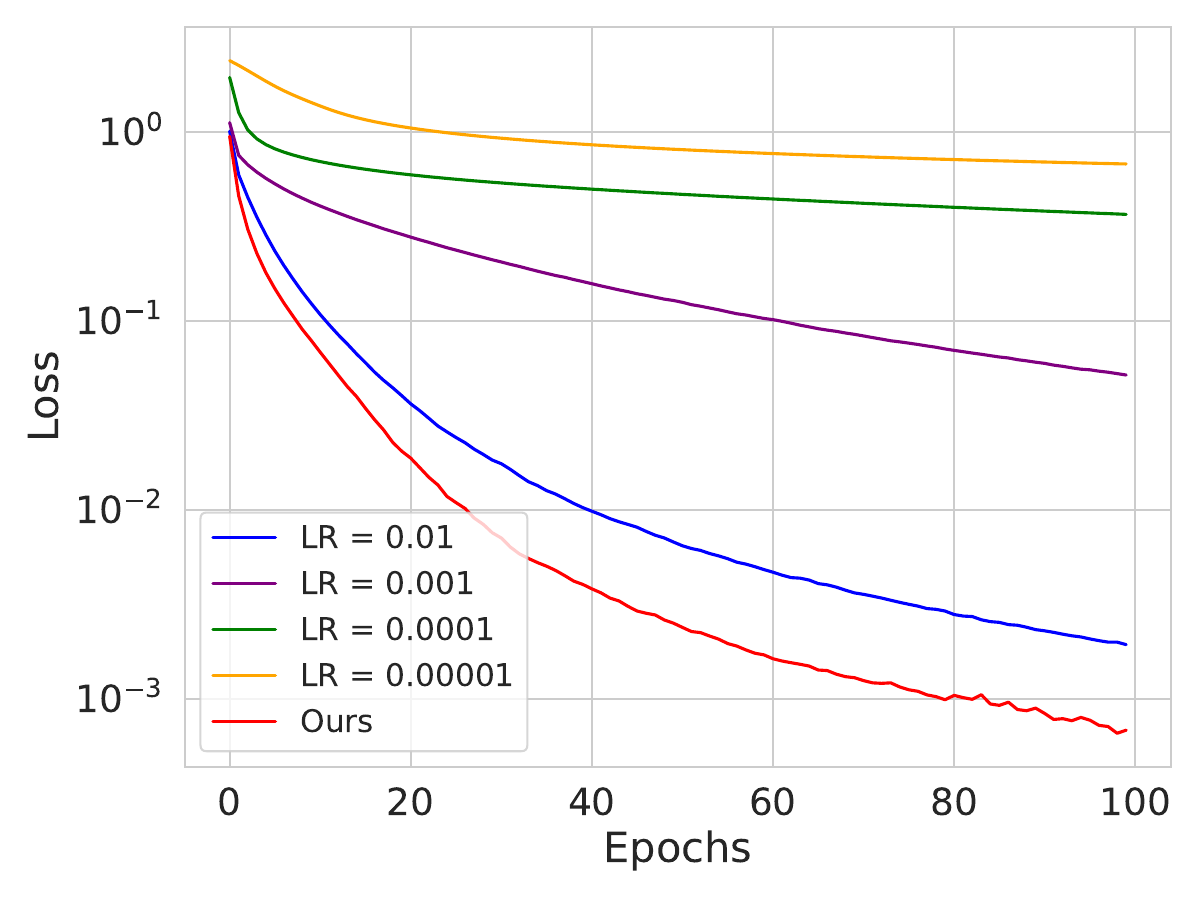}
    \caption{Training loss v/s Epochs}
  \end{subfigure}
  \hfill
  \begin{subfigure}[b]{0.31\textwidth}
    \includegraphics[width=\textwidth]{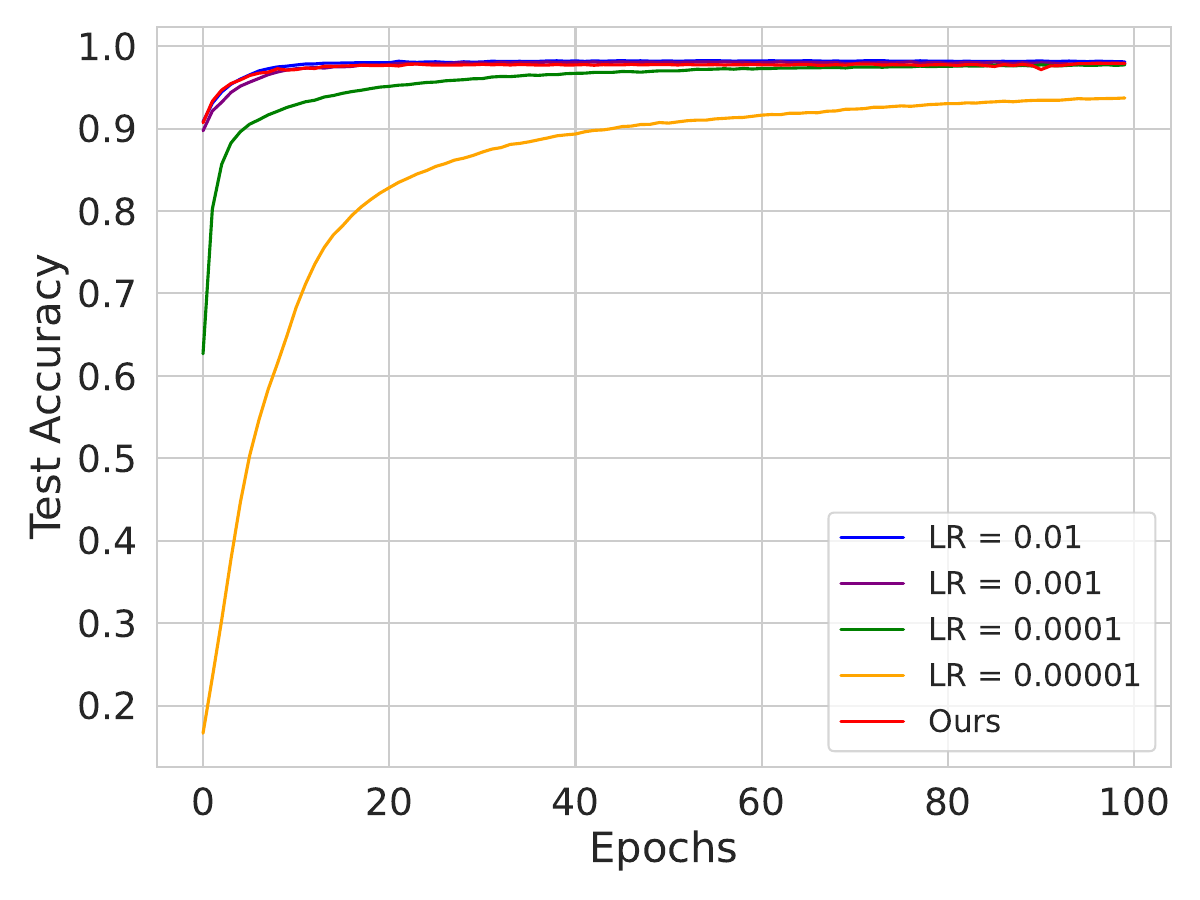}
    \caption{Validation acc. v/s Epochs}
  \end{subfigure}
  \caption{\textbf{Comparision with constant step sizes}: Mini-batch experiments on a single layer network with 300 nodes in each layer, trained on MNIST.}
\label{fig:1_300_mini_step}
\end{figure}
\begin{figure}[htbp]
  \centering
  \begin{subfigure}[b]{0.31\textwidth}
    \includegraphics[width=\textwidth]{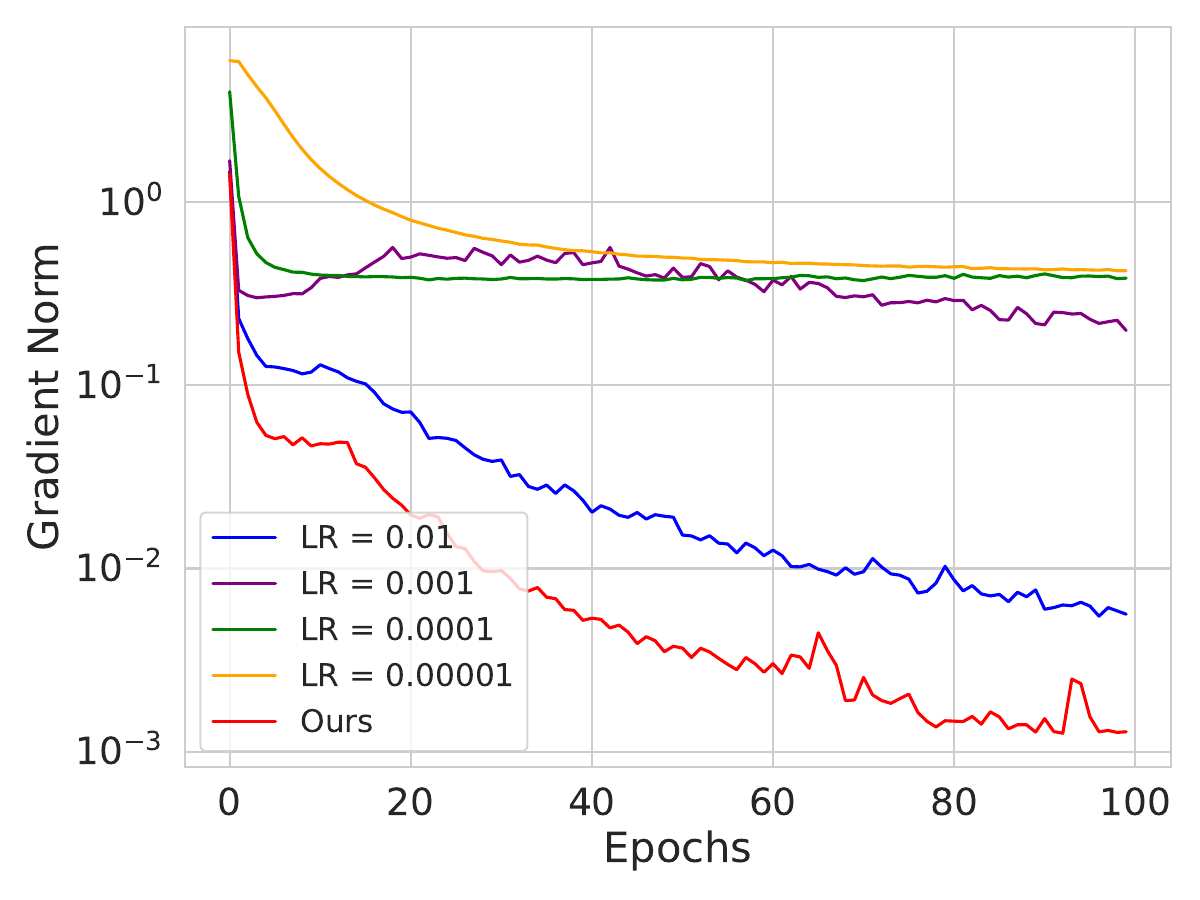}
    \caption{Gradient norm v/s Epochs}
  \end{subfigure}
  \hfill
  \begin{subfigure}[b]{0.31\textwidth}
    \includegraphics[width=\textwidth]{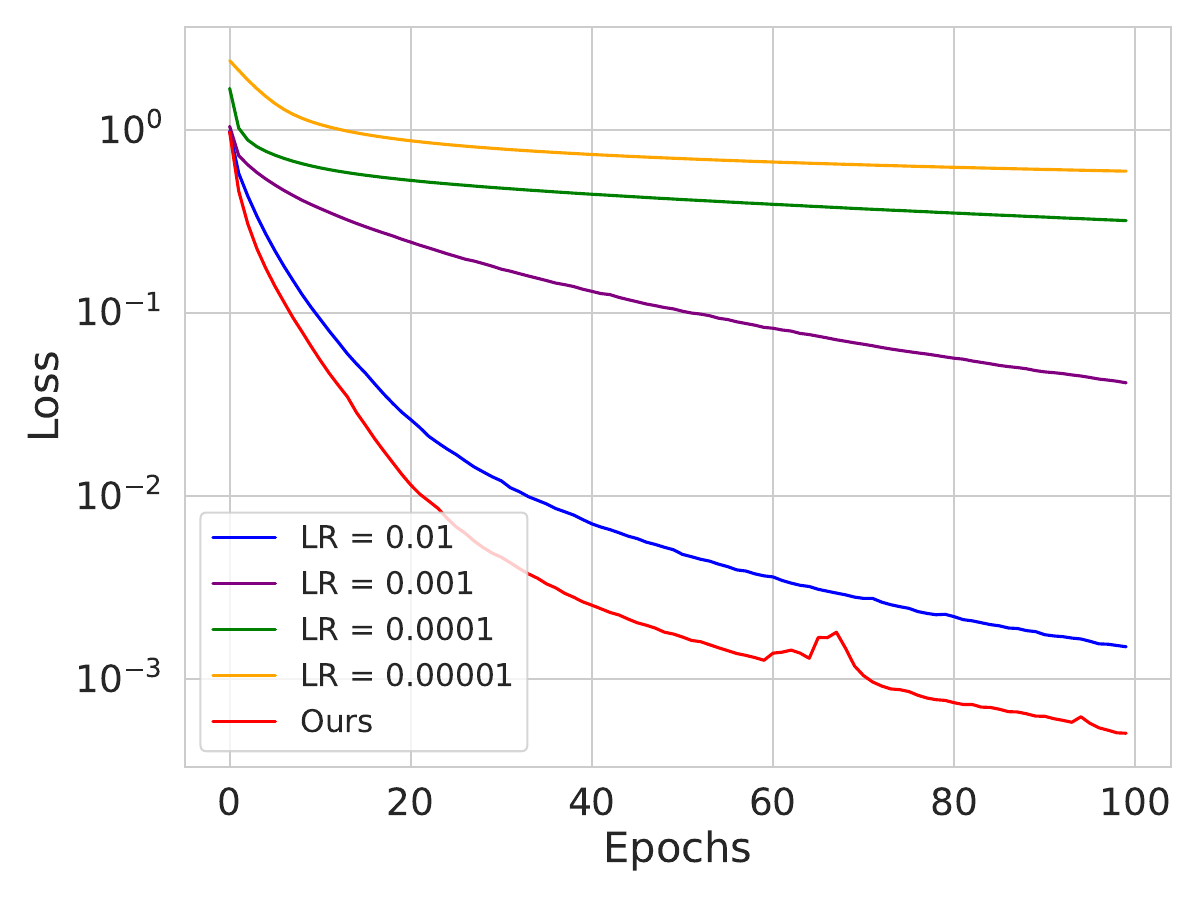}
    \caption{Training loss v/s Epochs}
  \end{subfigure}
  \hfill
  \begin{subfigure}[b]{0.31\textwidth}
    \includegraphics[width=\textwidth]{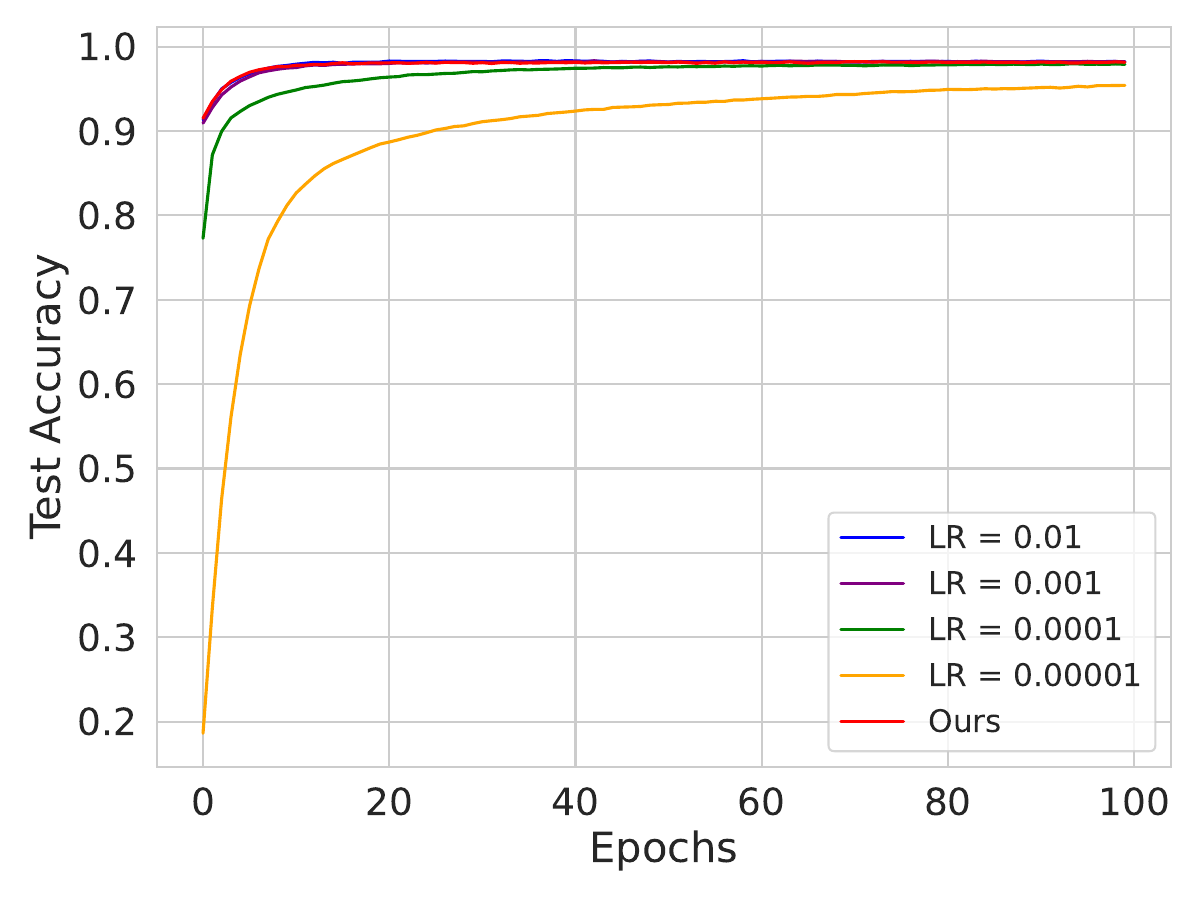}
    \caption{Validation acc. v/s Epochs}
  \end{subfigure}
  \caption{\textbf{Comparision with constant step sizes}: Mini-batch experiments on a single layer network with 1000 nodes in each layer, trained on MNIST.}
\label{fig:1__1000_mini_step}
\end{figure}
\begin{figure}[htbp]
  \centering
  \begin{subfigure}[b]{0.31\textwidth}
    \includegraphics[width=\textwidth]{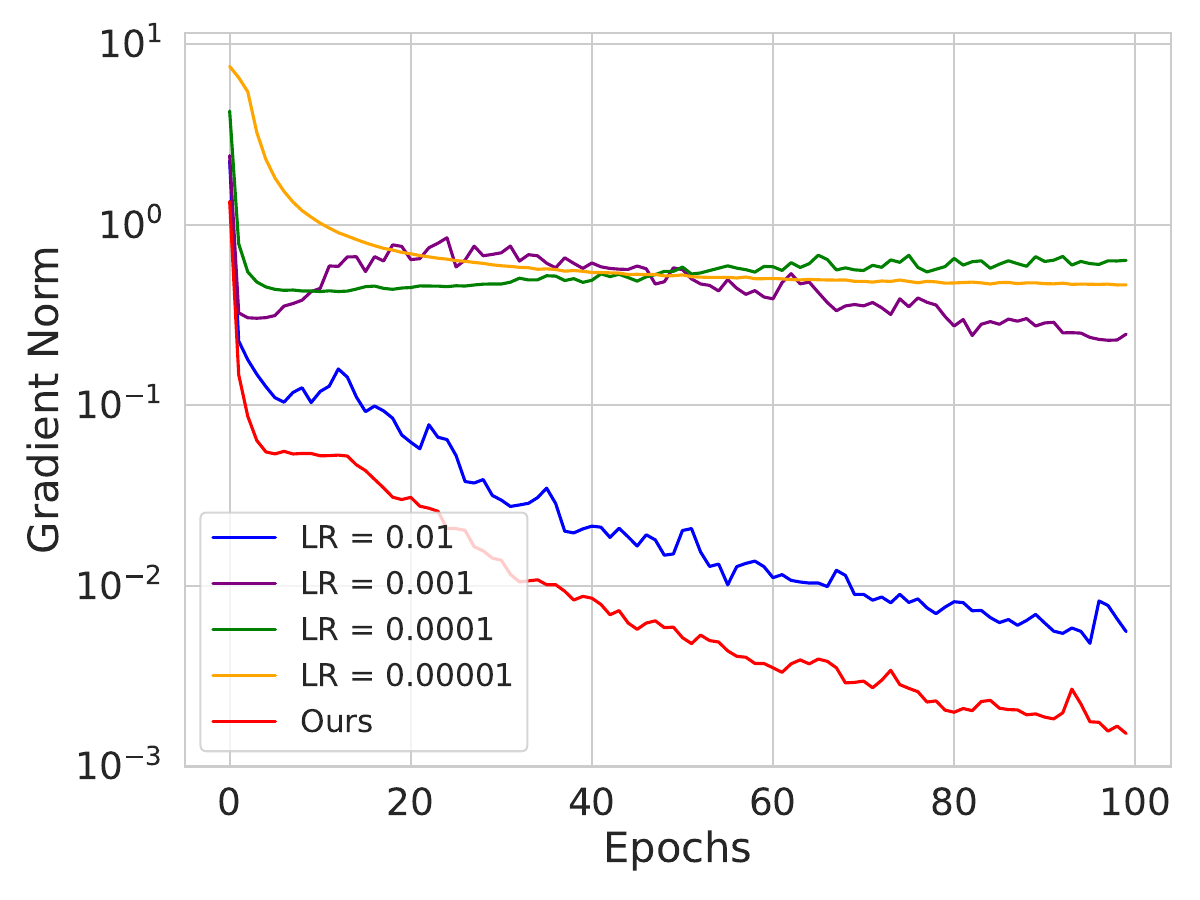}
    \caption{Gradient norm v/s Epochs}
  \end{subfigure}
  \hfill
  \begin{subfigure}[b]{0.31\textwidth}
    \includegraphics[width=\textwidth]{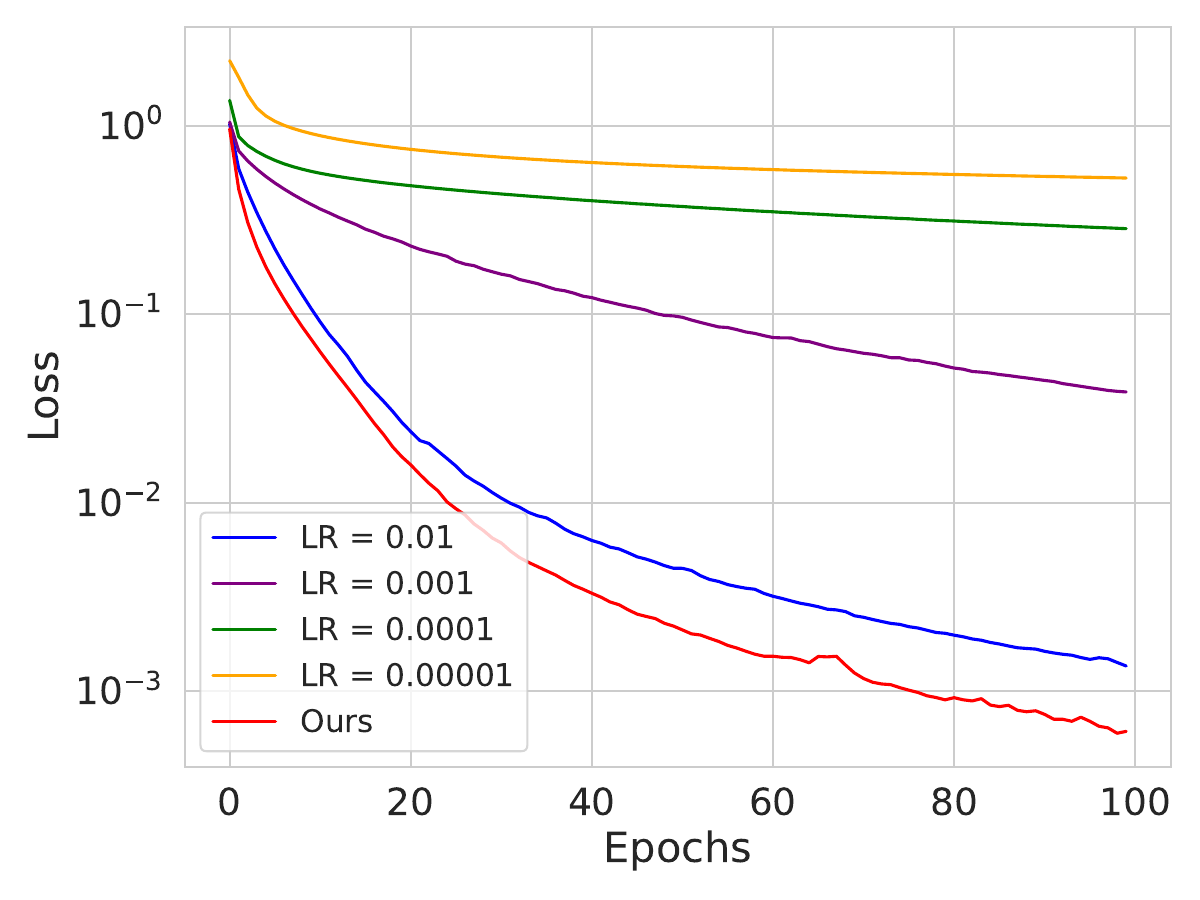}
    \caption{Training loss v/s Epochs}
  \end{subfigure}
  \hfill
  \begin{subfigure}[b]{0.31\textwidth}
    \includegraphics[width=\textwidth]{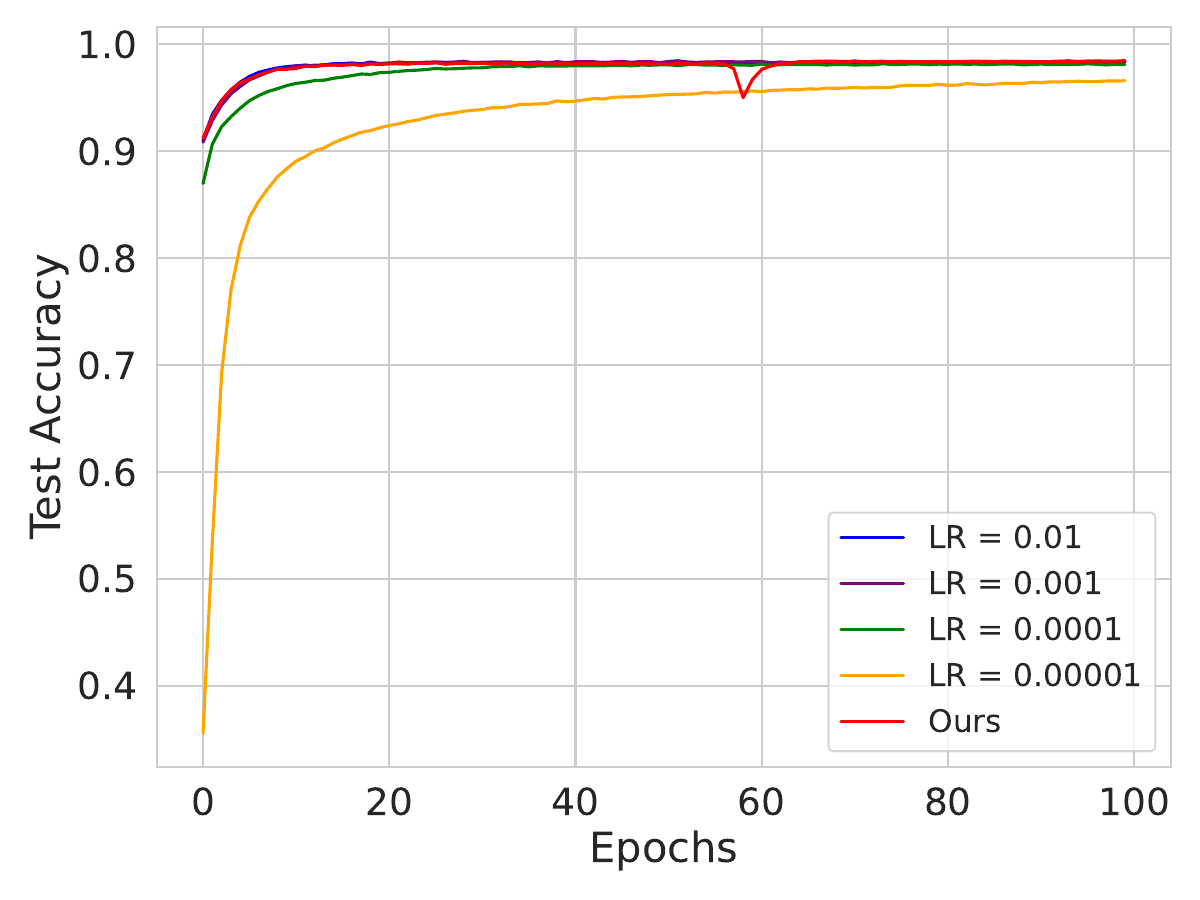}
    \caption{Validation acc. v/s Epochs}
  \end{subfigure}
  \caption{\textbf{Comparision with constant step sizes}: Mini-batch experiments on a single layer network with 3000 nodes in each layer, trained on MNIST.}
\label{fig:1__3000_mini_step}
\end{figure}
\begin{figure}[htbp]
  \centering
  \begin{subfigure}[b]{0.31\textwidth}
    \includegraphics[width=\textwidth]{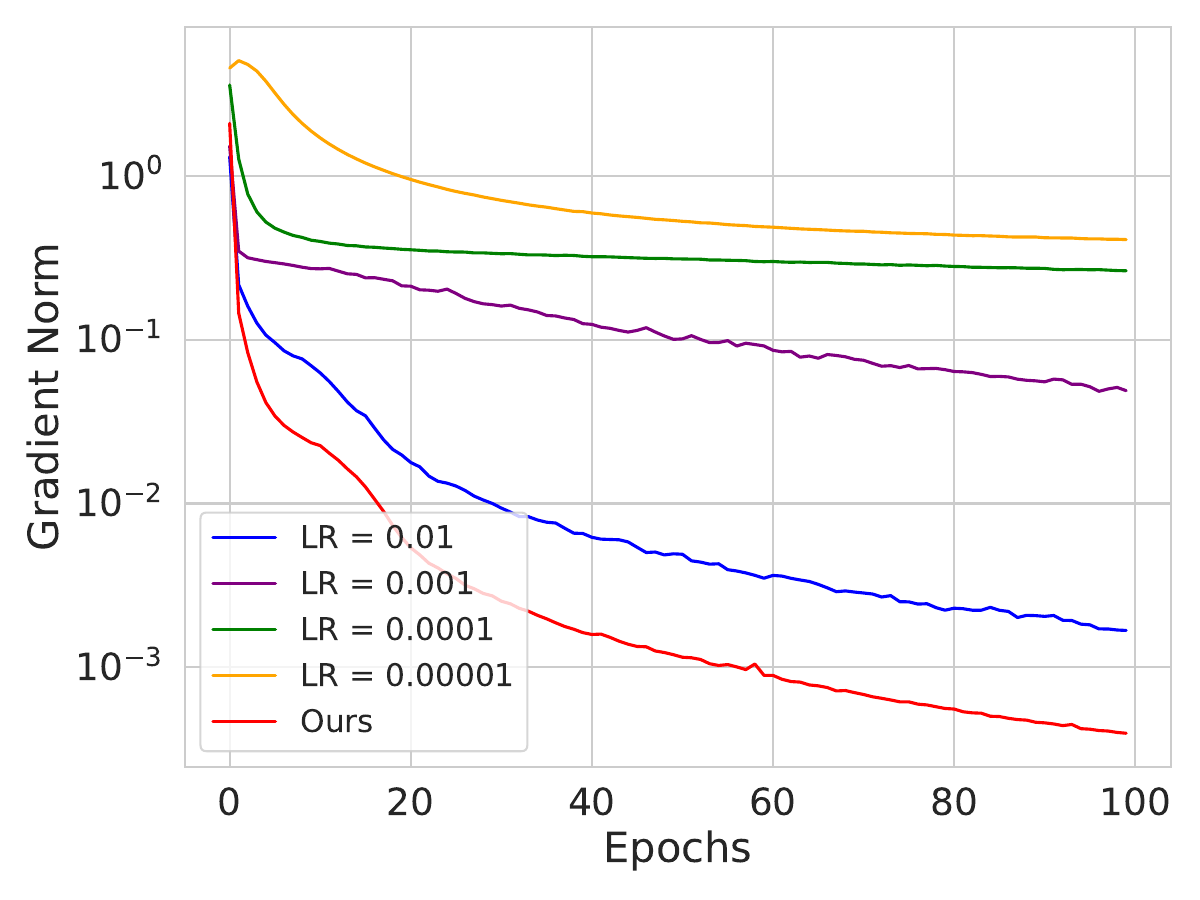}
    \caption{Gradient norm v/s Epochs}
  \end{subfigure}
  \hfill
  \begin{subfigure}[b]{0.31\textwidth}
    \includegraphics[width=\textwidth]{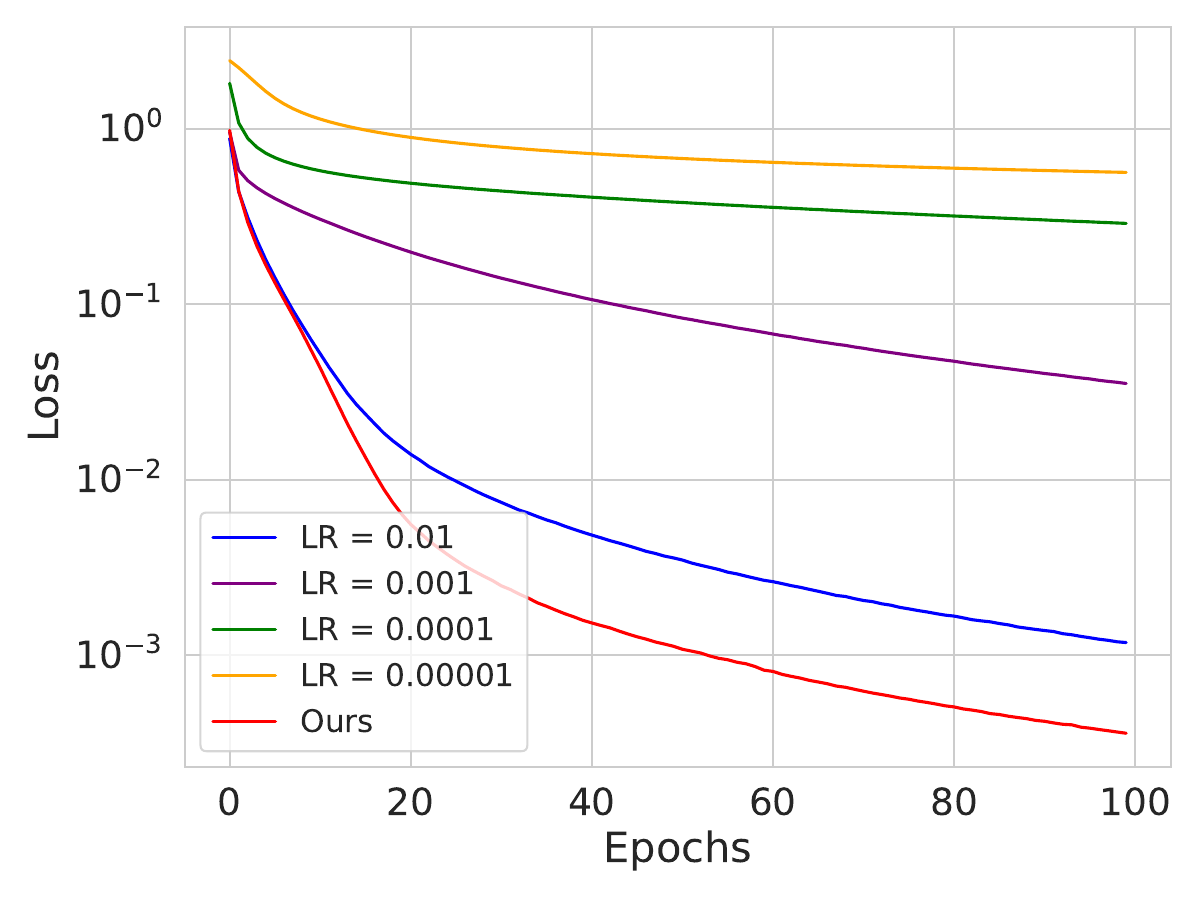}
    \caption{Training loss v/s Epochs}
  \end{subfigure}
  \hfill
  \begin{subfigure}[b]{0.31\textwidth}
    \includegraphics[width=\textwidth]{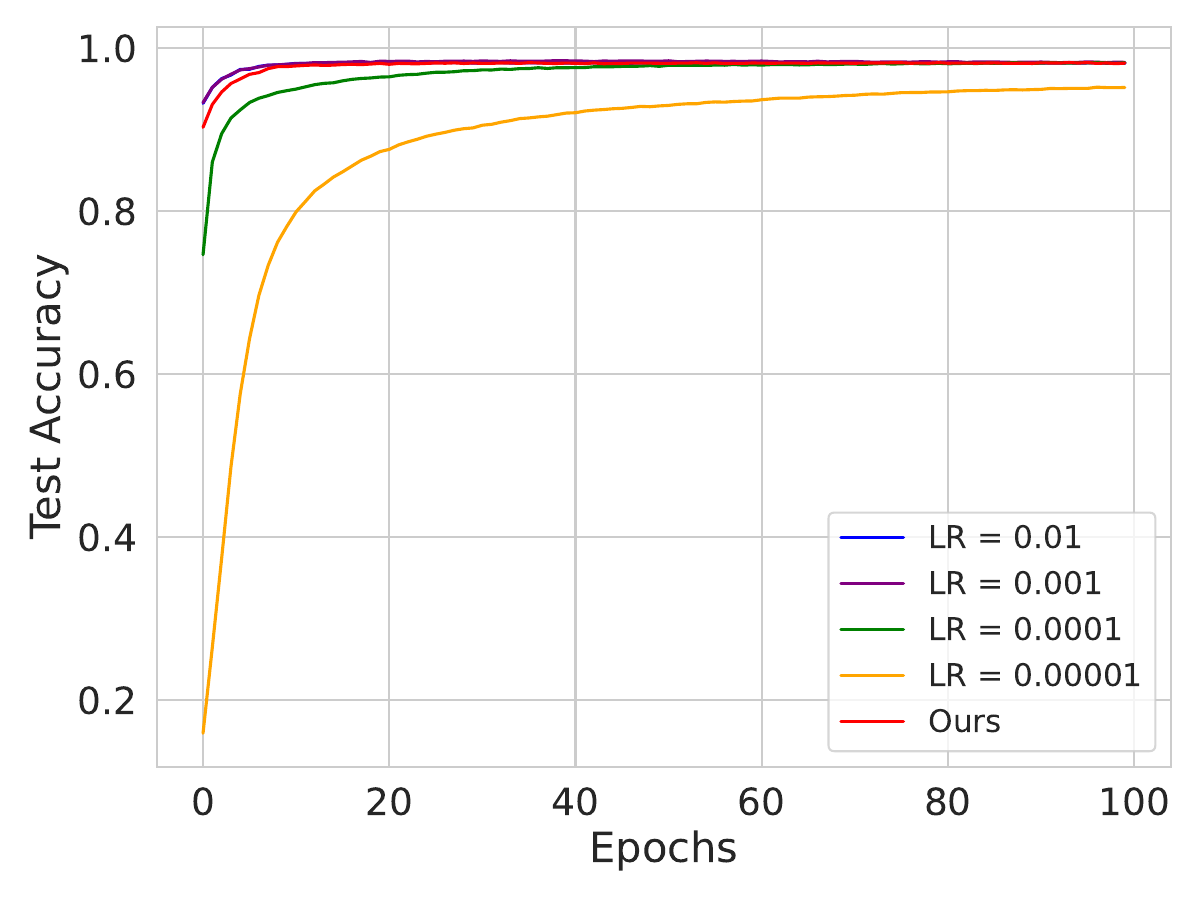}
    \caption{Validation acc. v/s Epochs}
  \end{subfigure}
  \caption{\textbf{Comparision with constant step sizes}: Mini-batch experiments on a 3 layer network with 300 nodes in each layer, trained on MNIST.}
\label{fig:3__300_mini_step}
\end{figure}
\begin{figure}[htbp]
  \centering
  \begin{subfigure}[b]{0.31\textwidth}
    \includegraphics[width=\textwidth]{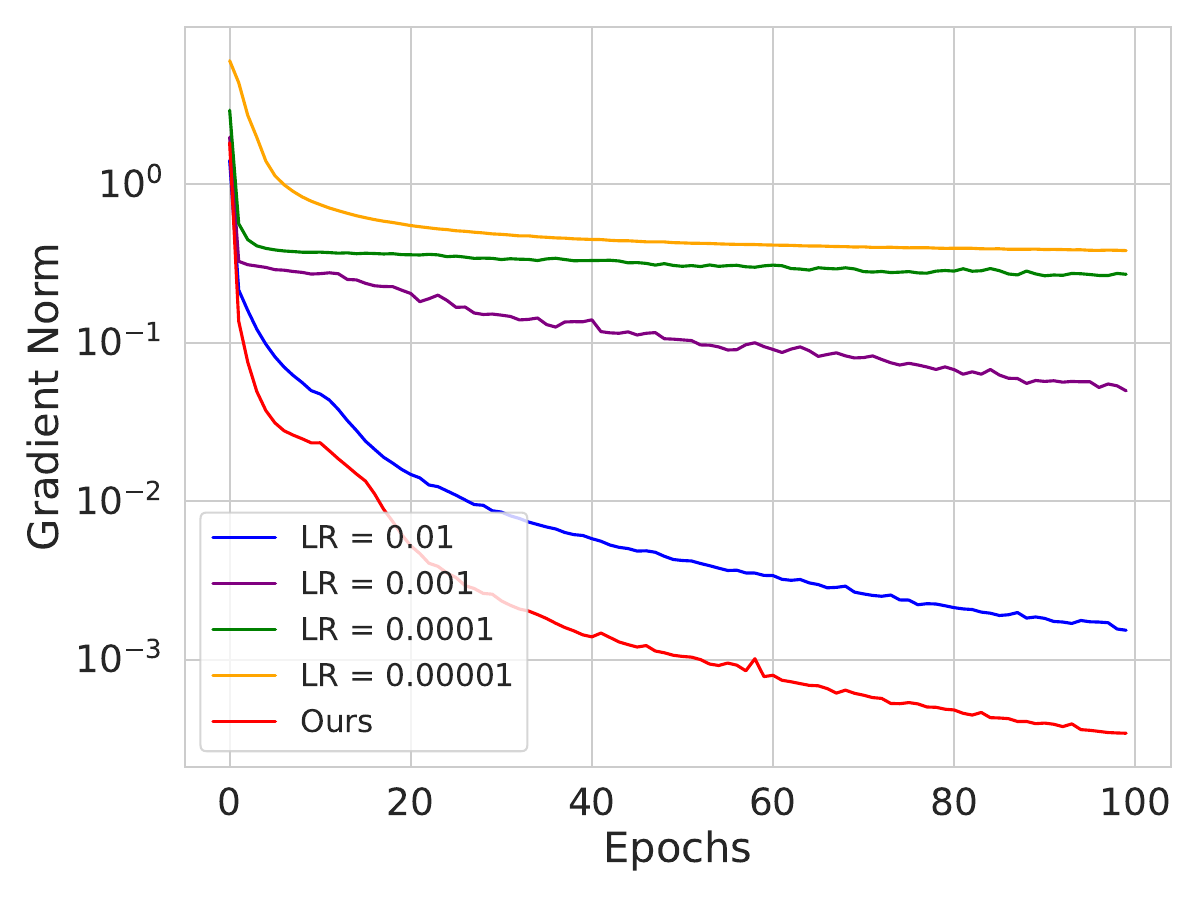}
    \caption{Gradient norm v/s Epochs}
  \end{subfigure}
  \hfill
  \begin{subfigure}[b]{0.31\textwidth}
    \includegraphics[width=\textwidth]{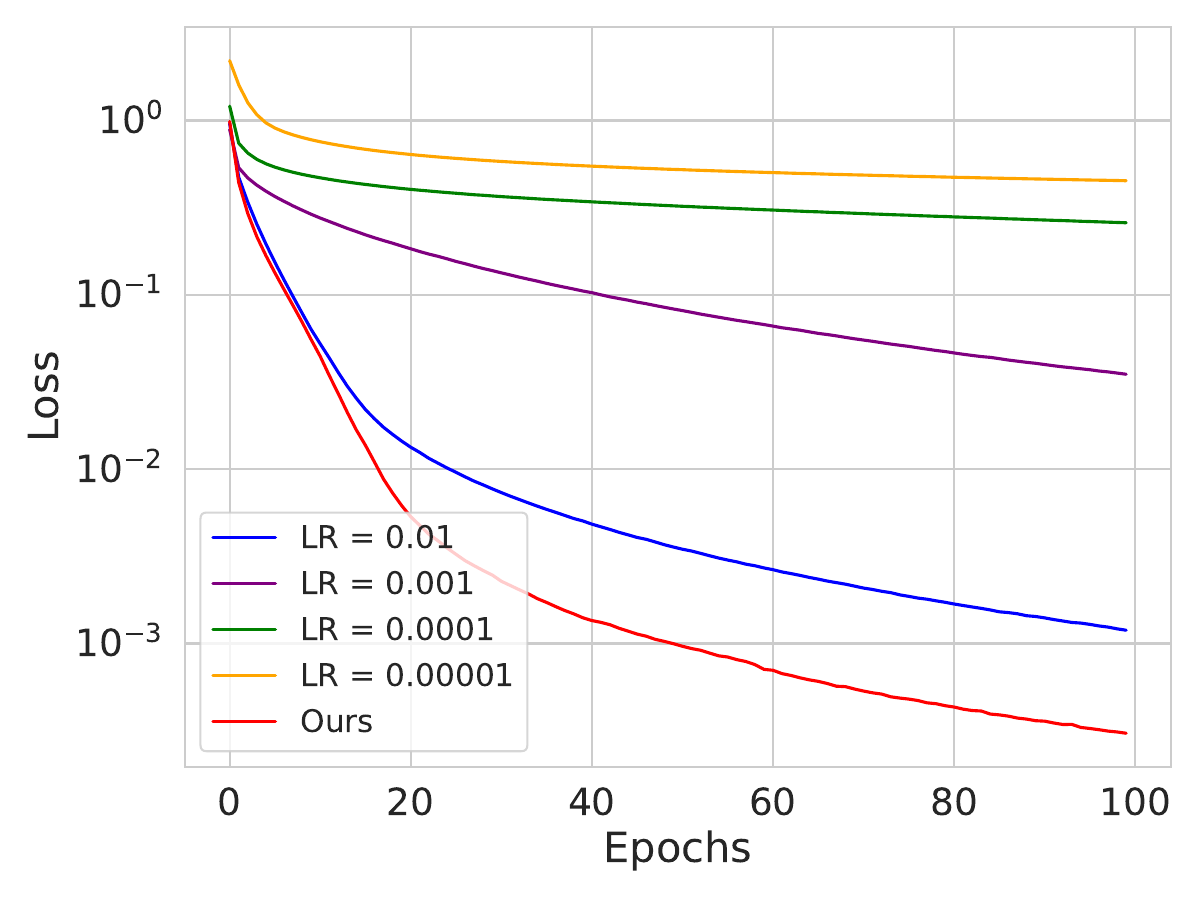}
    \caption{Training loss v/s Epochs}
  \end{subfigure}
  \hfill
  \begin{subfigure}[b]{0.31\textwidth}
    \includegraphics[width=\textwidth]{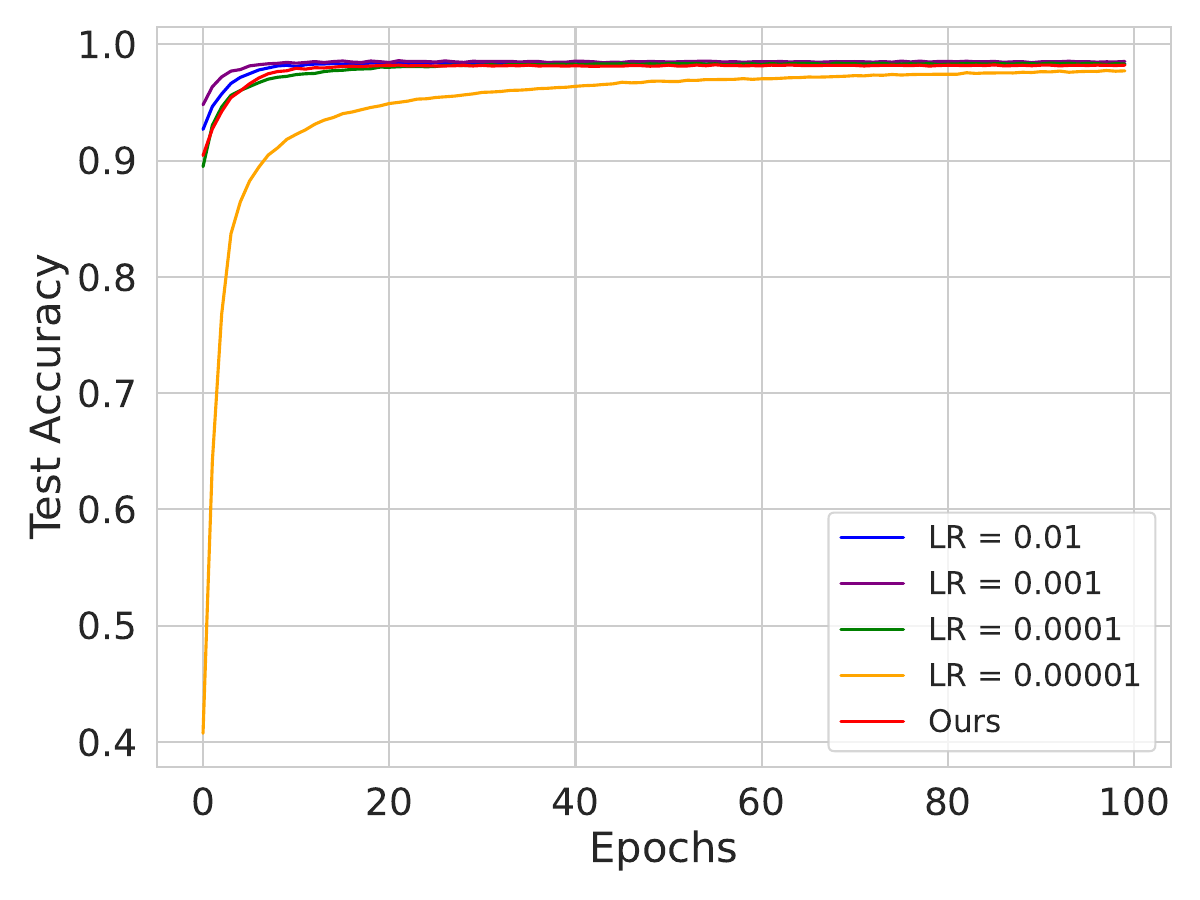}
    \caption{Validation acc. v/s Epochs}
  \end{subfigure}
  \caption{\textbf{Comparision with constant step sizes}: Mini-batch experiments on a 3 layer network with 1000 nodes in each layer, trained on MNIST.}
\label{fig:3__1000_mini_step}
\end{figure}
\begin{figure}[htbp]
  \centering
  \begin{subfigure}[b]{0.31\textwidth}
    \includegraphics[width=\textwidth]{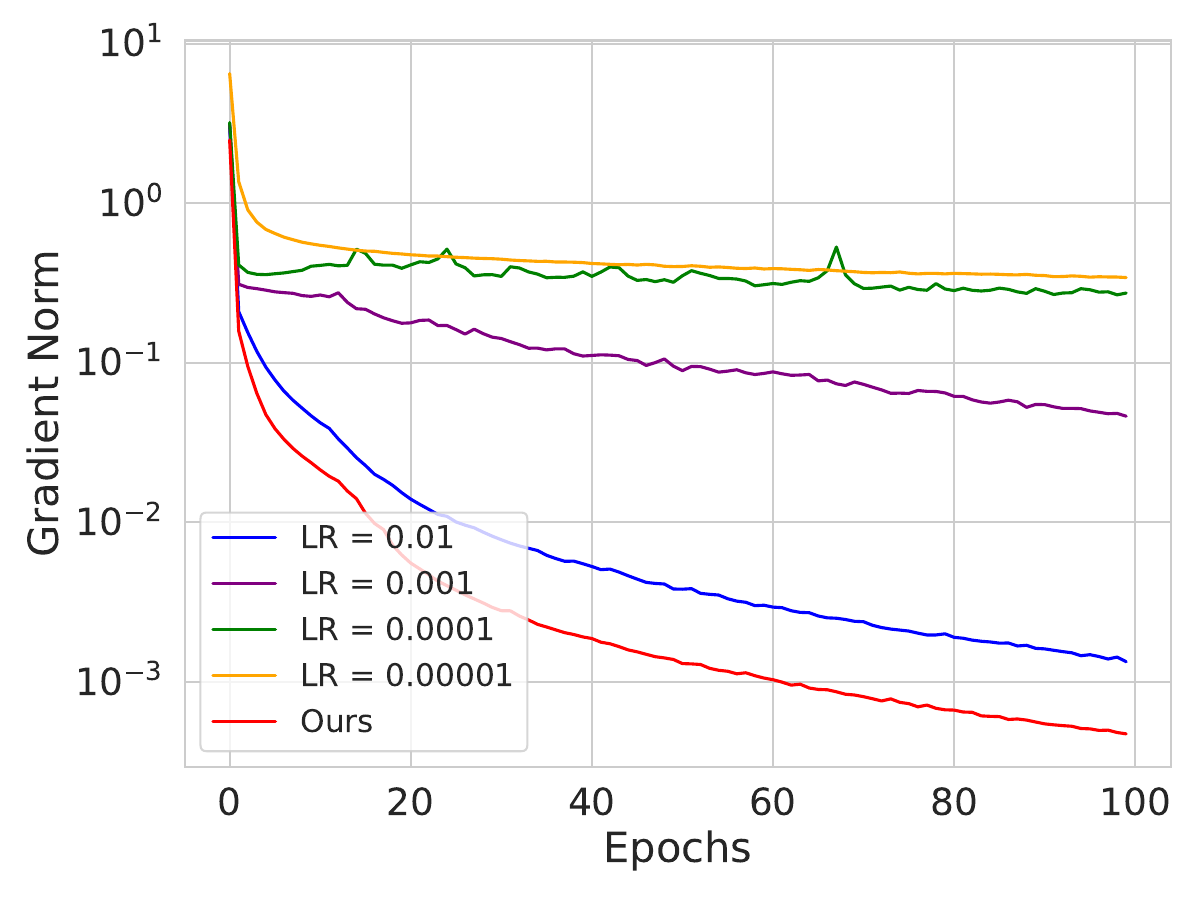}
    \caption{Gradient norm v/s Epochs}
  \end{subfigure}
  \hfill
  \begin{subfigure}[b]{0.31\textwidth}
    \includegraphics[width=\textwidth]{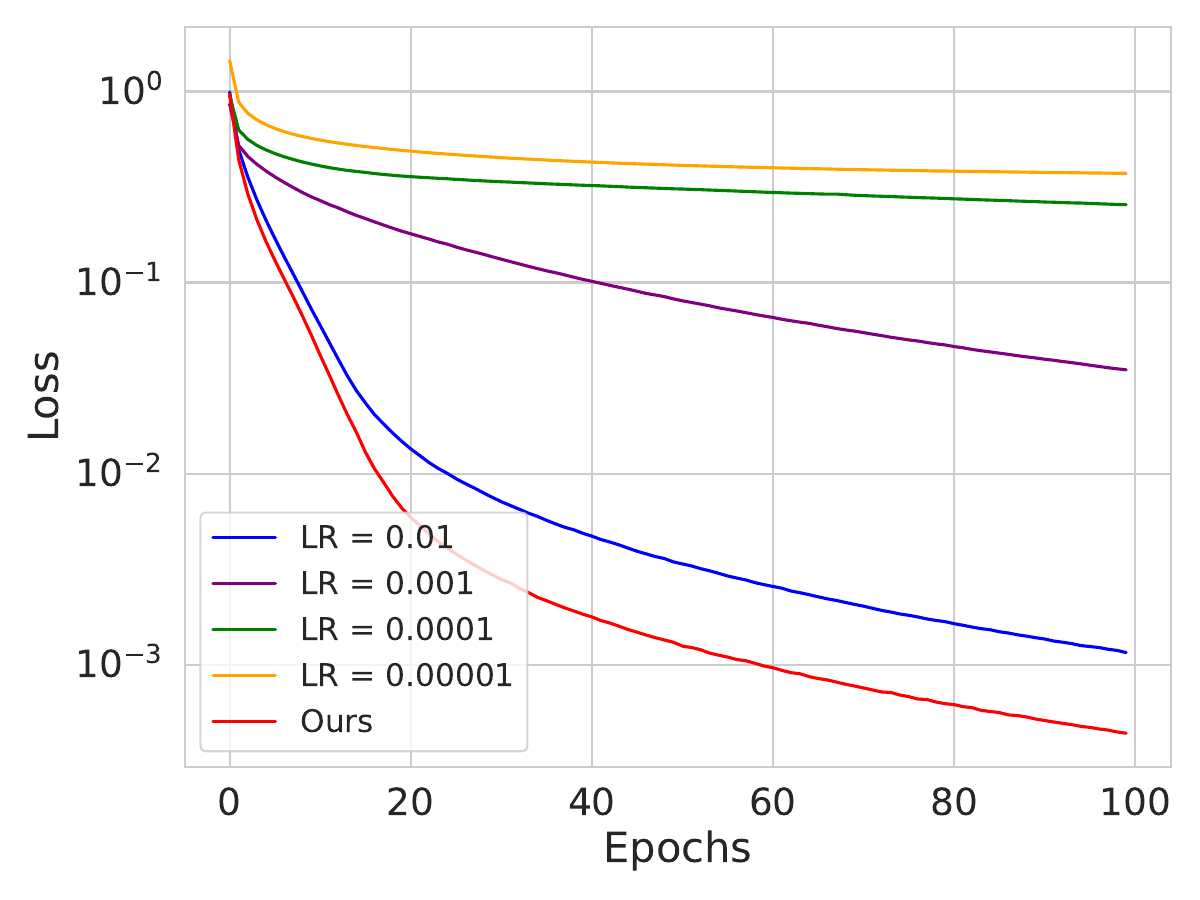}
    \caption{Training loss v/s Epochs}
  \end{subfigure}
  \hfill
  \begin{subfigure}[b]{0.31\textwidth}
    \includegraphics[width=\textwidth]{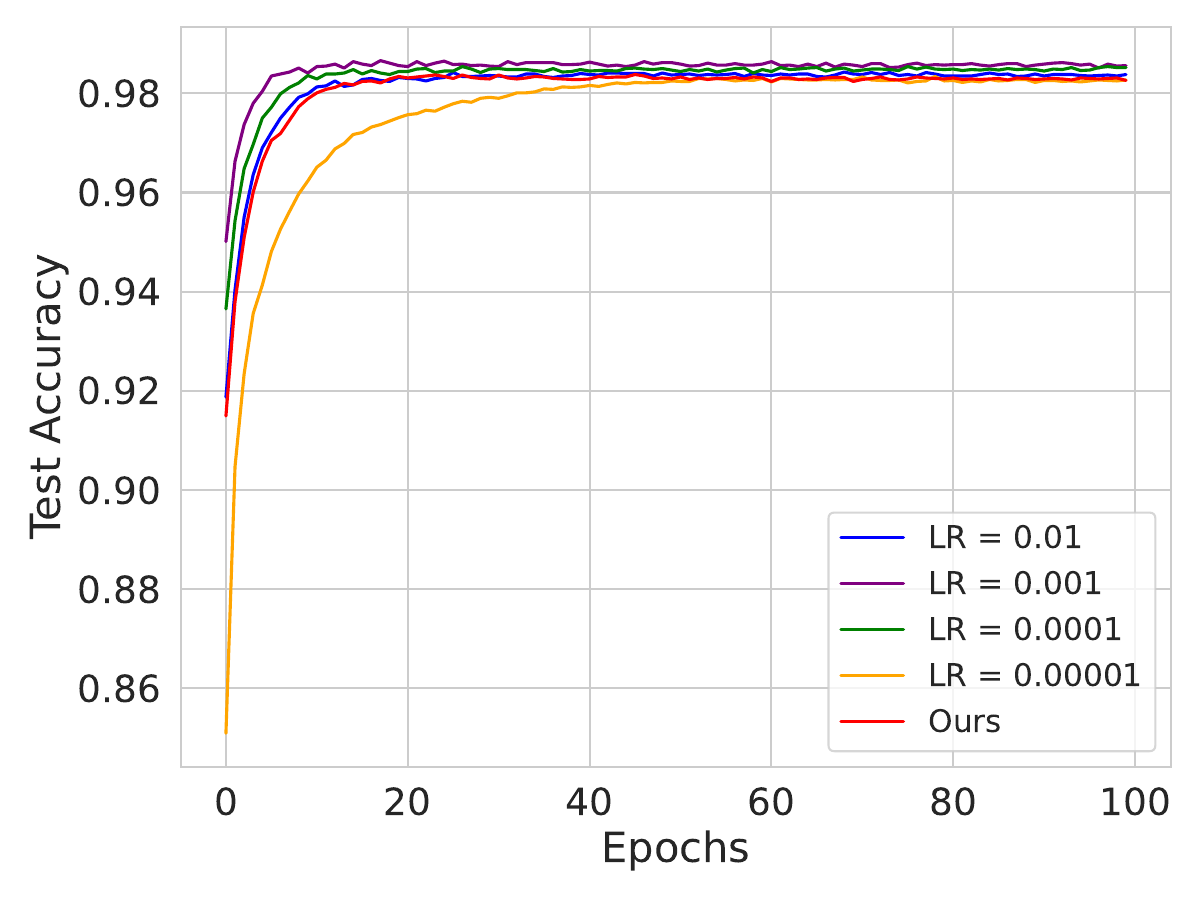}
    \caption{Validation acc. v/s Epochs}
  \end{subfigure}
  \caption{\textbf{Comparision with constant step sizes}: Mini-batch experiments on a 3 layer network with 3000 nodes in each layer, trained on MNIST.}
\label{fig:3__3000_mini_step}
\end{figure}
\begin{figure}[htbp]
  \centering
  \begin{subfigure}[b]{0.31\textwidth}
    \includegraphics[width=\textwidth]{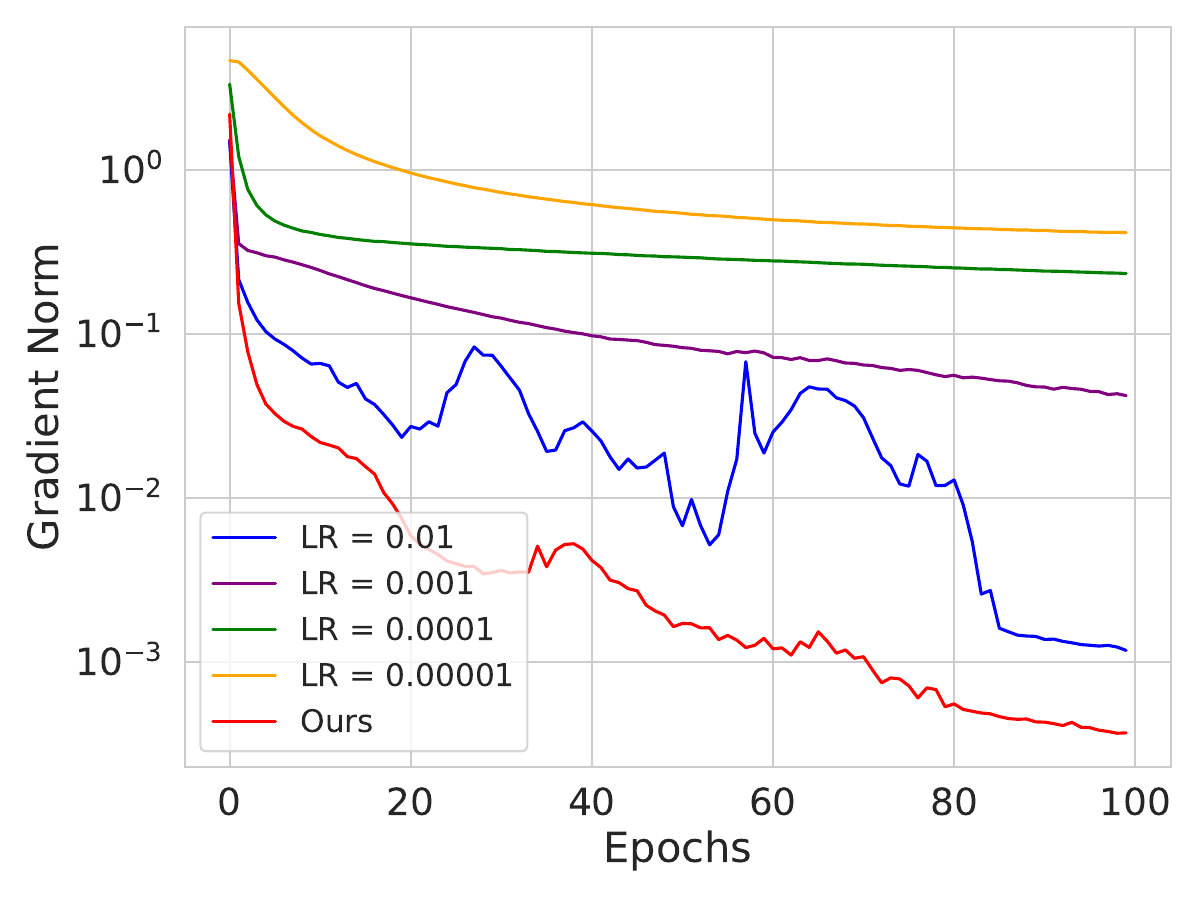}
    \caption{Gradient norm v/s Epochs}
  \end{subfigure}
  \hfill
  \begin{subfigure}[b]{0.31\textwidth}
    \includegraphics[width=\textwidth]{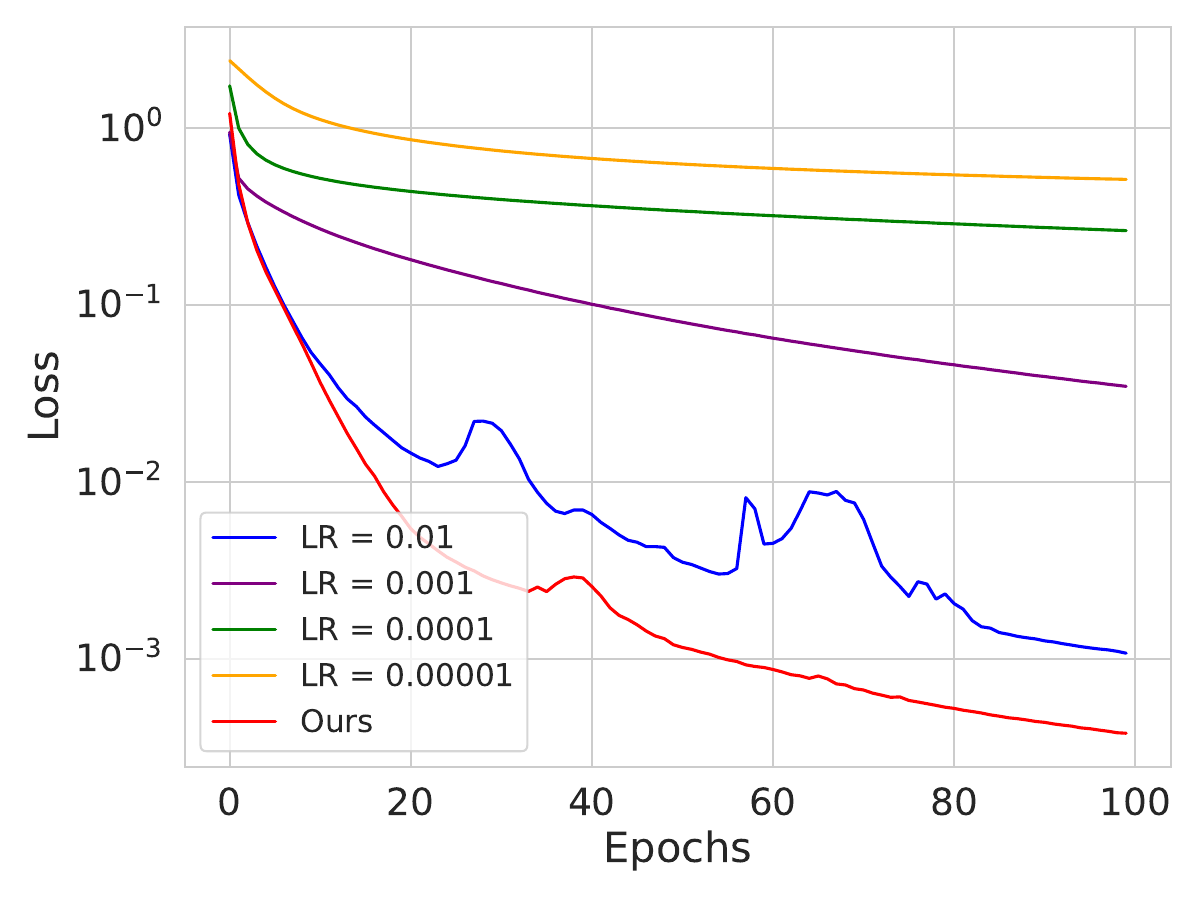}
    \caption{Training loss v/s Epochs}
  \end{subfigure}
  \hfill
  \begin{subfigure}[b]{0.31\textwidth}
    \includegraphics[width=\textwidth]{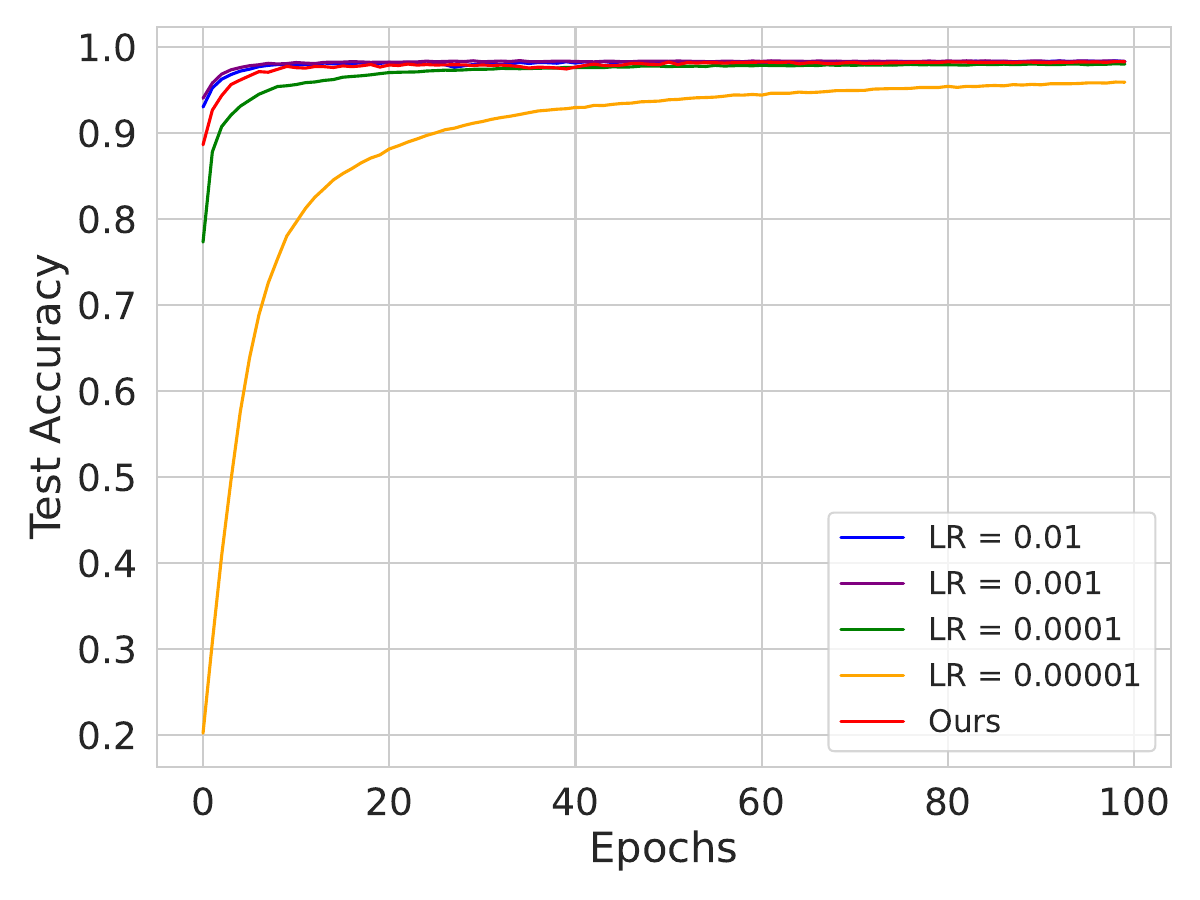}
    \caption{Validation acc. v/s Epochs}
  \end{subfigure}
  \caption{\textbf{Comparision with constant step sizes}: Mini-batch experiments on a 5 layer network with 300 nodes in each layer, trained on MNIST.}
\label{fig:5__300_mini_step}
\end{figure}
\begin{figure}[htbp]
  \centering
  \begin{subfigure}[b]{0.31\textwidth}
    \includegraphics[width=\textwidth]{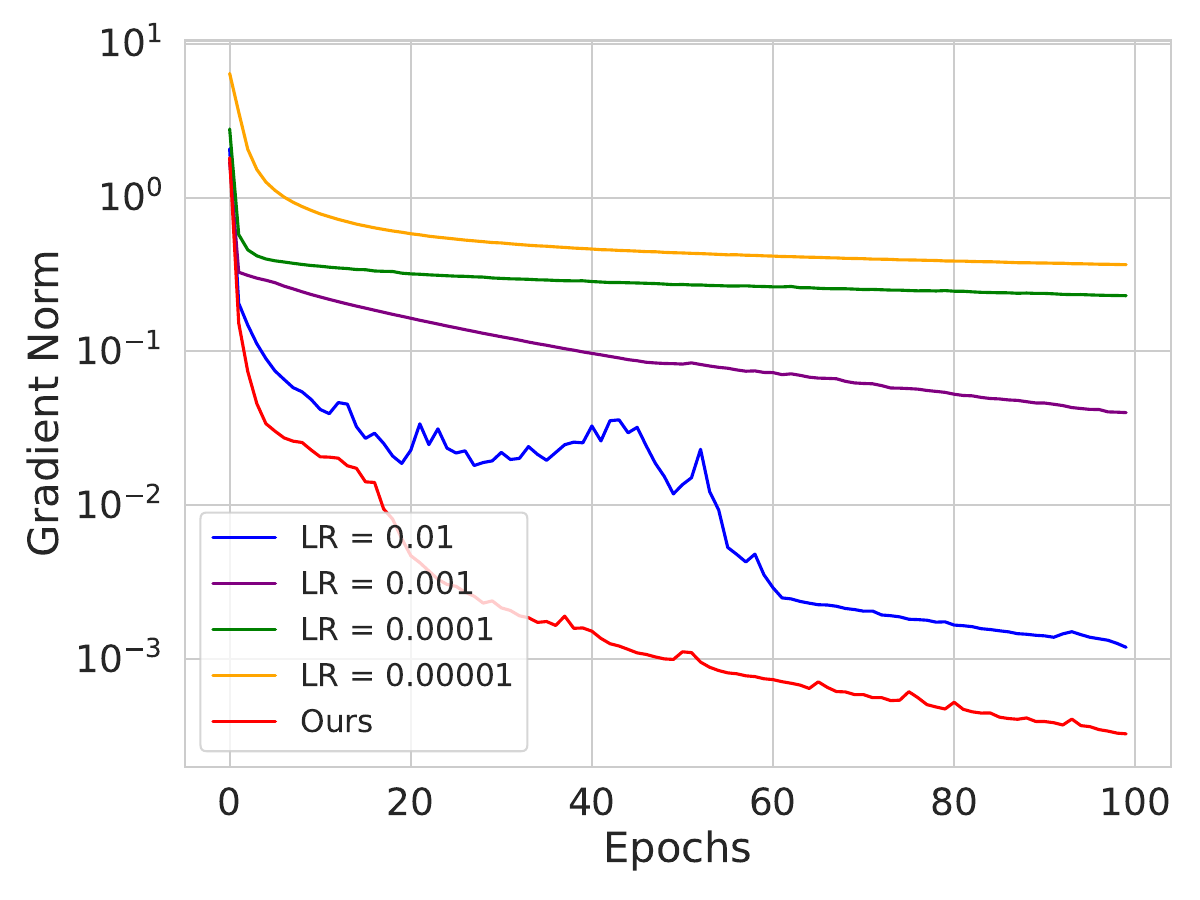}
    \caption{Gradient norm v/s Epochs}
  \end{subfigure}
  \hfill
  \begin{subfigure}[b]{0.31\textwidth}
    \includegraphics[width=\textwidth]{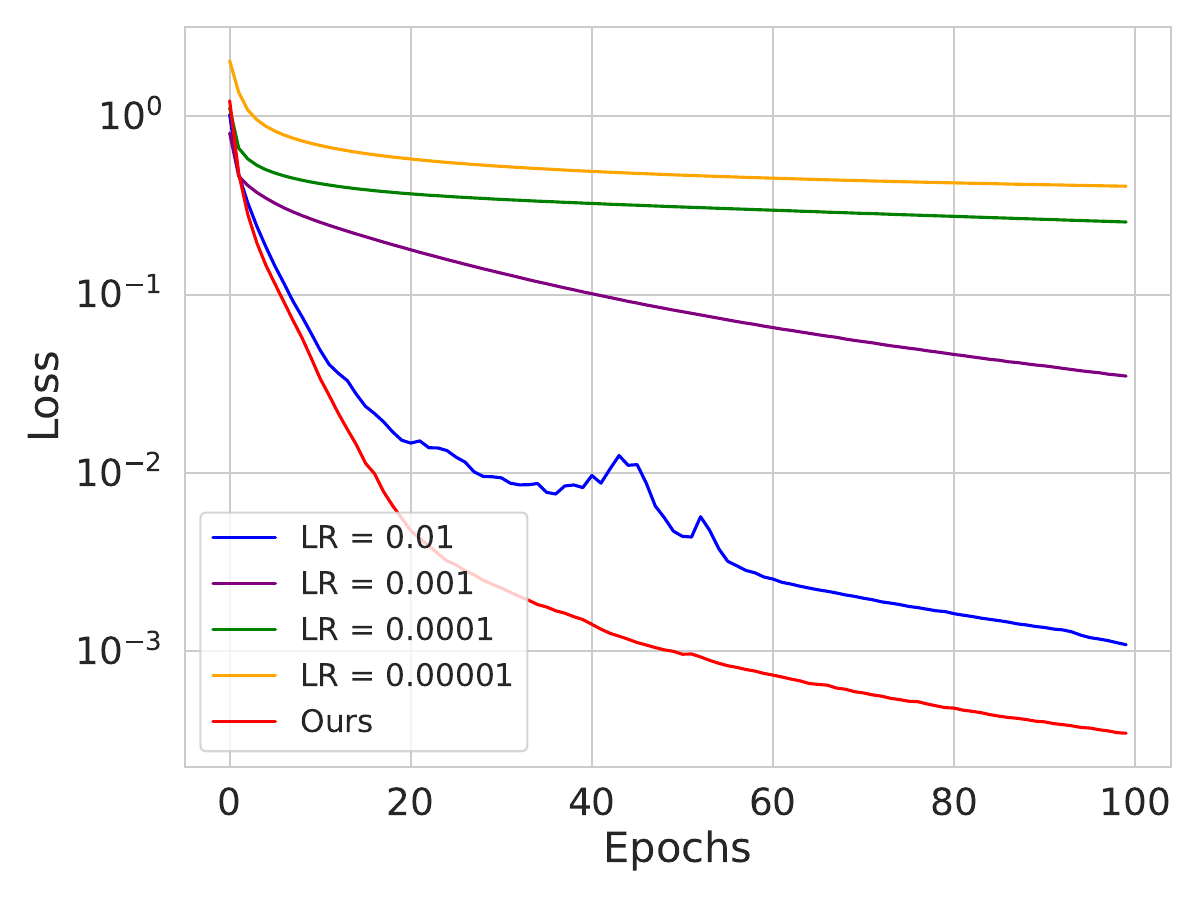}
    \caption{Training loss v/s Epochs}
  \end{subfigure}
  \hfill
  \begin{subfigure}[b]{0.31\textwidth}
    \includegraphics[width=\textwidth]{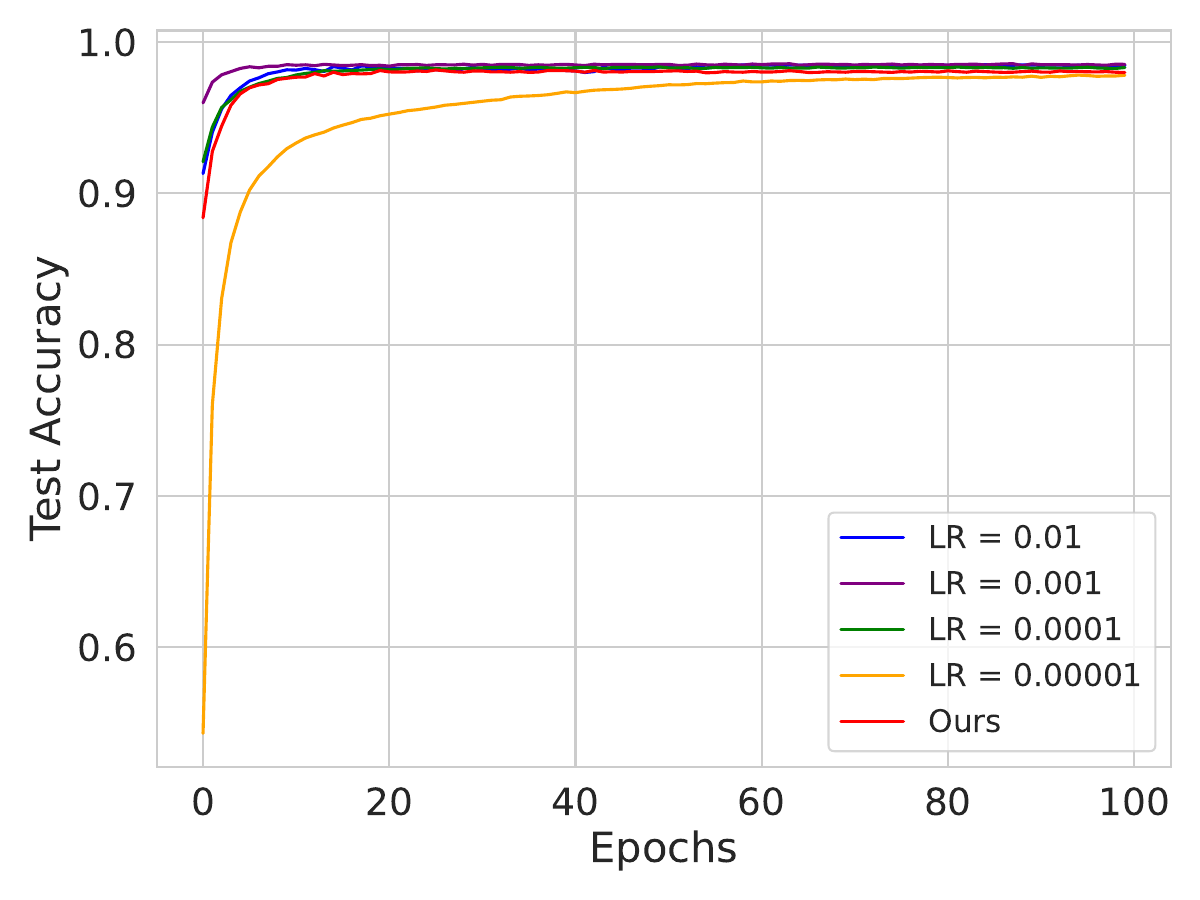}
    \caption{Validation acc. v/s Epochs}
  \end{subfigure}
  \caption{\textbf{Comparision with constant step sizes}: Mini-batch experiments on a 5 layer network with 1000 nodes in each layer, trained on MNIST.}
\label{fig:5__1000_mini_step}
\end{figure}
\begin{figure}[htbp]
  \centering
  \begin{subfigure}[b]{0.31\textwidth}
    \includegraphics[width=\textwidth]{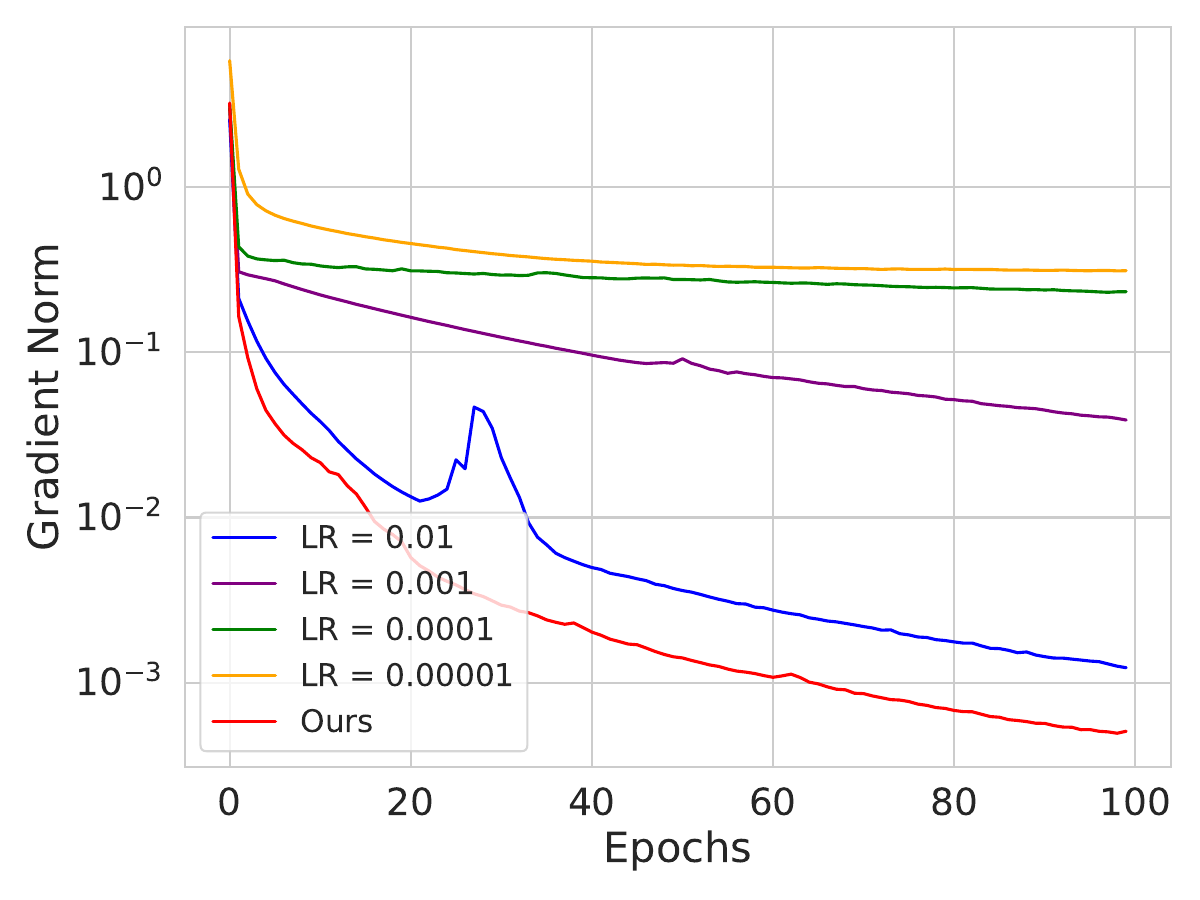}
    \caption{Gradient norm v/s Epochs}
  \end{subfigure}
  \hfill
  \begin{subfigure}[b]{0.31\textwidth}
    \includegraphics[width=\textwidth]{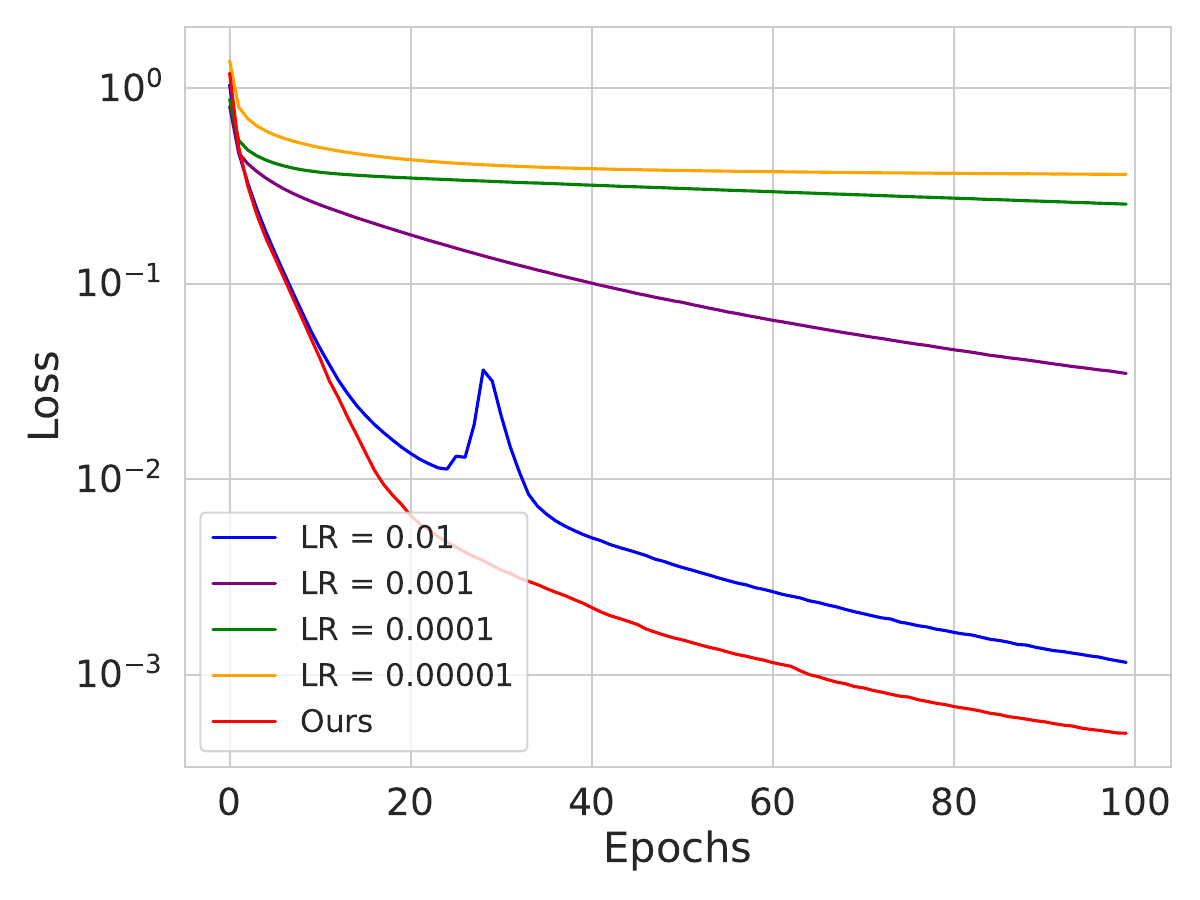}
    \caption{Training loss v/s Epochs}
  \end{subfigure}
  \hfill
  \begin{subfigure}[b]{0.31\textwidth}
    \includegraphics[width=\textwidth]{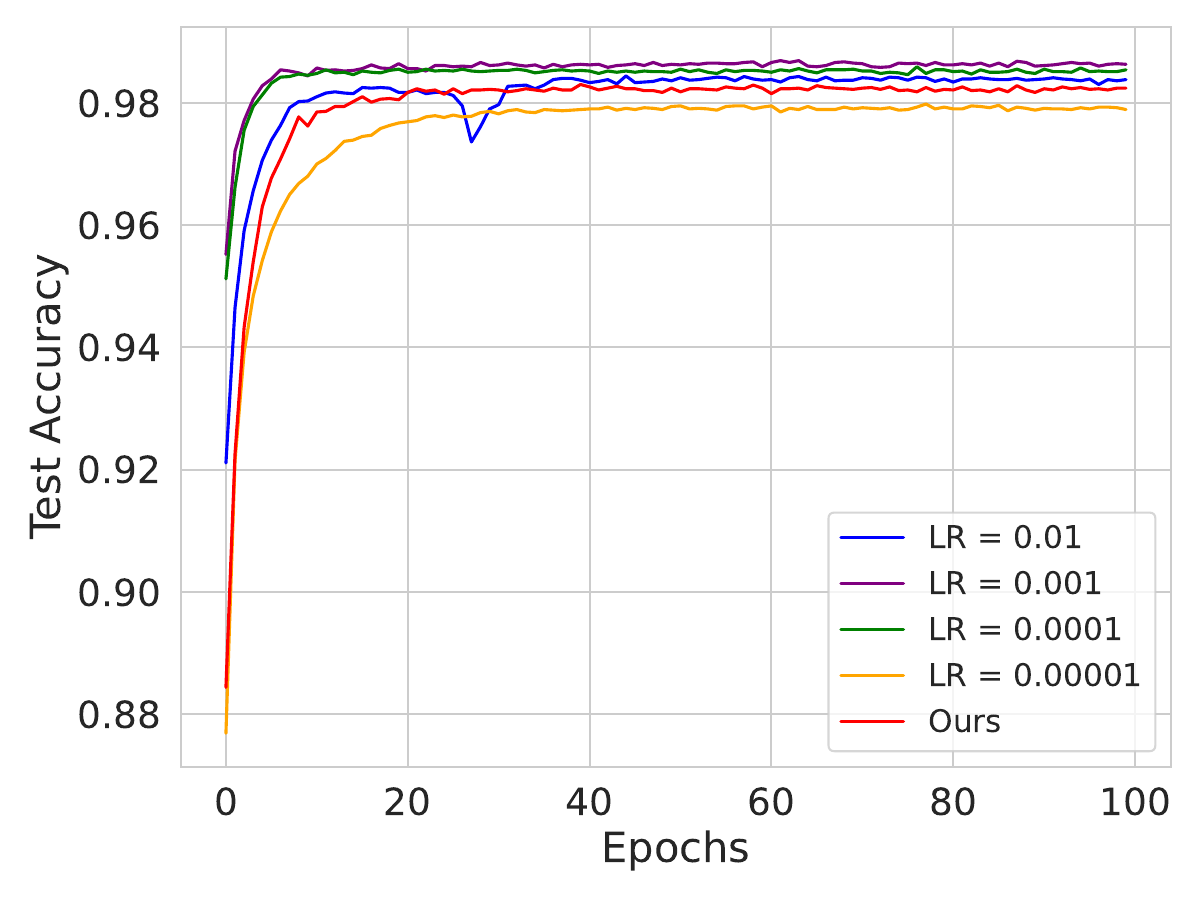}
    \caption{Validation acc. v/s Epochs}
  \end{subfigure}
  \caption{\textbf{Comparision with constant step sizes}: Mini-batch experiments on a 5 layer network with 3000 nodes in each layer, trained on MNIST.}
\label{fig:5__3000_mini_step}
\end{figure}
\newpage

Based on the preceding figures, it's evident that our chosen learning rate effectively reduces the gradient norm as compared to the series of constant learning rates and other learning rate schedulers in mini-batch setup, leading to a commendable validation accuracy. The outcomes remain consistent across various fully connected layer architectures. It achieves at least $10^{-1}$ less value for gradient norm with the best-performing learning rate scheduler and constant learning rate.
\\

\textbf{C.4 Increasing batch-size experiments.}\\

In this section, we conduct a series of experiments with various batch-size $\{256, 512, 1024, 2048\}$ on different depths and widths of linear layers using our learning rate. We also perform this experiment with CNN's like VGG-9 network on CIFAR-10, LeNet on MNIST and MobileNet on CIFAR-10. This experiment aims to show that, our learning rate with a larger batch size reduces the gradient norm more effectively. This section serves as a continuation for Section~\ref{sec:mb_exp}.

\begin{figure}[htbp]
  \centering
  \begin{subfigure}[b]{0.31\textwidth}
    \includegraphics[width=\textwidth]{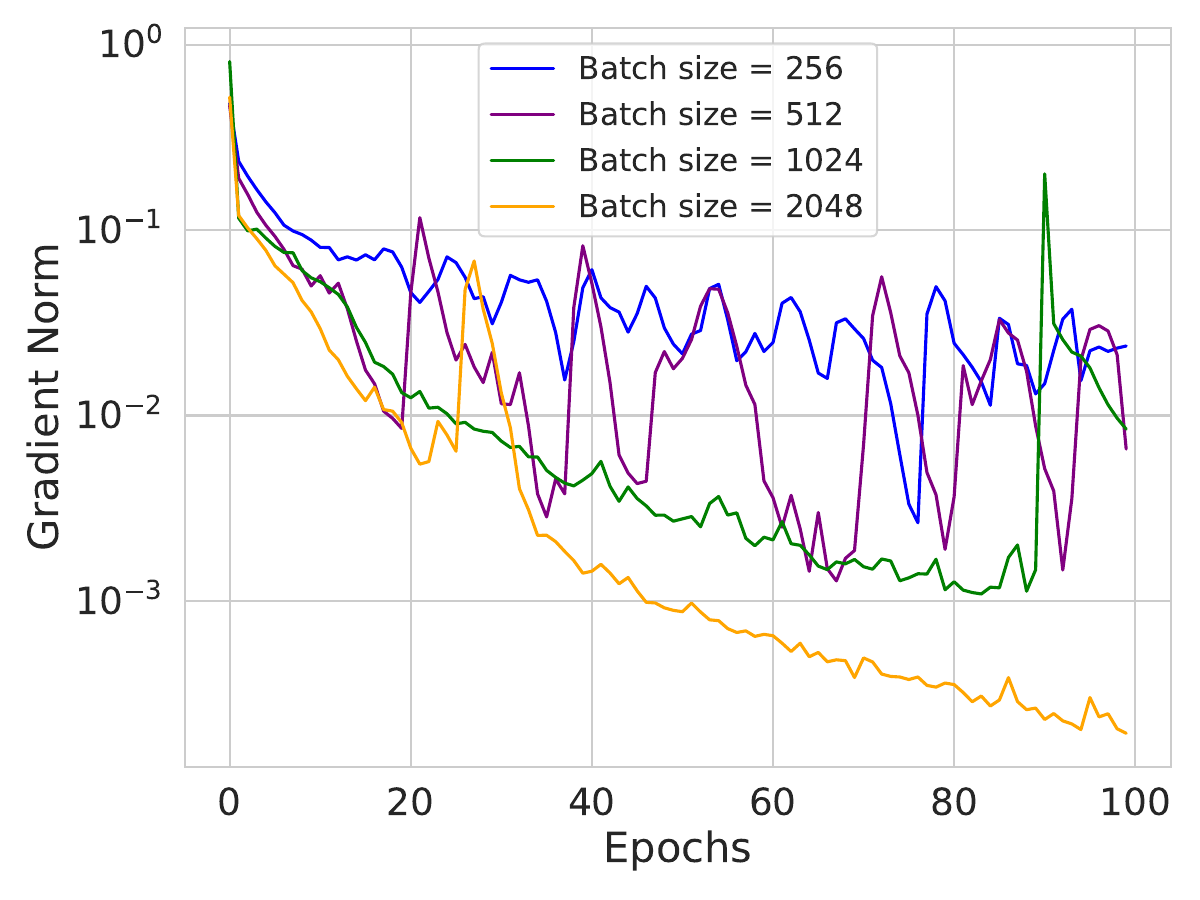}
    \caption{1 layer and 300 nodes.}
  \end{subfigure}
  \hfill
  \begin{subfigure}[b]{0.31\textwidth}
    \includegraphics[width=\textwidth]{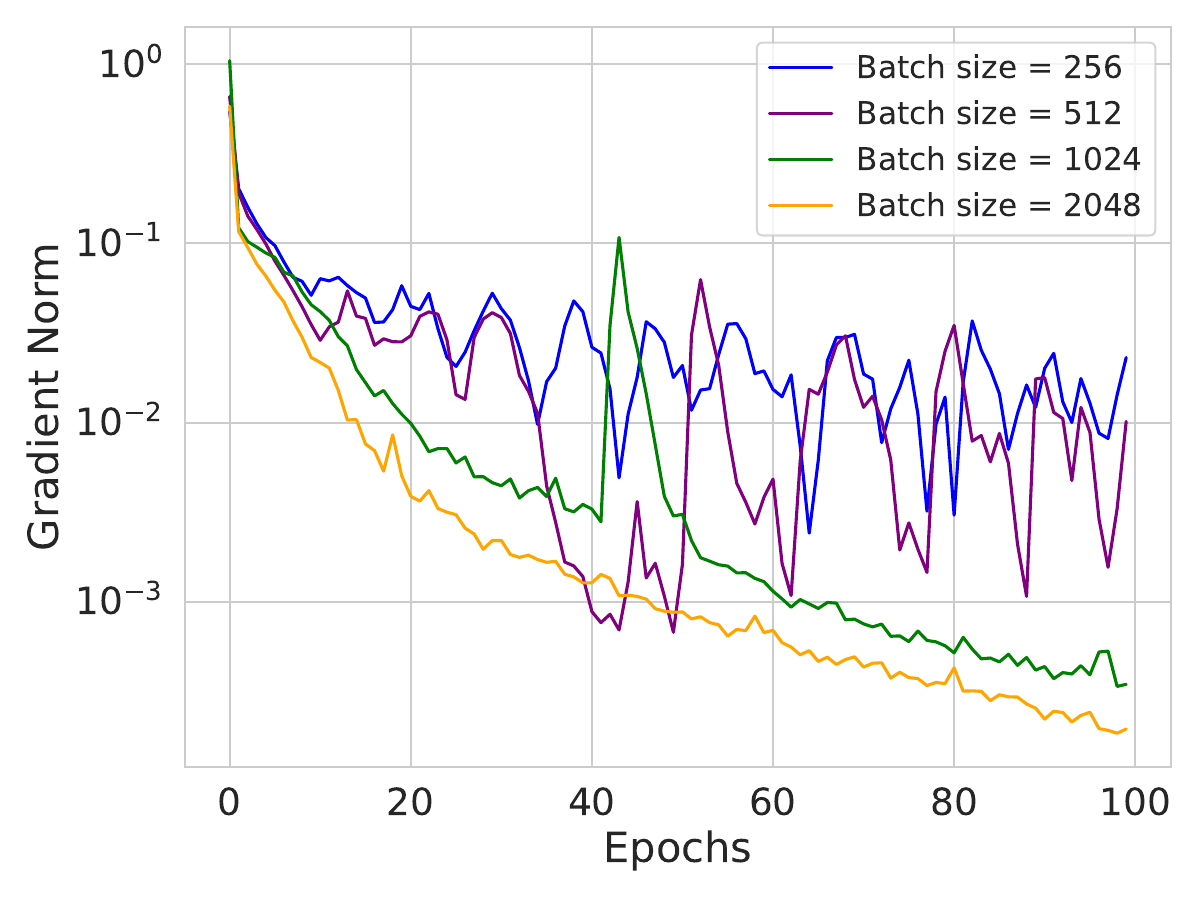}
    \caption{1 layer and 1000 nodes.}
  \end{subfigure}
  \hfill
  \begin{subfigure}[b]{0.31\textwidth}
    \includegraphics[width=\textwidth]{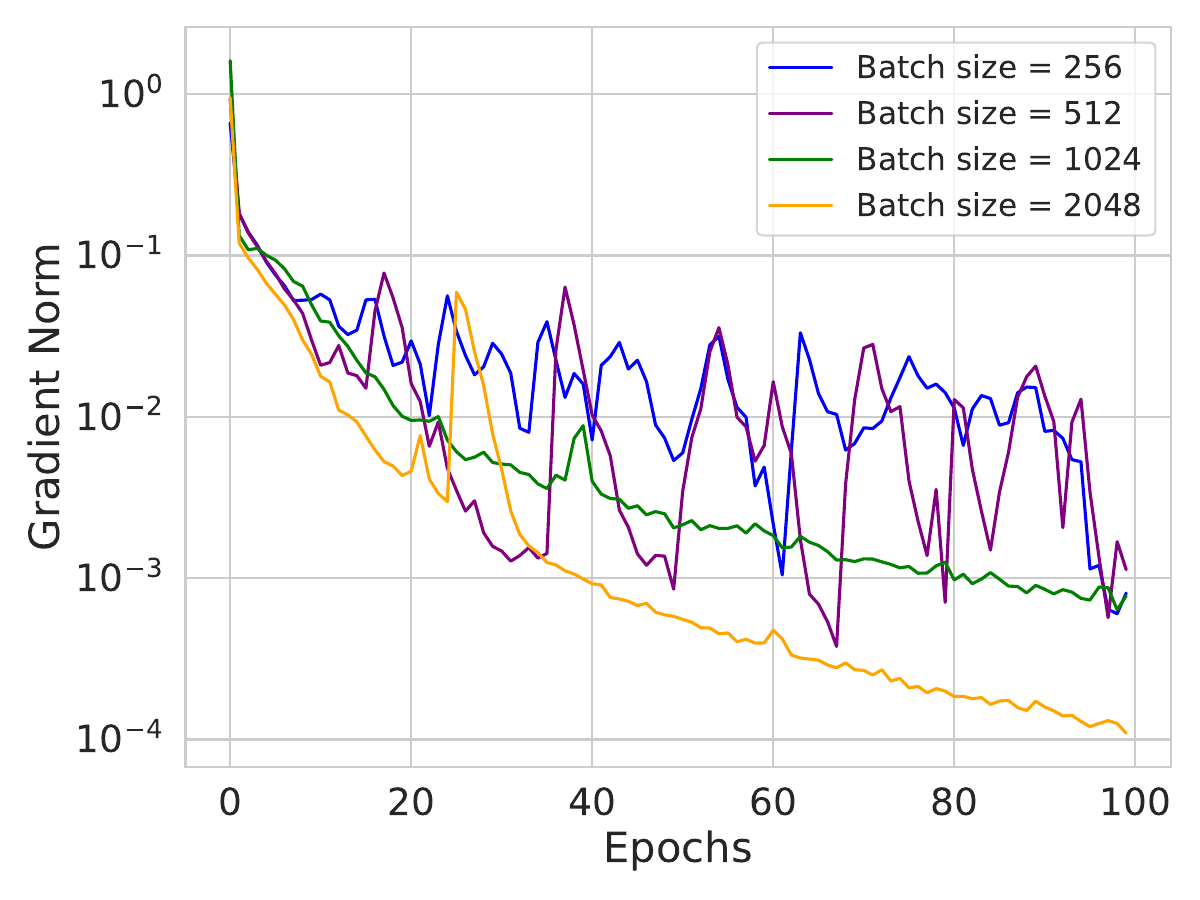}
    \caption{1 layer and 3000 nodes.}
  \end{subfigure}
  \caption{Gradient norm v/s epochs plot with various batch sizes on a fully connected network with a single layer, trained on MNIST.}
\label{fig:1__layer_batch_size}
\end{figure}
\begin{figure}[htbp]
  \centering
  \begin{subfigure}[b]{0.31\textwidth}
    \includegraphics[width=\textwidth]{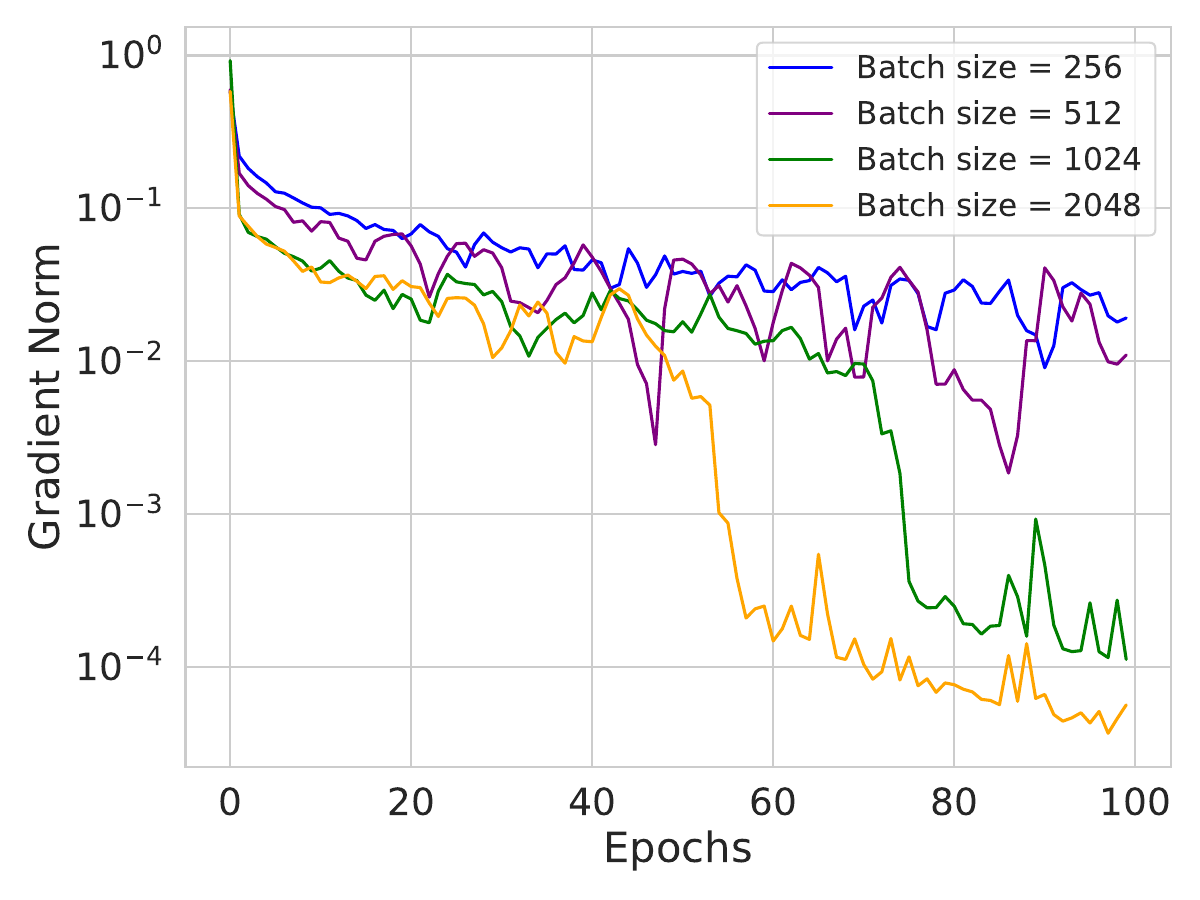}
    \caption{3 layer and 300 nodes.}
  \end{subfigure}
  \hfill
  \begin{subfigure}[b]{0.31\textwidth}
    \includegraphics[width=\textwidth]{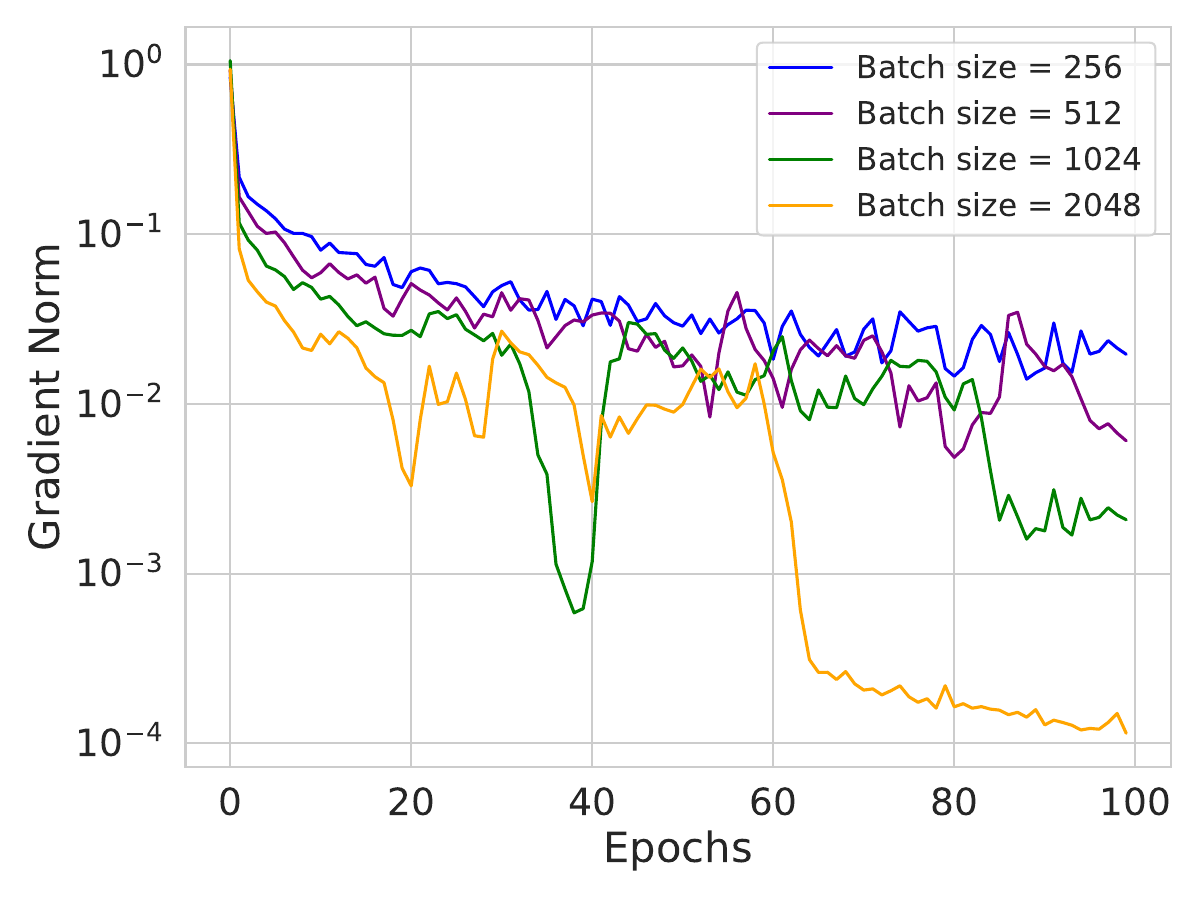}
    \caption{3 layer and 1000 nodes.}
  \end{subfigure}
  \hfill
  \begin{subfigure}[b]{0.31\textwidth}
    \includegraphics[width=\textwidth]{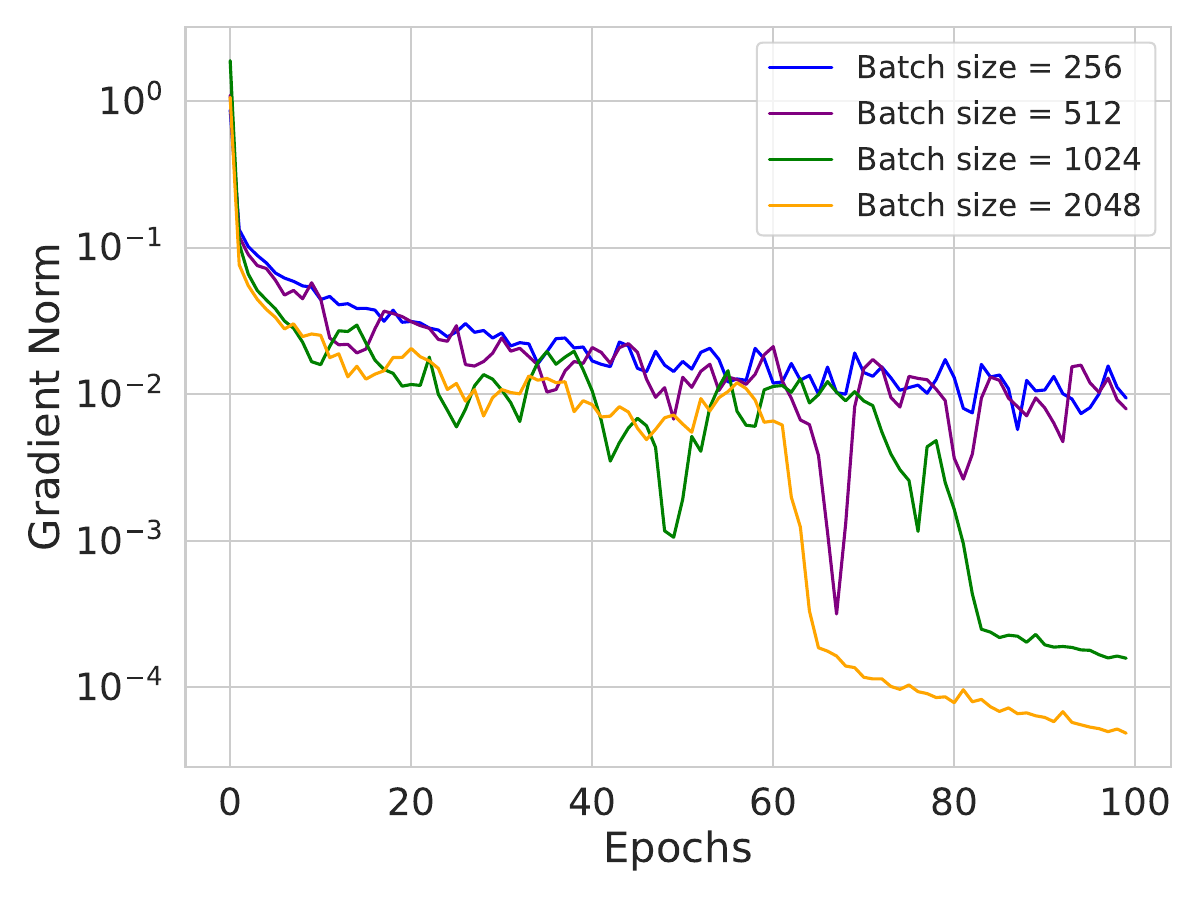}
    \caption{3 layer and 3000 nodes.}
  \end{subfigure}
  \caption{Gradient norm v/s epochs plot with various batch sizes on a fully connected with a 3 layers, trained on MNIST.}
\label{fig:3__layer_batch_size}
\end{figure}
\begin{figure}[htbp]
  \centering
  \begin{subfigure}[b]{0.31\textwidth}
    \includegraphics[width=\textwidth]{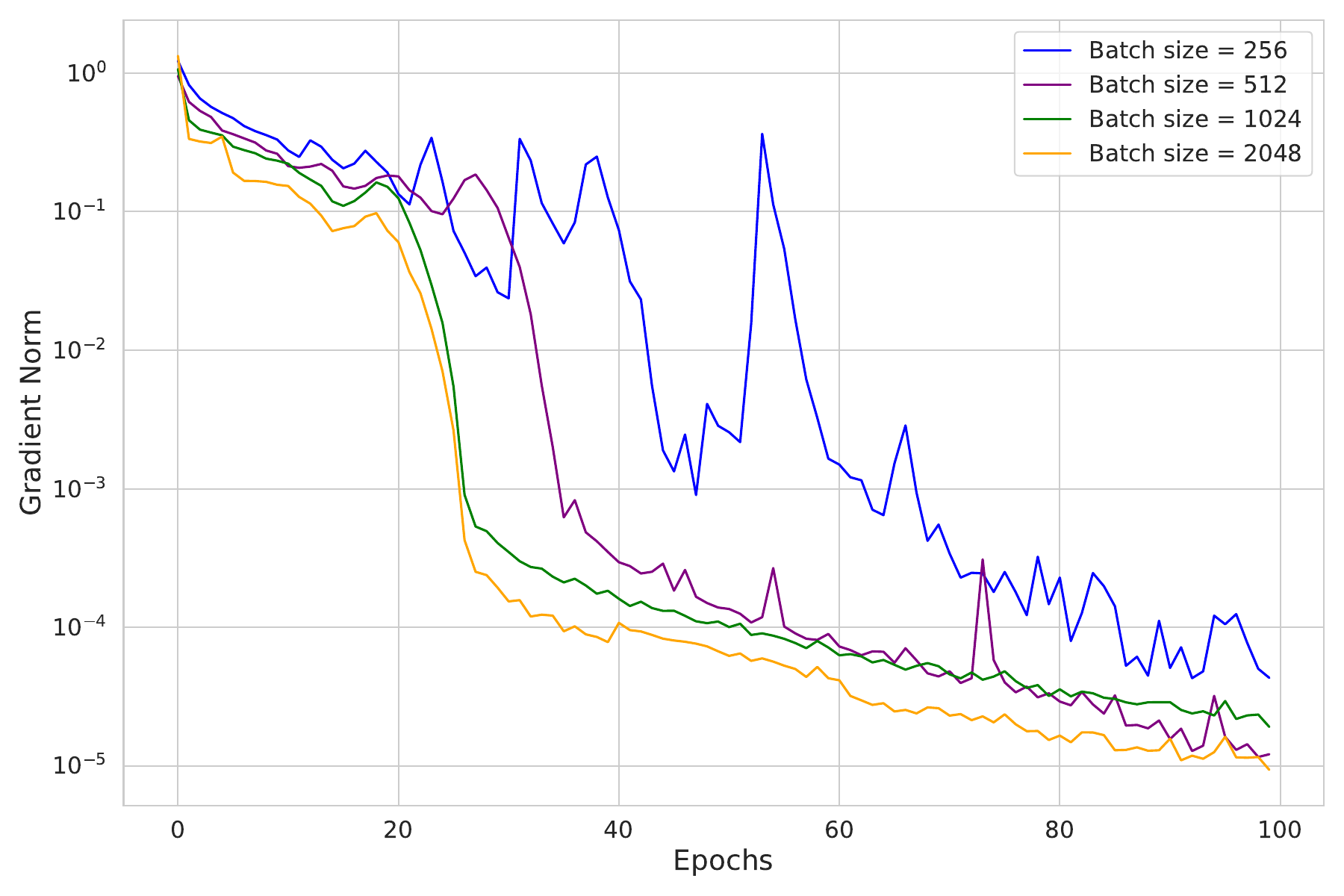}
    \caption{VGG-9 on CIFAR-10.}
  \end{subfigure}
  \hfill
  \begin{subfigure}[b]{0.31\textwidth}
    \includegraphics[width=\textwidth]{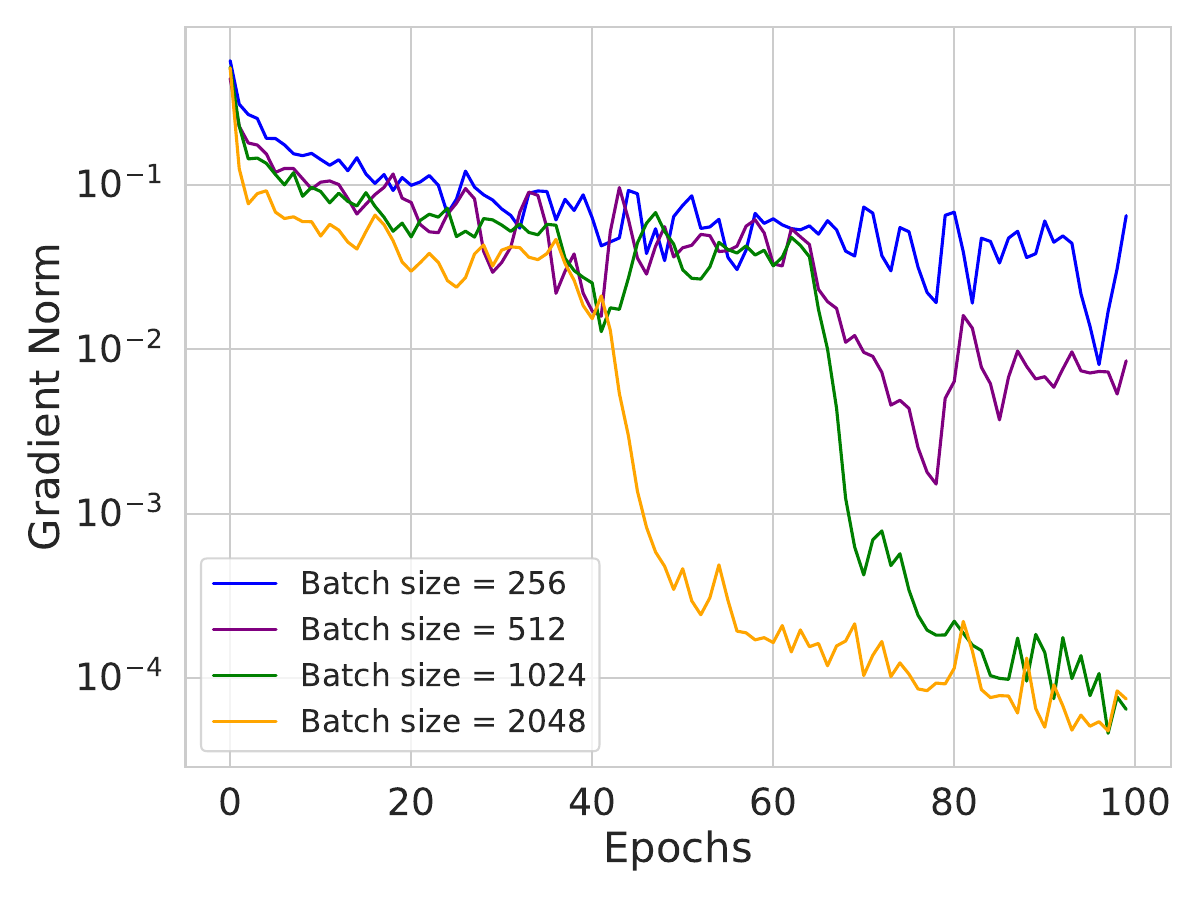}
    \caption{LeNet on MNIST.}
  \end{subfigure}
  \hfill
  \begin{subfigure}[b]{0.31\textwidth}
    \includegraphics[width=\textwidth]{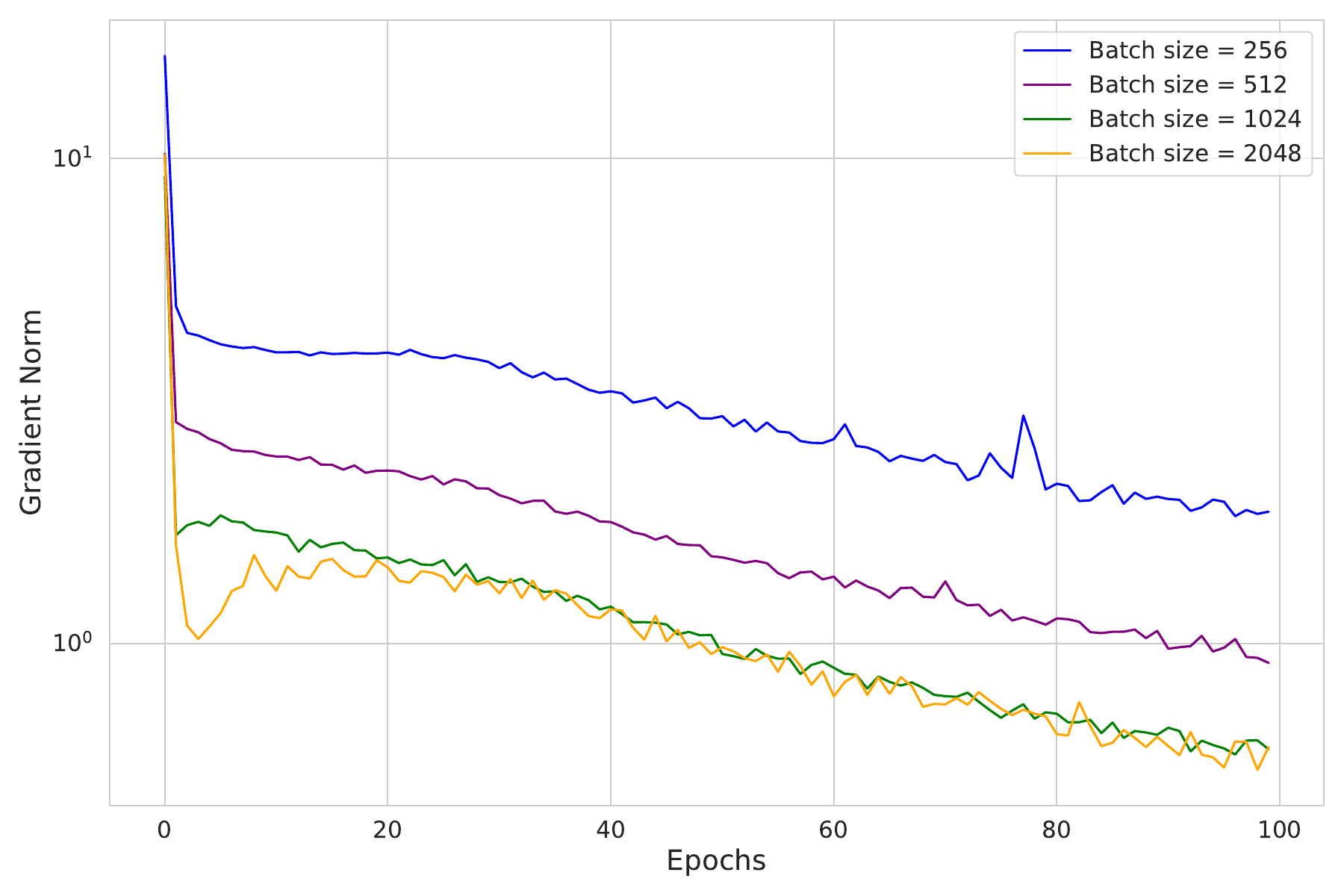}
    \caption{MobileNet on CIFAR-10.}
  \end{subfigure}
  \caption{Gradient norm v/s epochs plot with various batch sizes on VGG-9, LeNet and MobileNET.}
\label{fig:cnn__layer_batch_size}
\end{figure}

It is evident that as the batch size increases, our chosen learning rate effectively diminishes the gradient norm towards zero. This empirical evidence demonstrates that increasing the batch size leads to improved convergence.
\newpage
\textbf{C.5 Effect of Initialization on Our Learning Rate}
\\

In this section, we empirically demonstrated that our learning is not sensitive to various types of initialization strategies in both full-batch and mini-batch setups. Mini batch size of 5,000 is used for MNIST in all experiments in this section. This is a continuation of Section~\ref{sec:dis_exp}.

\begin{figure}
  \centering
  \begin{subfigure}[b]{0.31\textwidth}
    \includegraphics[width=\textwidth]{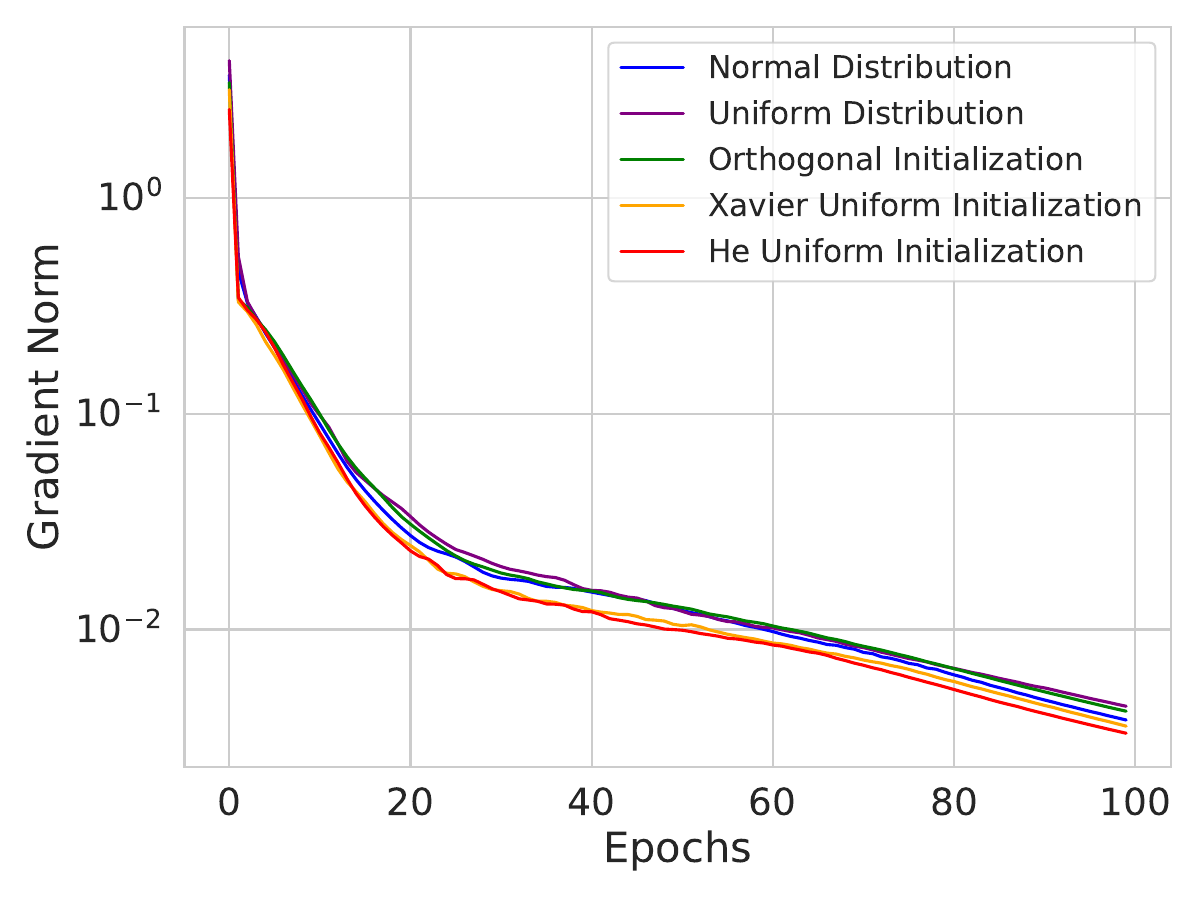}
    \caption{Gradient norm v/s Epochs}
  \end{subfigure}
  \hfill
  \begin{subfigure}[b]{0.31\textwidth}
    \includegraphics[width=\textwidth]{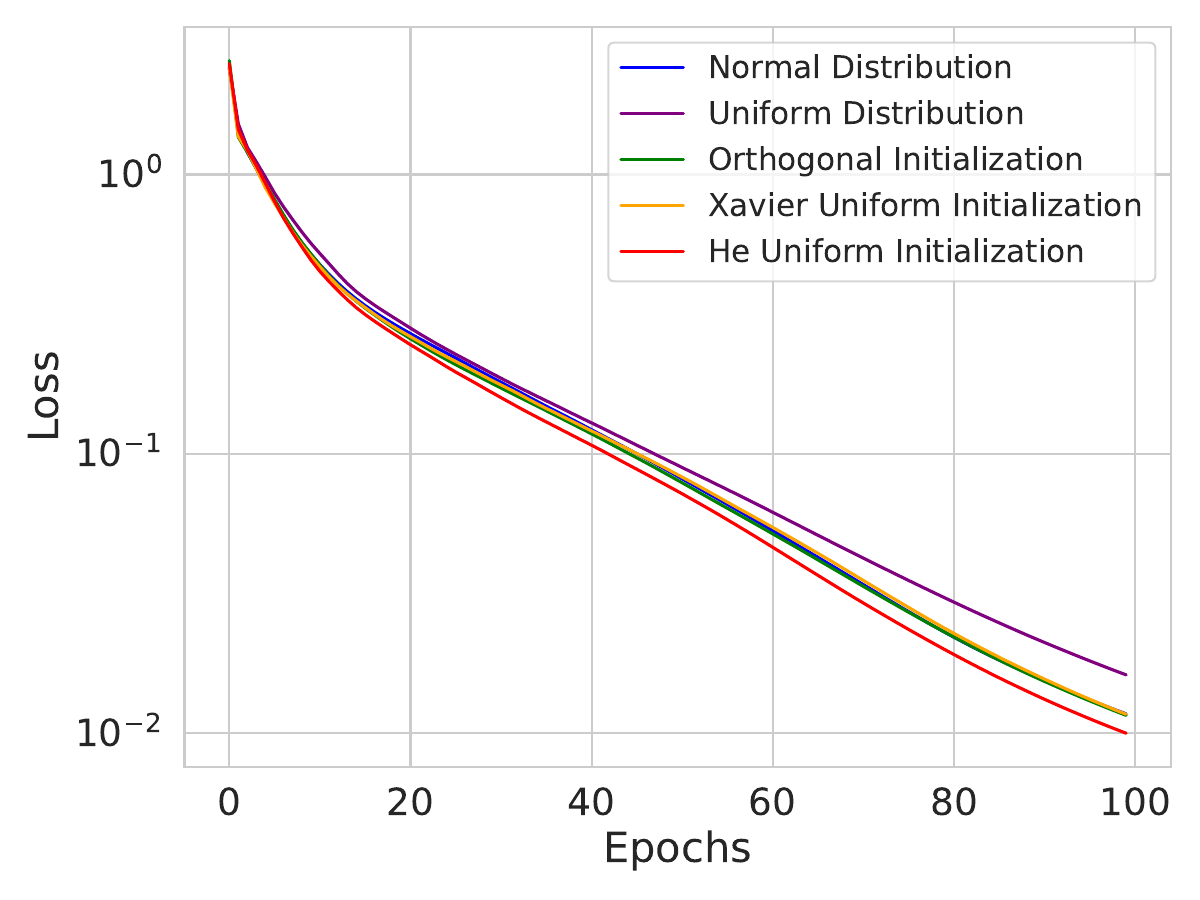}
    \caption{Training loss v/s Epochs}
  \end{subfigure}
  \hfill
  \begin{subfigure}[b]{0.31\textwidth}
    \includegraphics[width=\textwidth]{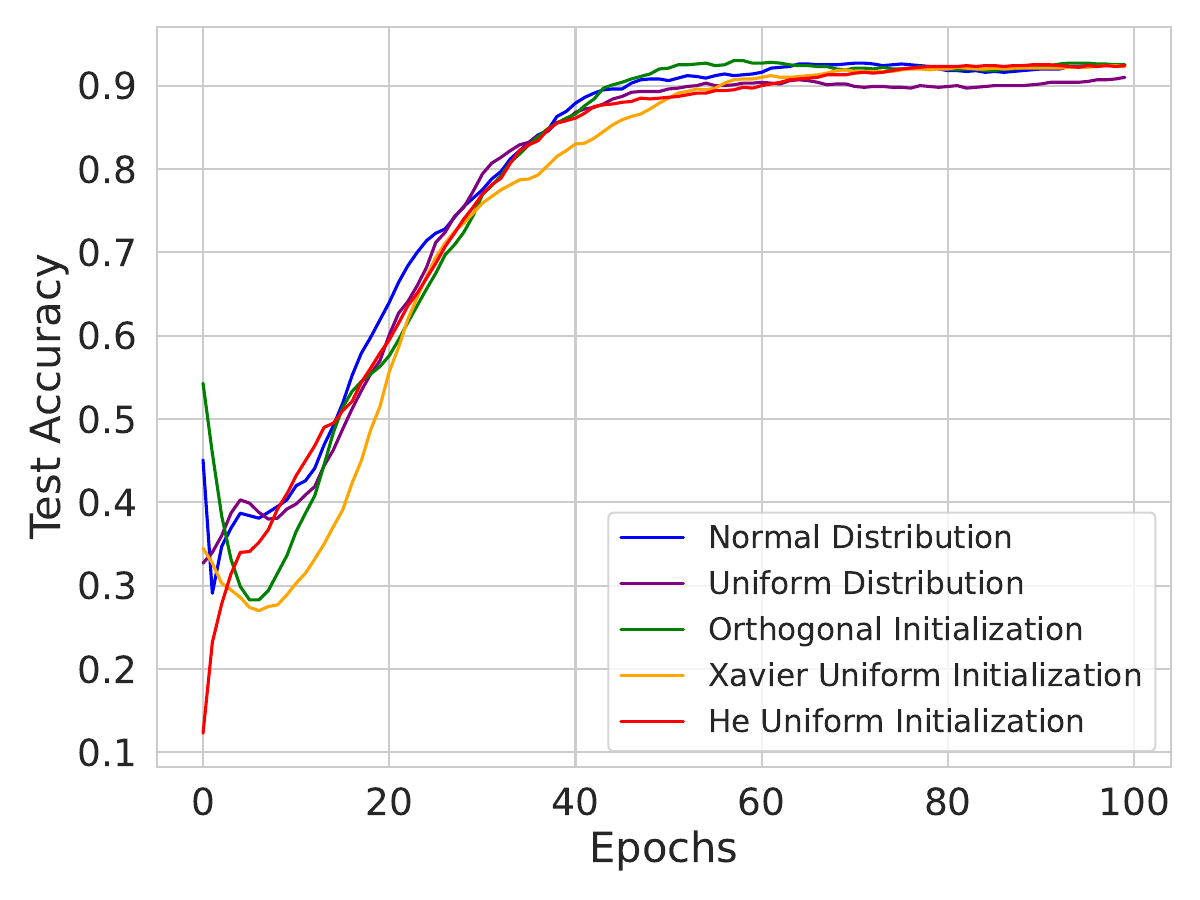}
    \caption{Validation acc. v/s Epochs}
  \end{subfigure}
  \caption{Effect of initialization on our learning rate and model performance. \textbf{Full batch} experiments conducted on a single layer network with 300 nodes on the MNIST.}
\label{fig:fb_distribution_1__300}
\end{figure}
\begin{figure}
  \centering
  \begin{subfigure}[b]{0.31\textwidth}
    \includegraphics[width=\textwidth]{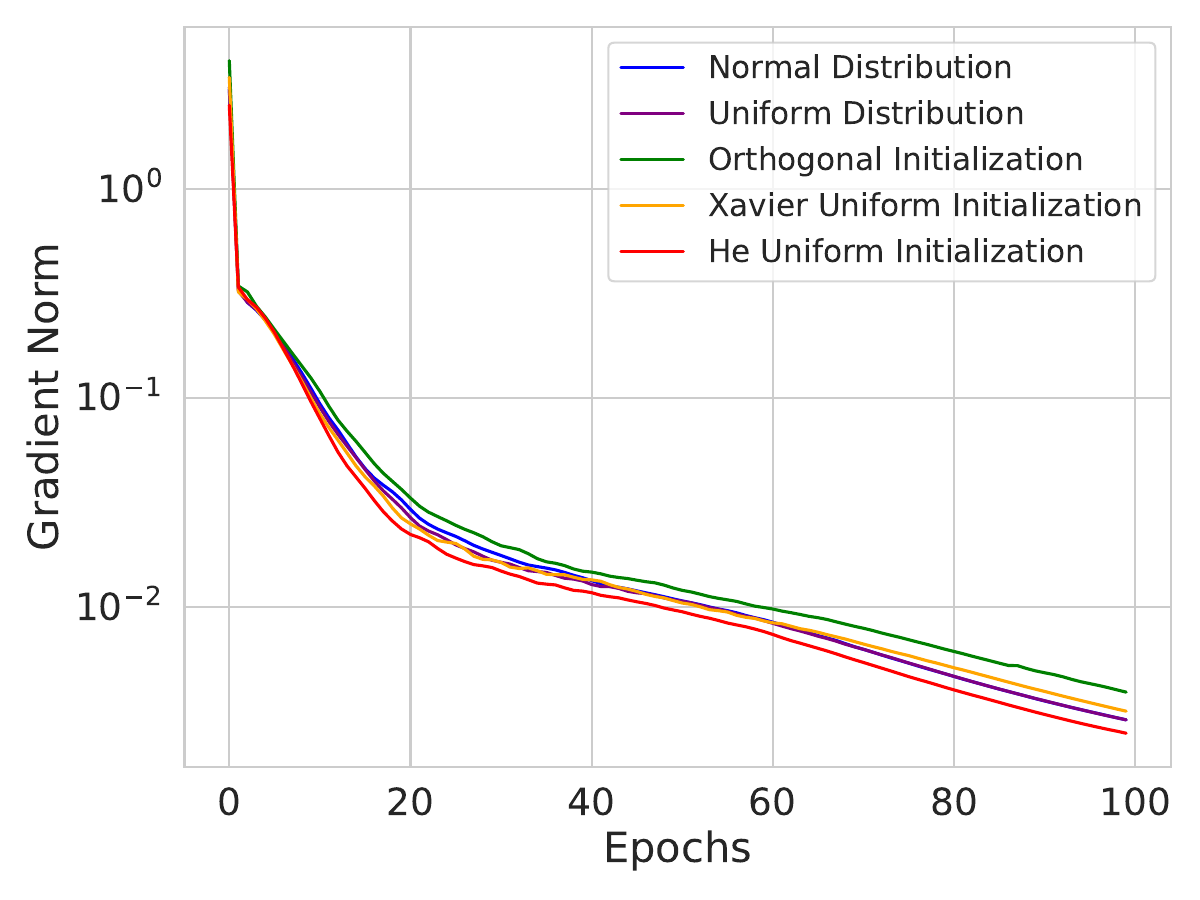}
    \caption{Gradient norm v/s Epochs}
  \end{subfigure}
  \hfill
  \begin{subfigure}[b]{0.31\textwidth}
    \includegraphics[width=\textwidth]{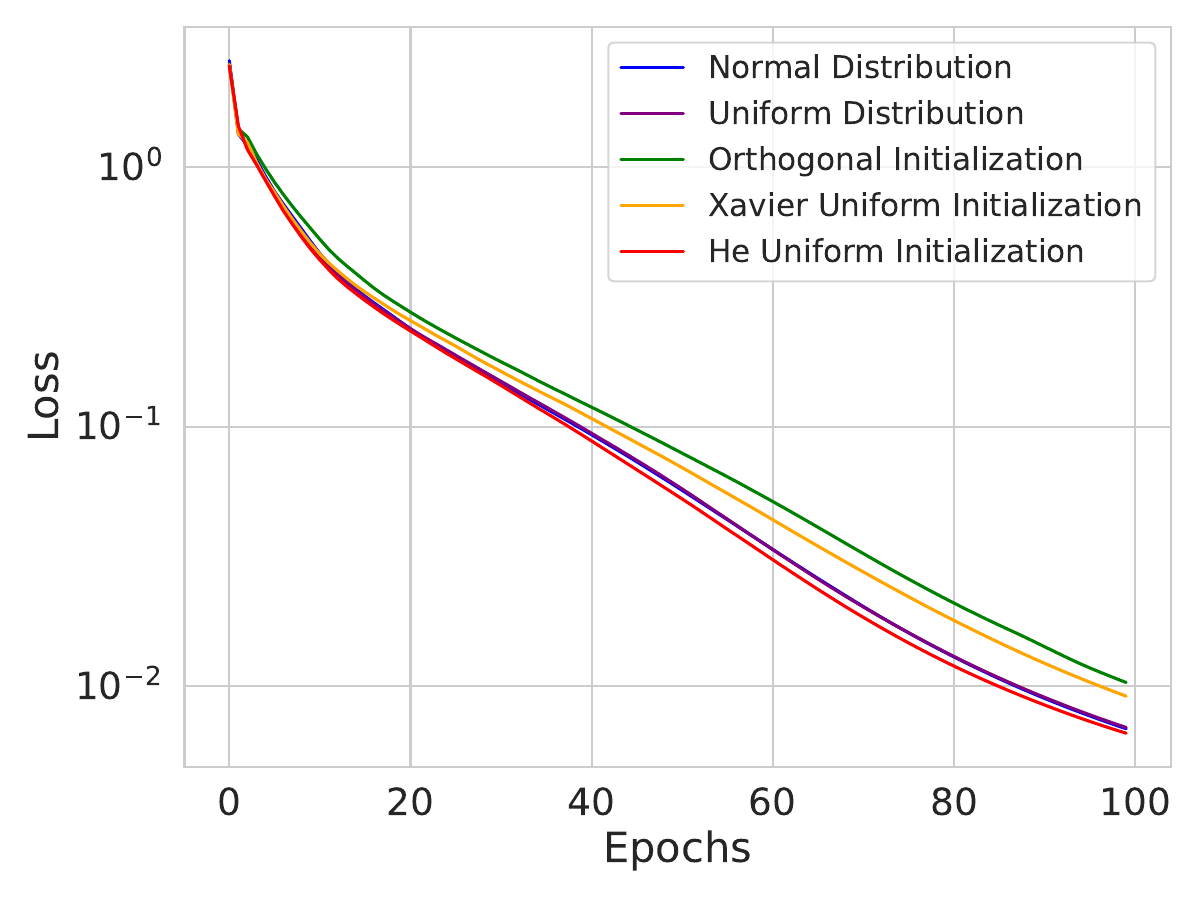}
    \caption{Training loss v/s Epochs}
  \end{subfigure}
  \hfill
  \begin{subfigure}[b]{0.31\textwidth}
    \includegraphics[width=\textwidth]{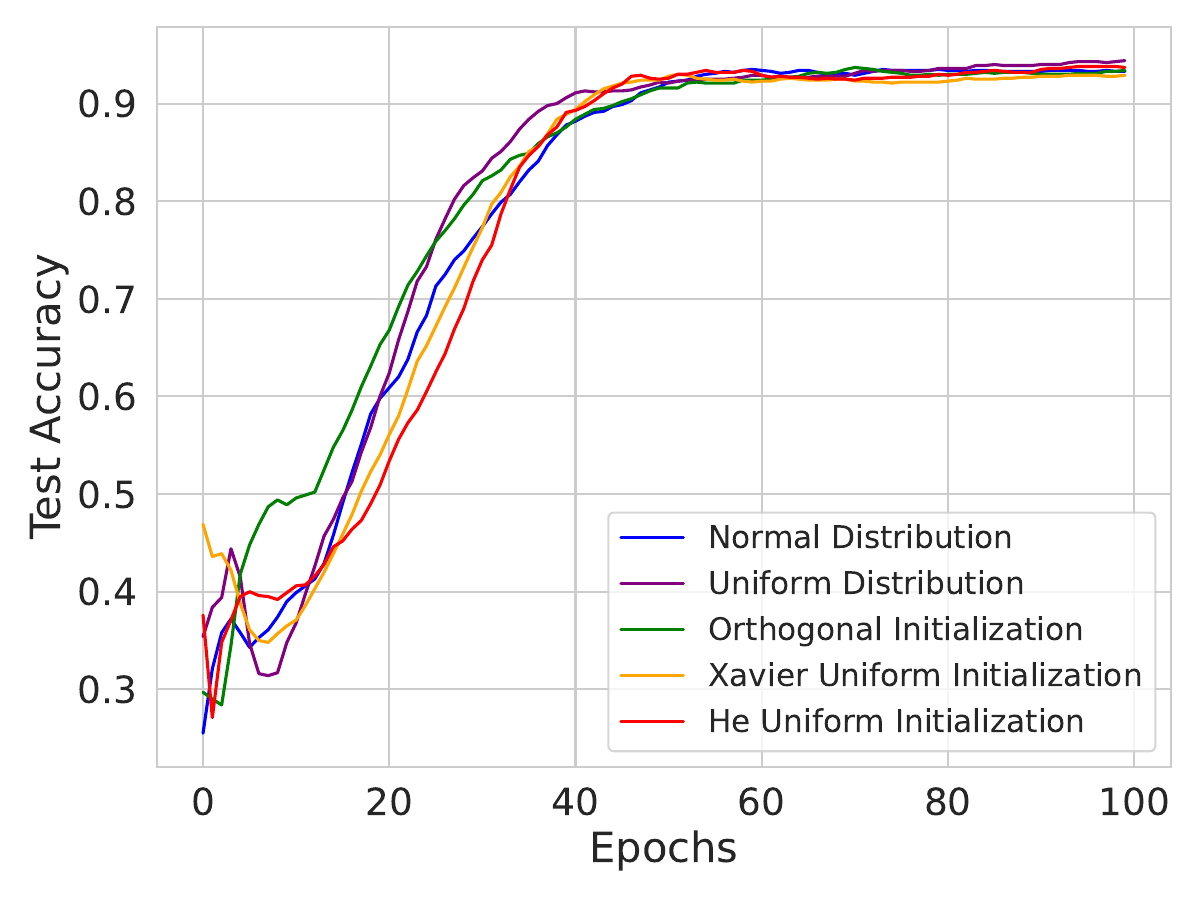}
    \caption{Validation acc. v/s Epochs}
  \end{subfigure}
  \caption{Effect of initialization on our learning rate and model performance. \textbf{Full batch} experiments conducted on a single layer network with 1000 nodes on the MNIST.}
\label{fig:fb_distribution_1_1000}
\end{figure}
\begin{figure}
  \centering
  \begin{subfigure}[b]{0.31\textwidth}
    \includegraphics[width=\textwidth]{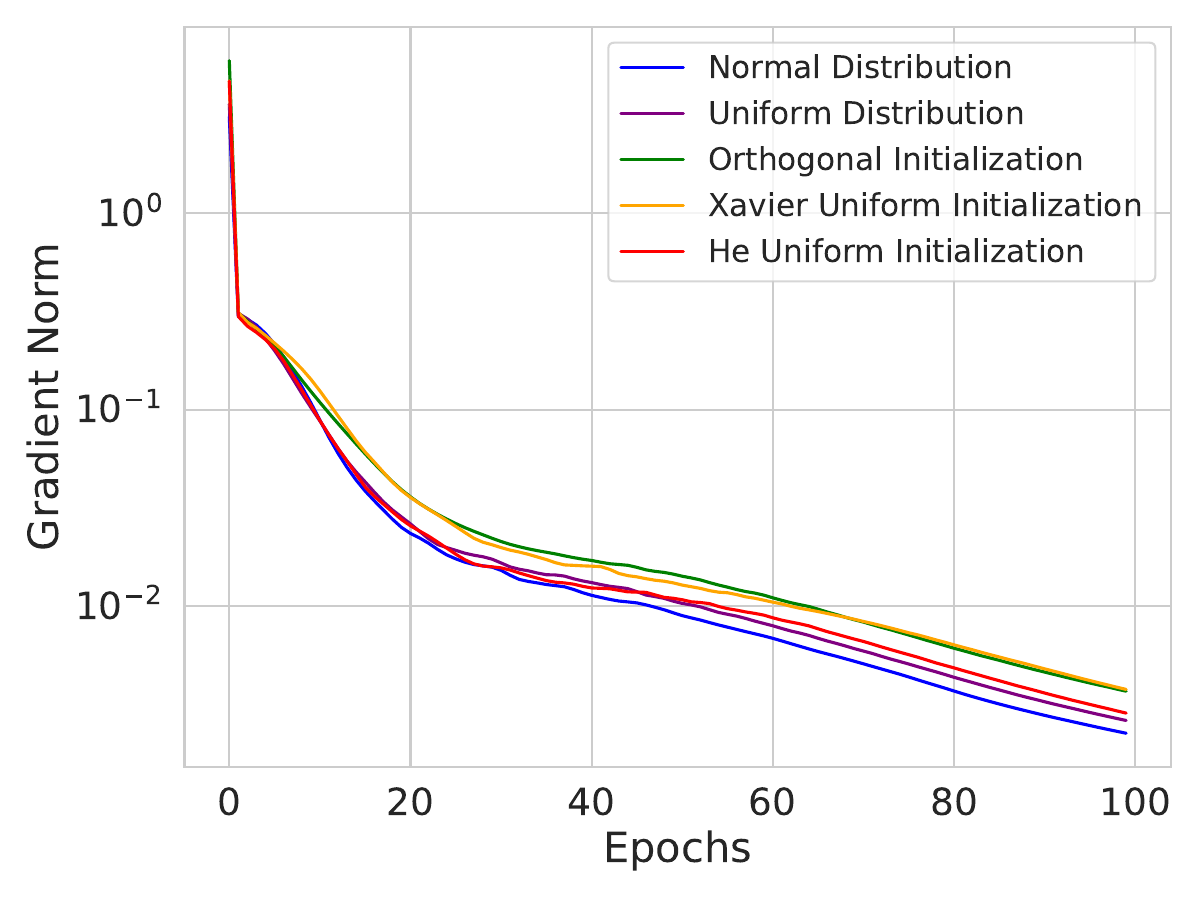}
    \caption{Gradient norm v/s Epochs}
  \end{subfigure}
  \hfill
  \begin{subfigure}[b]{0.31\textwidth}
    \includegraphics[width=\textwidth]{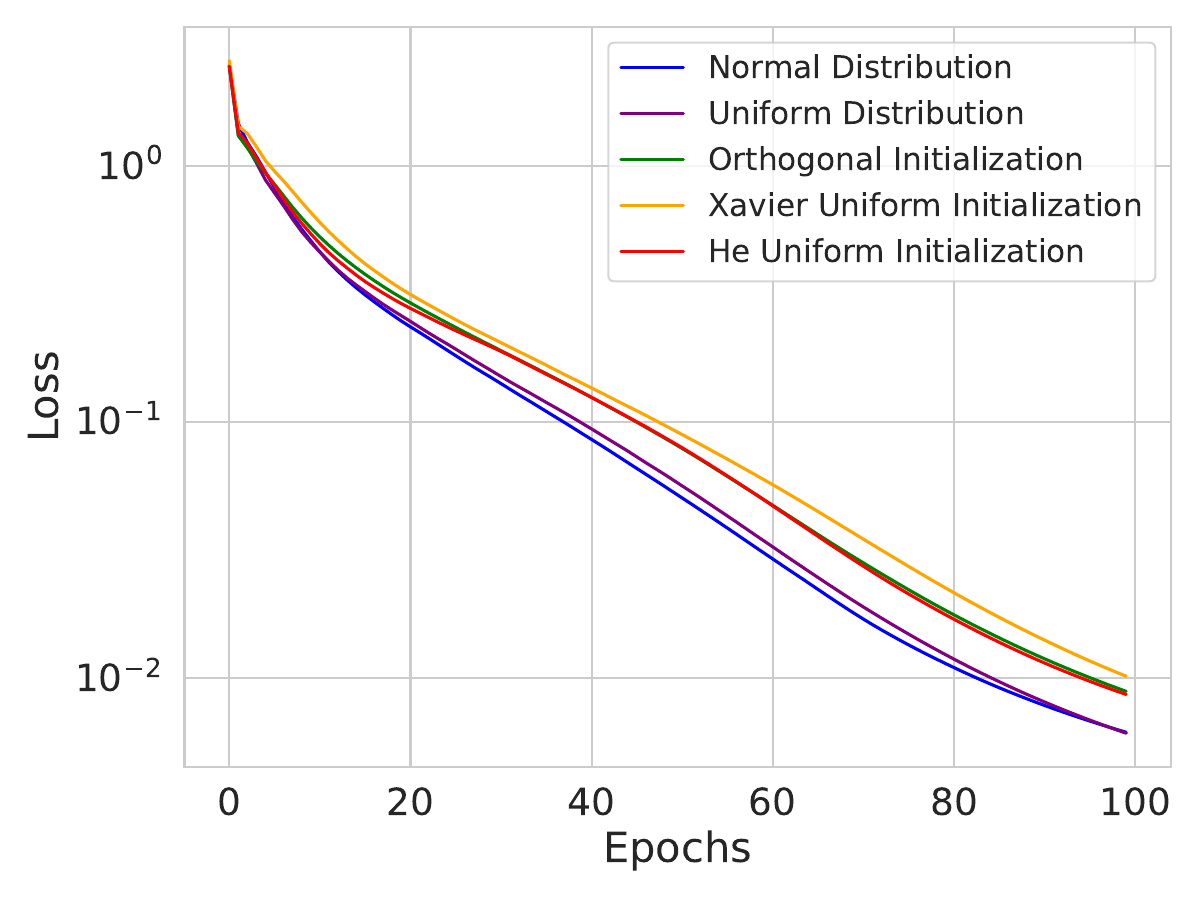}
    \caption{Training loss v/s Epochs}
  \end{subfigure}
  \hfill
  \begin{subfigure}[b]{0.31\textwidth}
    \includegraphics[width=\textwidth]{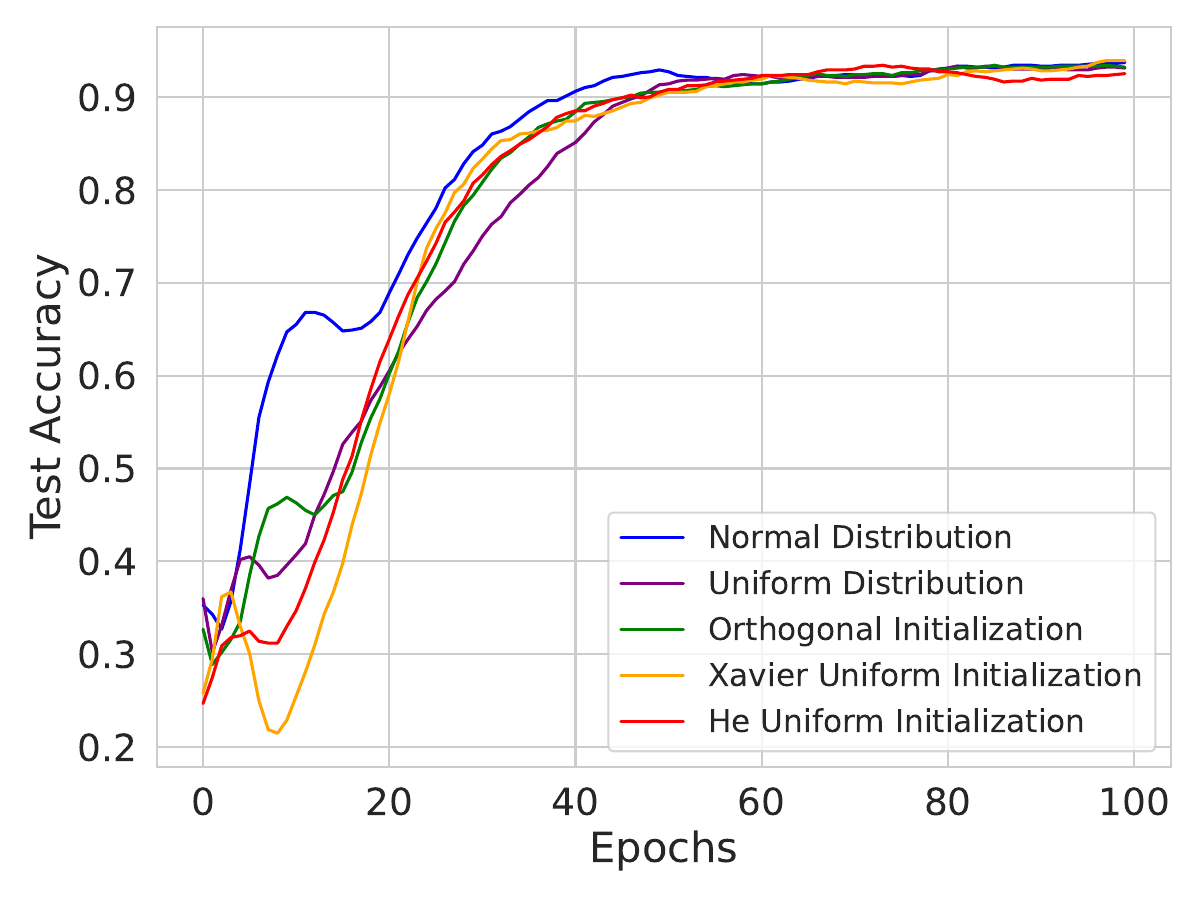}
    \caption{Validation acc. v/s Epochs}
  \end{subfigure}
  \caption{Effect of initialization on our learning rate and model performance. \textbf{Full batch} experiments conducted on a single layer network with 3000 nodes on the MNIST.}
\label{fig:fb_distribution_1__3000}
\end{figure}
\begin{figure}
  \centering
  \begin{subfigure}[b]{0.31\textwidth}
    \includegraphics[width=\textwidth]{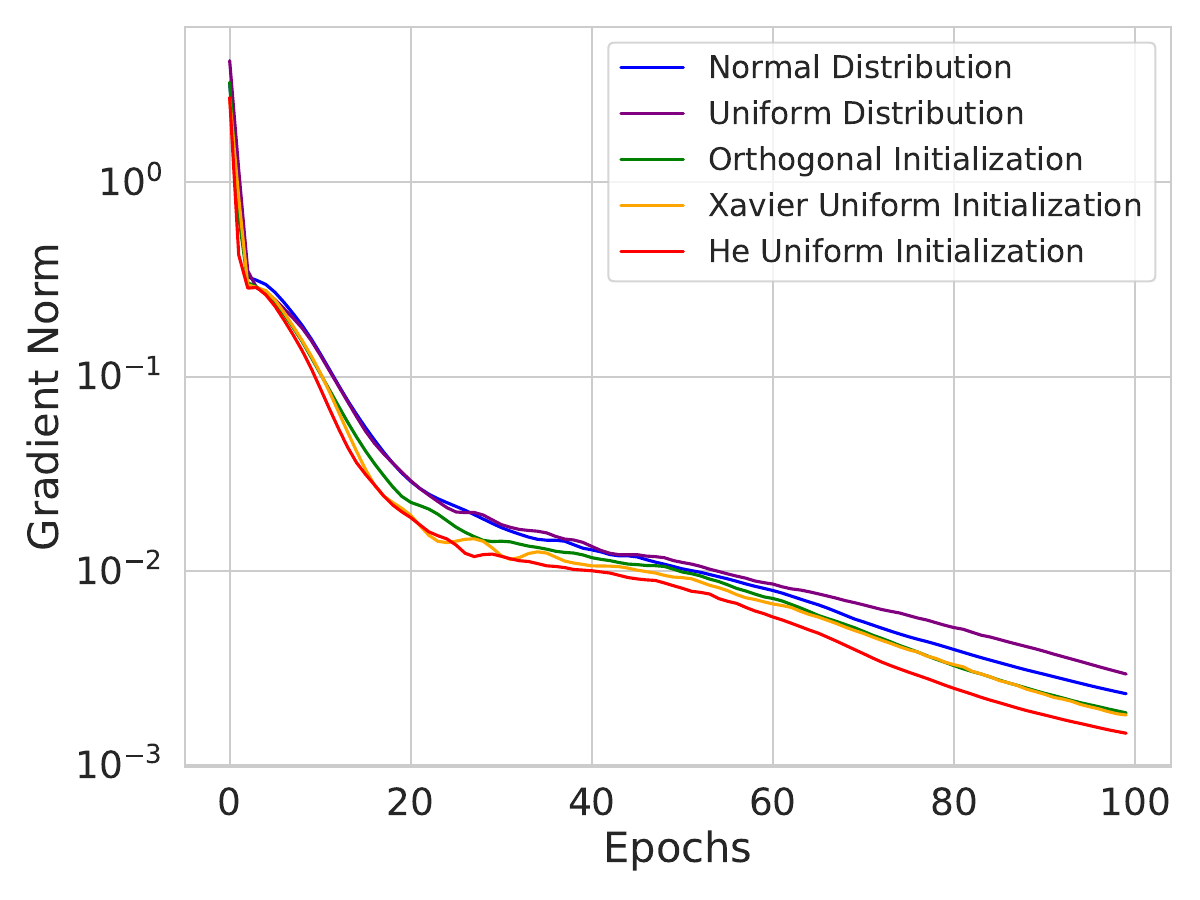}
    \caption{Gradient norm v/s Epochs}
  \end{subfigure}
  \hfill
  \begin{subfigure}[b]{0.31\textwidth}
    \includegraphics[width=\textwidth]{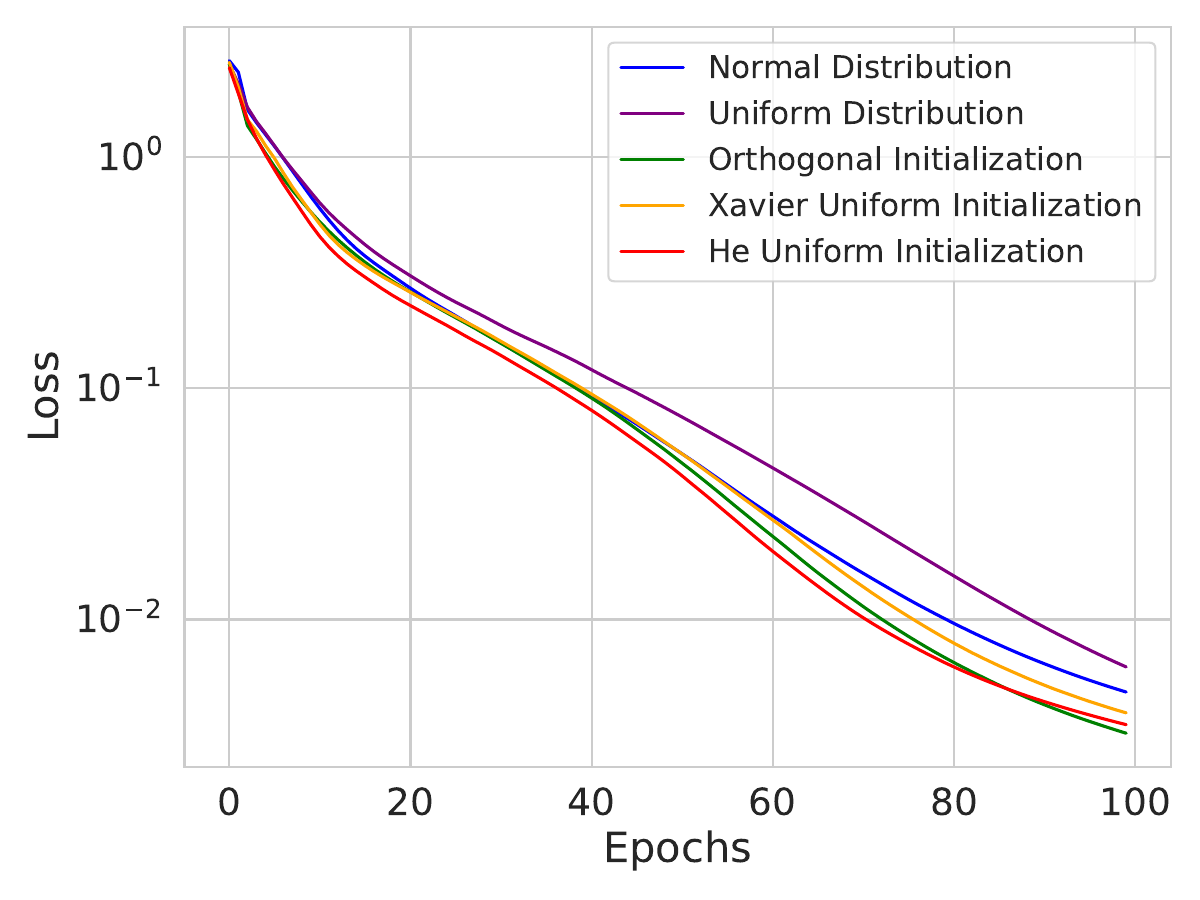}
    \caption{Training loss v/s Epochs}
  \end{subfigure}
  \hfill
  \begin{subfigure}[b]{0.31\textwidth}
    \includegraphics[width=\textwidth]{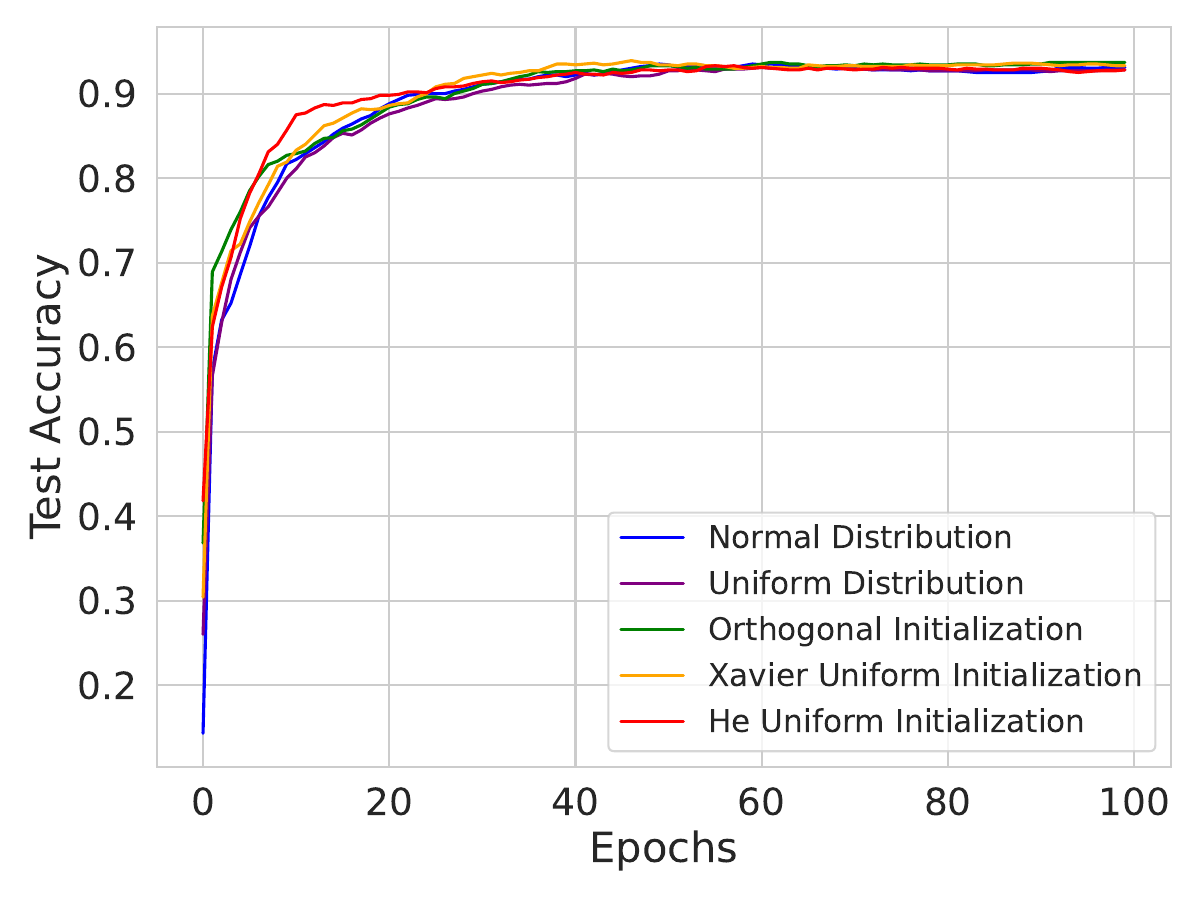}
    \caption{Validation acc. v/s Epochs}
  \end{subfigure}
  \caption{Effect of initialization on our learning rate and model performance. \textbf{Full batch} experiments conducted on a 3 layer network with 100 nodes on the MNIST.}
\label{fig:fb_distribution_3__100}
\end{figure}
\begin{figure}
  \centering
  \begin{subfigure}[b]{0.31\textwidth}
    \includegraphics[width=\textwidth]{results/linear/distributions_exp/full_batch/3__1000/grad_norm_new.pdf}
    \caption{Gradient norm v/s Epochs}
  \end{subfigure}
  \hfill
  \begin{subfigure}[b]{0.31\textwidth}
    \includegraphics[width=\textwidth]{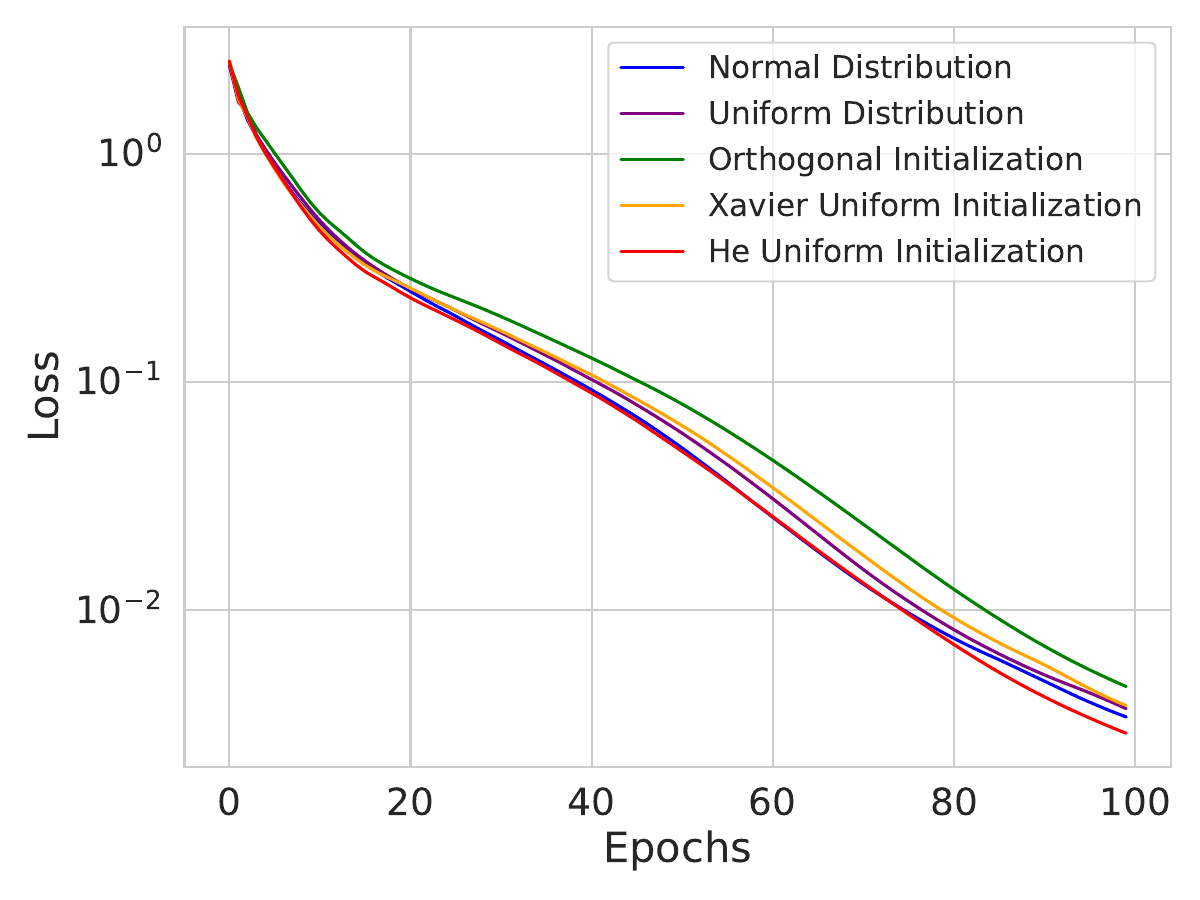}
    \caption{Training loss v/s Epochs}
  \end{subfigure}
  \hfill
  \begin{subfigure}[b]{0.31\textwidth}
    \includegraphics[width=\textwidth]{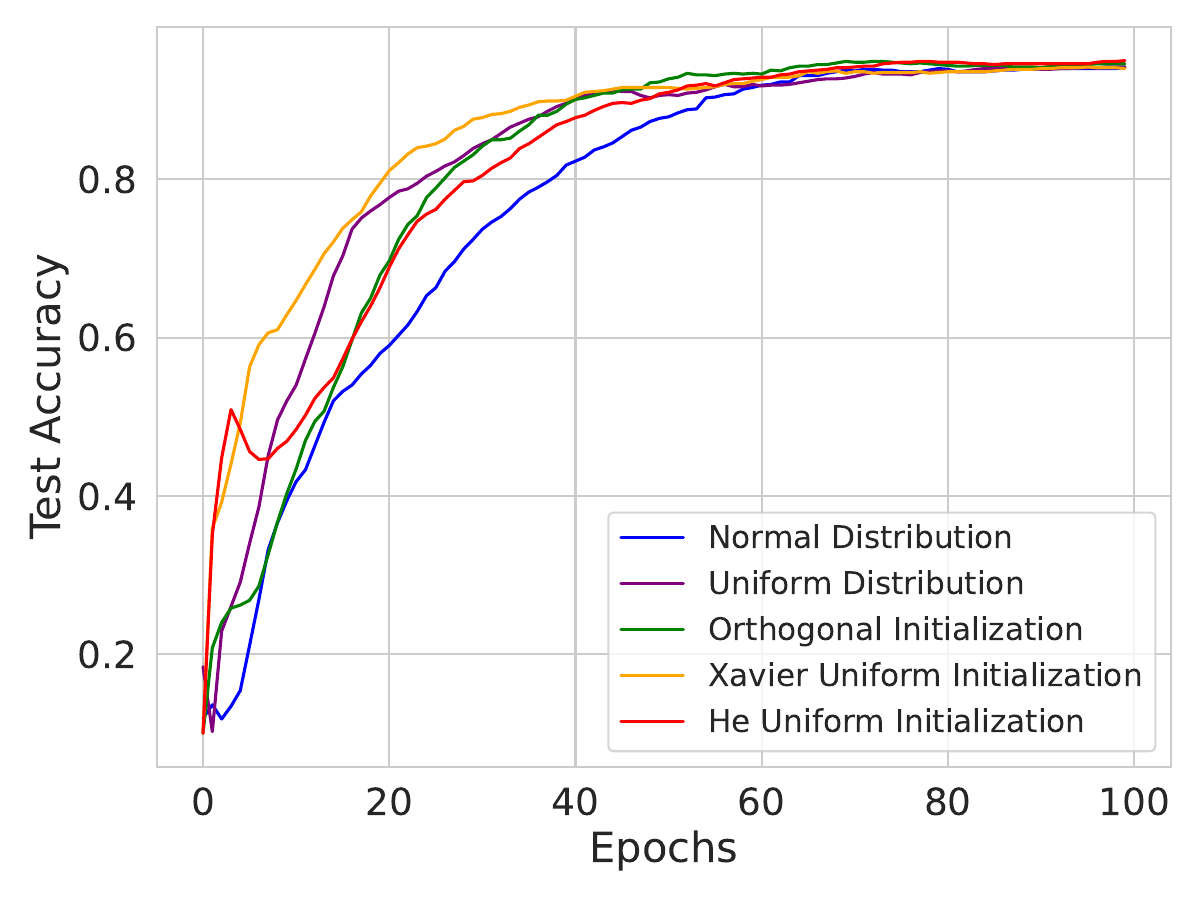}
    \caption{Validation acc. v/s Epochs}
  \end{subfigure}
  \caption{Effect of initialization on our learning rate and model performance. \textbf{Full batch} experiments conducted on a 3 layer network with 1000 nodes on the MNIST.}
\label{fig:fb_distribution_3__1000}
\end{figure}
\begin{figure}
  \centering
  \begin{subfigure}[b]{0.31\textwidth}
    \includegraphics[width=\textwidth]{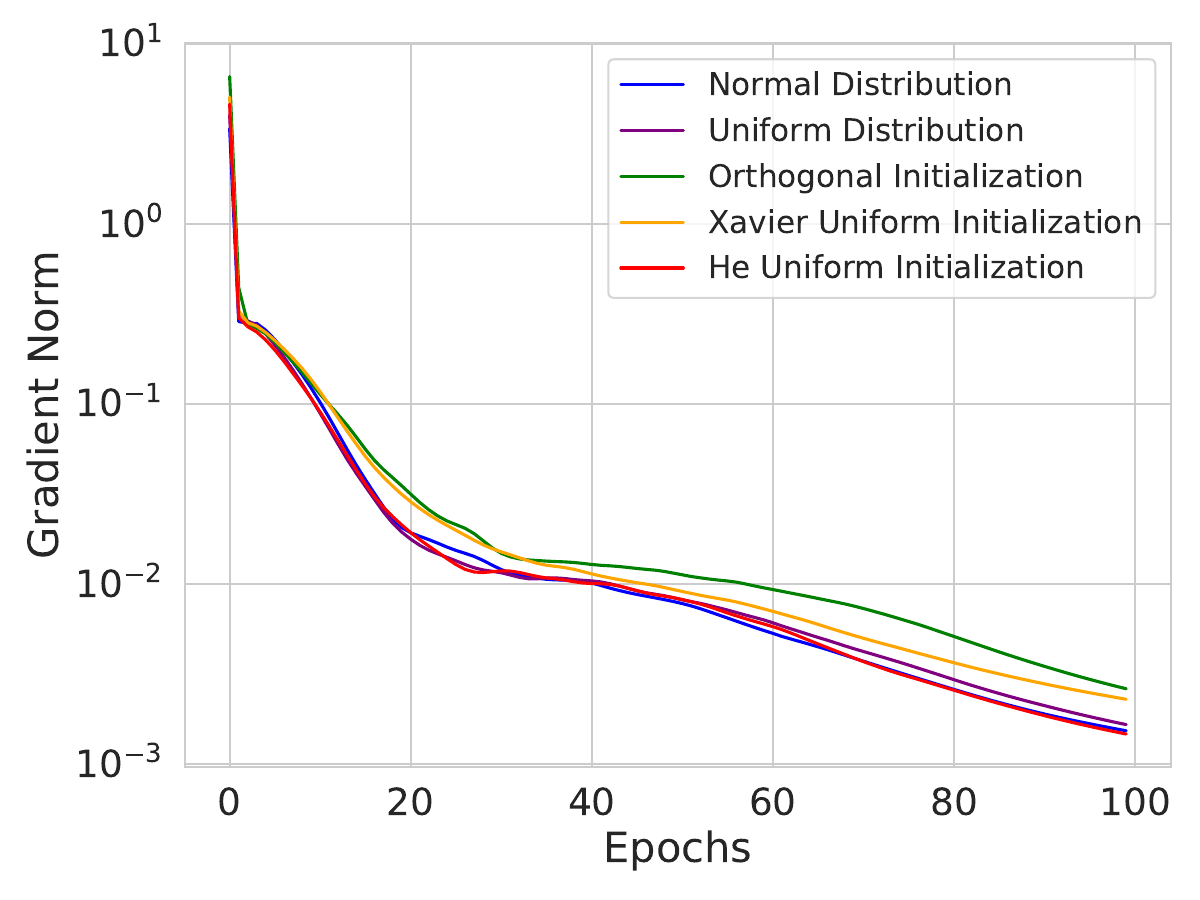}
    \caption{Gradient norm v/s Epochs}
  \end{subfigure}
  \hfill
  \begin{subfigure}[b]{0.31\textwidth}
    \includegraphics[width=\textwidth]{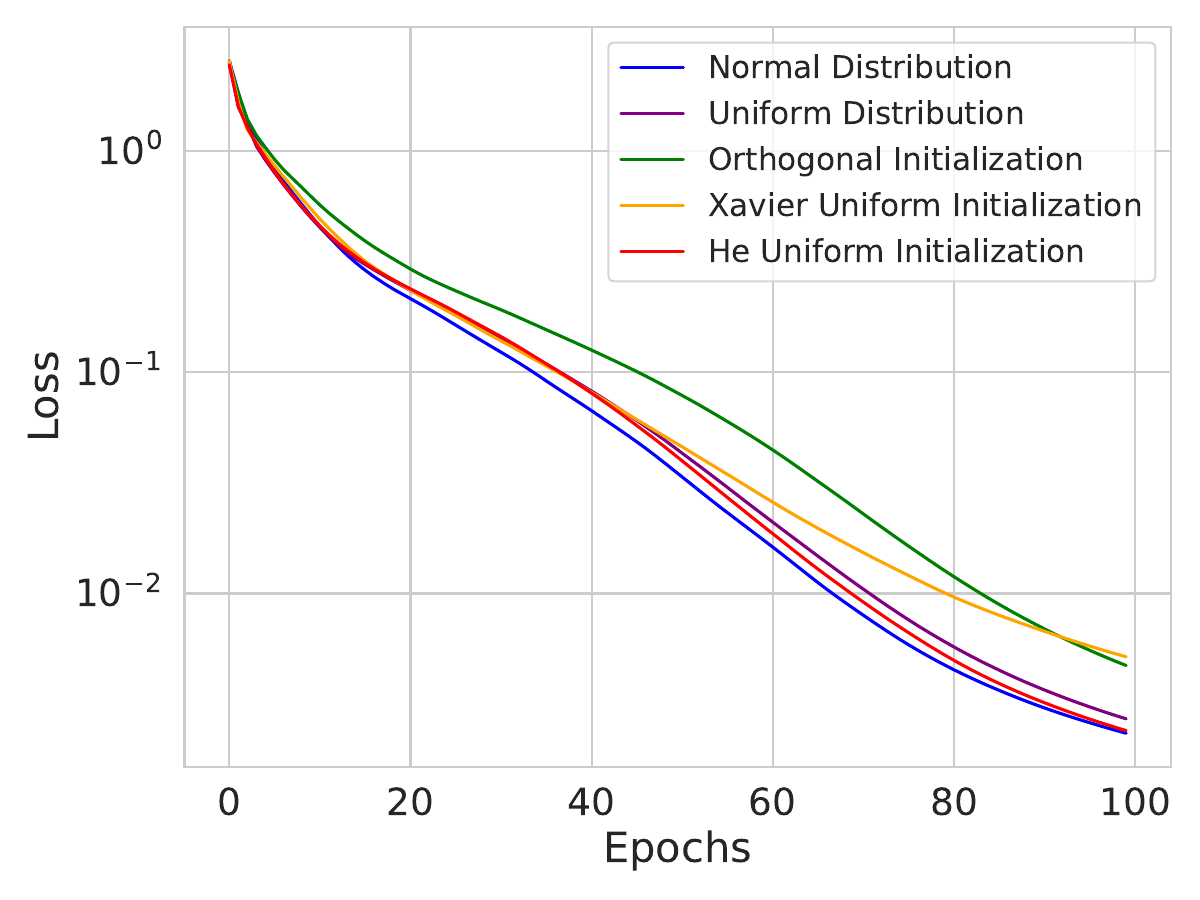}
    \caption{Training loss v/s Epochs}
  \end{subfigure}
  \hfill
  \begin{subfigure}[b]{0.31\textwidth}
    \includegraphics[width=\textwidth]{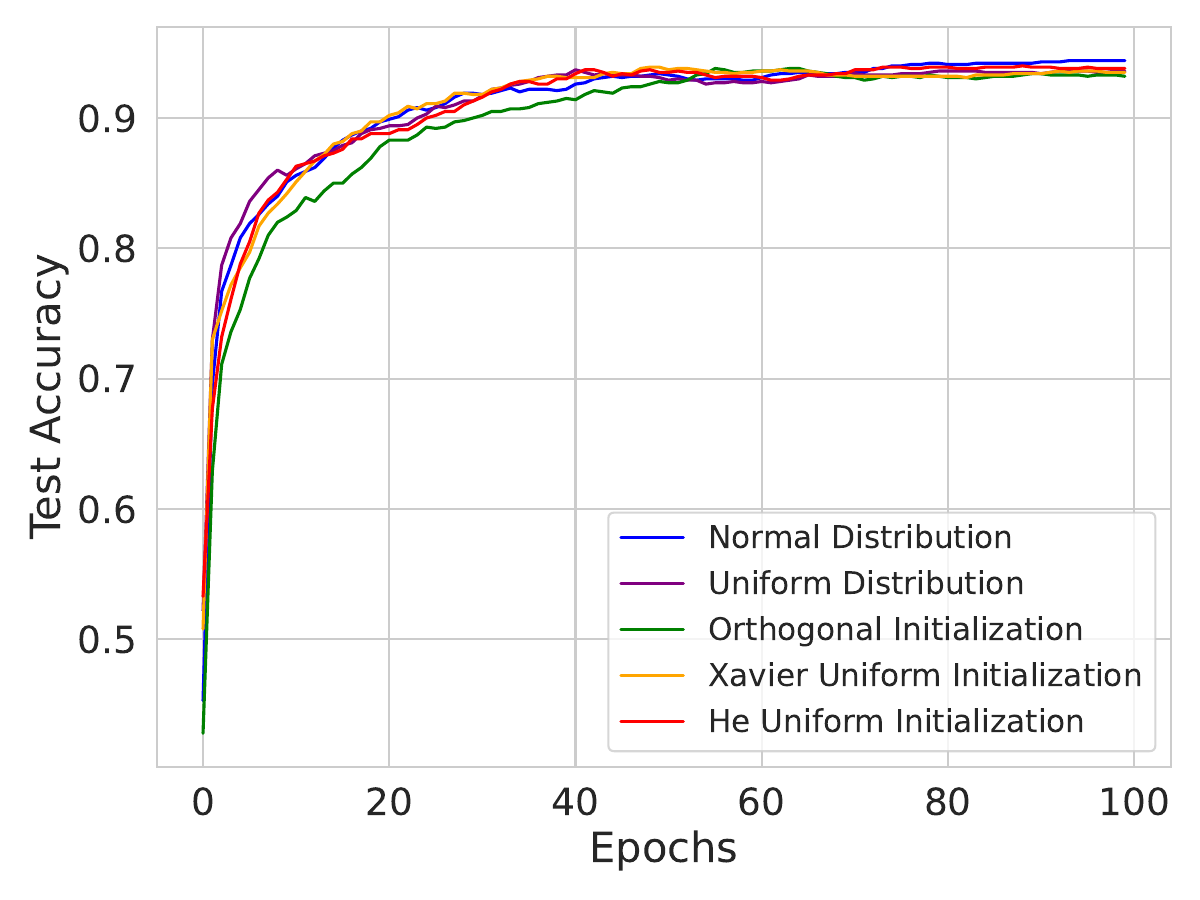}
    \caption{Validation acc. v/s Epochs}
  \end{subfigure}
  \caption{Effect of initialization on our learning rate and model performance. \textbf{Full batch} experiments conducted on a 3 layer network with 3000 nodes on the MNIST.}
\label{fig:fb_distribution_3__3000}
\end{figure}
\begin{figure}
  \centering
  \begin{subfigure}[b]{0.31\textwidth}
    \includegraphics[width=\textwidth]{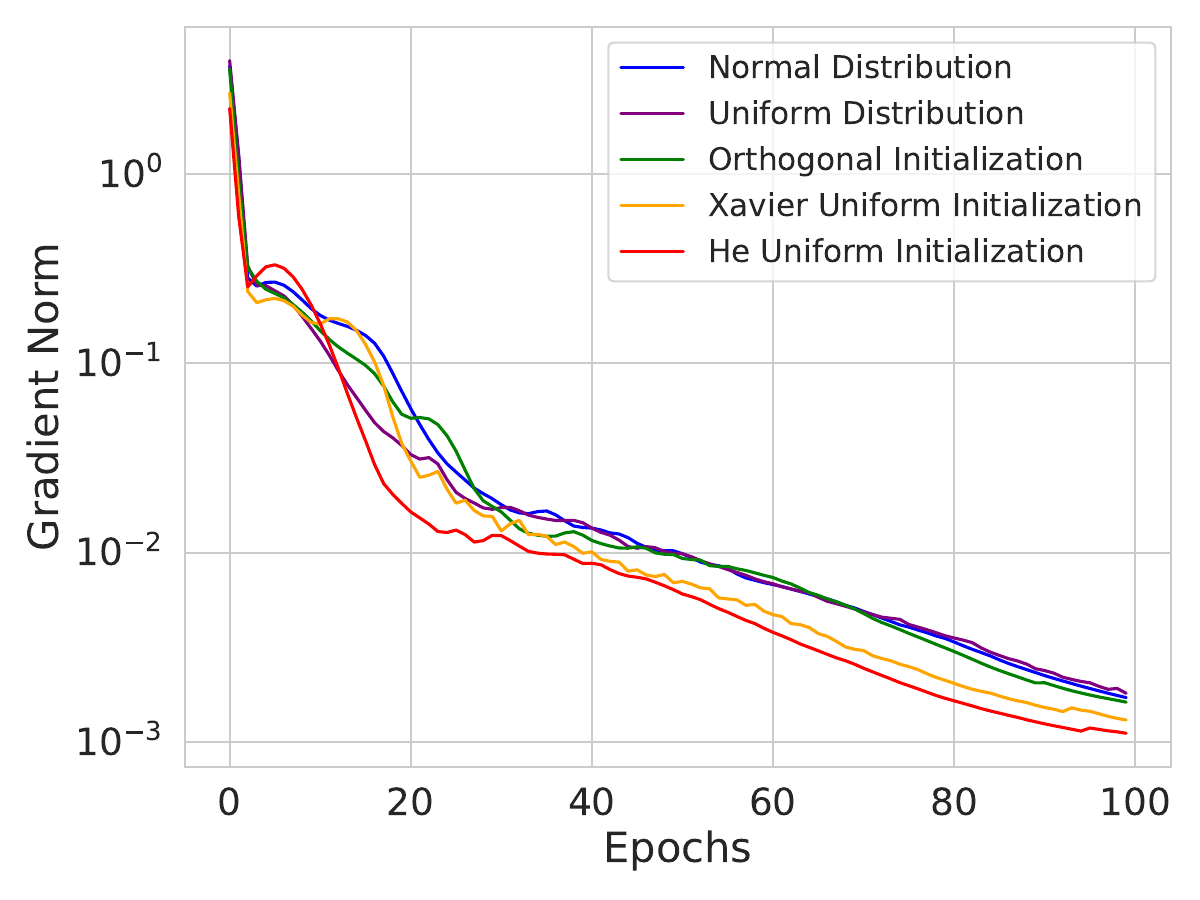}
    \caption{Gradient norm v/s Epochs}
  \end{subfigure}
  \hfill
  \begin{subfigure}[b]{0.31\textwidth}
    \includegraphics[width=\textwidth]{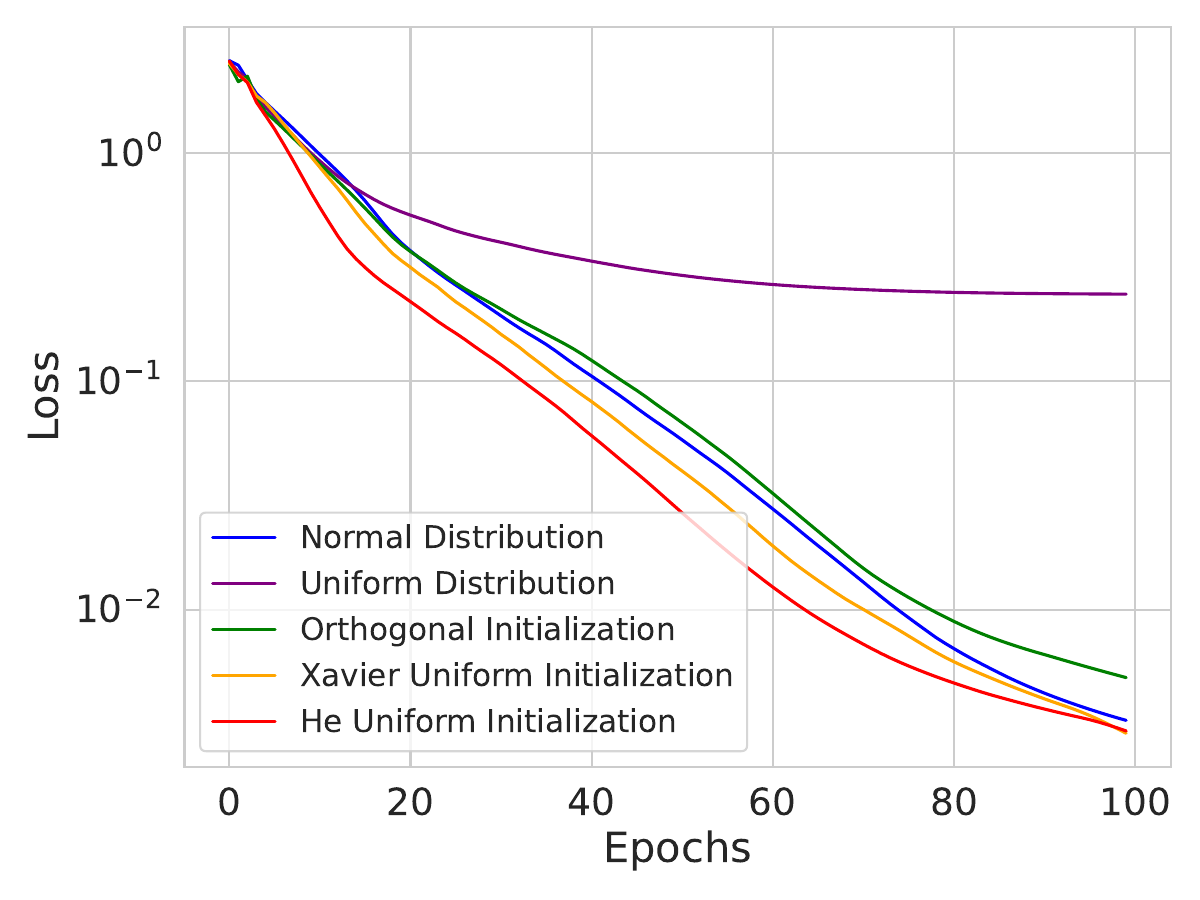}
    \caption{Training loss v/s Epochs}
  \end{subfigure}
  \hfill
  \begin{subfigure}[b]{0.31\textwidth}
    \includegraphics[width=\textwidth]{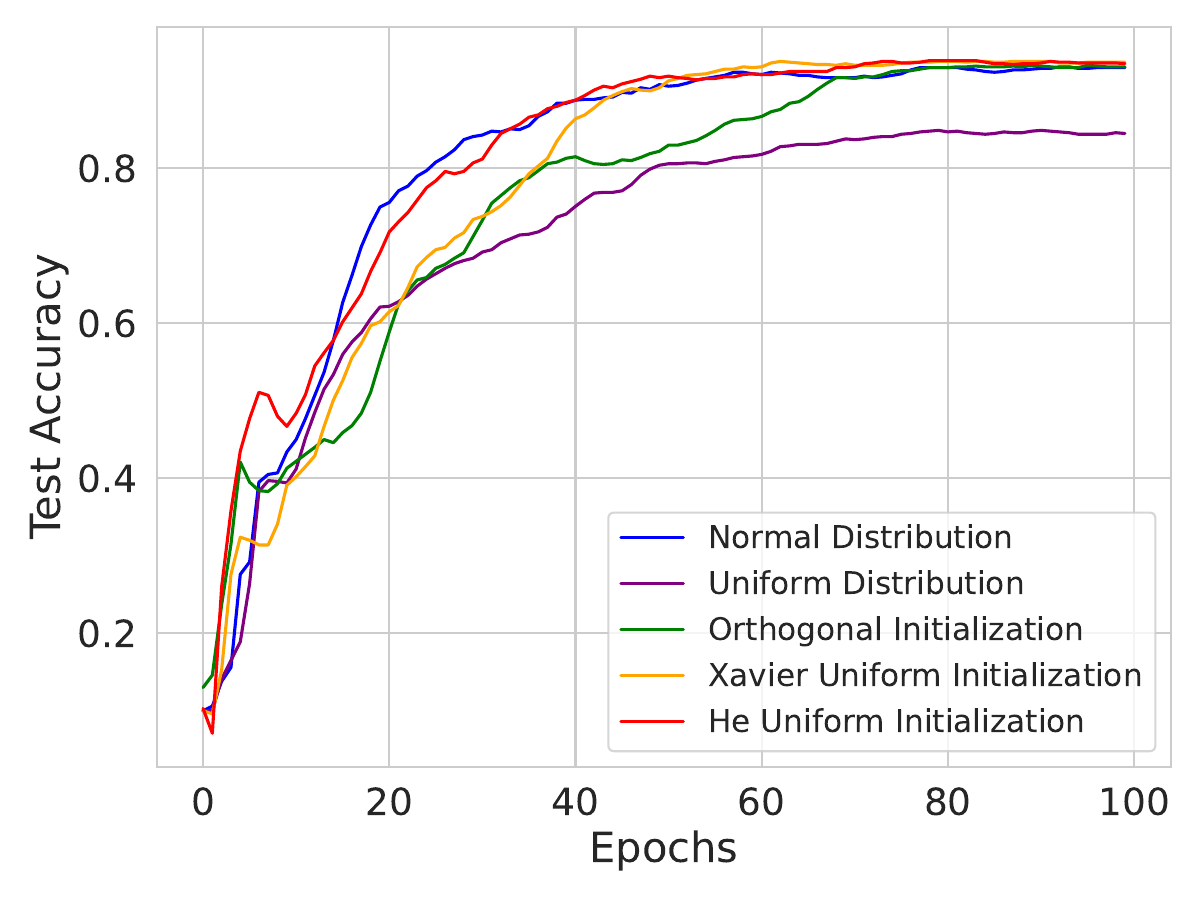}
    \caption{Validation acc. v/s Epochs}
  \end{subfigure}
  \caption{Effect of initialization on our learning rate and model performance. \textbf{Full batch} experiments conducted on a 5 layer network with 300 nodes on the MNIST.}
\label{fig:fb_distribution_5__300}
\end{figure}
\begin{figure}
  \centering
  \begin{subfigure}[b]{0.31\textwidth}
    \includegraphics[width=\textwidth]{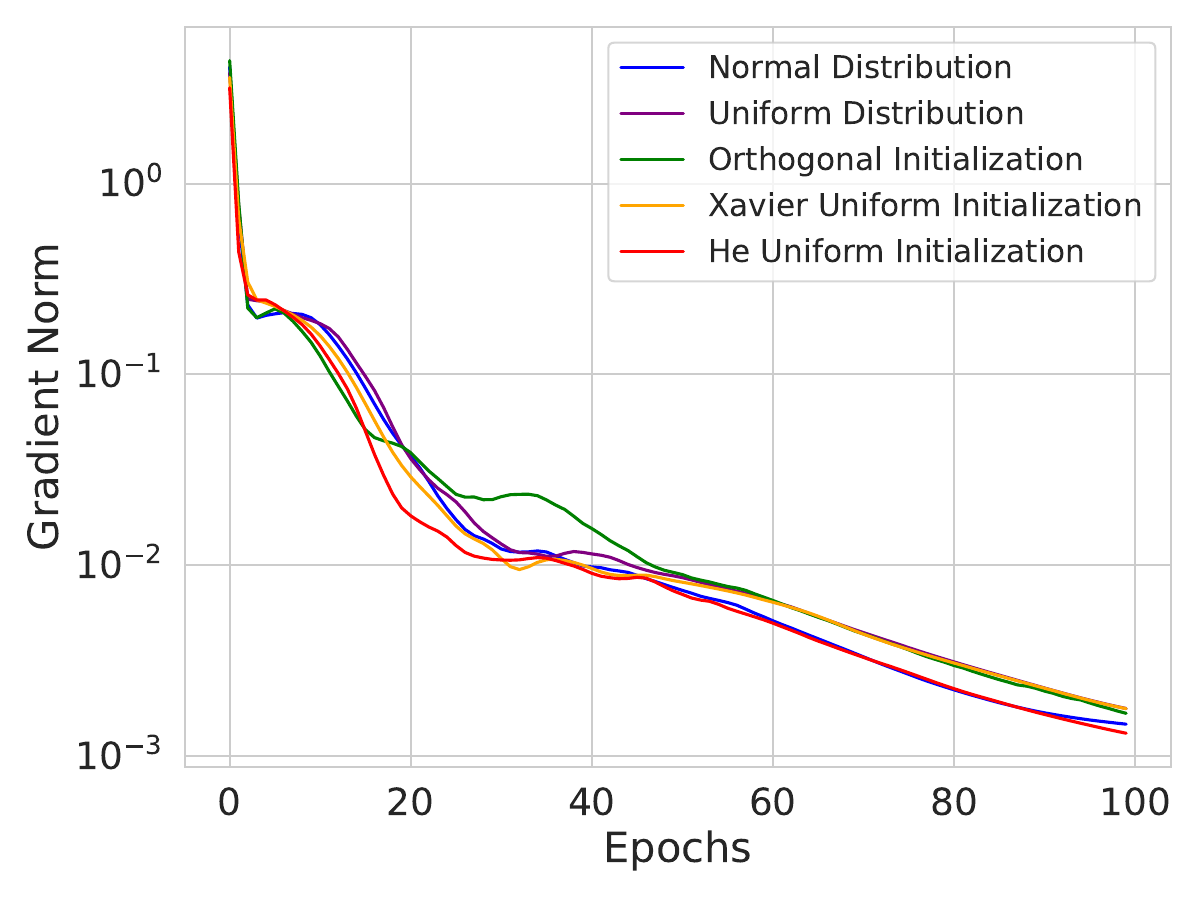}
    \caption{Gradient norm v/s Epochs}
  \end{subfigure}
  \hfill
  \begin{subfigure}[b]{0.31\textwidth}
    \includegraphics[width=\textwidth]{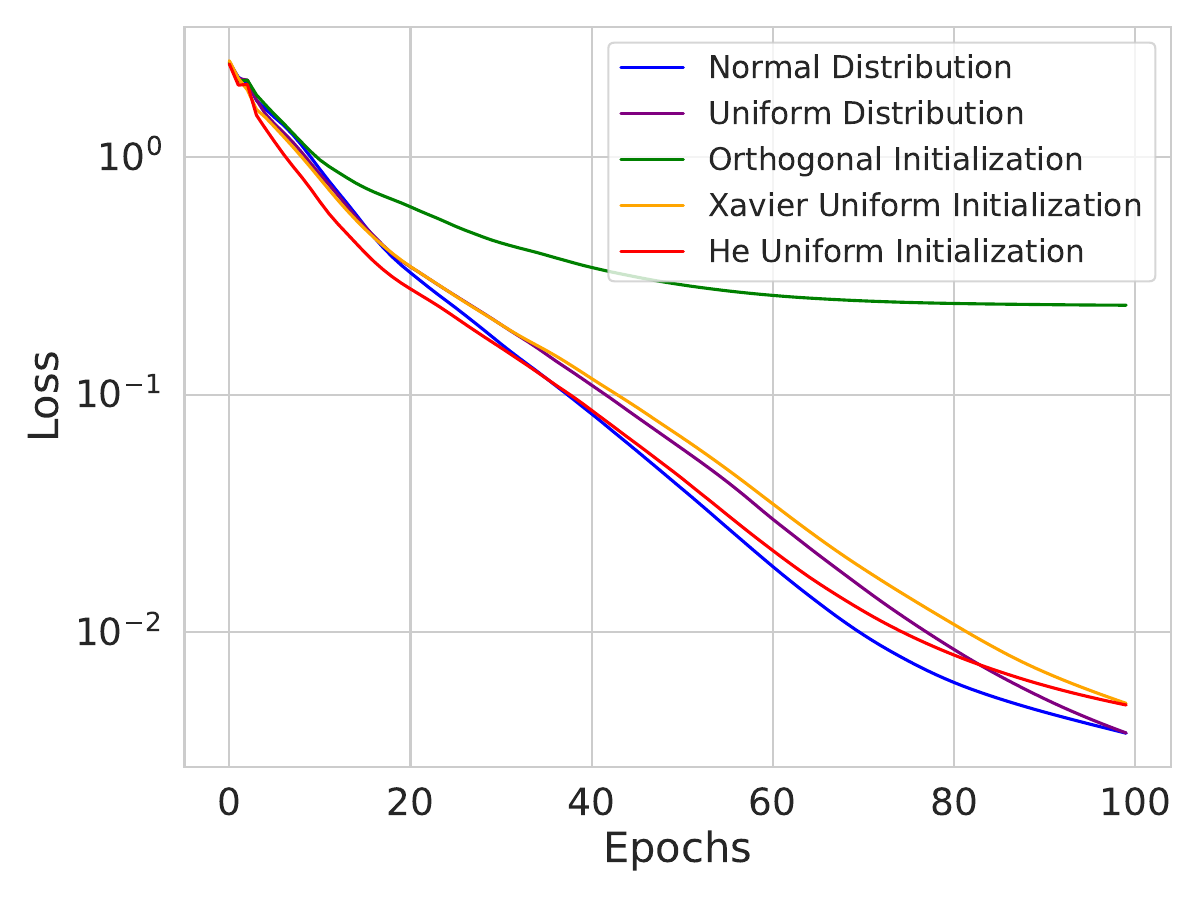}
    \caption{Training loss v/s Epochs}
  \end{subfigure}
  \hfill
  \begin{subfigure}[b]{0.31\textwidth}
    \includegraphics[width=\textwidth]{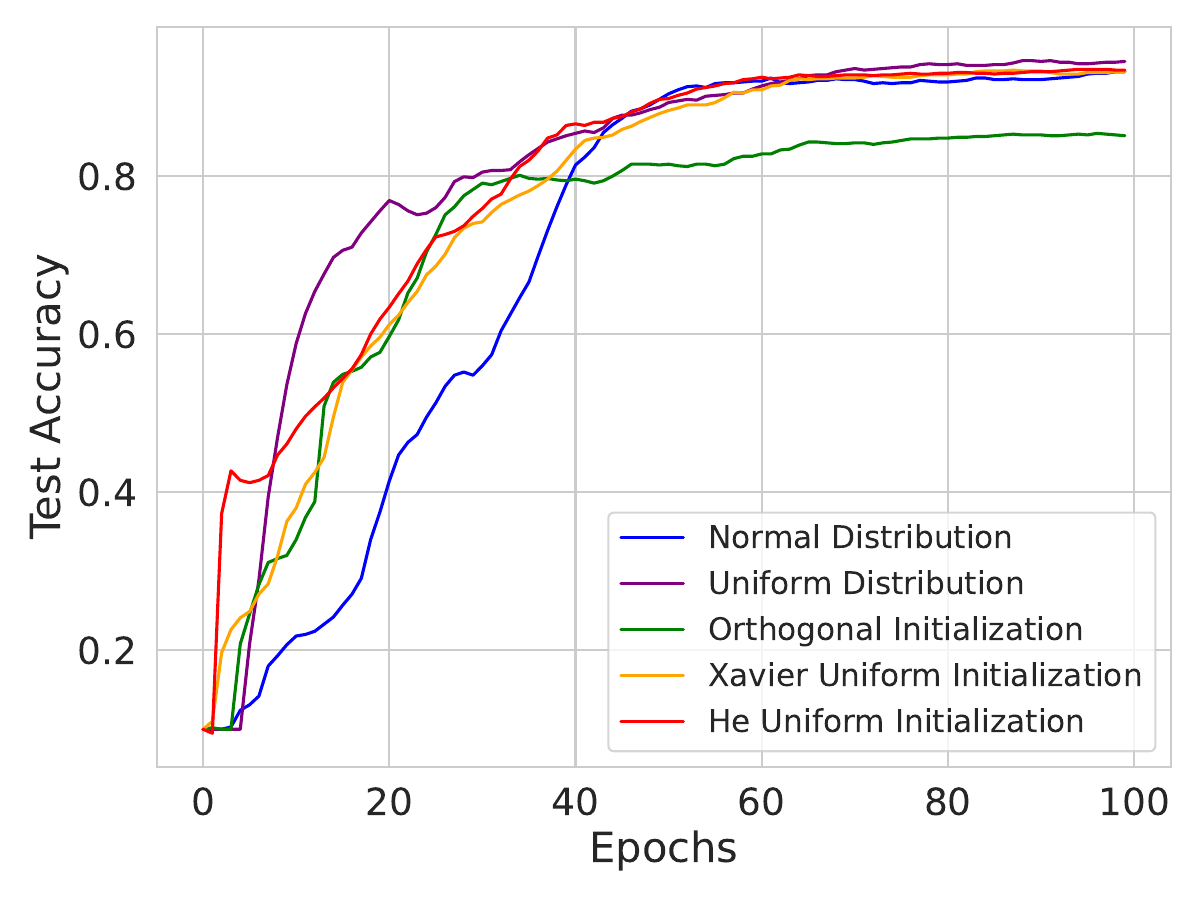}
    \caption{Validation acc. v/s Epochs}
  \end{subfigure}
  \caption{Effect of initialization on our learning rate and model performance. \textbf{Full batch} experiments conducted on a 5 layer network with 1000 nodes on the MNIST.}
\label{fig:fb_distribution_5__1000}
\end{figure}
\begin{figure}
  \centering
  \begin{subfigure}[b]{0.31\textwidth}
    \includegraphics[width=\textwidth]{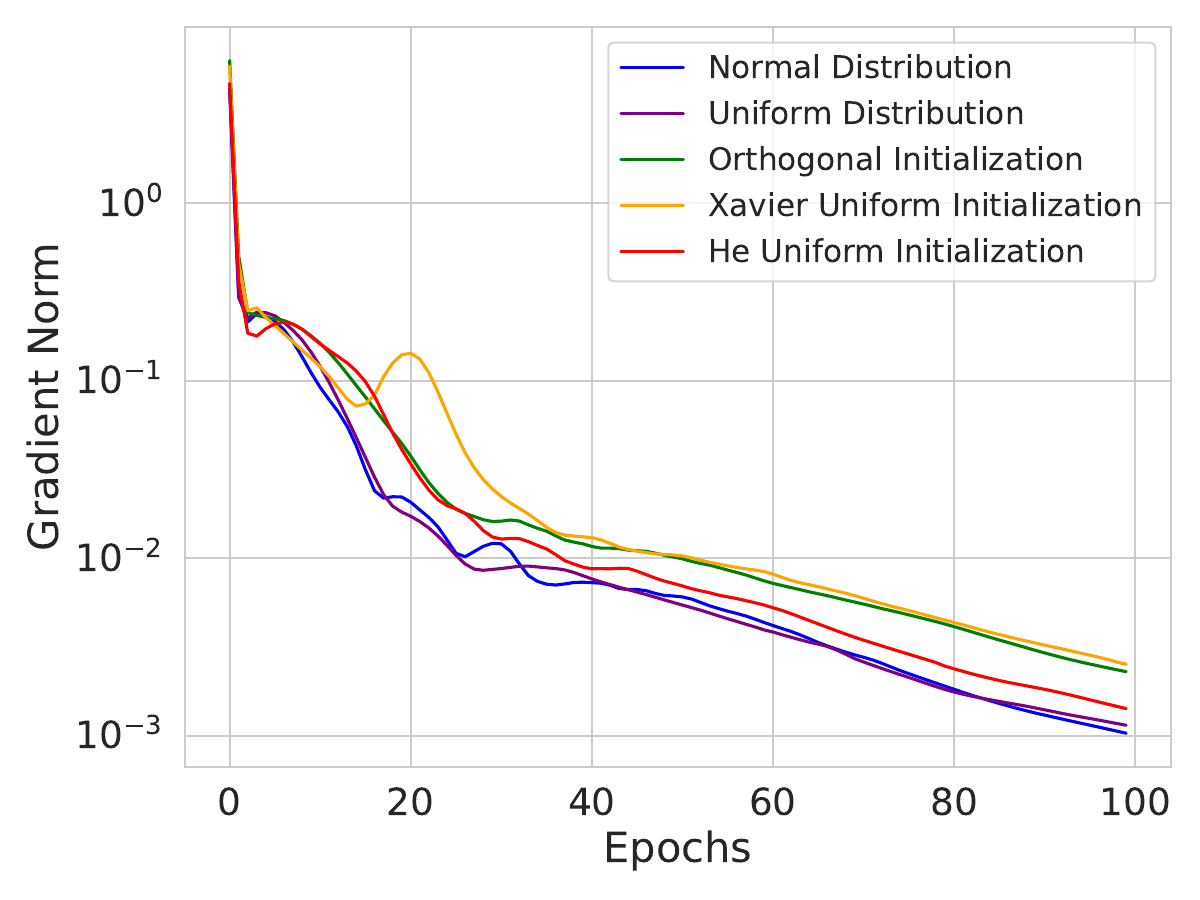}
    \caption{Gradient norm v/s Epochs}
  \end{subfigure}
  \hfill
  \begin{subfigure}[b]{0.31\textwidth}
    \includegraphics[width=\textwidth]{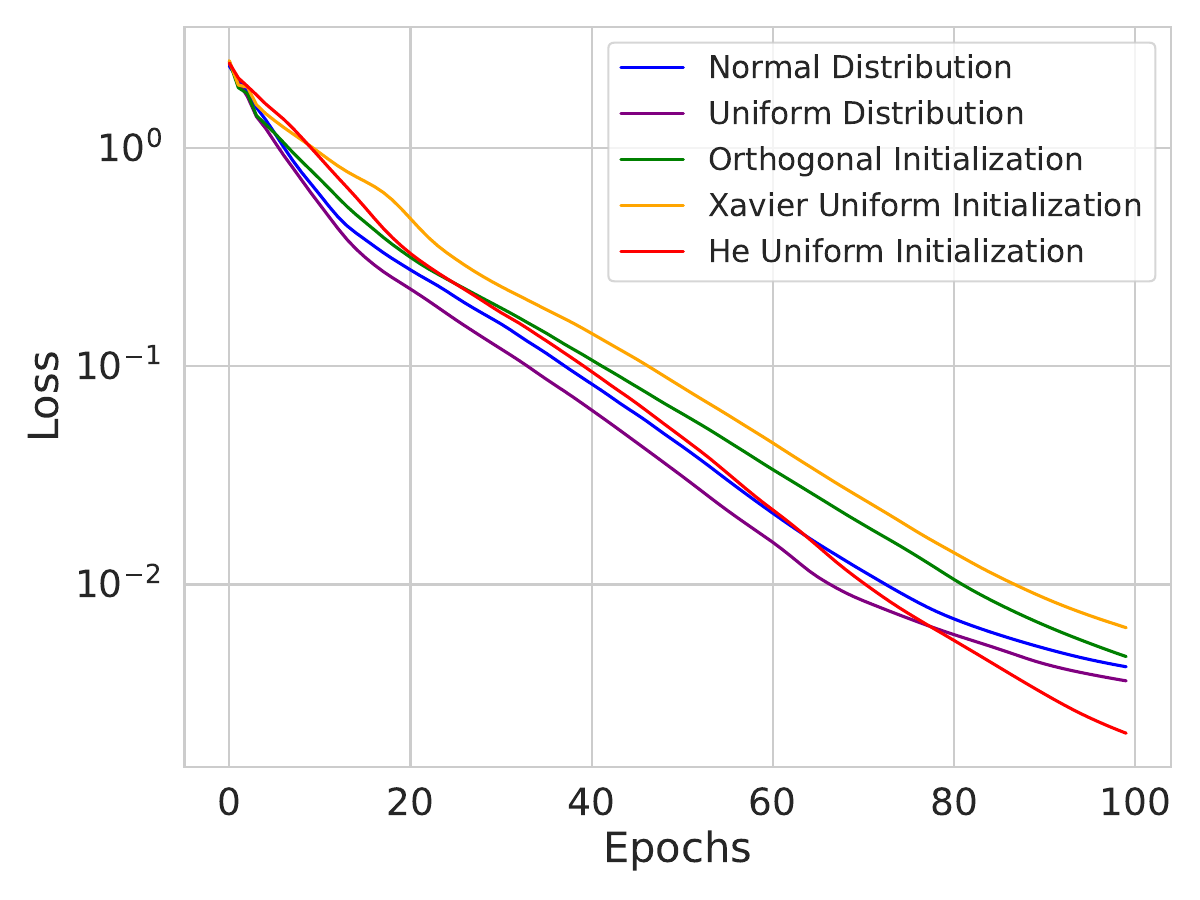}
    \caption{Training loss v/s Epochs}
  \end{subfigure}
  \hfill
  \begin{subfigure}[b]{0.31\textwidth}
    \includegraphics[width=\textwidth]{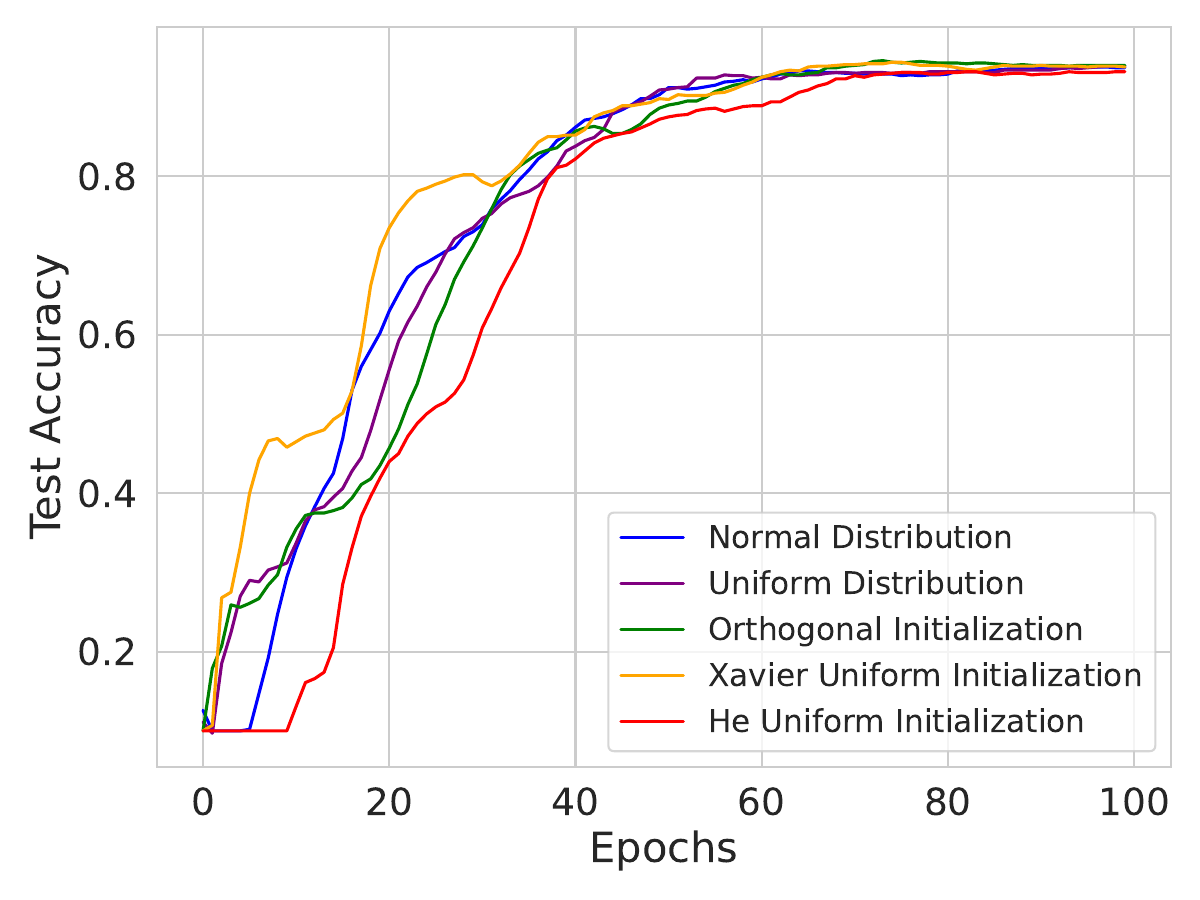}
    \caption{Validation acc. v/s Epochs}
  \end{subfigure}
  \caption{Effect of initialization on our learning rate and model performance. \textbf{Full batch} experiments conducted on a 5 layer network with 3000 nodes on the MNIST.}
\label{fig:fb_distribution_5__3000}
\end{figure}
\begin{figure}
  \centering
  \begin{subfigure}[b]{0.31\textwidth}
    \includegraphics[width=\textwidth]{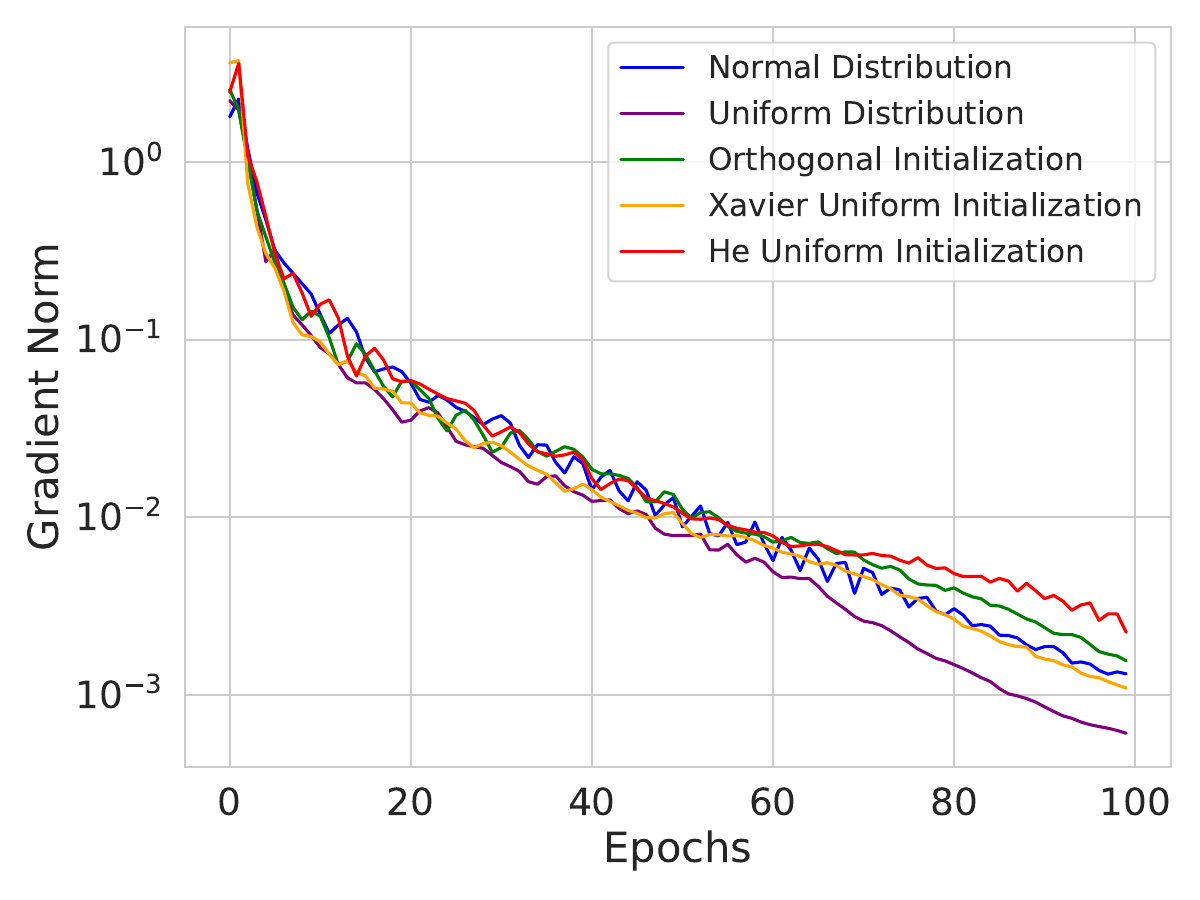}
    \caption{Gradient norm v/s Epochs}
  \end{subfigure}
  \hfill
  \begin{subfigure}[b]{0.31\textwidth}
    \includegraphics[width=\textwidth]{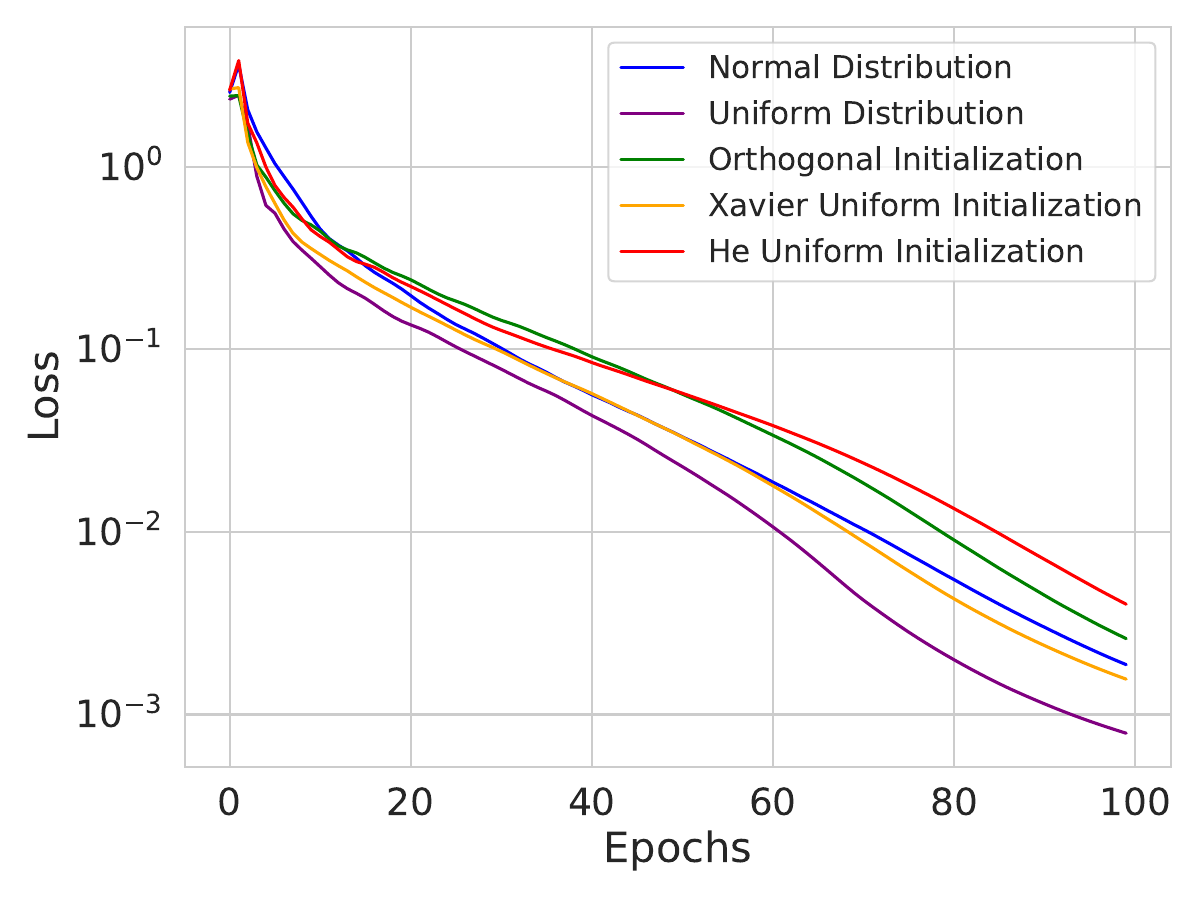}
    \caption{Training loss v/s Epochs}
  \end{subfigure}
  \hfill
  \begin{subfigure}[b]{0.31\textwidth}
    \includegraphics[width=\textwidth]{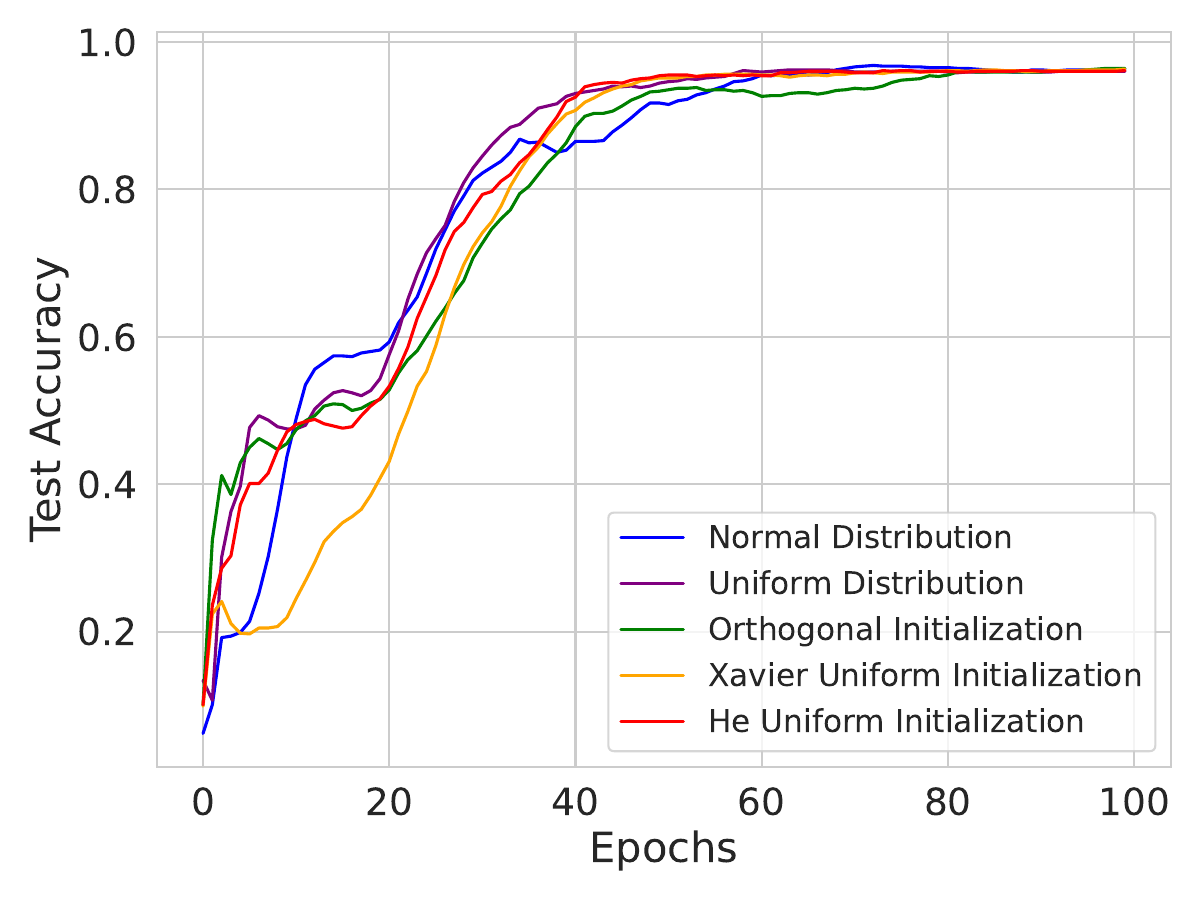}
    \caption{Validation acc. v/s Epochs}
  \end{subfigure}
  \caption{Effect of initialization on our learning rate and model performance. \textbf{Full batch} experiments conducted on LeNet with MNIST.}
\label{fig:fb_distribution_lenet}
\end{figure}
\begin{figure}
  \centering
  \begin{subfigure}[b]{0.31\textwidth}
    \includegraphics[width=\textwidth]{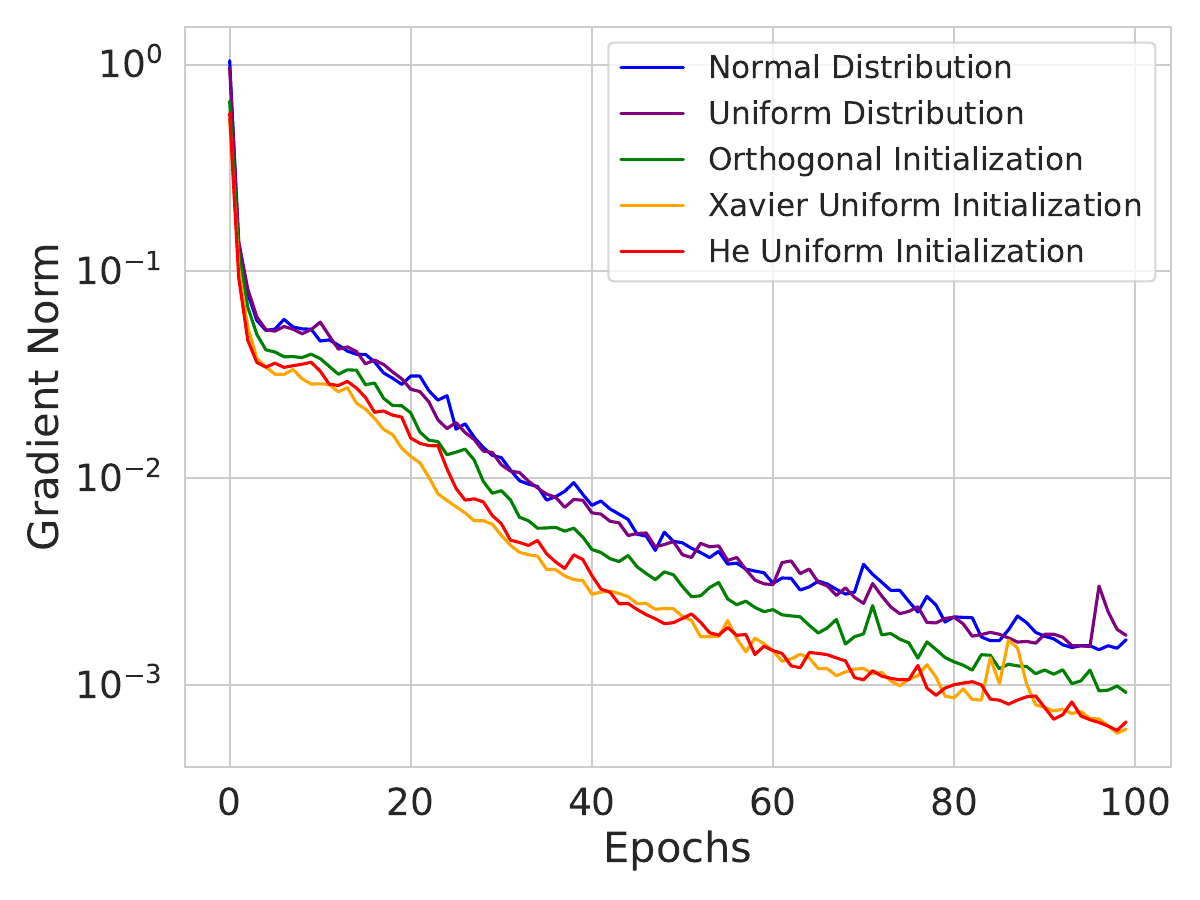}
    \caption{Gradient norm v/s Epochs}
  \end{subfigure}
  \hfill
  \begin{subfigure}[b]{0.31\textwidth}
    \includegraphics[width=\textwidth]{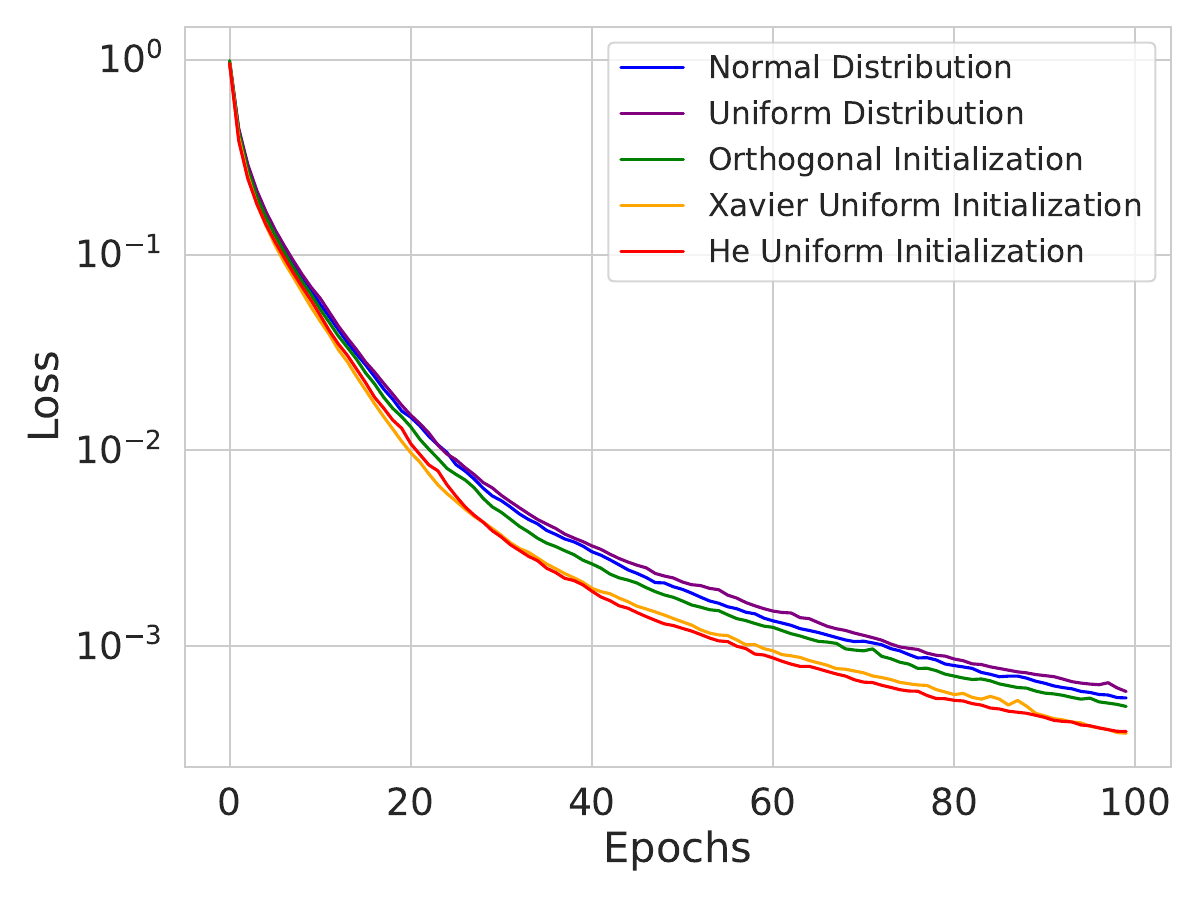}
    \caption{Training loss v/s Epochs}
  \end{subfigure}
  \hfill
  \begin{subfigure}[b]{0.31\textwidth}
    \includegraphics[width=\textwidth]{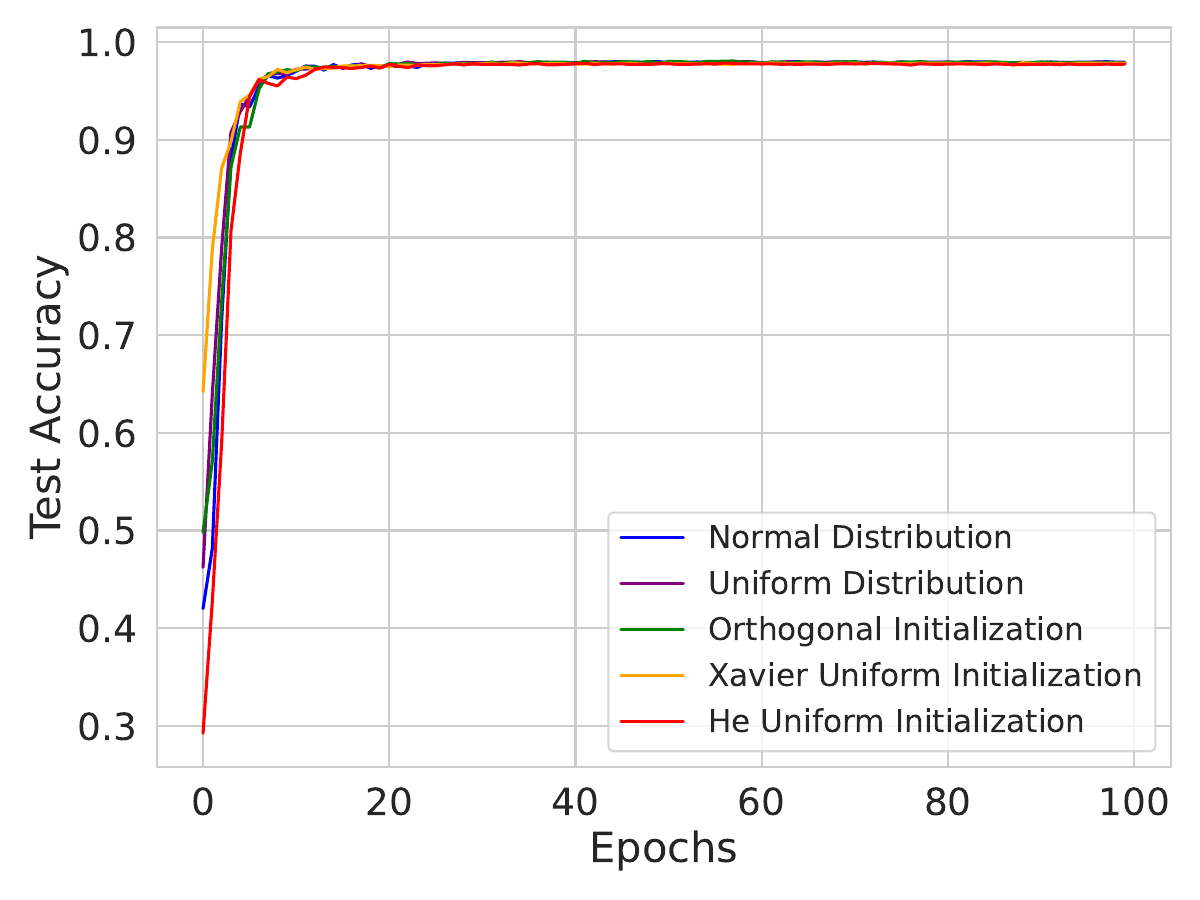}
    \caption{Validation acc. v/s Epochs}
  \end{subfigure}
  \caption{Effect of initialization on our learning rate and model performance. \textbf{Mini batch} experiments conducted on a single layer network with 300 nodes on the MNIST.}
\label{fig:mb_distribution_1__300}
\end{figure}
\begin{figure}
  \centering
  \begin{subfigure}[b]{0.31\textwidth}
    \includegraphics[width=\textwidth]{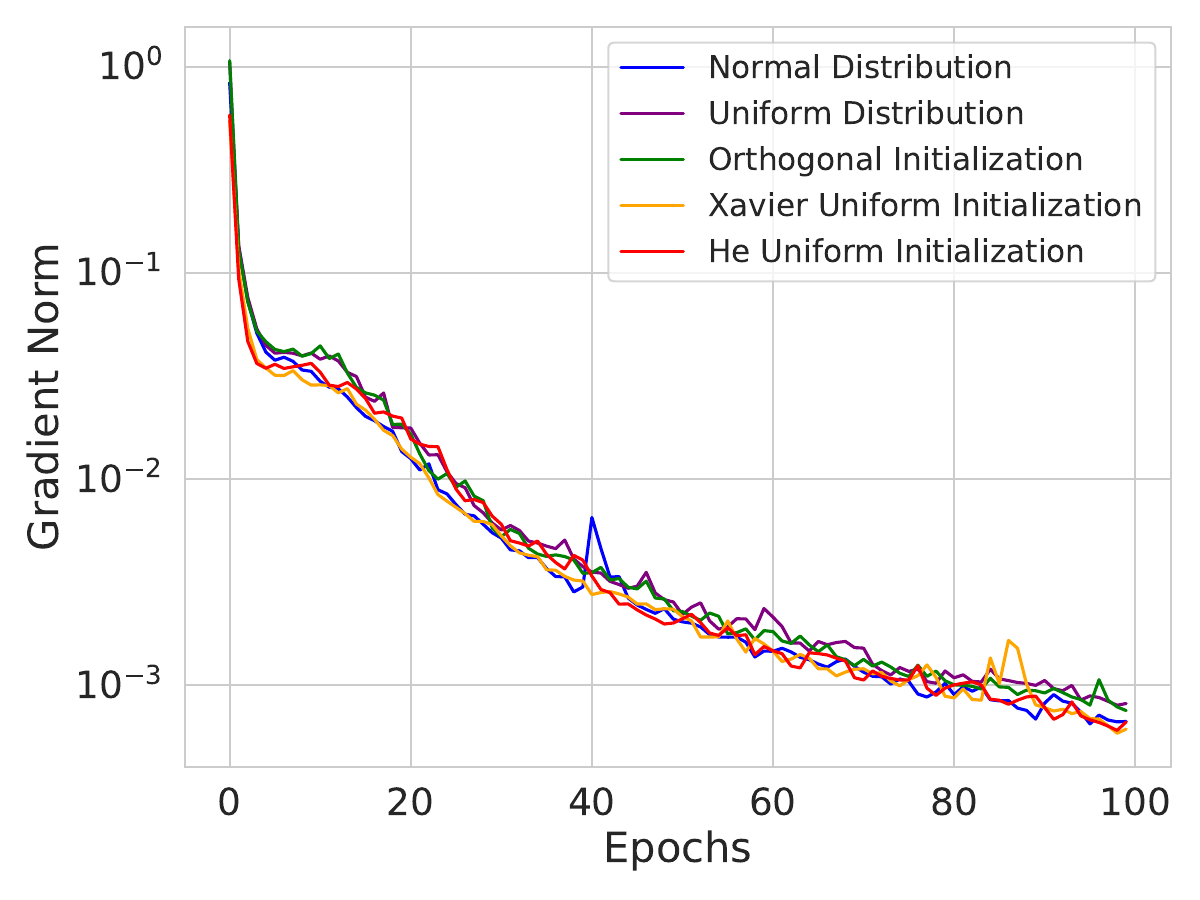}
    \caption{Gradient norm v/s Epochs}
  \end{subfigure}
  \hfill
  \begin{subfigure}[b]{0.31\textwidth}
    \includegraphics[width=\textwidth]{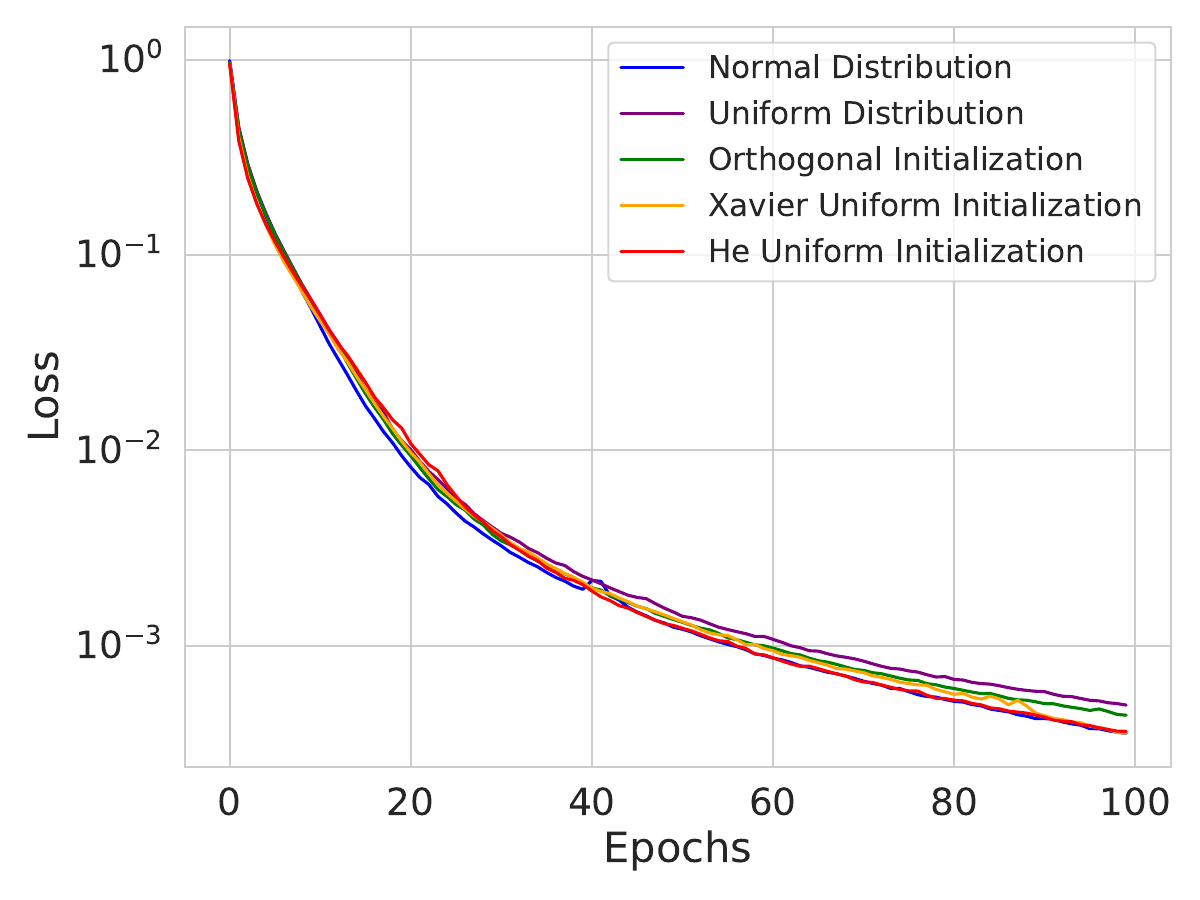}
    \caption{Training loss v/s Epochs}
  \end{subfigure}
  \hfill
  \begin{subfigure}[b]{0.31\textwidth}
    \includegraphics[width=\textwidth]{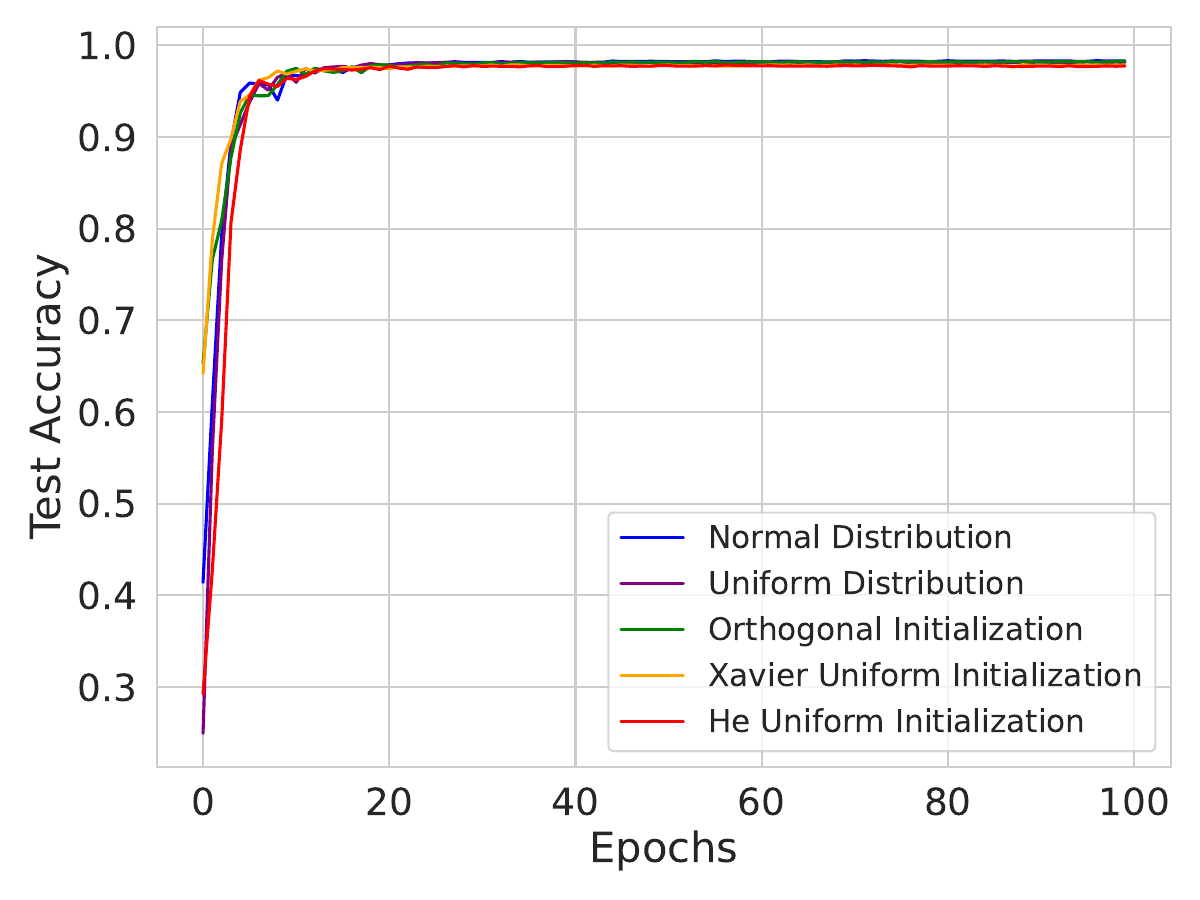}
    \caption{Validation acc. v/s Epochs}
  \end{subfigure}
  \caption{Effect of initialization on our learning rate and model performance. \textbf{Mini batch} experiments conducted on a single layer network with 1000 nodes on the MNIST.}
\label{fig:mb_distribution_1_1000}
\end{figure}
\begin{figure}
  \centering
  \begin{subfigure}[b]{0.31\textwidth}
    \includegraphics[width=\textwidth]{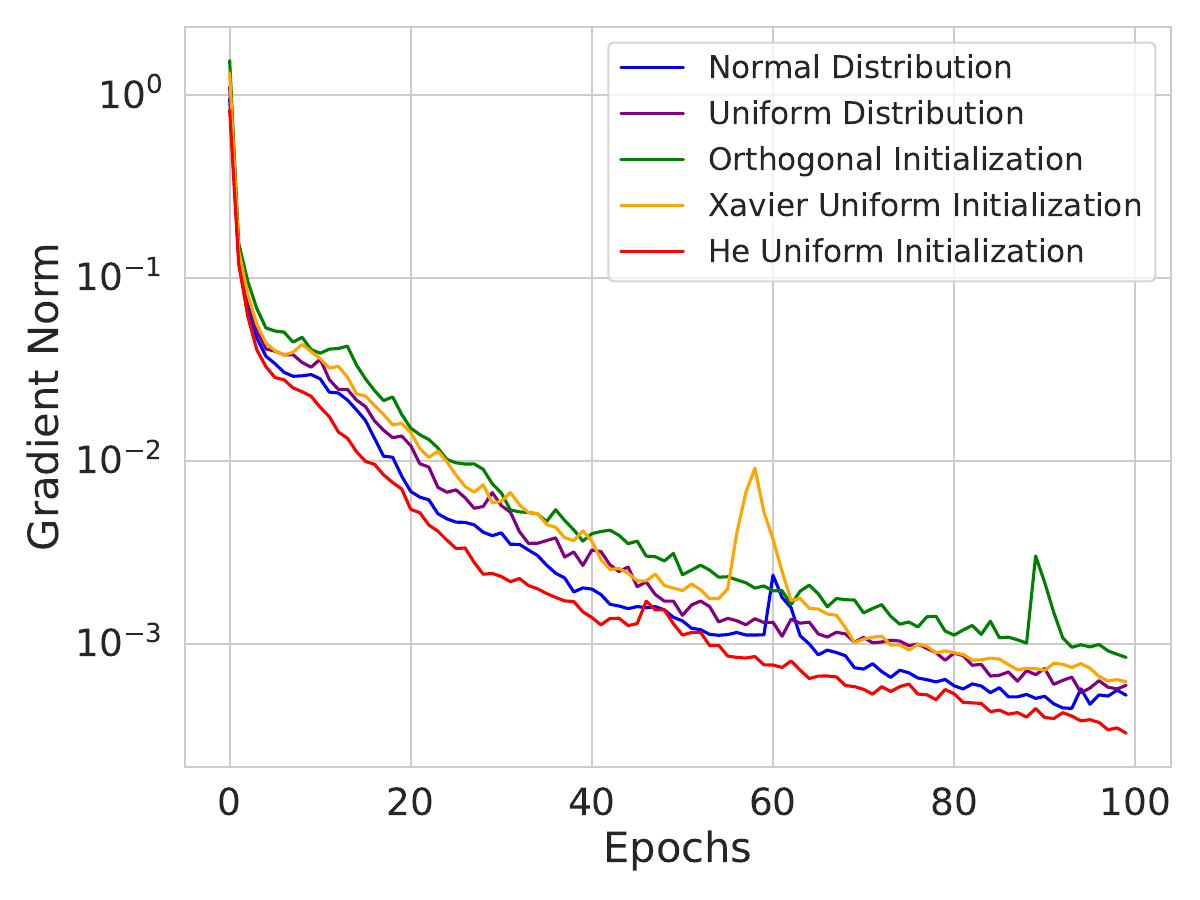}
    \caption{Gradient norm v/s Epochs}
  \end{subfigure}
  \hfill
  \begin{subfigure}[b]{0.31\textwidth}
    \includegraphics[width=\textwidth]{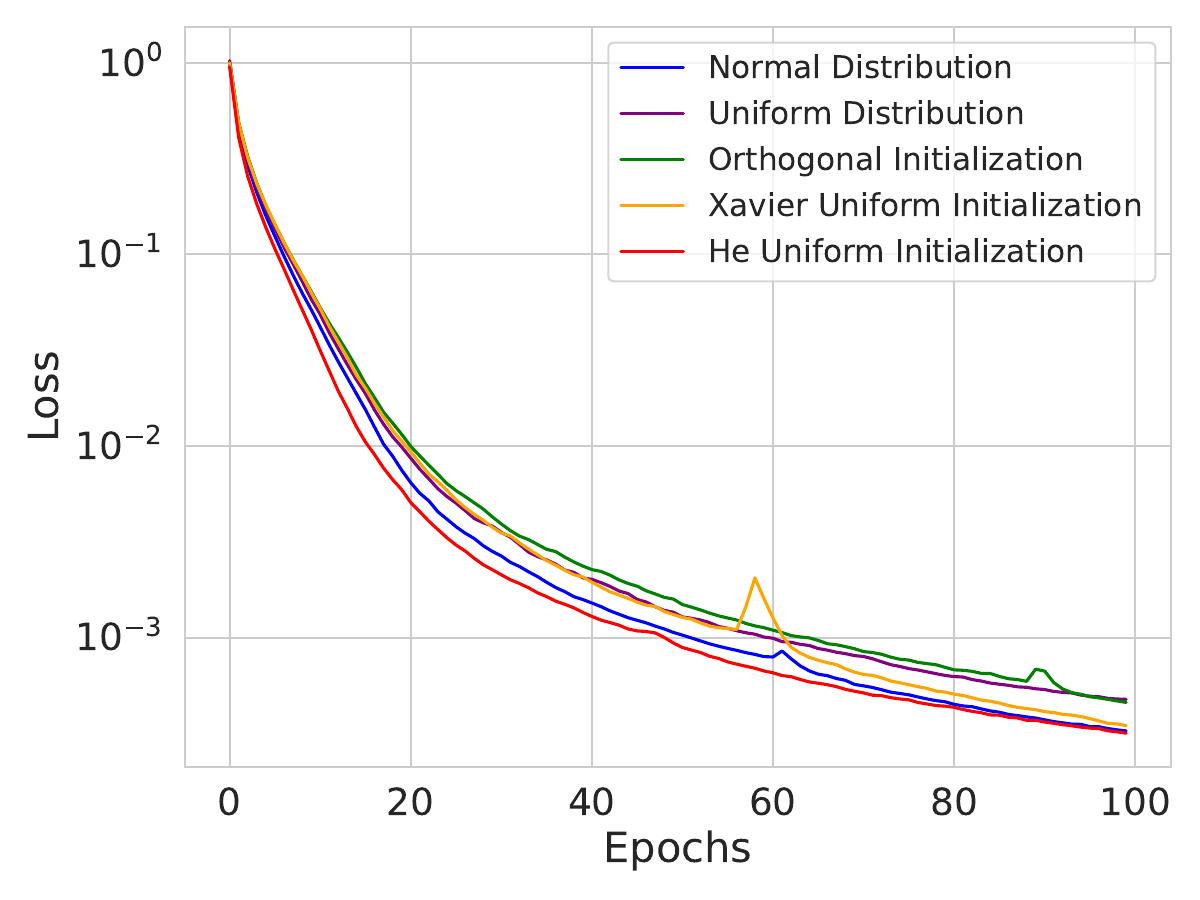}
    \caption{Training loss v/s Epochs}
  \end{subfigure}
  \hfill
  \begin{subfigure}[b]{0.31\textwidth}
    \includegraphics[width=\textwidth]{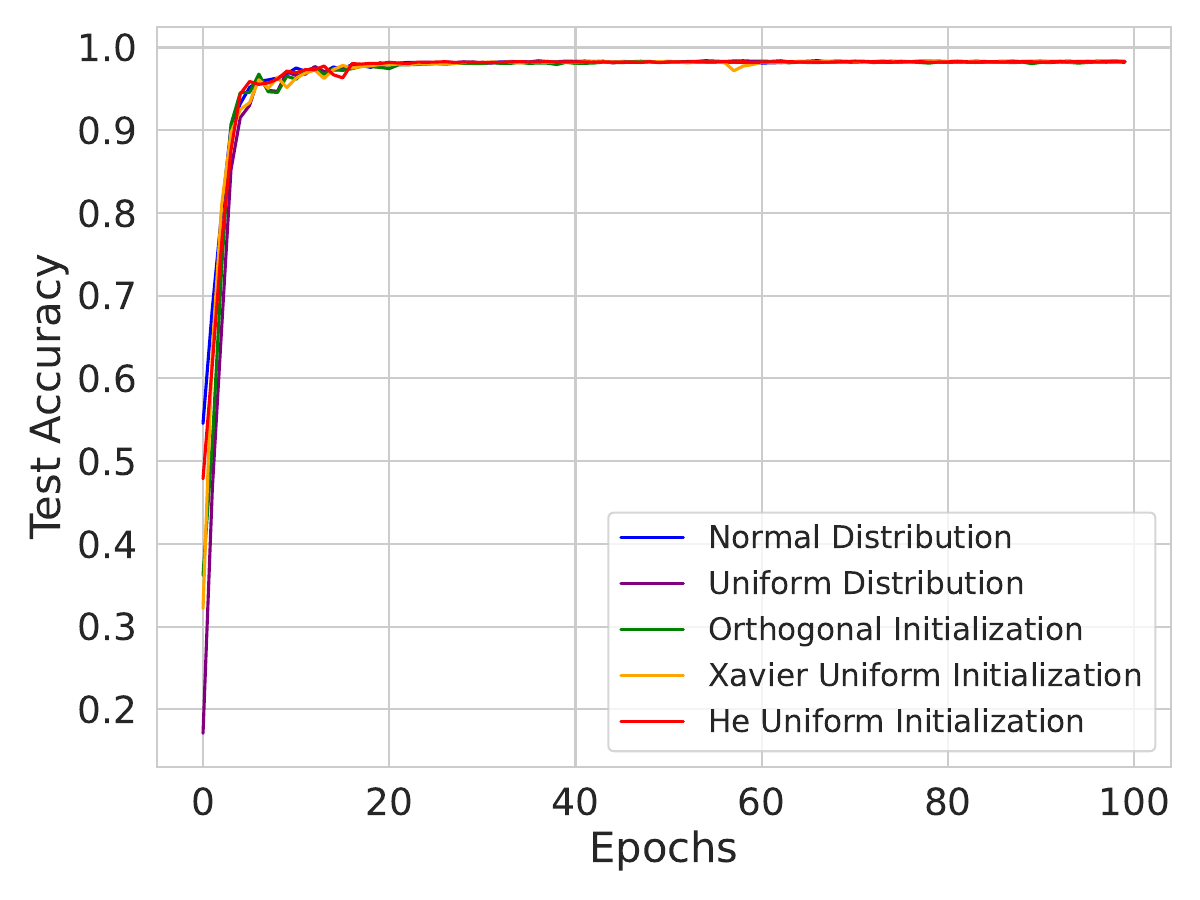}
    \caption{Validation acc. v/s Epochs}
  \end{subfigure}
  \caption{Effect of initialization on our learning rate and model performance. \textbf{Mini batch} experiments conducted on a single layer network with 3000 nodes on the MNIST.}
\label{fig:mb_distribution_1__3000}
\end{figure}
\begin{figure}
  \centering
  \begin{subfigure}[b]{0.31\textwidth}
    \includegraphics[width=\textwidth]{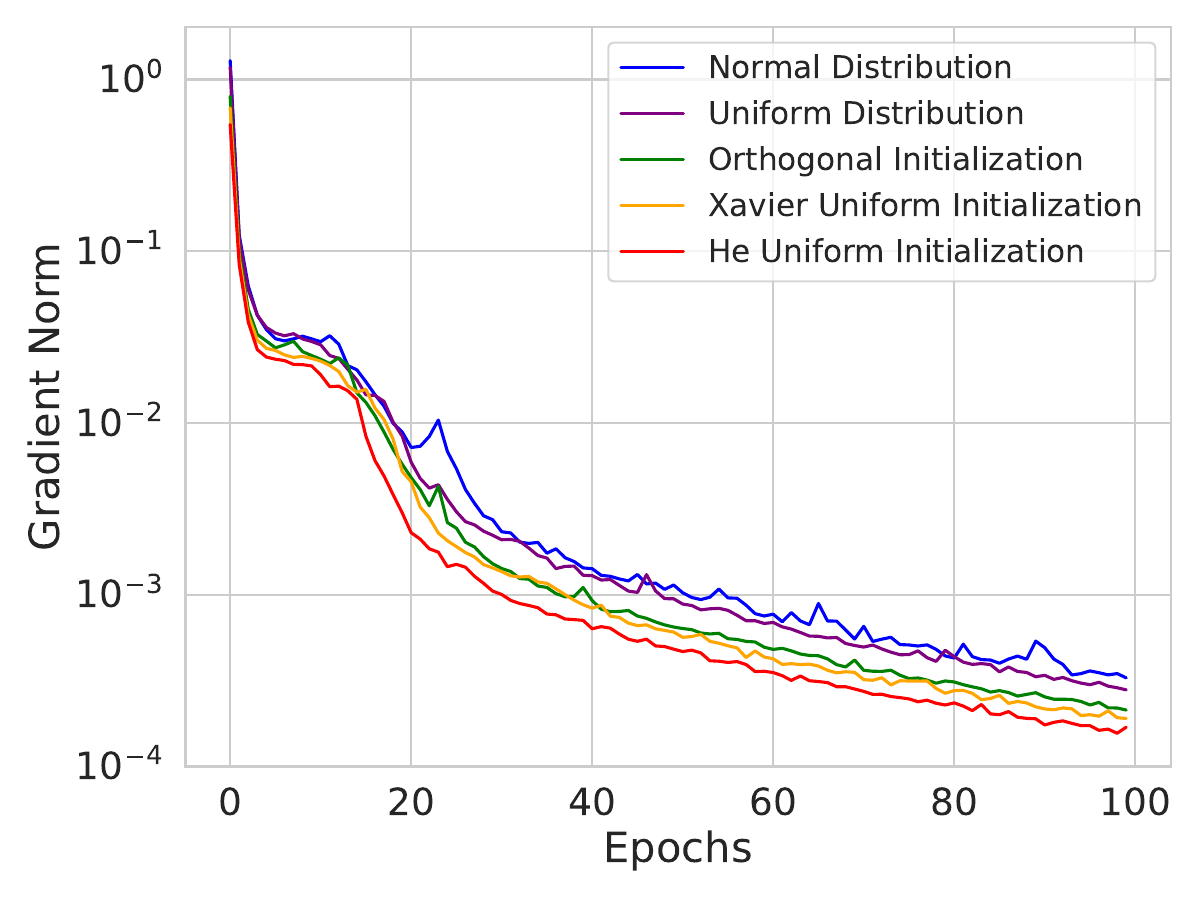}
    \caption{Gradient norm v/s Epochs}
  \end{subfigure}
  \hfill
  \begin{subfigure}[b]{0.31\textwidth}
    \includegraphics[width=\textwidth]{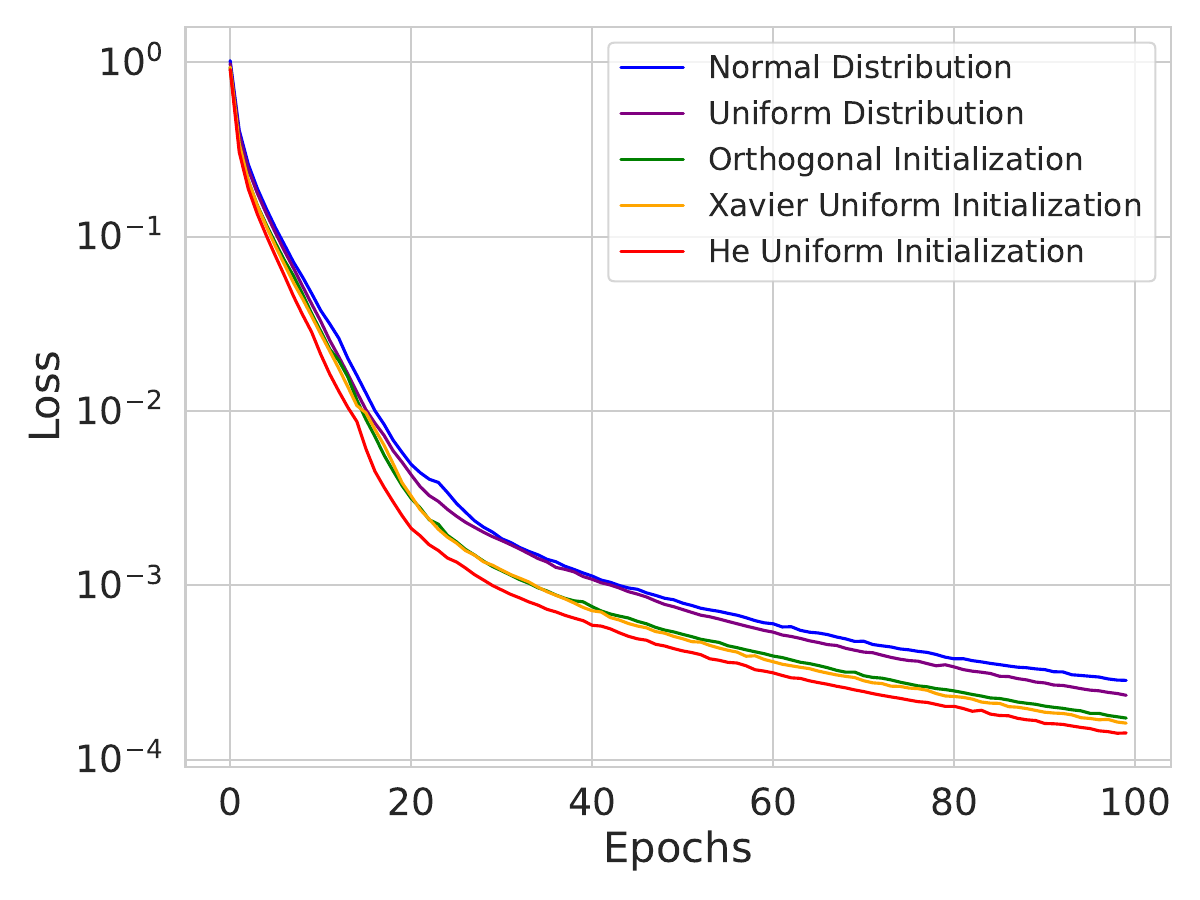}
    \caption{Training loss v/s Epochs}
  \end{subfigure}
  \hfill
  \begin{subfigure}[b]{0.31\textwidth}
    \includegraphics[width=\textwidth]{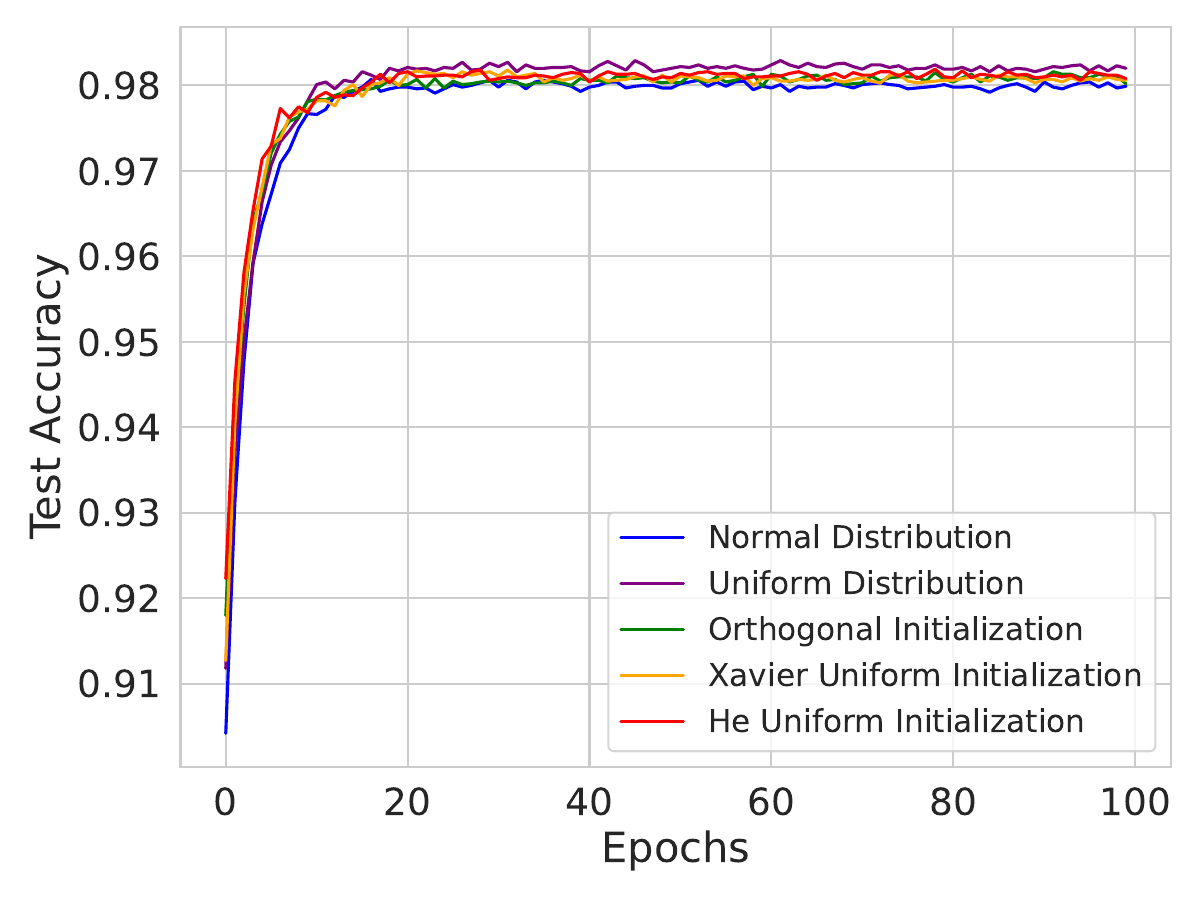}
    \caption{Validation acc. v/s Epochs}
  \end{subfigure}
  \caption{Effect of initialization on our learning rate and model performance. \textbf{Mini batch} experiments conducted on a 3 layer network with 100 nodes on the MNIST.}
\label{fig:mb_distribution_3__100}
\end{figure}
\begin{figure}
  \centering
  \begin{subfigure}[b]{0.31\textwidth}
    \includegraphics[width=\textwidth]{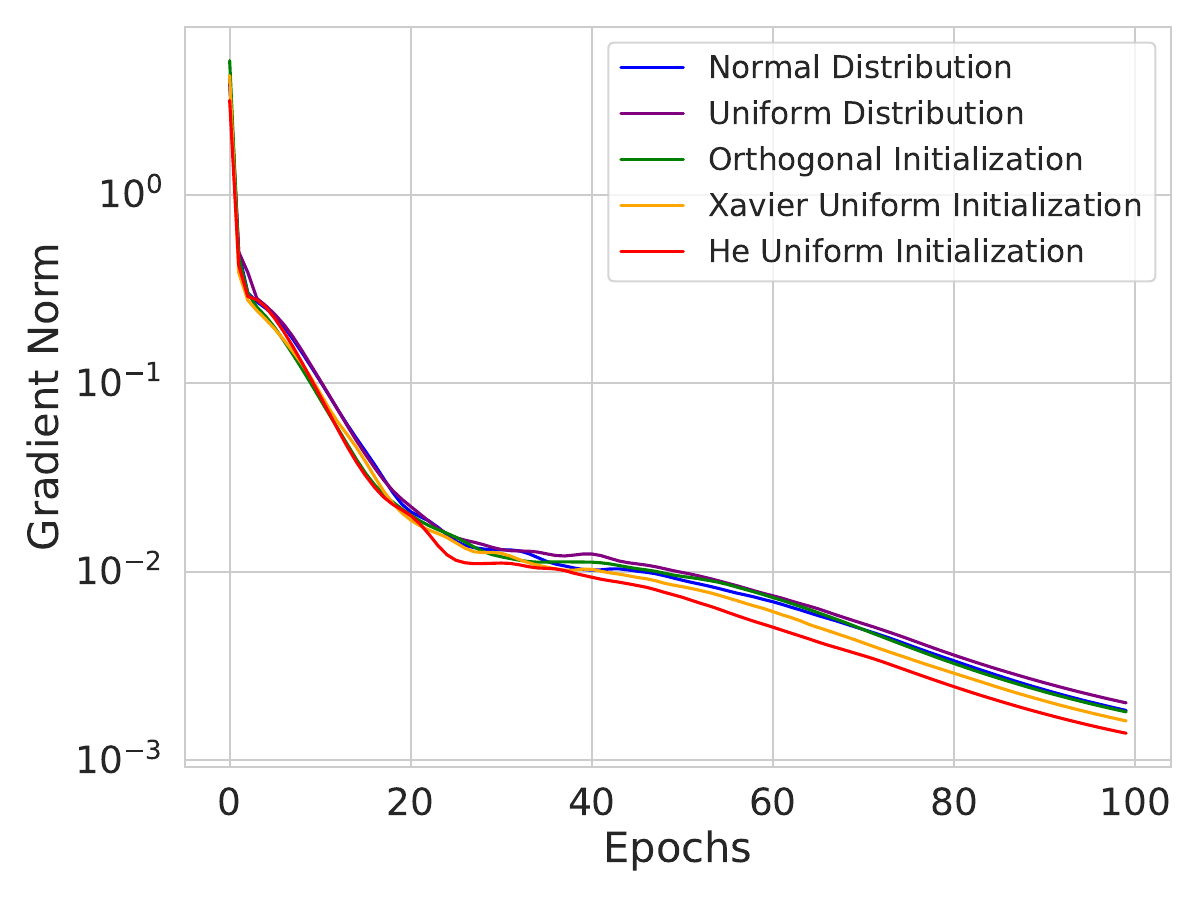}
    \caption{Gradient norm v/s Epochs}
  \end{subfigure}
  \hfill
  \begin{subfigure}[b]{0.31\textwidth}
    \includegraphics[width=\textwidth]{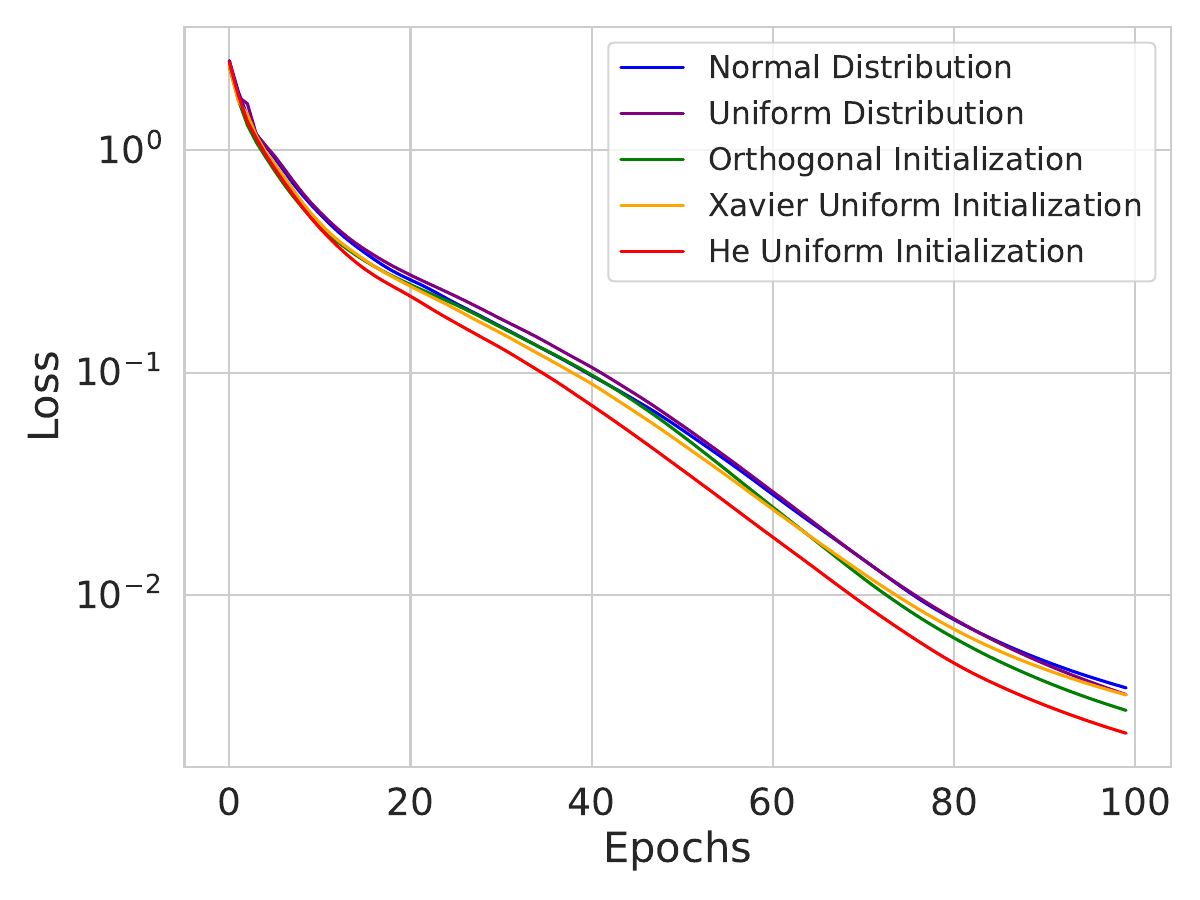}
    \caption{Training loss v/s Epochs}
  \end{subfigure}
  \hfill
  \begin{subfigure}[b]{0.31\textwidth}
    \includegraphics[width=\textwidth]{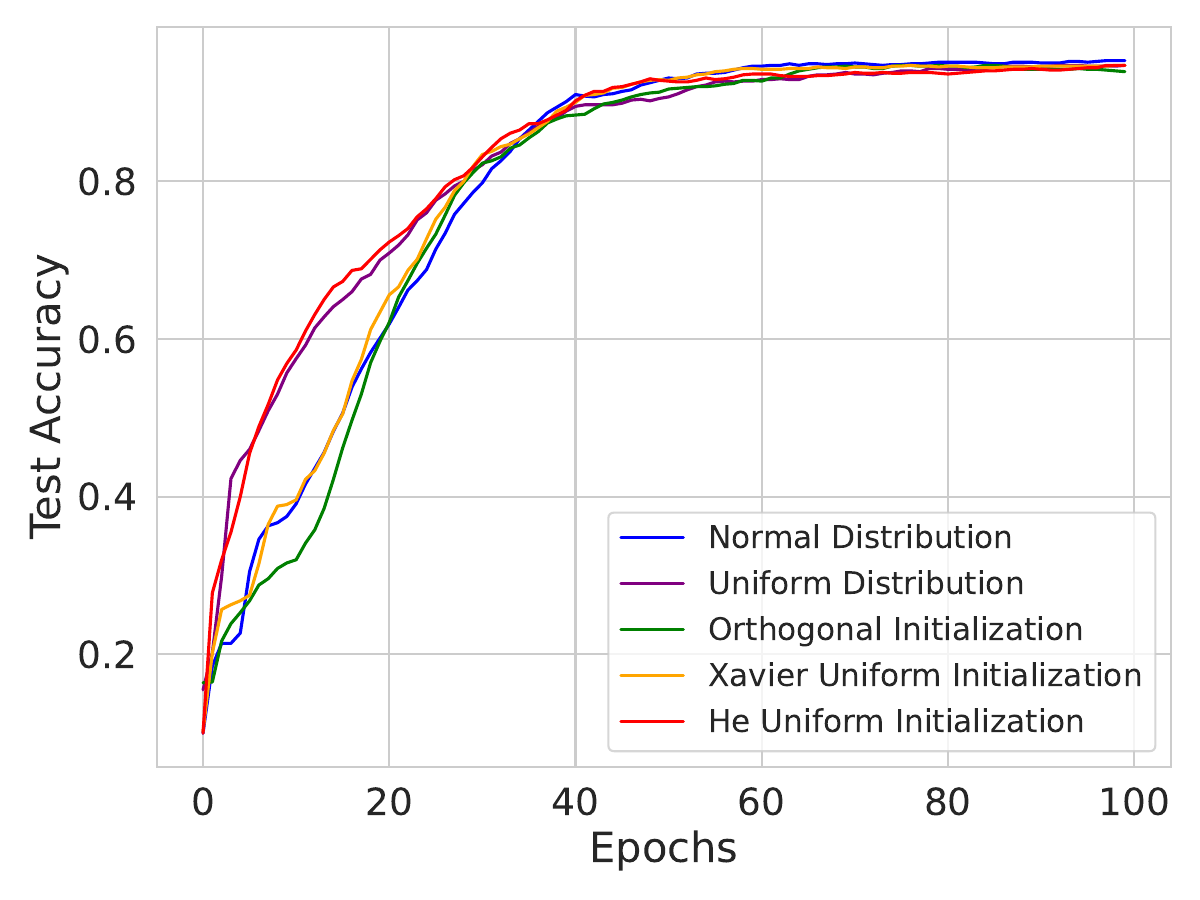}
    \caption{Validation acc. v/s Epochs}
  \end{subfigure}
  \caption{Effect of initialization on our learning rate and model performance. \textbf{Mini batch} experiments conducted on a 3 layer network with 1000 nodes on the MNIST.}
\label{fig:mb_distribution_3__1000}
\end{figure}
\begin{figure}
  \centering
  \begin{subfigure}[b]{0.31\textwidth}
    \includegraphics[width=\textwidth]{results/linear/distributions_exp/mini_batch/3__3000/grad_norm_new.pdf}
    \caption{Gradient norm v/s Epochs}
  \end{subfigure}
  \hfill
  \begin{subfigure}[b]{0.31\textwidth}
    \includegraphics[width=\textwidth]{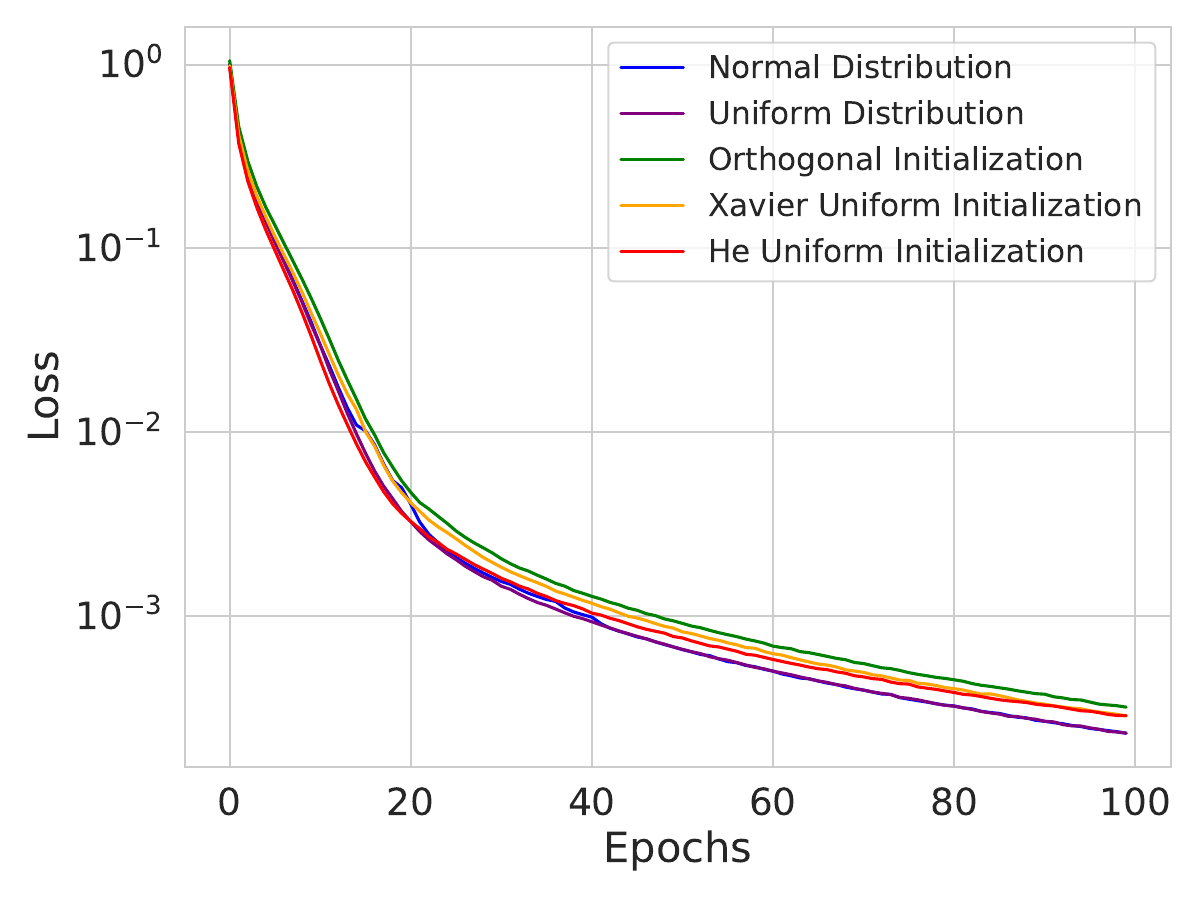}
    \caption{Training loss v/s Epochs}
  \end{subfigure}
  \hfill
  \begin{subfigure}[b]{0.31\textwidth}
    \includegraphics[width=\textwidth]{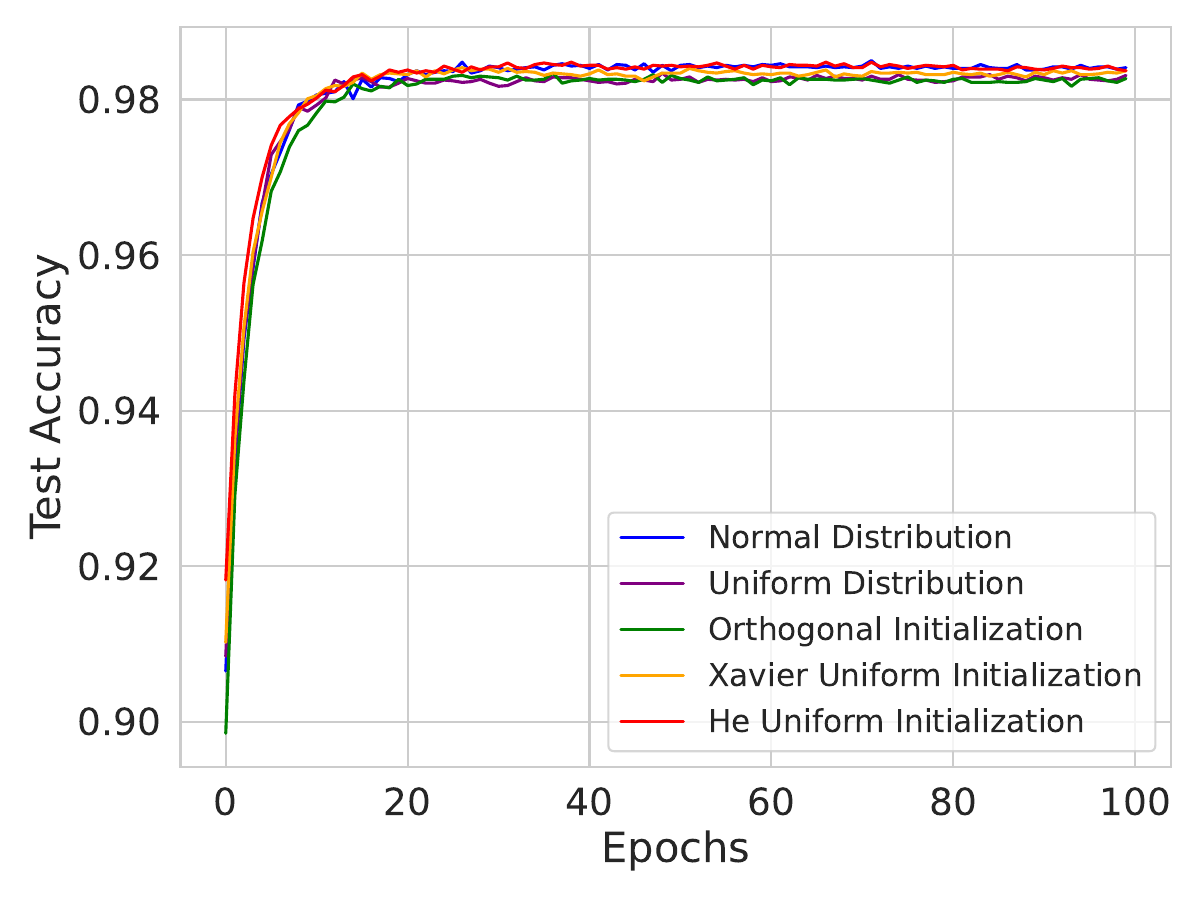}
    \caption{Validation acc. v/s Epochs}
  \end{subfigure}
  \caption{Effect of initialization on our learning rate and model performance. \textbf{Mini batch} experiments conducted on a 3 layer network with 3000 nodes on the MNIST.}
\label{fig:mb_distribution_3__3000}
\end{figure}
\begin{figure}
  \centering
  \begin{subfigure}[b]{0.31\textwidth}
    \includegraphics[width=\textwidth]{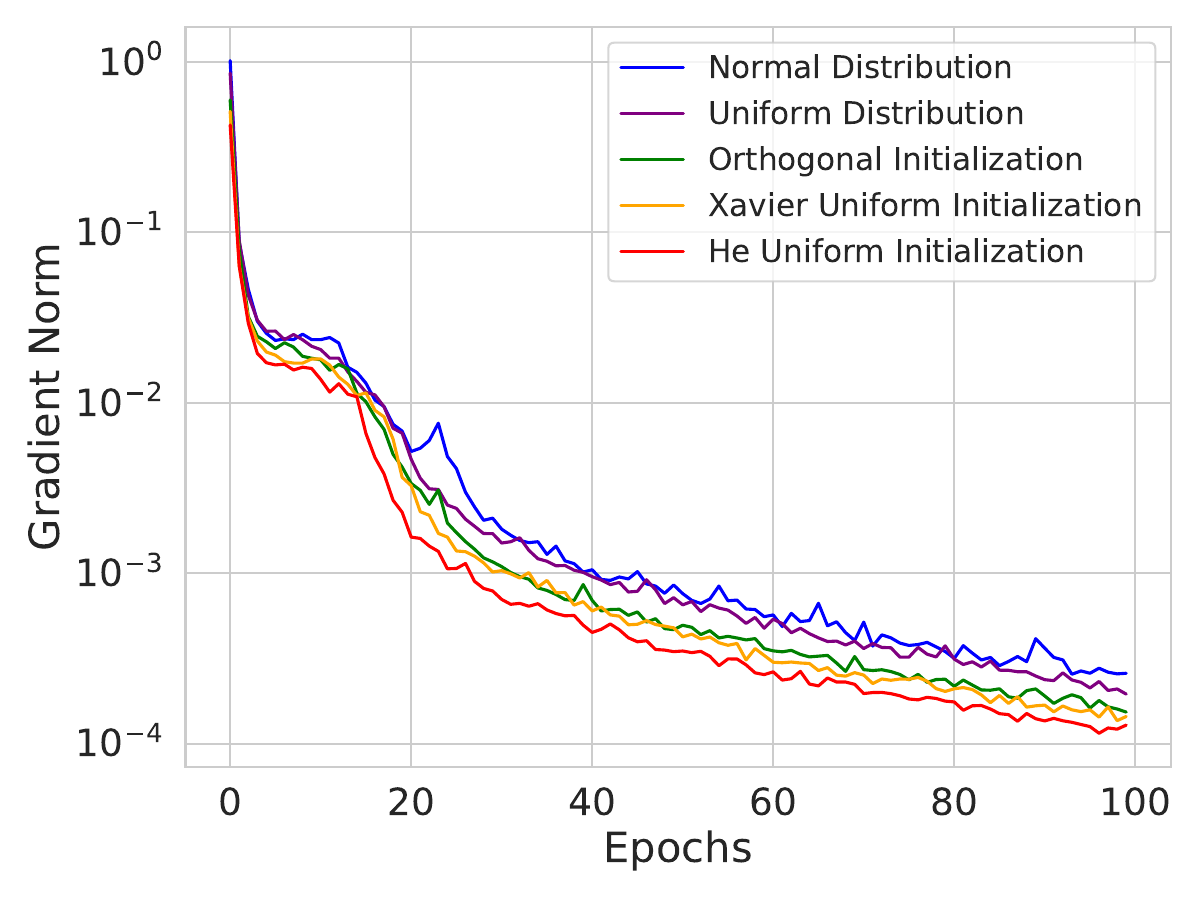}
    \caption{Gradient norm v/s Epochs}
  \end{subfigure}
  \hfill
  \begin{subfigure}[b]{0.31\textwidth}
    \includegraphics[width=\textwidth]{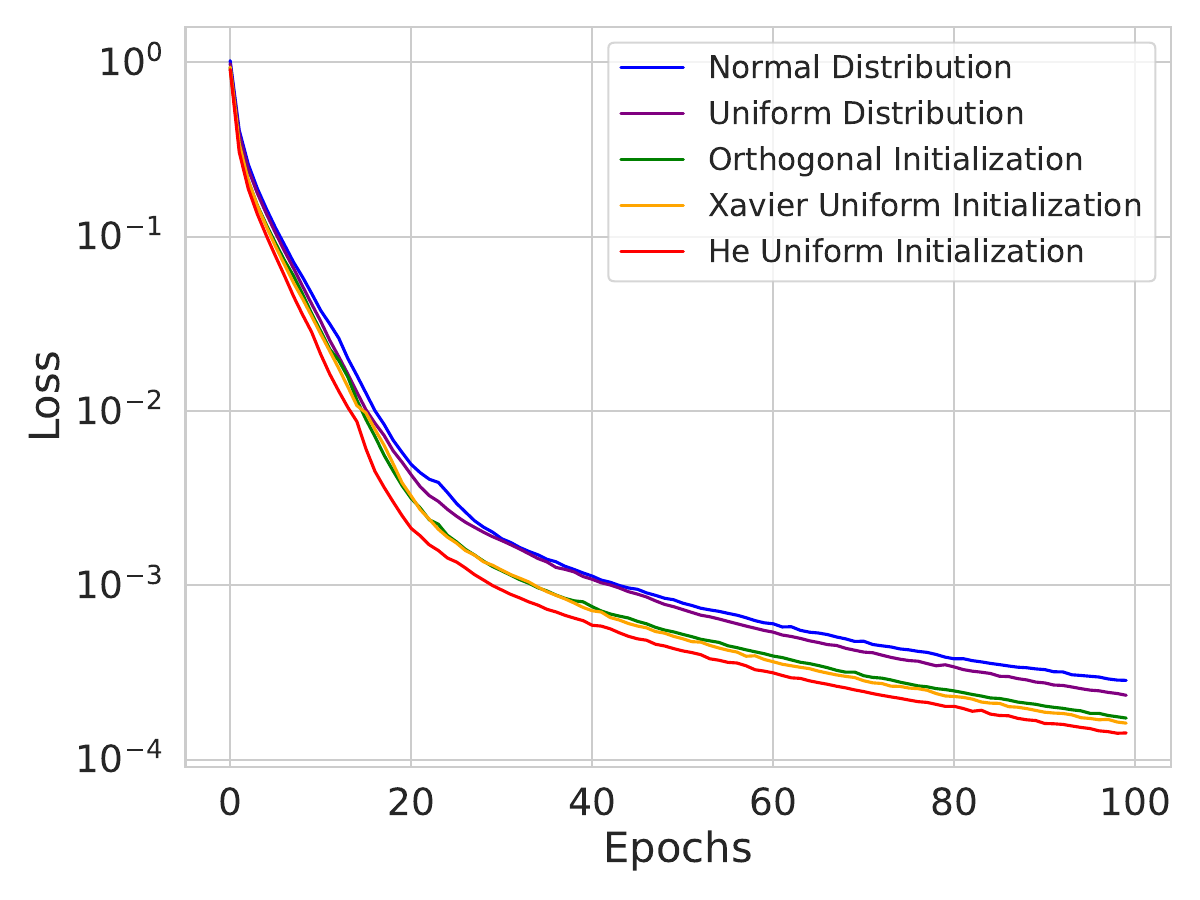}
    \caption{Training loss v/s Epochs}
  \end{subfigure}
  \hfill
  \begin{subfigure}[b]{0.31\textwidth}
    \includegraphics[width=\textwidth]{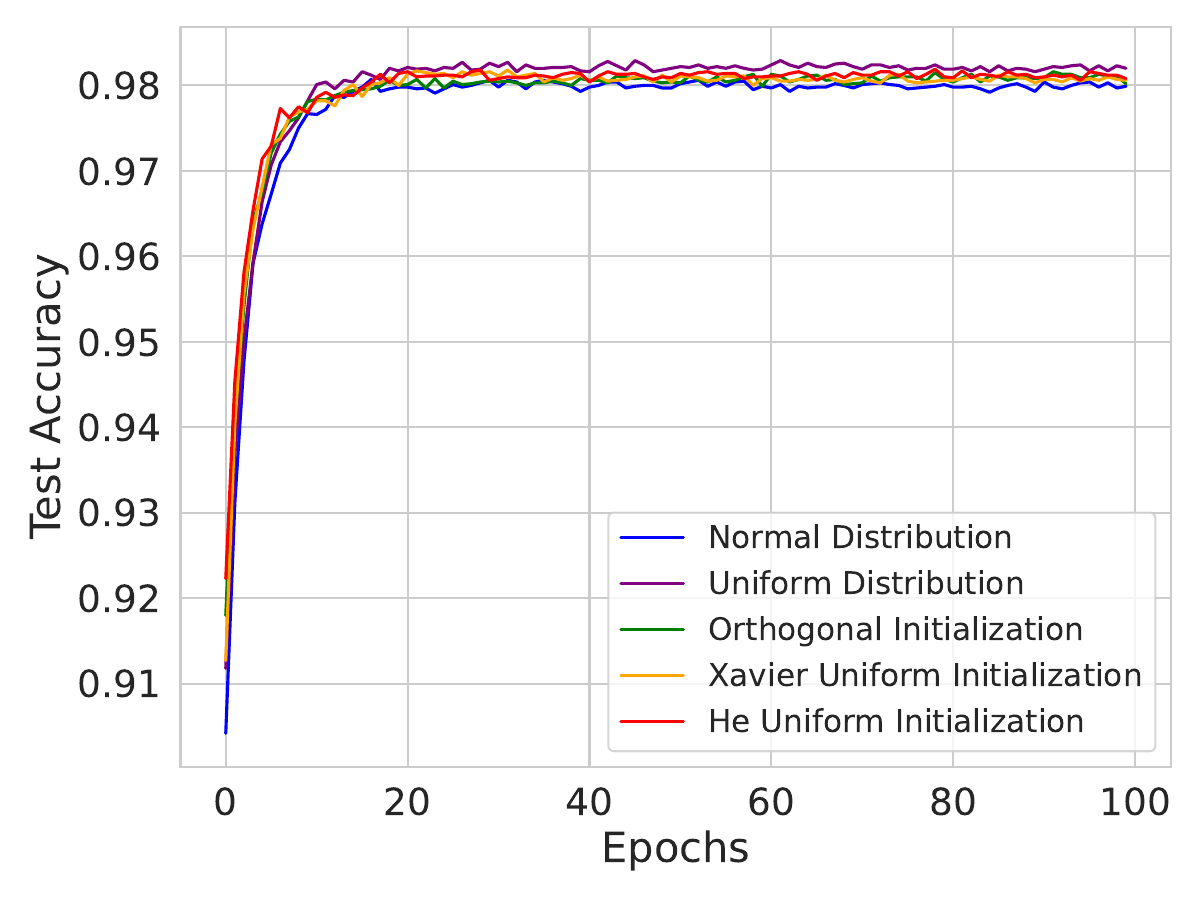}
    \caption{Validation acc. v/s Epochs}
  \end{subfigure}
  \caption{Effect of initialization on our learning rate and model performance. \textbf{Mini batch} experiments conducted on a 5 layer network with 300 nodes on the MNIST.}
\label{fig:mb_distribution_5__300}
\end{figure}
\begin{figure}
  \centering
  \begin{subfigure}[b]{0.31\textwidth}
    \includegraphics[width=\textwidth]{results/linear/distributions_exp/full_batch/5__1000/grad_norm_new.pdf}
    \caption{Gradient norm v/s Epochs}
  \end{subfigure}
  \hfill
  \begin{subfigure}[b]{0.31\textwidth}
    \includegraphics[width=\textwidth]{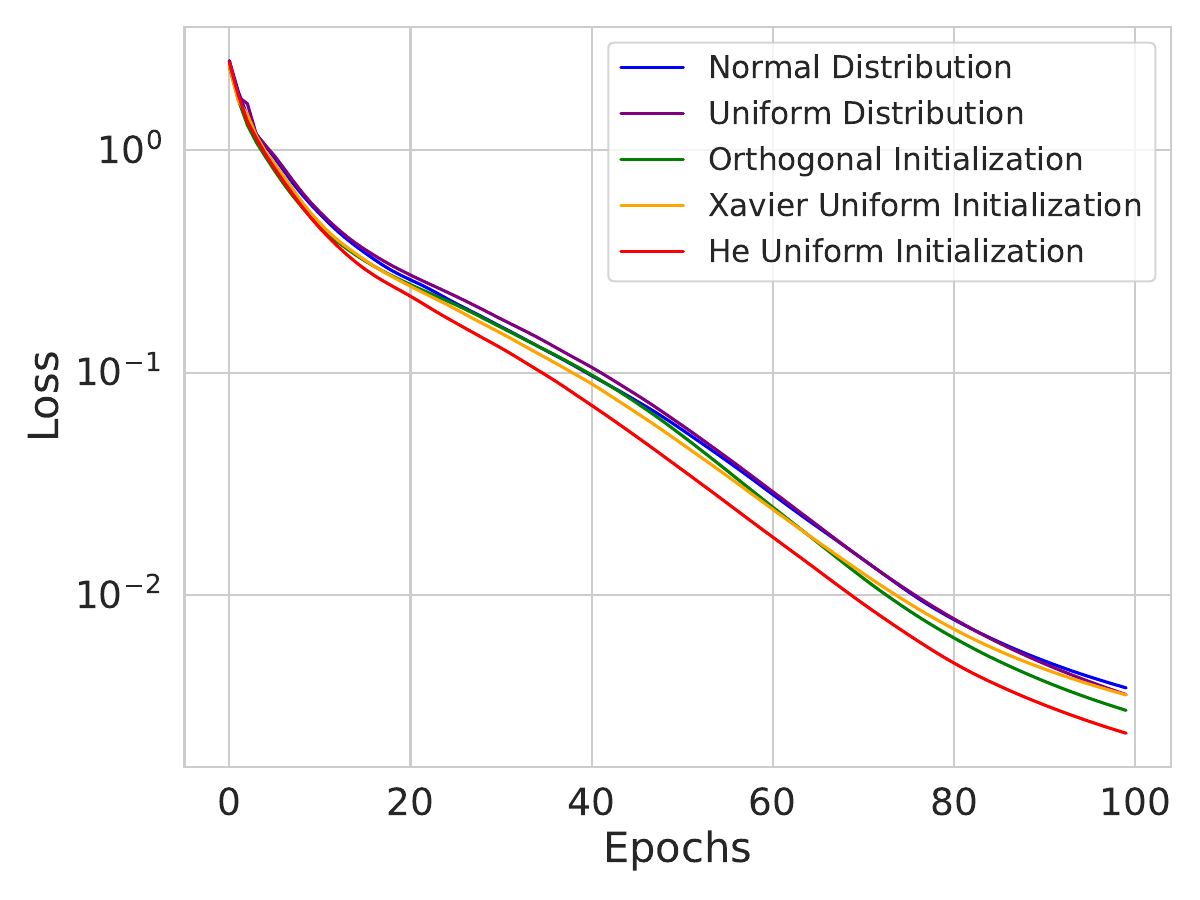}
    \caption{Training loss v/s Epochs}
  \end{subfigure}
  \hfill
  \begin{subfigure}[b]{0.31\textwidth}
    \includegraphics[width=\textwidth]{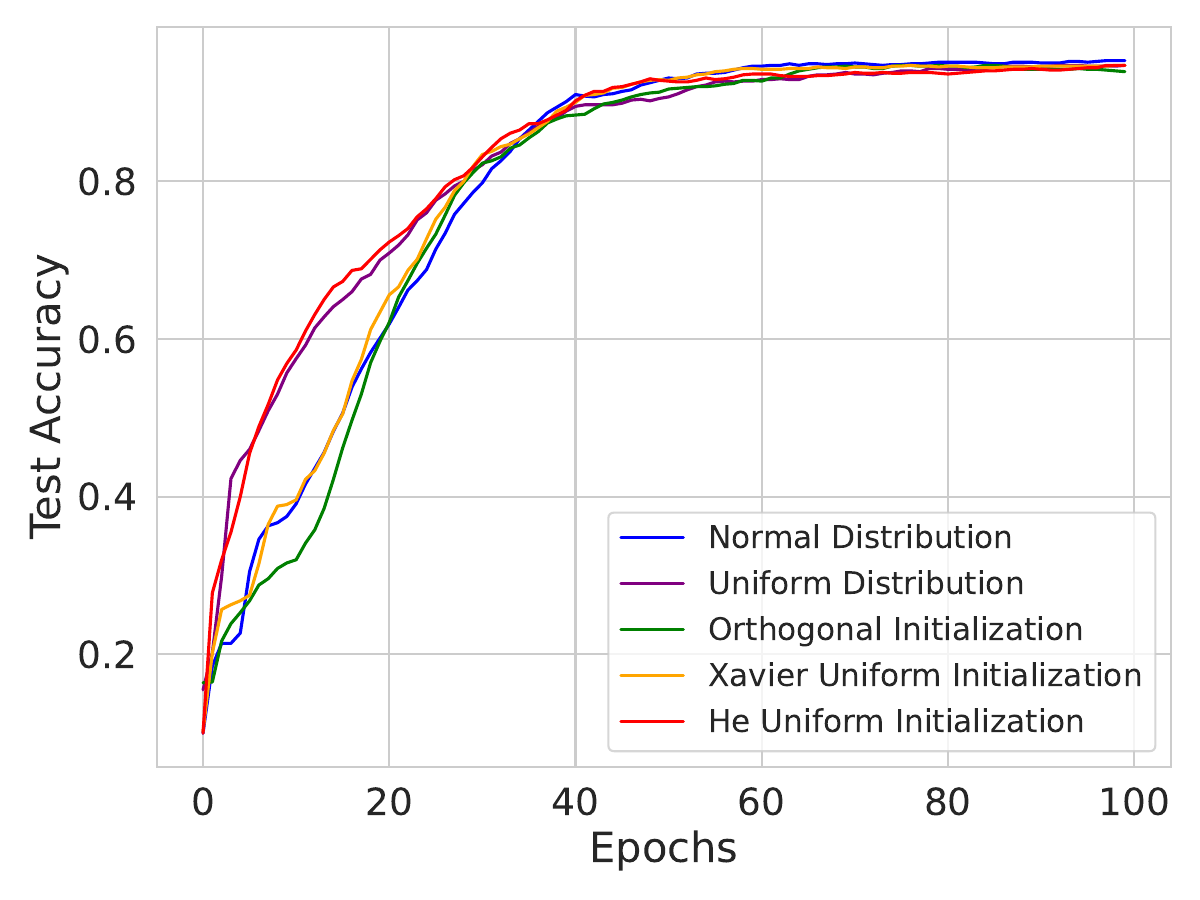}
    \caption{Validation acc. v/s Epochs}
  \end{subfigure}
  \caption{Effect of initialization on our learning rate and model performance. \textbf{Mini batch} experiments conducted on a 5 layer network with 1000 nodes on the MNIST.}
\label{fig:mb_distribution_5__1000}
\end{figure}
\begin{figure}
  \centering
  \begin{subfigure}[b]{0.31\textwidth}
    \includegraphics[width=\textwidth]{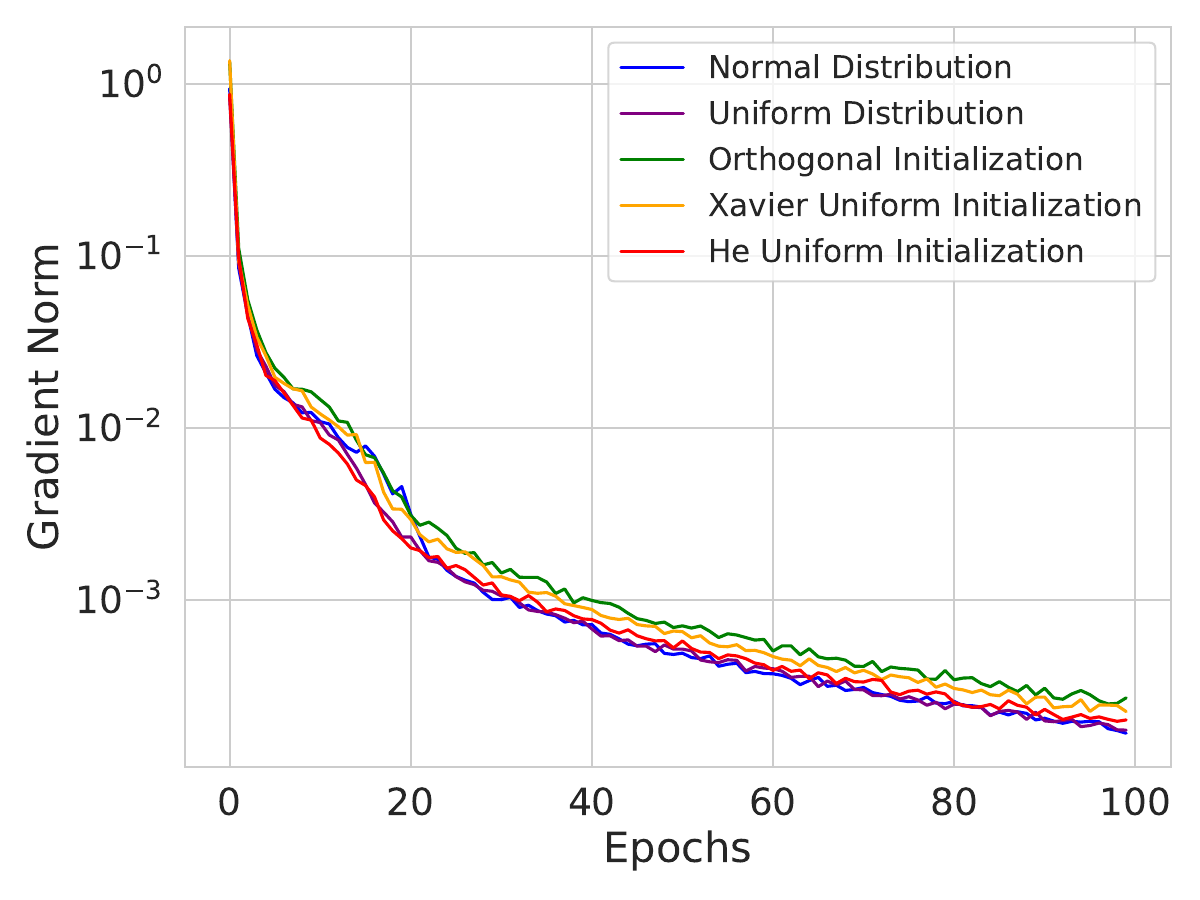}
    \caption{Gradient norm v/s Epochs}
  \end{subfigure}
  \hfill
  \begin{subfigure}[b]{0.31\textwidth}
    \includegraphics[width=\textwidth]{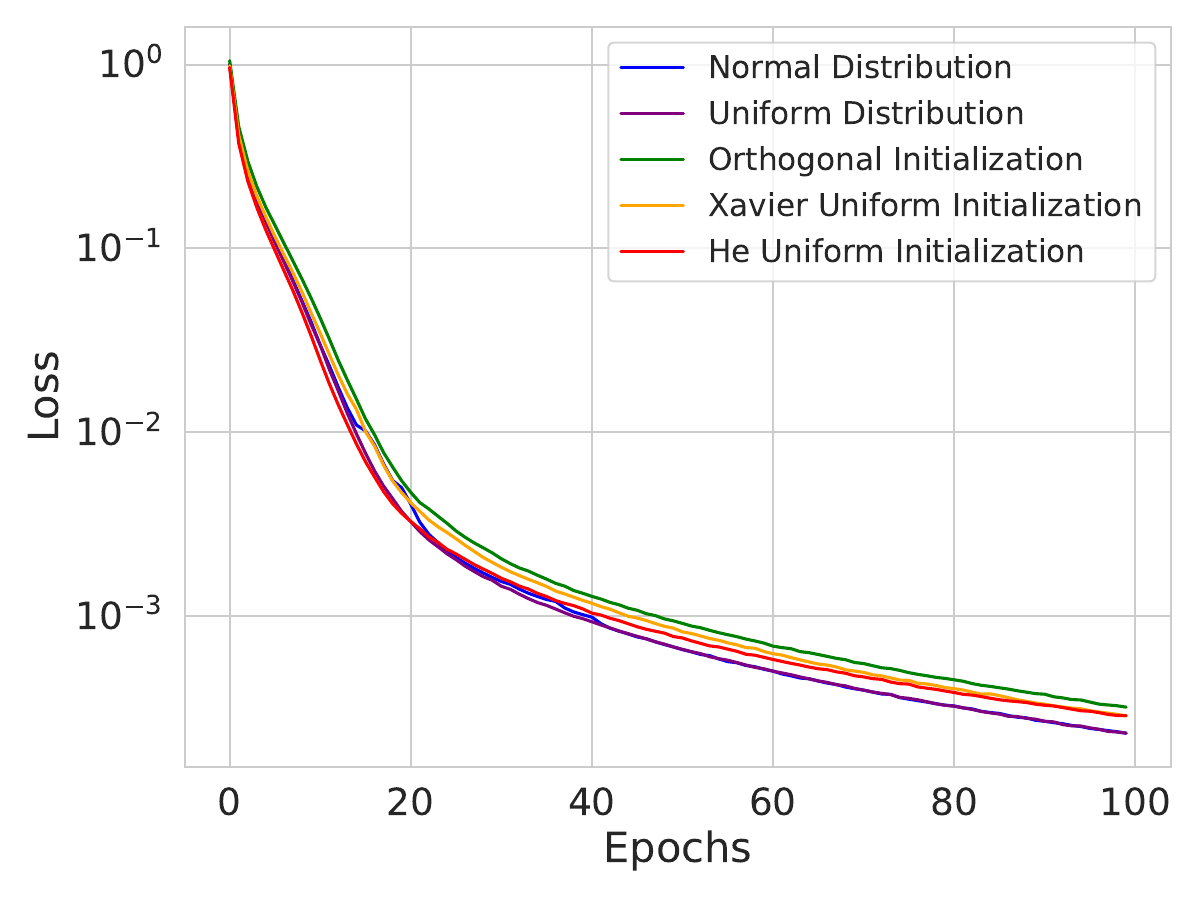}
    \caption{Training loss v/s Epochs}
  \end{subfigure}
  \hfill
  \begin{subfigure}[b]{0.31\textwidth}
    \includegraphics[width=\textwidth]{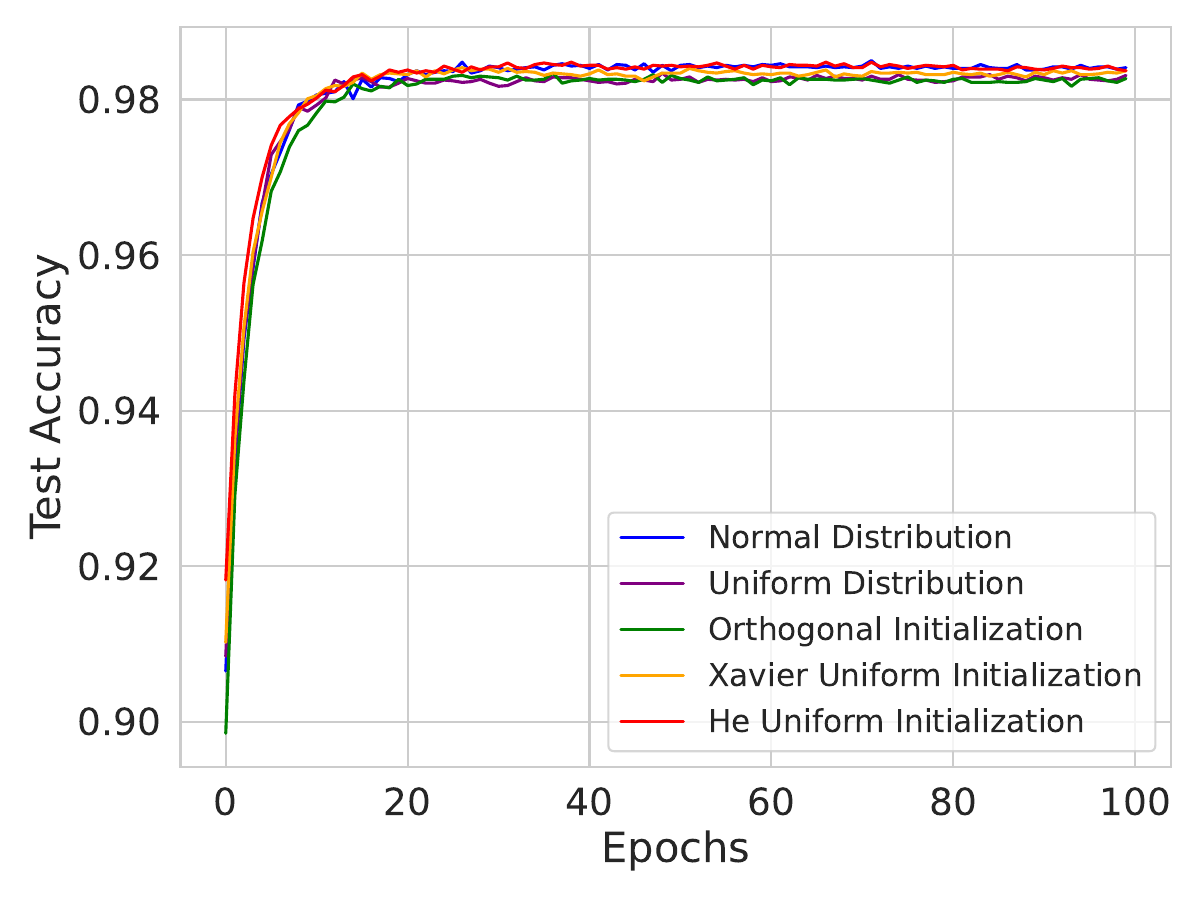}
    \caption{Validation acc. v/s Epochs}
  \end{subfigure}
  \caption{Effect of initialization on our learning rate and model performance. \textbf{Mini batch} experiments conducted on a 5 layer network with 3000 nodes on the MNIST.}
\label{fig:mb_distribution_5__3000}
\end{figure}
\begin{figure}
  \centering
  \begin{subfigure}[b]{0.31\textwidth}
    \includegraphics[width=\textwidth]{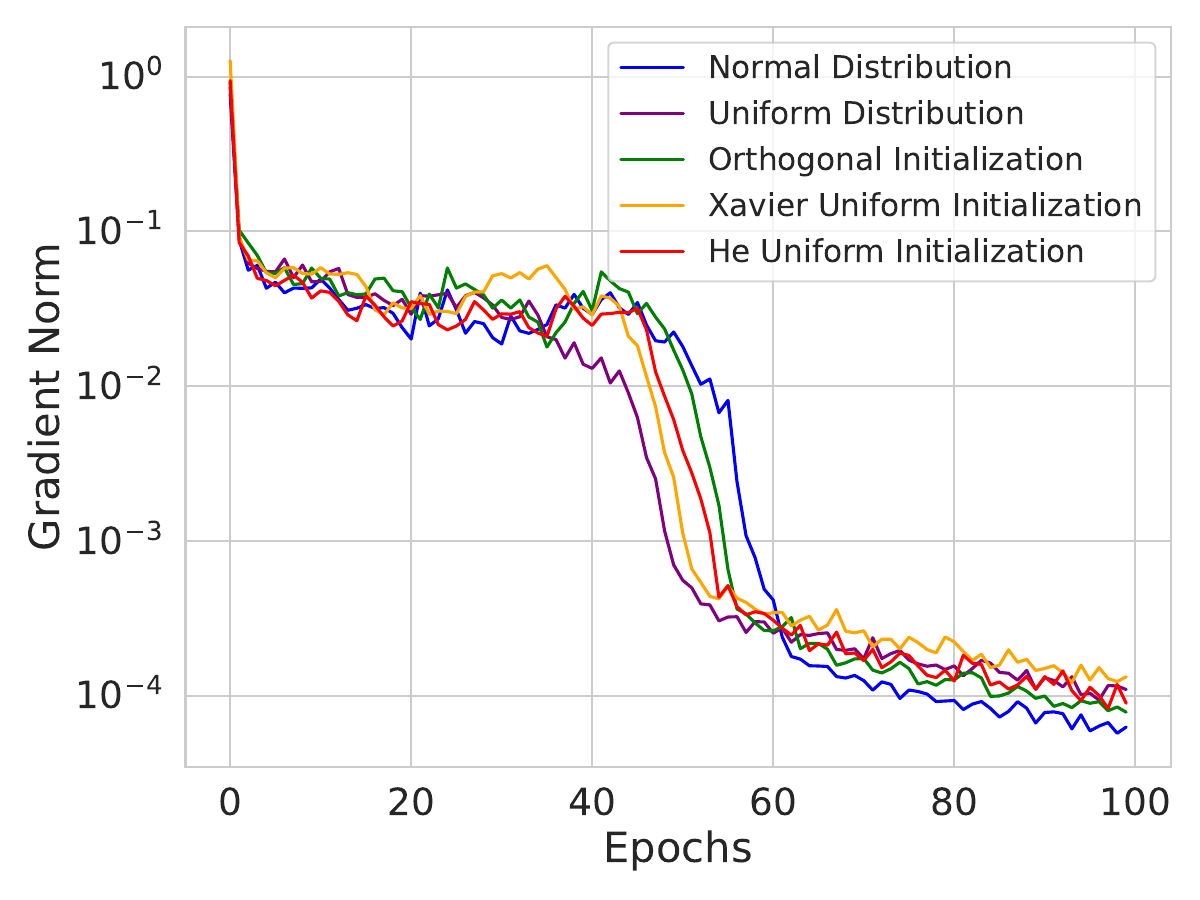}
    \caption{Gradient norm v/s Epochs}
  \end{subfigure}
  \hfill
  \begin{subfigure}[b]{0.31\textwidth}
    \includegraphics[width=\textwidth]{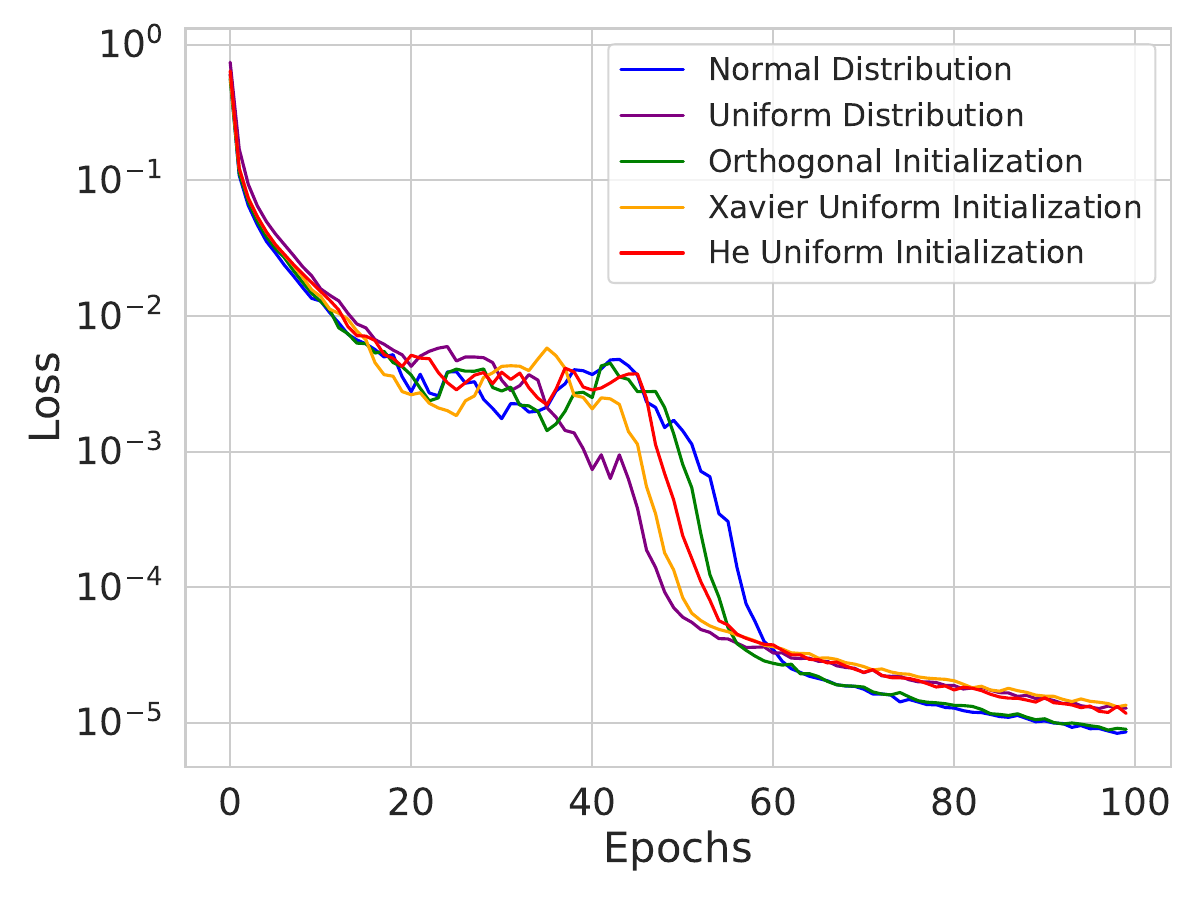}
    \caption{Training loss v/s Epochs}
  \end{subfigure}
  \hfill
  \begin{subfigure}[b]{0.31\textwidth}
    \includegraphics[width=\textwidth]{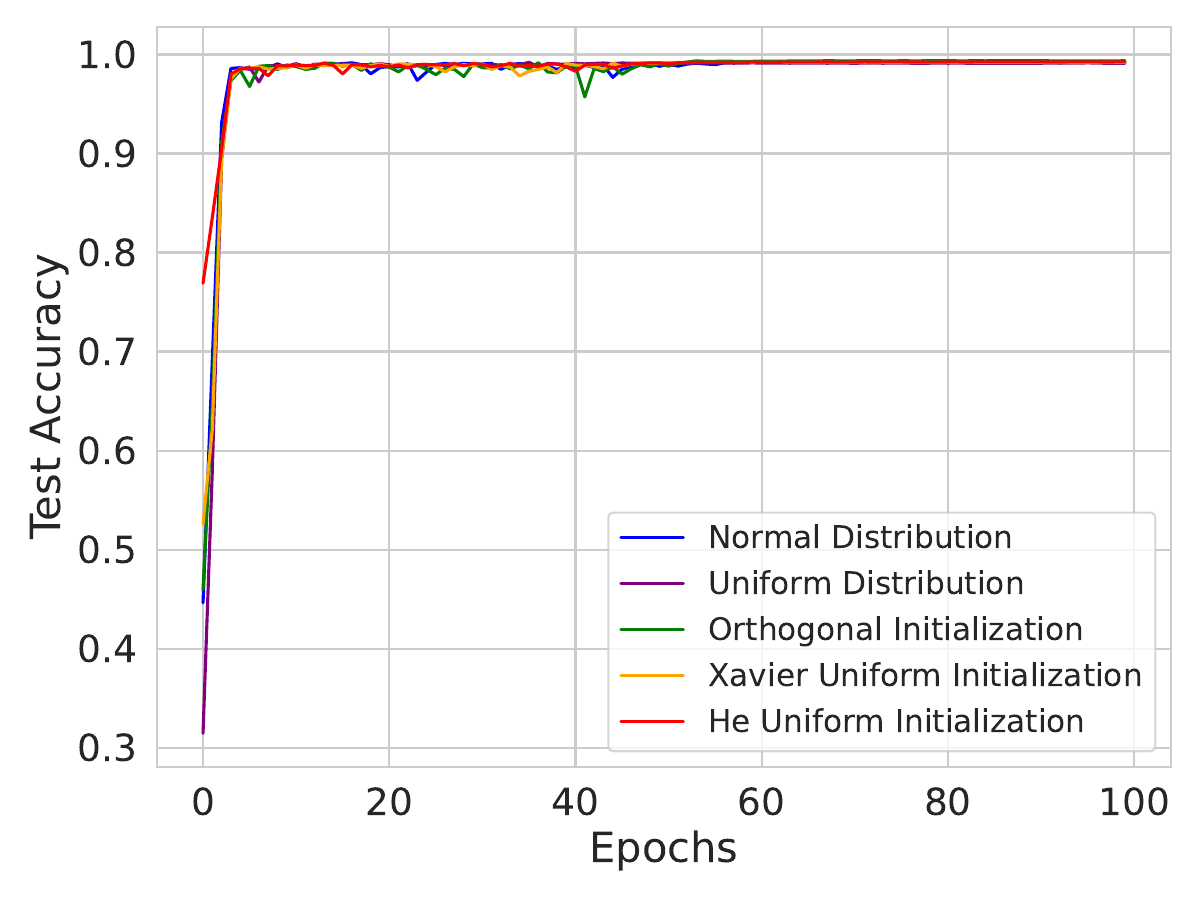}
    \caption{Validation acc. v/s Epochs}
  \end{subfigure}
  \caption{Effect of initialization on our learning rate and model performance. \textbf{Mini batch} experiments conducted on LeNet with MNIST.}
\label{fig:mb_distribution_lenet}
\end{figure}
\newpage

\textbf{C.6 VGG-9 on CIFAR-10 and LeNet on MNIST}\\

We compared the performance of our learning rate on VGG-9 and LeNet against an array of constant step sizes. VGG-9 is trained on CIFAR-10 with a batch size of 2,500 and LeNet is trained on MNIST with a batch size of 5,000. The training lasted for 100 epochs in each case. This section serves as a continuation for Section~\ref{sec:cnn_exp}.

\begin{figure}
  \centering
  \begin{subfigure}[b]{0.31\textwidth}
    \includegraphics[width=\textwidth]{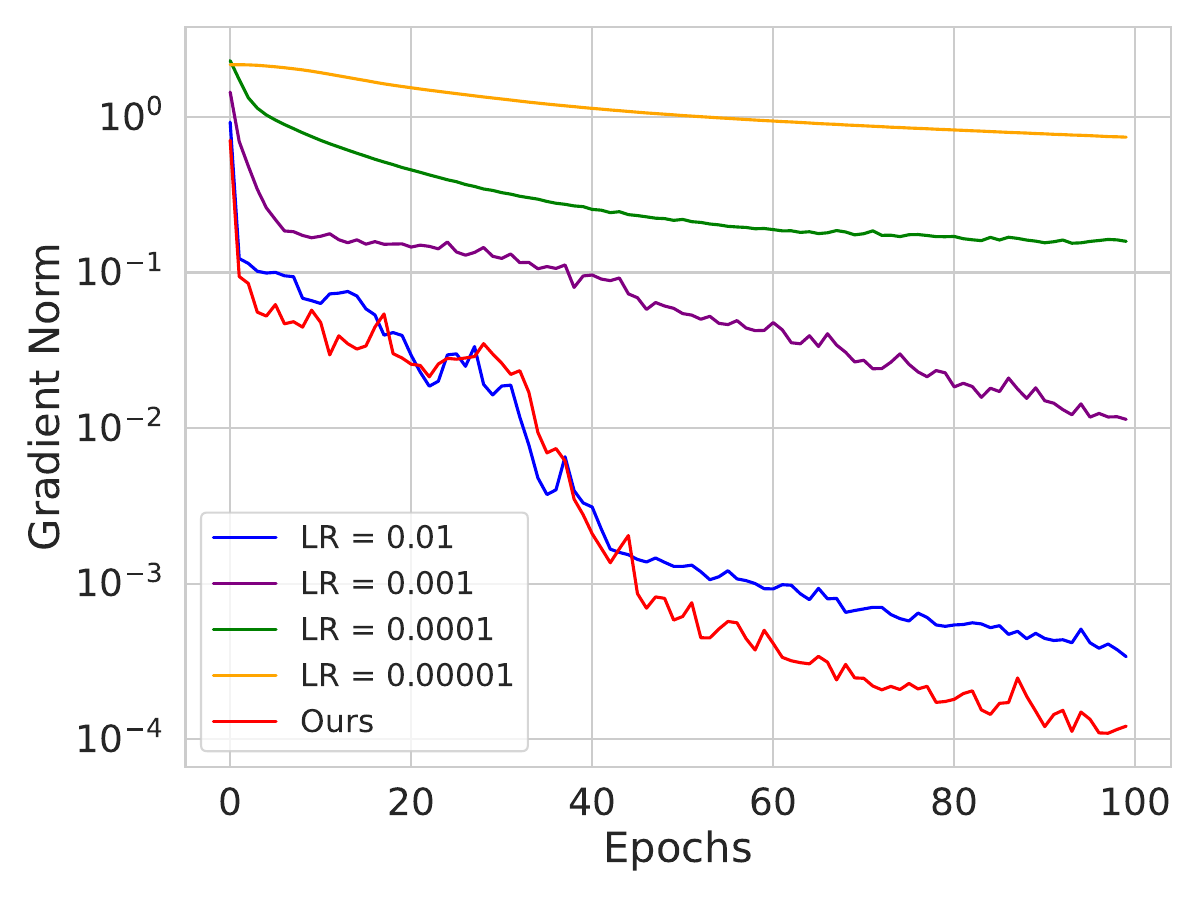}
    \caption{Gradient norm v/s Epochs}
  \end{subfigure}
  \hfill
  \begin{subfigure}[b]{0.31\textwidth}
    \includegraphics[width=\textwidth]{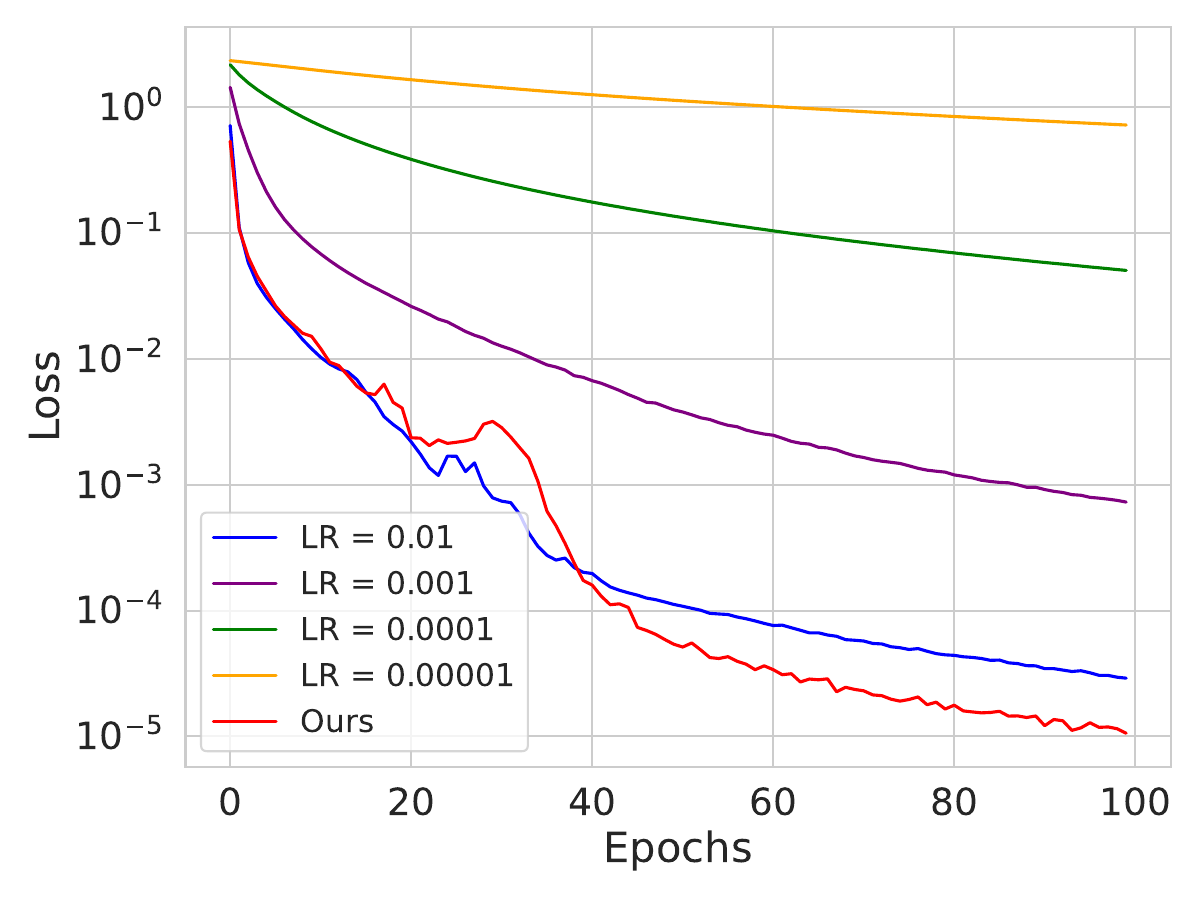}
    \caption{Training loss v/s Epochs}
  \end{subfigure}
  \hfill
  \begin{subfigure}[b]{0.31\textwidth}
    \includegraphics[width=\textwidth]{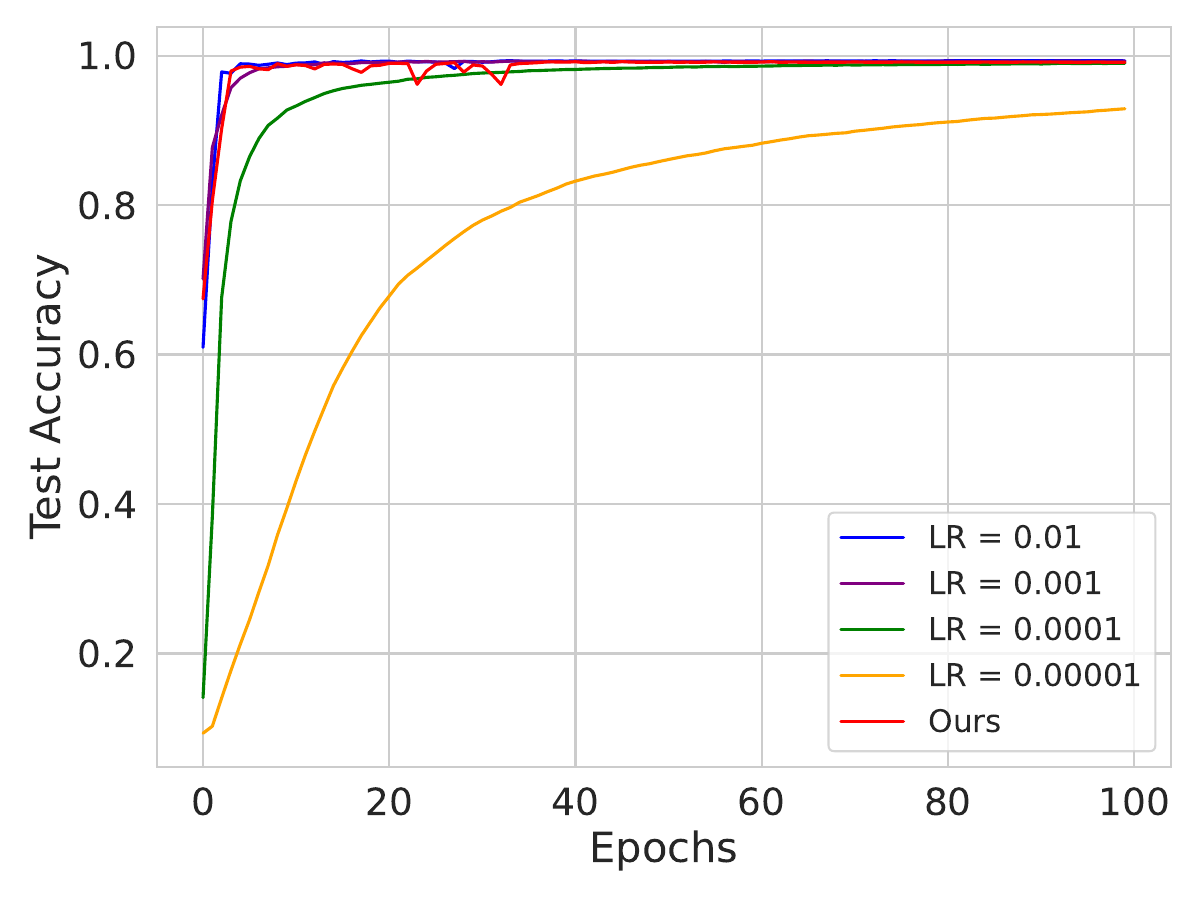}
    \caption{Validation acc. v/s Epochs}
  \end{subfigure}
  \caption{Mini-batch experiments on LeNet architecture with MNIST data.}
\label{fig:lenet_step_app}
\end{figure}
\begin{figure}
  \centering
  \begin{subfigure}[b]{0.31\textwidth}
    \includegraphics[width=\textwidth]{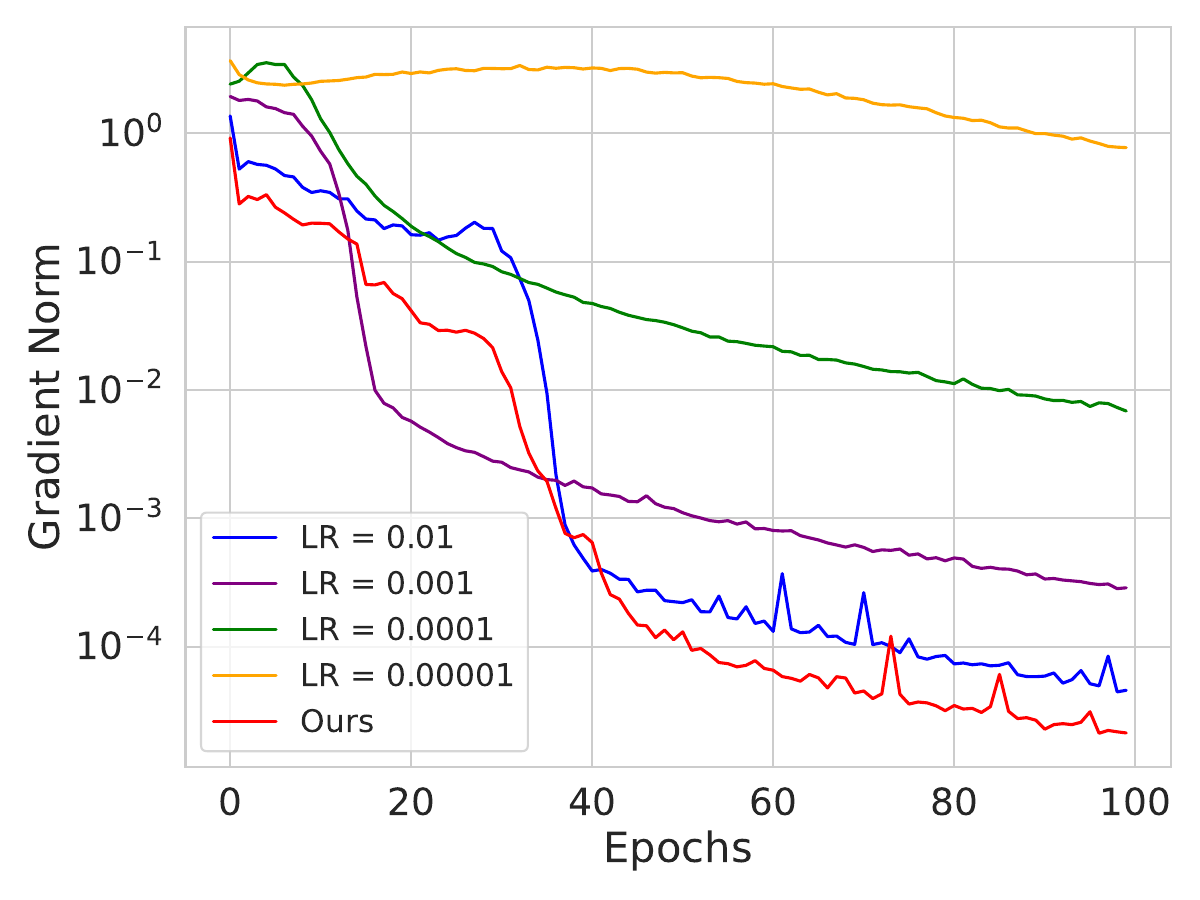}
    \caption{Gradient norm v/s Epochs}
  \end{subfigure}
  \hfill
  \begin{subfigure}[b]{0.31\textwidth}
    \includegraphics[width=\textwidth]{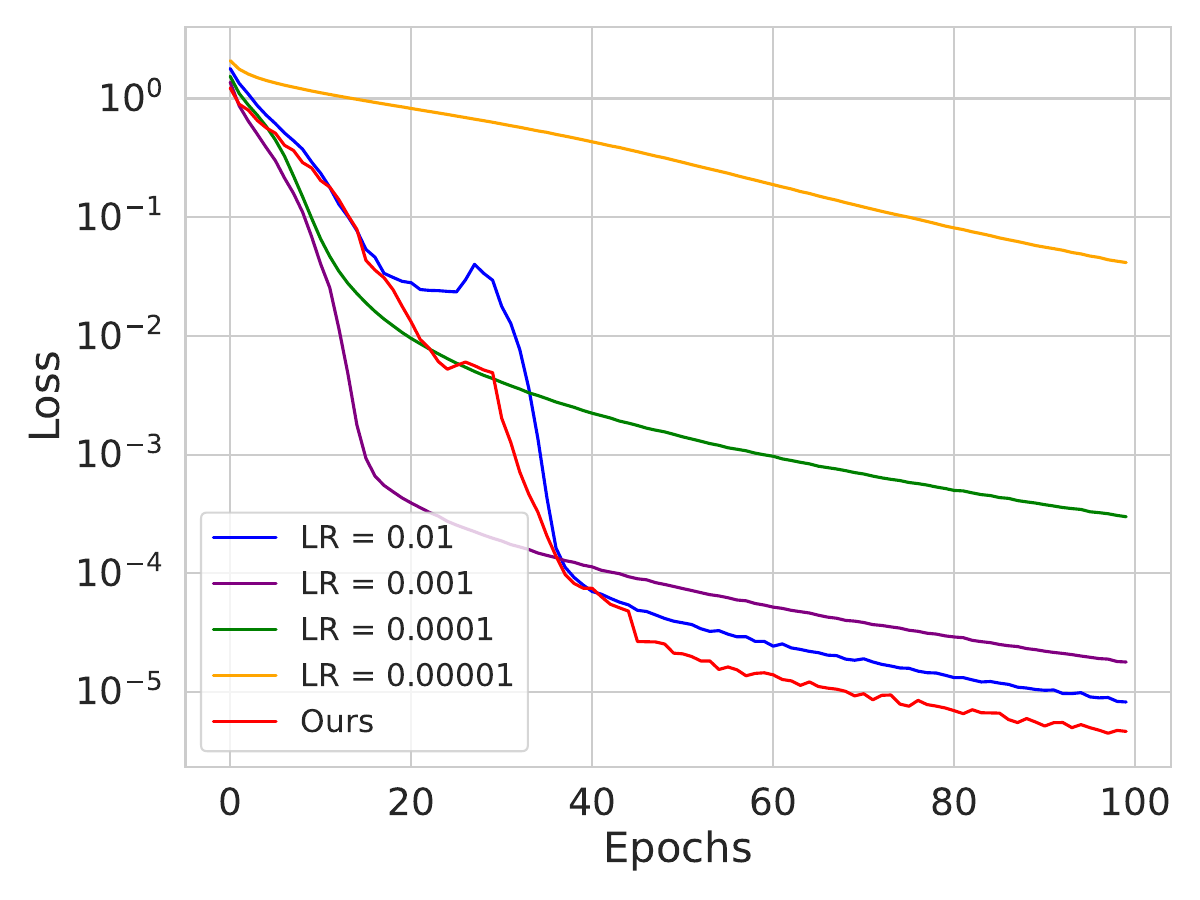}
    \caption{Training loss v/s Epochs}
  \end{subfigure}
  \hfill
  \begin{subfigure}[b]{0.31\textwidth}
    \includegraphics[width=\textwidth]{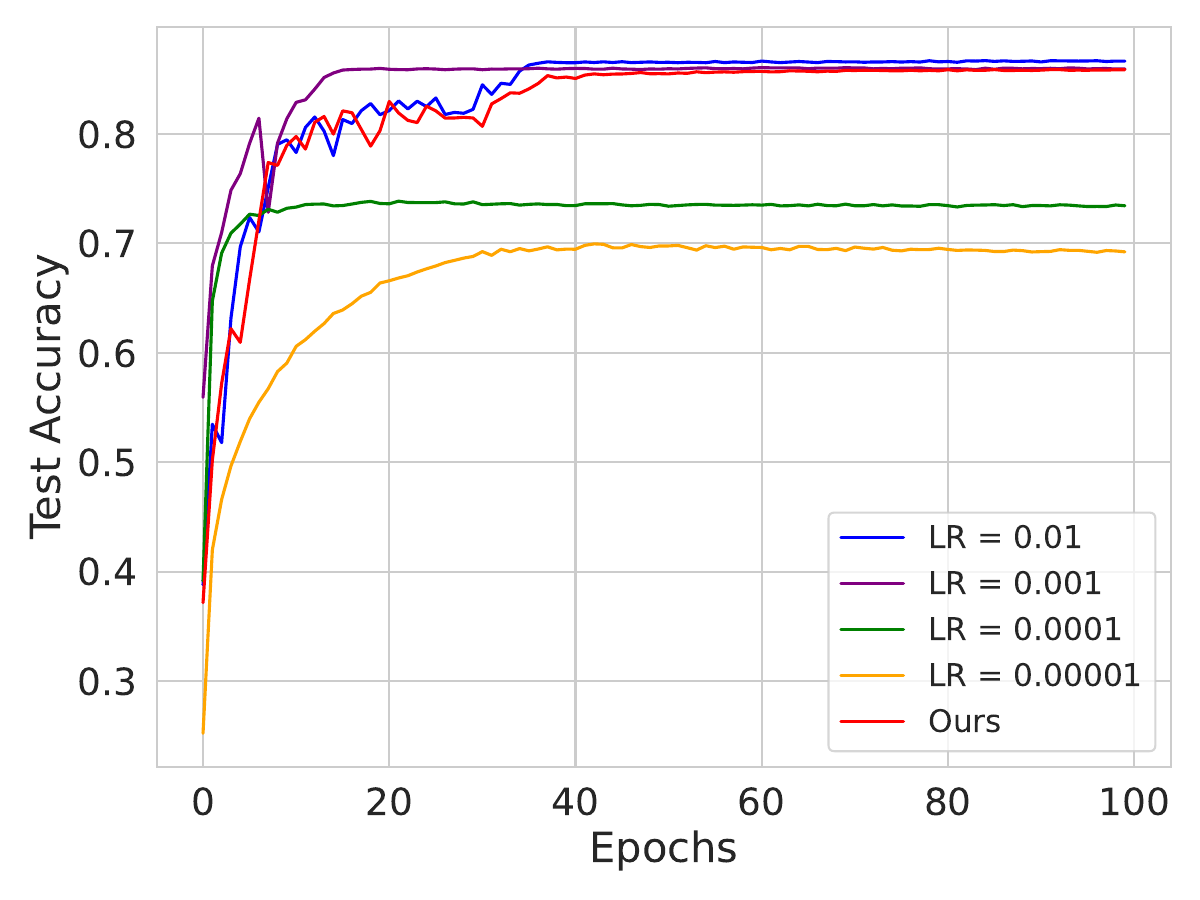}
    \caption{Validation acc. v/s Epochs}
  \end{subfigure}
  \caption{Mini-batch experiments on VGG-9 architecture with CIFAR-10 data.}
\label{fig:vgg_step_app}
\end{figure}
\end{document}